%% file: main.tex
\documentclass[10pt,twocolumn,letterpaper]{article}

\usepackage{cvpr}
\usepackage{times}
\usepackage{epsfig}
\usepackage{graphicx}
\usepackage{amsmath,amsthm}
\usepackage{amssymb}
\usepackage{enumitem}
\usepackage{bm}
\usepackage{subfig}

\usepackage[pagebackref=true,breaklinks=true,letterpaper=true,colorlinks,bookmarks=false]{hyperref}

\cvprfinalcopy 

\usepackage[vlined,ruled,linesnumbered]{algorithm2e}
\usepackage[noend]{algorithmic}

\input{preamble_packages_selected} 
\input{preamble_symbols}

\input{shortcuts}

\def\cvprPaperID{4787} 

\ifcvprfinal\fi
\begin{document}

\title{\vspace{-30mm} \small{\normalfont This paper has been accepted for publication at the IEEE Conference on Computer Vision and Pattern Recognition, 2020. \\ Please cite the paper as: Heng Yang and Luca Carlone, \\ ``In Perfect Shape: Certifiably Optimal 3D Shape Reconstruction from 2D Landmarks'', \\ In \emph{IEEE Conf. on Computer Vision and Pattern Recognition (CVPR)}, 2020.} \\ \vspace{10mm} \Large{In Perfect Shape: Certifiably Optimal \\ 3D Shape Reconstruction from 2D Landmarks}}

\author{Heng Yang and Luca Carlone\\
Laboratory for Information \& Decision Systems (LIDS)\\
Massachusetts Institute of Technology\\
{\tt\small \{hankyang, lcarlone\}@mit.edu}  
}

\maketitle
\input{abstract}

\input{introduction}

\input{relatedWork}

\input{notationandPreliminaries}
\input{problemStatement}
\input{certifiableNonminimalSolver}

	\input{basisReduction}
\input{robustOutlierRejection}

\input{experiments}
\input{conclusions}

\begin{center}
{\bf \Large Supplementary Material}
\end{center}

\setcounter{equation}{0}
\setcounter{theorem}{0}
\setcounter{figure}{0}
\renewcommand{\theequation}{A\arabic{equation}}
\renewcommand{\thetheorem}{A\arabic{theorem}}
\renewcommand{\thefigure}{A\arabic{figure}}

\input{supp-proof-tran-free-shape-recon}
\input{supp-proof-SOS_relaxation}
\input{supp-proof-certificate_global_optimality}
\input{supp-derivation-basis_reduction}

\input{supp-derivation-robust_algorithm}
\input{supp-experiments}

\clearpage
{\small
\bibliographystyle{ieee_fullname}
\bibliography{../../references/refs,myRefs}
}

\end{document}

%% file: preamble_packages_selected.tex
\usepackage{amsfonts}
\usepackage{amsmath}
\usepackage{amssymb}

\usepackage{graphicx,url}

\usepackage{xspace}
\usepackage{bm}
\usepackage[T1]{fontenc}
\usepackage{latexsym}
\usepackage{xstring}
\usepackage{relsize}

\usepackage[noend]{algorithmic}
\usepackage{multirow}
\usepackage{xcolor}

%% file: preamble_symbols.tex
\newtheorem{theorem}{Theorem}

\newtheorem{lemma}[theorem]{Lemma}

\newtheorem{definition}[theorem]{Definition}
\newtheorem{proposition}[theorem]{Proposition}

\renewcommand{\cf}{\emph{cf.}\xspace}

\newcommand{\bdmath}{\begin{dmath}}
\newcommand{\edmath}{\end{dmath}}
\newcommand{\beq}{\begin{equation}}
\newcommand{\eeq}{\end{equation}}
\newcommand{\bdm}{\begin{displaymath}}
\newcommand{\edm}{\end{displaymath}}
\newcommand{\bea}{\begin{eqnarray}}
\newcommand{\eea}{\end{eqnarray}}
\newcommand{\beal}{\beq \begin{array}{ll}}
\newcommand{\eeal}{\end{array} \eeq}
\newcommand{\beas}{\begin{eqnarray*}}
\newcommand{\eeas}{\end{eqnarray*}}
\newcommand{\ba}{\begin{array}}
\newcommand{\ea}{\end{array}}
\newcommand{\bit}{\begin{itemize}}
\newcommand{\eit}{\end{itemize}}
\newcommand{\ben}{\begin{enumerate}}
\newcommand{\een}{\end{enumerate}}

\newcommand{\calN}{{\cal N}}

\newcommand{\calS}{{\cal S}}

\newcommand{\calX}{{\cal X}}

\renewcommand{\etal}{\emph{et~al.}\xspace}

\newcommand{\M}[1]{{\bm #1}} 

\newcommand{\hide}[1]{}
\renewcommand{\wrt}{w.r.t.\xspace}

\newcommand{\hiddenText}{{\color{gray} hidden text.}}
\newcommand{\hideWithText}[1]{\hiddenText}

\newcommand{\kron}{\otimes}

\newcommand{\subject}{\text{ subject to }}

\DeclareMathOperator*{\argmin}{arg\,min}

\newcommand{\tran}{^{\mathsf{T}}}

\newcommand{\rank}[1]{\mathrm{rank}\left(#1\right)}

\newcommand{\zero}{{\mathbf 0}}
\newcommand{\eye}{{\mathbf I}}

\newcommand{\Real}[1]{ { {\mathbb R}^{#1} } }

\newcommand{\at}[1]{^{(#1)}}

\newcommand{\SOthree}{\ensuremath{\mathrm{SO}(3)}\xspace}

\newcommand{\MA}{\M{A}}
\newcommand{\MB}{\M{B}}

\newcommand{\MM}{\M{M}}

\newcommand{\MQ}{\M{Q}}

\newcommand{\MR}{\M{R}}
\newcommand{\MS}{\M{S}}

\newcommand{\MZ}{\M{Z}}

\newcommand{\vh}{\boldsymbol{h}} 

\newcommand{\vc}{\boldsymbol{c}}

\newcommand{\vg}{\boldsymbol{g}}

\newcommand{\vr}{\boldsymbol{r}}

\newcommand{\vv}{\boldsymbol{v}}
\newcommand{\vt}{\boldsymbol{t}}
\newcommand{\vxx}{\boldsymbol{x}} 
\newcommand{\vy}{\boldsymbol{y}}

\newcommand{\vlambda}{\boldsymbol{\lambda}}

\newcommand{\valpha}{\boldsymbol{\alpha}}

\newcommand{\vepsilon}{\boldsymbol{\epsilon}}

\newcommand{\scenario}[1]{{\smaller \sf#1}\xspace}

\newcommand{\blue}[1]{{\color{blue}#1}}

\newcommand{\linkToPdf}[1]{\href{#1}{\blue{(pdf)}}}
\newcommand{\linkToPpt}[1]{\href{#1}{\blue{(ppt)}}}
\newcommand{\linkToCode}[1]{\href{#1}{\blue{(code)}}}
\newcommand{\linkToWeb}[1]{\href{#1}{\blue{(web)}}}
\newcommand{\linkToVideo}[1]{\href{#1}{\blue{(video)}}}
\newcommand{\award}[1]{\xspace} 

\newcommand{\vz}{\boldsymbol{z}}


%% file: shortcuts.tex
\newcommand{\nchoosek}[2]{\left( \substack{#1 \\ #2} \right)}

\newcommand{\ncksmall}{\left( \substack{n+d \\ d} \right)}

\newcommand{\FGCar}{\scenario{FG3DCar}}

\newcommand{\alternrobust}{\scenario{Altern+Robust}}
\newcommand{\convexrobust}{\scenario{Convex+Robust}}

\newcommand{\PMP}{\scenario{PMP}}
\newcommand{\convexrefine}{\scenario{convex+refine}}
\newcommand{\altern}{\scenario{Altern}}

\newcommand{\supp}{Supplementary Material\xspace}

\newcommand{\bmat}{\left[\begin{array}}
\newcommand{\emat}{\end{array}\right]}

\newcommand{\bmatc}{\begin{array}}
\newcommand{\ematc}{\end{array}}

\newcommand{\sumfeatures}{\sum_{i=1}^N}
\newcommand{\sumbasis}{\sum_{k=1}^K}

\newcommand{\Zero}{\boldsymbol{0}}

\newcommand{\vztilde}{\tilde{\vz}}
\newcommand{\vBtilde}{\tilde{\MB}}
\newcommand{\vzbar}{\bar{\vz}}
\newcommand{\vBbar}{\bar{\MB}}

\newcommand{\proj}{\text{proj}}

\newcommand{\barc}{\bar{c}}
\newcommand{\barcsq}{\bar{c}^2}

\renewcommand{\subject}{s.t.}

\newcommand{\name}{\scenario{Shape$^\star$}}
\newcommand{\namerobust}{\scenario{Shape${\#}$}}

\newcommand{\corank}{\text{corank}}

\newcommand{\mySize}{N}
\newcommand{\myNote}[2]{#2\xspace} 

\newcommand{\iOneToN}{i=1,\ldots,N}

\newcommand{\rorder}{\beta}

\newcommand{\GNC}{\scenario{GNC}}

\newcommand{\revise}[1]{#1\xspace}

%% file: abstract.tex

\begin{abstract}
We study the problem of 3D shape reconstruction from 2D landmarks extracted in a single image. We adopt the 3D deformable shape model and formulate the reconstruction as a joint optimization of the camera pose and the linear shape parameters. 
Our first contribution is to apply Lasserre's hierarchy of convex Sums-of-Squares (SOS) relaxations to solve the shape reconstruction problem and show that the SOS relaxation of~\revise{minimum} order 2 empirically
 solves the original non-convex problem \emph{exactly}. 
 Our second contribution is to exploit the structure of the polynomial in the objective function and find a reduced set of basis monomials for the SOS relaxation that significantly decreases the size of the resulting semidefinite program (SDP) without compromising its accuracy. 
 These two contributions, to the best of our knowledge, lead to the \emph{first certifiably optimal} 
 solver for 3D shape reconstruction, that we name~\name. 
 Our third contribution is to add an outlier rejection layer to~\name using a~\revise{truncated least squares (TLS) robust cost function and leveraging 
  \emph{graduated non-convexity} to solve TLS \emph{without initialization}.}
  The result is a robust reconstruction algorithm, named~\namerobust, that tolerates a large amount of outlier measurements. 
  We evaluate the performance of \name and \namerobust in both simulated and real experiments, showing that \name outperforms local optimization and previous convex relaxation techniques, while \namerobust achieves state-of-the-art performance and is robust against $70\%$ outliers in the~\FGCar dataset.
\end{abstract}

%% file: introduction.tex

\section{Introduction}
\label{sec:introduction}

3D object detection and pose estimation from a single image is a fundamental problem in computer vision. Despite the progress in semantic segmentation~\cite{He17iccv-maskRCNN}, depth estimation~\cite{Lasinger19arXiv-robustMonocularDepthEstimation}, and pose estimation~\cite{Kneip2014ECCV-UPnP,Xiang17RSS-posecnn}, reconstructing the 3D shape and pose of an object from a single image remains a challenging task~\cite{Aubry14cvpr-seeing3Dchairs,Zhou17pami-shapeEstimationConvex,Tome17cvpr-liftfromdeep,Wu16eccv-3DInterpreter,Kolotouros19cvpr-shapeRec,Ramakrishna12eccv-humanPose}.

A typical approach for 3D shape reconstruction is to first detect 2D landmarks in a single image, and then solve a model-based optimization to lift the 2D landmarks to form a 3D model~\cite{Zhou15cvpr,Zhou17pami-shapeEstimationConvex,Ramakrishna12eccv-humanPose,Lin14eccv-modelFitting,Wang14cvpr-humanPose}. For the optimization to be well-posed, the unknown shape is assumed to be a \emph{3D deformable model}, composed by a linear combination of \emph{basis shapes}, handcrafted or learned from a large corpus of training data~\cite{Cootes95cviu}. The optimization then seeks to jointly optimize the coefficients of the linear combination (\emph{shape parameters}) and the \emph{camera pose} to minimize the \emph{reprojection errors} between the 3D model and the 2D landmarks. This model-based paradigm has been successful in several applications such as face recognition~\cite{Blanz03pami-modelFitting,Gu06cvpr-faceAlignment}, car model fitting~\cite{Lin14eccv-modelFitting,Hejrati12nips-objectAnalysis}, and human pose estimation~\cite{Zhou17pami-shapeEstimationConvex,Ramakrishna12eccv-humanPose}.

Despite its long history and broad range of applications, there is still no \emph{globally optimal} solver for the  
non-convex optimization problem arising in 3D shape reconstruction.
Therefore, most existing solutions adopt a local optimization strategy, which alternates between solving for the camera pose and the shape parameters. These techniques, as shown in prior works~\cite{Ramakrishna12eccv-humanPose,Hejrati12nips-objectAnalysis}, require an initial guess for the solution and often get stuck in local minima. In addition, 2D landmark detectors are prone to produce \emph{outliers}, causing existing methods to be 
brittle~\cite{Wang14cvpr-humanPose}.
Therefore, the motivation for this paper is two-fold: (i) to develop a \emph{certifiably optimal} shape reconstruction solver, and (ii) to develop a robust reconstruction algorithm that is insensitive to a large amount of outlier 2D measurements (\eg, $70\%$).

{\bf Contributions.} Our first contribution is to formulate the shape reconstruction problem as a \emph{polynomial optimization} problem and apply Lasserre's hierarchy of \emph{Sums-of-Squares} (SOS) relaxations to relax the non-convex polynomial optimization into~\revise{a} convex \emph{semidefinite program} (SDP). We show the SOS relaxation 
of~\revise{minimum} order 2 empirically solves the non-convex shape reconstruction problem exactly and provides a global optimality certificate. 
The second contribution is to apply \emph{basis reduction}, a technique that exploits the sparse structure of the  polynomial in the objective function, to reduce the size of the resulting SDP. We show that basis reduction significantly improves the efficiency of the SOS relaxation without compromising global optimality. To the best of our knowledge, this is the first certifiably optimal solver for shape reconstruction, and we name it \name. 
Our third contribution is to robustify \name by adopting a~\revise{truncated least squares (TLS)} robust cost function and solving the resulting 
robust estimation problem using graduated non-convexity~\cite{Blake1987book-visualReconstruction}.
The resulting algorithm, named \namerobust, is robust against $70\%$ outliers and does not require an initial guess.

The rest of this paper is organized as follows. Section~\ref{sec:relatedWork} reviews related work. 
Section~\ref{sec:notationandPreliminaries} introduces notation and preliminaries on SOS relaxations. 
Section~\ref{sec:problemStatement} introduces the shape reconstruction problem.
Section~\ref{sec:certifiableNonminimalSolver} develops our SOS solver (\name).
Section~\ref{sec:robustOutlierRejection} presents an algorithm (\namerobust) to robustify the SOS relaxation against outliers.
 Section~\ref{sec:experiments} provides experimental results in both simulations and real datasets, while 
Section~\ref{sec:conclusions} concludes the paper. 


%% file: relatedWork.tex
\section{Related Work}
\label{sec:relatedWork}
We limit our review to optimization-based approaches for 3D shape reconstruction from 2D landmarks. The 
interested reader can find a review of end-to-end shape and pose reconstruction using deep learning in~\cite{Kolotouros19cvpr-shapeRec,Tome17cvpr-liftfromdeep,Kolotouros19iccv-humanlearnplusmodel}.

{\bf Local Optimization.} Most existing methods resort to local optimization to solve the non-convex joint optimization of shape parameters and camera pose. Blanz and Vetter~\cite{Blanz03pami-modelFitting} propose a method for face recognition by fitting a morphable model of the 3D face shape and texture to a single image using stochastic Newton's method to escape local minima. Gu and Kanade~\cite{Gu06cvpr-faceAlignment} align a deformable point-based 3D face model by alternatively deforming the 3D model and updating the 3D pose. Using similar alternating optimization, Ramakrishna~\etal~\cite{Ramakrishna12eccv-humanPose} tackle 3D human pose estimation by finding a sparse set of basis shapes from an over-complete human shape dictionary using projected matching pursuit;
the approach is further improved by Fan~\etal~\cite{Fan14eccv-plcrHumanPose} to include pose locality constraints. 
Lin~\etal~\cite{Lin14eccv-modelFitting} demonstrate joint 3D car model fitting and fine-grained classification;
 car model fitting in cluttered images is investigated in~\cite{Hejrati12nips-objectAnalysis}. To mitigate the impact of outlying 2D landmarks, Li~\etal~\cite{Li2011PAMI-RobustRANSACShapeAlignment} propose a RANSAC-type method for car model fitting and Wang~\etal~\cite{Wang14cvpr-humanPose} replace the least squares estimation with an $\ell_1$-norm minimization.

{\bf Convex Relaxation.} More recently, Zhou~\etal~\cite{Zhou15cvpr} develop a convex relaxation, where they first over-parametrize the 3D deformable shape model by associating one rotation with each basis and then relax the resulting Stiefel manifold constraint to its convex envelope. Although showing superior performance compared to local optimization, the convex relaxation in~\cite{Zhou15cvpr} comes with no optimality guarantee and is typically loose in practice. In addition, Zhou~\etal~\cite{Zhou17pami-shapeEstimationConvex} model outliers using a sparse matrix and augment the optimization with an $\ell_1$ regularization to achieve robustness against $40\%$ outliers. 
In contrast, we will show that our convex relaxation comes with certifiable optimality, 
and our robust reconstruction approach can handle $70\%$ outliers.

%% file: notationandPreliminaries.tex
\section{Notation and Preliminaries}
\label{sec:notationandPreliminaries}
We use $\calS^n$ to denote the set of $n\times n$ symmetric matrices. 
We write $\MA \in \calS^n_{+}$ (resp. $\MA \in \calS^n_{++}$) to denote that the matrix $\MA \in \calS^n$ is positive semidefinite (PSD) (resp. positive definite (PD)). 
Given $\vxx = [x_1,\dots,x_n]\tran$, we let $\Real{}[\vxx]$ (resp. $\Real{}[\vxx]_{d}$) be the ring of polynomials in $n$ variables with real coefficients (resp. with degree at most $d$), and $[\vxx]_d$ be the vector of all $\ncksmall$ monomials with degree up to $d$. 

We now give a brief summary of SOS relaxations for polynomial optimization. 
Our review is based on~\cite{Blekherman12Book-sdpandConvexAlgebraicGeometry,Nie14mp-finiteConvergenceLassere,lasserre10book-momentsOpt}.
We first introduce the notion of SOS polynomial.

\begin{definition}[SOS Polynomial~\cite{Blekherman12Book-sdpandConvexAlgebraicGeometry}] A polynomial $p(\vxx) \!\!\in\!\! \Real{}[\vxx]_{2d}$ is said to be a sums-of-squares (SOS) polynomial if there exist polynomials $q_1,\ldots,q_m\!\! \in \!\!\Real{}[\vxx]_d$ such that:
\bea \label{eq:SOSdefinition}
p(\vxx) = \sum_{i=1}^m q_i^2(\vxx).
\eea
We use $\Sigma_n$ (resp. $\Sigma_{n,2d})$ to denote the set of SOS polynomials in $n$ variables (resp. with degree at most $2d$). A polynomial $p(\vxx) \in \Real{}[\vxx]_{2d}$ is SOS if and only if there exists a PSD matrix 
\myNote{$\MQ \in \calS^{\ncksmall}_{+}$}{
	$\MQ \in \calS^{\mySize_Q}_{+}$ with $\mySize_Q = \ncksmall$,
} such that:
\bea \label{eq:PSDDescriptionSOS}
p(\vxx) = [\vxx]_d\tran \MQ [\vxx]_d,
\eea
and $\MQ$ is called the Gram matrix of $p(\vxx)$.
\end{definition}

Now consider the following polynomial optimization:
\bea \label{eq:polynomialOptGeneric}
\min_{\vxx \in \Real{n}} & f(\vxx) \\
\subject & h_i(\vxx) = 0, i=1,\dots,m, \nonumber \\
& g_k(\vxx) \geq 0, k=1,\dots,l, \nonumber
\eea
where $f,h_i,g_k \in \Real{}[\vxx]$ are all polynomials and let $\calX$ be the feasible set defined by $h_i,g_k$. For convenience, denote $\vh := (h_1,\dots,h_m)$, $g_0:=1$ and $\vg = (g_0,\dots,g_l)$. 
We call
\bea
& \langle \vh \rangle := \{h\in \Real{}[\vxx]: h = \sum_{i=1}^m \lambda_i h_i, \lambda_i \in \Real{}[\vxx]\}, \label{eq:ideal}\\
& \langle \vh \rangle_{2\rorder} := \{h \in \langle \vh \rangle: \text{deg}(\lambda_i h_i) \leq 2\rorder \}, \label{eq:2tideal}
\eea
the \emph{ideal} and the \emph{$2\rorder$-th truncated ideal} of $\vh$, where $\text{deg}(\cdot)$ is the degree of a polynomial.
The \emph{ideal} is simply a summation of polynomials with polynomial coefficients, a construct that will simplify the notation later on.
 We call
\bea
& Q(\vg) := \{ g \in \Real{}[\vxx]: g = \sum_{k=0}^m s_k g_k, s_k \in \Sigma_n \}, \label{eq:qmodule} \\
& Q_\rorder(\vg) := \{g \in Q(\vg): \text{deg}(s_k g_k) \leq 2\rorder \}, \label{eq:tqmodule}
\eea
the \emph{quadratic module} and the \emph{$\rorder$-th truncated quadratic module} generated from $\vg$. 
Note that the quadratic module is similar to the \emph{ideal}, except now we require the polynomial coefficients to 
be SOS.
%
Apparently, if $p(\vxx) \in \langle \vh \rangle + Q(\vg)$, then $p(\vxx)$ is nonnegative on $\calX$\footnote{If $p \in \langle \vh \rangle + Q(\vg)$, then $p = h + g$, with $h \in \langle \vh \rangle$ and $g \in Q(\vg)$. For any $\vxx \in \calX$, since $h_i(\vxx)=0$, so $h(\vxx) = \sum \lambda_i h_i = 0$; since $g_k(\vxx) \geq 0$ and $s_k(\vxx) \geq 0$, so $g = \sum s_k g_k \geq 0$. Therefore, $p = h+g \geq 0$}. Putinar’s Positivstellensatz~\cite{Putinar93IUMJ-putinarCertificate} describes when the reverse is also true.

\begin{theorem}[Putinar's Positivstellensatz~\cite{Putinar93IUMJ-putinarCertificate}] \label{thm:putinarPositivstellensatz}
Let $\calX$ be the feasible set of problem~\eqref{eq:polynomialOptGeneric}. Assume $\langle \vh \rangle + Q(\vg)$ is Archimedean, \ie, $M - \|\vxx\|_2^2 \in \langle \vh \rangle_{2\rorder} + Q_\rorder(\vg)$ for some $\rorder \in \mathbb{N}$ and $ M>0 $. If $p(\vxx) \in \Real{}[\vxx]$ is positive on $\calX$, then $p(\vxx) \in \langle \vh \rangle + Q(\vg)$.
\end{theorem}

Based on Putinar's Positivstellensatz, Lasserre~\cite{Lasserre01siopt-LasserreHierarchy} derived a sequence of SOS relaxations that approximates the global minimum of problem~\eqref{eq:polynomialOptGeneric} with increasing accuracy.
The key insight behind Lasserre's hierarchy is twofold.
 The first insight is that problem~\eqref{eq:polynomialOptGeneric}, which we can write succinctly as
 $\min_{\vxx \in \calX} f(x)$, can be equivalently written as 
$\displaystyle\max_{\vxx, \gamma} \gamma, \subject f(x) - \gamma \geq 0 \text{ on } \calX$
(intuition: the latter pushes the lower bound $\gamma$ to reach the global minimum of $f(\vxx)$).
The second intuition is that we can rewrite the condition $f(x) - \gamma \geq 0 \text{ on } \calX$, using 
Putinar's Positivstellensatz (Theorem~\ref{thm:putinarPositivstellensatz}), leading to the following hierarchy of Sums-of-Squares relaxations.

\begin{theorem}[Lasserre's Hierarchy~\cite{Lasserre01siopt-LasserreHierarchy}] \label{thm:lasserreHierarchy} Lasserre's hierarchy of order $\rorder$ is the following SOS program:
\bea \label{eq:LasserreHierarchyordert}
\max \quad \gamma, \quad \subject \;\; f(\vxx) - \gamma \in \langle \vh \rangle_{2\rorder} + Q_\rorder(\vg),
\eea
which can be written as a standard SDP. 
Moreover, let $f^{\star}$ be the global minimum of~\eqref{eq:polynomialOptGeneric} and $f^{\star}_\rorder$ be the optimal value of~\eqref{eq:LasserreHierarchyordert}, then $f^{\star}_\rorder$ monotonically increases and $f^{\star}_\rorder \rightarrow f^{\star}$ when $\rorder \rightarrow \infty$. More recently, Nie~\cite{Nie14mp-finiteConvergenceLassere} proved that under Archimedeanness, Lasserre’s hierarchy has finite convergence generically (\ie, $f^{\star}_\rorder = f^{\star}$ for some finite $\rorder$).
\end{theorem}

In computer vision, Lasserre's hierarchy was first used by Kahl and Henrion~\cite{Kahl07IJCV-GlobalOptGeometricReconstruction} to minimize rational functions arising in geometric reconstruction problems, and more recently by Probst \etal~\cite{Probst19ICCV-convexRelaxationNonminimal} as a framework to solve 
a set of 3D vision problems. 
In this paper we will show that the SOS relaxation as written in eq.~\eqref{eq:LasserreHierarchyordert} allows using \emph{basis reduction} to exploit the \emph{sparsity pattern} of polynomials and leads to significantly smaller semidefinite programs.                                                                                                                                                                                                                                                                                                                                                                                                                                                                                                                                                                                                                                                                                                                                                                                                                                                                                                

%% file: problemStatement.tex
\section{Problem Statement: Shape Reconstruction}
\label{sec:problemStatement}

Assume we are given $N$ pixel measurements $\MZ = [\vz_1,\dots,\vz_N] \in \Real{2 \times N}$ 
(the 2D \emph{landmarks}), generated from 
the projection of points belonging to an unknown 3D shape $\MS\in\Real{3 \times N}$ onto an image.
Further assume the shape $\MS$ that can be represented as a linear combination of $K$ predefined basis shapes $\MB_k \in \Real{3 \times N}$, \ie $\MS = \sumbasis c_k \MB_{k}$, where $\{c_k\}_{k=1}^K$ are (unknown) 
\emph{shape coefficients}. Then, the generative model of the 2D landmarks reads:
\bea \label{eq:basisgenmodel}
\vz_i = \Pi \MR  \left( \sumbasis c_k \MB_{ki} \right)   + \vt + \vepsilon_i,i=1,\dots,N,
\eea
where $\MB_{ki}$ denotes the $i$-th 3D point on the $k$-th basis shape, 
$\vepsilon_i \in \Real{2}$ models the measurement noise, 
and $\Pi$ is the (known) weak perspective projection matrix:
\bea \label{eq:weakprojectionmatrix}
\Pi = \bmat{ccc} 
s_x & 0 & 0 \\ 0 & s_y & 0
\emat,
\eea
with $s_x$ and $s_y$ being constants\footnote{\revise{The weak perspective camera model is a good approximation of the full perspective camera model when \emph{the distance from the object to the camera} is much larger than \emph{the depth of the object itself~\cite{Zhou15cvpr}}.~\cite{Zhou18PAMI-monocap} showed that the solution obtained using the weak perspective model provides a good initialization when refining the pose for the full perspective model. } }.
In eq.~\eqref{eq:basisgenmodel}, $\MR \in \SOthree$ and $\vt \in \Real{2}$ model the (unknown) rotation and translation of the shape $\MS$ relative to the camera (only a 2D translation can be estimated).  \emph{The shape reconstruction problem consists in the joint estimation of the shape parameters $\{c_k\}_{k=1}^K$ and the camera pose $(\MR, \vt)$}\footnote{\revise{Shape reconstruction in the case of a single 3D model,~\ie,~$K=1$, is called shape alignment and has been solved recently in~\cite{Yang20ral-GNC}.}}. 

Without loss of generality, we adopt the \emph{nonnegative sparse coding} (NNSC) convention~\cite{Zhou17pami-shapeEstimationConvex} and assume all the coefficients $c_k$ are nonnegative\footnote{\revise{The general case of real coefficients is equivalent to the NNSC case where for each basis $\MB_k$ we also add the basis $-\MB_k$}.}. Due to the existence of noise, we solve the following \emph{weighted least squares optimization} with \emph{Lasso ($\ell_1$-norm) regularization}:
\small
\bea \label{eq:weightedleastsquares}
\hspace{-5mm}
\min_{\substack{c_k \geq 0, k=1,\dots,K \\ \vt \in \Real{2}, \MR\in \SOthree} } & 
\!\!\displaystyle \!\!\sumfeatures w_i \!\! \left\| \vz_i \!-\! \Pi \MR \left( \sumbasis c_k \MB_{ki} \right) \!-\! \vt \right\|^2 \!\!\!\!+\! \alpha \!\!\sumbasis \left| c_k \right| 
\eea
\normalsize
The $\ell_1$-norm regularization (controlled by a given constant $\alpha$) encourages the coefficients $c_k$ to be sparse when the shape $\MS$ is generated from only a subset of the basis shapes~\cite{Zhou17pami-shapeEstimationConvex} (note that the 
$\ell_1$-norm becomes redundant when using the NNSC convention). 
Contrary to previous approaches~\cite{Zhou17pami-shapeEstimationConvex,Ramakrishna12eccv-humanPose}, we explicitly associate a given weight $w_i \geq 0$ to each 2D measurement $\vz_i$ in eq.~\eqref{eq:weightedleastsquares}. On the one hand, this allows  
 accommodating heterogeneous noise in the 2D landmarks~\revise{(\eg,~$w_i = 1/\sigma_i^2$ when the noise $\vepsilon_i$ is Gaussian, $\vepsilon_i \sim \calN(\zero,\sigma_i^2 \eye_2)$)}.
On the other hand, as shown in Section~\ref{sec:robustOutlierRejection}, the weighted least squares framework 
is useful to robustify~\eqref{eq:weightedleastsquares} against outliers. 

%% file: certifiableNonminimalSolver.tex
\section{Certifiably Optimal Shape Reconstruction}
\label{sec:certifiableNonminimalSolver}

This section shows how to develop a \emph{certifiably optimal} solver for problem~\eqref{eq:weightedleastsquares}. Our first step is to algebraically eliminate the translation $\vt$ and obtain a translation-free shape reconstruction problem, as shown below.

\begin{theorem}[Translation-free Shape Reconstruction]\label{thm:trans-freeRecon}
The shape reconstruction problem~\eqref{eq:weightedleastsquares} is equivalent to the following translation-free optimization:
\small
\bea \label{eq:trans-freeshaperecon}
\min_{\substack{c_k \geq 0, k=1,\dots,K \\ \MR\in \SOthree} }  & 
\!\!\! \displaystyle \sumfeatures \left\| \vztilde_i \!-\! \Pi \MR \left( \sumbasis c_k \vBtilde_{ki} \right)  
 \right\|^2 \!\! +\! \alpha \sumbasis | c_k | 
\eea
\normalsize
where $\vztilde_i$ and $\vBtilde_{ki}$ can be computed as follows:
\bea
& \hspace{-3mm} \vztilde_i = \sqrt{w_i}(\vz_i - \vzbar^w), \text{ with } \vzbar^w = \frac{ \sumfeatures w_i \vz_i }{\sumfeatures w_i}, \\
& \hspace{-3mm} \vBtilde_{ki} = \sqrt{w_i}(\MB_{ki} - \vBbar^w_{k}), \text{ with } \vBbar^w_k = \frac{\sumfeatures w_i \MB_{ki} }{ \sumfeatures w_i }.
\eea
Further, let $\MR^{\star}$ and $c_k^{\star},k=1,\dots,K,$ be the global minimizer of the above translation-free optimization~\eqref{eq:trans-freeshaperecon}, then the optimal translation $\vt^\star$ can be recovered as:
\bea \label{eq:recovertfromRc}
\vt^{\star} = \vzbar^w - \Pi \MR^\star \left( \sumbasis c_k^\star \vBbar^w_k \right).
\eea
\end{theorem}

A formal proof of Theorem~\ref{thm:trans-freeRecon} can be found in the \supp. The intuition behind 
Theorem~\ref{thm:trans-freeRecon} is that if we express the landmark coordinates and 3D basis shapes 
with respect to their (weighted) centroids $\vzbar^w$ and $\vBbar^w_k,k=1,\dots,K$, we can remove the 
dependence on the translation $\vt$.
 This strategy is inspired by Horn's method for point cloud registration~\cite{Horn87josa}, 
 and generalizes~\cite{Zhou17pami-shapeEstimationConvex} to the weighted and non-centered case.

\subsection{SOS Relaxation}
\label{sec:sosRelaxation}
This section  applies Lasserre's hierarchy as described in Theorem~\ref{thm:lasserreHierarchy} to solve the translation-free shape reconstruction problem~\eqref{eq:trans-freeshaperecon}. We do this in two steps: we first show problem~\eqref{eq:trans-freeshaperecon} can be formulated as a polynomial optimization in the form~\eqref{eq:polynomialOptGeneric}; and then we add valid constraints to make the feasible set Archimedean.

{\bf Polynomial Optimization Formulation.} Denote $\vc = [c_1,\dots,c_k]\tran \in \Real{K}$, $\vr = \text{vec}(\MR) = [\vr_1\tran,\vr_2\tran,\vr_3\tran]\tran \in \Real{9}$, with $\vr_i,i=1,2,3$ being the $i$-th column of $\MR$, then $\vxx := [\vc\tran,\vr\tran]\tran \in \Real{K+9}$ is the unknown decision vector in~\eqref{eq:polynomialOptGeneric}. 
Consider the first term in the objective function of~\eqref{eq:trans-freeshaperecon}. We can write:
\bea \label{eq:costqi}
\scriptstyle q_i(\vxx) := \left\| \vztilde_i - \Pi \MR \left( \sumbasis c_k \vBtilde_{ki} \right)   \right\|^2 = \left\| \vztilde_i - \Pi \sumbasis c_k \MR \vBtilde_{ki}   \right\|^2,
\eea
then it becomes clear that $q_i(\vxx)$ is a polynomial function of $\vxx$ with degree 4. 
Because the Lasso regularization is linear in $\vc$, the objective function $f(\vxx)$ is a degree-4 polynomial.

Now we consider the feasible set of~\eqref{eq:trans-freeshaperecon}. The inequality constraints $c_k \geq 0$ are already in generic form~\eqref{eq:polynomialOptGeneric} with $g_k(\vxx) = c_k,k=1,\dots,K$, being degree-1 polynomials.
As for the $\MR \in \SOthree$ constraint, 
it has already been shown in related work~\cite{Tron15rssws3D-dualityPGO3D-duplicate,Briales17cvpr-registration}
that enforcing $\MR \in \SOthree$ is equivalent to imposing 15 \revise{quadratic} equality constraints.

\begin{lemma}[Quadratic Constraints for $\SOthree$~\cite{Tron15rssws3D-dualityPGO3D-duplicate,Briales17cvpr-registration}] For a matrix $\MR \in \Real{3 \times 3}$, 
the constraint $\MR \in \SOthree$ (where $\SOthree :=\{\MR : \MR\tran\MR=\eye_3, \det{\MR}=+1\}$ is the set of proper rotation matrices) is equivalent to the following set of degree-2 polynomial equality constraints
($h_i(\vxx) = 0, i=1,\dots,15$):
\bea
\begin{cases} \label{eq:polyConstraintsSOthree}
h_1 = 1-\| \vr_1 \|^2,h_2 = 1-\| \vr_2 \|^2 ,h_3 = 1-\| \vr_3 \|^2 \\
h_4 = \vr_1\tran \vr_2, \quad h_5 = \vr_2\tran \vr_3, \quad h_6 = \vr_3\tran \vr_1 \\
h_{7,8,9} = \vr_1 \times \vr_2 - \vr_3 \\
h_{10,11,12} = \vr_2 \times \vr_3 - \vr_1 \\
h_{13,14,15} = \vr_3 \times \vr_1 - \vr_2
\end{cases}
\eea
where $\vr_i \in \Real{3},i=1,2,3$, denotes the $i$-th column of $\MR$ and ``$\times$'' represents the vector cross product.
\end{lemma}
In eq.~\eqref{eq:polyConstraintsSOthree}, $h_{1,2,3}$ constrain the columns to be unit vectors, $h_{4,5,6}$ constrain the columns to be mutually orthogonal, and $h_{7-15}$ constrain the columns to satisfy the right-hand rule (\ie, the determinant constraint)\footnote{\revise{We remark that the 15 equality constraints in~\eqref{eq:polyConstraintsSOthree} are redundant. For example, $h_{1,2,3,7,8,9}$ are sufficient to fully constrain $\MR \in \SOthree$. We also found that, empirically, choosing $h_{1,2,3}$ and $h_{7-15}$ yields similar tightness results as choosing all 15 constraints.}}. 

In summary, the translation-free problem~\eqref{eq:trans-freeshaperecon} is equivalent to a polynomial optimization with a degree-4 objective $f(\vxx)$, constrained by 15 \revise{quadratic} equalities $h_i(\vxx)$ (eq.~\eqref{eq:polyConstraintsSOthree}) and $K$ \revise{linear} inequalities $g_k(\vxx)=c_k$.

{\bf Archimedean Feasible Set.} The issue with the feasible set defined by inequalities $c_k\geq 0$ and equalities~\eqref{eq:polyConstraintsSOthree} is that $\langle \vh \rangle + Q(\vg)$ is \emph{not} Archimedean, which can be easily seen from the \emph{unboundedness} of the linear inequality $c_k \geq 0$\footnote{$M - \| \vxx \|_2^2 \geq 0$ requires $\vxx$ to have bounded $\ell_2$-norm. }. However, we know the linear coefficients must be bounded because the pixel measurement values $\MZ$ lie in a bounded set (the 2D image). Therefore, we propose to \emph{normalize} the 2D measurements and the 3D basis shapes: (i) for 2D measurements $\MZ$, we first divide them by $s_x$ and $s_y$ (eq.~\eqref{eq:weakprojectionmatrix}), and then scale them such that they lie inside a unit circle; (ii) for each 3D basis shape $\MB_k$, we scale $\MB_k$ such that it lies inside a unit sphere. With this proper normalization, we can add the following degree-2 inequality constraints ($c_k^2\leq1$) that bound the linear coefficients:
\bea \label{eq:boundlinearcoeff}
g_{K+k}(\vxx) = 1 - c_k^2,k=1,\dots,K.
\eea
Now we can certify the Archimedeanness of $\langle \vh \rangle + Q(\vg)$:
\bea \label{eq:certifyArchimedean}
K+3 - \| \vxx \|_2^2 = \underbrace{ \sumbasis 1 \cdot g_{K+k}}_{\in Q_1(\vg)} + \underbrace{ h_1 +  h_2 +  h_3 }_{\in \langle \vh \rangle_2},
\eea
with $M=K+3$ and $\rorder=1$ (\cf Theorem~\ref{thm:putinarPositivstellensatz}).

{\bf Apply Lasserre's Hierarchy.} With Archimedeanness, we can now apply Lasserre's hierarchy of SOS relaxations.

\begin{proposition}[SOS Relaxations for Shape Reconstruction] \label{prop:naiveSOSRelax}
The SOS relaxation of order $\rorder$ ($\rorder \geq 2$)\footnote{The minimum relaxation order is 2 because $f(\vxx)$ has degree 4.} for the translation-free shape reconstruction problem~\eqref{eq:trans-freeshaperecon} is the following convex semidefinite program:
\bea \label{eq:naiveSOS}
\max_{ \substack{\gamma \in \Real{}, \MS_0 \in \calS^{\mySize_0}_+
\\ \MS_k \in \calS_{+}^{\mySize_s}, \; k=1,\dots,2K
\\ \vlambda_i \in \Real{ \mySize_{\lambda} }, i=1,\dots,15 } } & \gamma \\
\label{eq:naiveSOSLinearConstraints} \subject & 
f(\vxx) - \gamma = [\vxx]_\rorder\tran \MS_0 [\vxx]_\rorder + \nonumber \\
& \sum_{k=1}^{2K} \left( [\vxx]_{\rorder-1}\tran \MS_k [\vxx]_{\rorder-1} \right) g_k(\vxx) + \nonumber \\
& \sum_{i=1}^{15} \left(\vlambda_i\tran [\vxx]_{2\rorder-2}\right) h_i(\vxx),
\eea
where $f(\vxx)$ is the objective function defined in~\eqref{eq:trans-freeshaperecon}, $g_k(\vxx),k=1,\dots,2K$ are the inequality constraints $c_k,1-c_k^2$, $h_i(\vxx),i=1,\dots,15$ are the equality constraints defined in~\eqref{eq:polyConstraintsSOthree}, and 
$\mySize_0:=\nchoosek{K+9+\rorder}{\rorder}$, $\mySize_s:= \nchoosek{K+8+\rorder}{\rorder-1}$, $\mySize_{\lambda} := \nchoosek{K+7+2\rorder}{2\rorder-2}$ are the sizes of matrices and vectors.
\end{proposition}

While a formal proof of Proposition~\ref{prop:naiveSOSRelax} is given in the~\supp, we observe that~\eqref{eq:naiveSOS} immediately results from the application of Lasserre's hierarchy to~\eqref{eq:LasserreHierarchyordert}, by parametrizing $Q_\rorder(\vg)$ with monomial bases $[\vxx]_{\rorder-1}$, $[\vxx]_{\rorder}$ and PSD matrices $\MS_0$, $\MS_k,k=1,\dots,2K$ (one for each $g_k$), and by parametrizing $\langle \vh \rangle_{2\rorder}$ with monomial basis $[\vxx]_{2\rorder-2}$ and coefficient vectors $\vlambda_i,i=1,\dots,15$ (one for each $h_i$). Problem~\eqref{eq:naiveSOS} can be written as an 
SDP and solved globally using standard convex solvers (\eg YALMIP~\cite{Lofberg04cacsd-yalmip}). 
 We call the SDP written in~\eqref{eq:naiveSOS} the \emph{primal} SDP. The \emph{dual} SDP of~\eqref{eq:naiveSOS} can be derived using \emph{moment relaxation}~\cite{Lasserre01siopt-LasserreHierarchy,Laurent09eaag-SOSMomentOptimization,lasserre10book-momentsOpt}, which is readily available in GloptiPoly 3~\cite{henrion2009optimmethodsoftw-gloptipoly3}.

{\bf Extract Solutions from SDP.} After solving the SDP~\eqref{eq:naiveSOS}, we can extract solutions to the original non-convex problem~\eqref{eq:trans-freeshaperecon}, a procedure we call \emph{rounding}.
\begin{proposition}[Rounding and Duality Gap] \label{prop:extractSolutions}
Let $f^\star_\rorder = \gamma^\star$ and $\MS_0^{\rorder\star},\MS_k^{\rorder\star},\vlambda_i^{\rorder\star}$ be the optimal solutions to the SDP~\eqref{eq:naiveSOS} at order $\rorder$; compute $\vv^{\rorder\star}$ as the eigenvector corresponding to the minimum eigenvalue of $\MS_0^{\rorder\star}$, and then normalize $\vv^{\rorder\star}$ such that the first entry is equal to 1. Then an approximate solution to problem~\eqref{eq:trans-freeshaperecon} can be obtained as:
\bea
\hat{\vc}^\rorder = \text{proj}_{\vg}([\vv^{\rorder\star}]_{\vc}); \quad \hat{\vr}^\rorder = \text{proj}_{\vh}([\vv^{\rorder\star}]_{\vr}),
\eea
where $[\vv^{\rorder\star}]_{\vc}$ (resp. $[\vv^{\rorder\star}]_{\vr}$) denotes the entries of $\vv^{\rorder\star}$ corresponding to monomials $\vc$ (resp. $\vr$), and $\proj_{\vg}$ (resp. $\proj_{\vh}$) denotes projection to the feasible set defined by $\vg$ (resp. $\vh$). Specifically for problem~\eqref{eq:trans-freeshaperecon}, $\proj_{\vg}$ is rounding each coefficient $c_k$ to the $[0,1]$ interval, and $\proj_{\vh}$ is the projection to $\SOthree$. 
Moreover, let $\hat{f}_{\rorder}$ be the value of the objective function evaluated at the approximate solution $\hat{\vxx}^\rorder := [(\hat{\vc}^\rorder)\tran, (\hat{\vr}^\rorder)\tran]\tran$, 
then the following inequality holds (weak duality):
\bea
f^\star_\rorder \leq f^\star \leq \hat{f}_\rorder,
\eea
where $f^\star$ is the true (unknown) global minimum of problem~\eqref{eq:trans-freeshaperecon}. 
We define the relative duality gap $\eta_\rorder$ as:
\bea
\eta_\rorder = (\hat{f}_\rorder - f^\star_\rorder) / \hat{f}_\rorder,
\eea
which quantifies the quality of the SOS relaxation.
\end{proposition}

{\bf Certifiable Global Optimality.} Besides extracting solutions to the original problem, we can also verify when the SOS relaxation solves the original problem \emph{exactly}.
\revise{
\begin{theorem}[Certificate of Global Optimality] \label{thm:certificateGlobalOptimality} Let $f^\star_\rorder = \gamma^\star$ and $\MS_0^{\rorder \star}$ be the optimal solutions to the SDP~\eqref{eq:naiveSOS} at order $\rorder$. If $\text{corank}(\MS_0^{\rorder \star}) = 1$ (the corank is the dimension of the null space of $\MS_0^{\rorder \star}$), 
then $f^\star_\rorder$ is the global minimum of problem~\eqref{eq:trans-freeshaperecon}, and the relaxation is said to be \emph{tight} at order $\rorder$. Moreover, the relative duality gap $\eta_{\rorder}=0$ and the solution $\hat{\vxx}^{\rorder}$ extracted using Proposition~\ref{prop:extractSolutions} is the unique global minimizer of problem~\eqref{eq:trans-freeshaperecon}.
\end{theorem}
}


The proof of Theorem~\ref{thm:certificateGlobalOptimality} is given in the \supp. Empirically (Section~\ref{sec:experiments}), we observed that the relaxation is always tight at the minimum relaxation order $\rorder=2$.
Note that even when the relaxation is not tight, one can still obtain an approximate solution using Proposition~\ref{prop:extractSolutions} and quantify how suboptimal the approximate solution is using the relative duality gap $\eta_\rorder$.


%% file: basisReduction.tex

\subsection{Basis Reduction}
\label{sec:basisReduction}
Despite the theoretical soundness and finite convergence at order $\rorder=2$, the size of the SDP~\eqref{eq:naiveSOS} is $\mySize_0=\nchoosek{K+9+\rorder}{\rorder}$,
 which for $\rorder=2$ becomes $\nchoosek{K+11}{2}$, implying that the size of the SDP grows \emph{quadratically} in the number of bases $K$. Although there have been promising advances in improving the scalability of SDP solvers (see~\cite{majumdar19arXiv-surveySDPScalability} for a thorough review), such as exploiting sparsity~\cite{Weisser18mpc-SBSOS,Waki06jopt-SOSSparsity,Mangelson19icra-PGOwithSBSOS} and low-rankness~\cite{Burer03mp,Rosen19ijrr-sesync-duplicate}, in this section we demonstrate a simple yet effective approach, called \emph{basis reduction}, that exploits the structure of the objective function to significantly reduce the size of the SDP in~\eqref{eq:naiveSOS}. 

In a nutshell, basis reduction methods seek to find a smaller, but still expressive enough,  subset of the full vector of monomials $[\vxx]_\rorder$ on the right-hand side (RHS) of eq.~\eqref{eq:naiveSOSLinearConstraints}, to explain the objective function $f(\vxx)$ on the left-hand side (LHS). There exist standard approximation algorithms for basis reduction, discussed in~\cite{Permenter18mp-partialFacialReduction,Permenter14cdc-basisSelectionSOS} and implemented in YALMIP~\cite{lofberg09tac-PrePostProcessSOS}. However, in practice we found the basis selection method in YALMIP failed to find any reduction for the SDP~\eqref{eq:naiveSOS}. Therefore, here we propose a problem-specific reduction, 
which follows from the examination of which monomials appear on the LHS of~\eqref{eq:naiveSOSLinearConstraints}.

\begin{proposition}[SOS Relaxation with Basis Reduction]\label{prop:sosBasisReduction}
The SOS relaxation of order $\rorder=2$ with basis reduction for the translation-free shape reconstruction problem~\eqref{eq:trans-freeshaperecon} is the following convex semidefinite program:
\bea \label{eq:SOSBasisReduction}
\max_{ 
\substack{\gamma \in \Real{}, \MS_0 \in \calS_+^{\mySize_0'}
\\ \MS_k \in \calS_+^{\mySize_s'}, k=1,\dots,2K, 
\\ \vlambda_i \in \Real{\mySize_\lambda'}, i=1,\dots,15} } 
& \gamma \\
\subject & f(\vxx) - \gamma = m_2(\vxx)\tran \MS_0 m_2(\vxx) +  \nonumber \\
& \sum_{k=1}^{2K} ( [\vr]_1\tran \MS_k [\vr] ) g_k(\vxx) + \nonumber \\
& \sum_{i=1}^{15} (\vlambda_i\tran [\vc]_2) h_i(\vxx),
\eea
where 
$\mySize_0' = 10K+10$, $\mySize_s'=10$, 
$\mySize_\lambda' = \nchoosek{K+2}{2}$, and
 $m_2(\vxx) = [1,\vc\tran,\vr\tran,\vc\tran \kron \vr\tran]\tran \in \Real{\mySize_0}$, 
 and where $\kron$ is the Kronecker product.
\end{proposition}

\revise{Comparing the SDP~\eqref{eq:SOSBasisReduction} and~\eqref{eq:naiveSOS}, the most significant change is replacing the full monomial basis $[\vxx]_\rorder$ in~\eqref{eq:naiveSOS} with a much smaller monomial basis $m_2(\vxx)$ that excludes degree-2 monomials purely supported in $\vc$ and $\vr$. This reduction is motivated by analyzing the monomial terms in $f(\vxx)$. Although a formal proof of the equivalence between~\eqref{eq:naiveSOS} and~\eqref{eq:SOSBasisReduction} remains open, we provide an intuitive explanation in the \supp.  After basis reduction, the size of the SDP~\eqref{eq:SOSBasisReduction} is $\mySize_0'=10K+10$, which is \emph{linear} in $K$ and much smaller than the size of the original SDP~\eqref{eq:naiveSOS} $\mySize_k=\nchoosek{K+11}{2}$\footnote{For $K=5,10,20$, $\mySize_0=120,210,465$, while $\mySize_0'=60,110,210$.}. Section~\ref{sec:experiments} numerically shows that the SDP after basis reduction gives the same (tight) solution as the original SDP.} 


\subsection{\name: Algorithm Summary}
To summarize the derivation in this section, our solver for the shape reconstruction problem~\eqref{eq:weightedleastsquares}, named 
\name, works as follows. It first solves the SDP~\eqref{eq:SOSBasisReduction} and applies the rounding described 
in Proposition~\ref{prop:extractSolutions} to compute an estimate of the shape parameters $c_k$ and rotation $\MR$ and possibly certify its optimality. Then, \name uses the closed-form expression~\eqref{eq:recovertfromRc}  
to retrieve the translation estimate $\vt$.

%% file: robustOutlierRejection.tex
\section{Robust Outlier Rejection}
\label{sec:robustOutlierRejection}

Section~\ref{sec:certifiableNonminimalSolver} proposed a certifiably optimal solver for problem~\eqref{eq:weightedleastsquares}. However, the least squares formulation~\eqref{eq:weightedleastsquares} tends to be sensitive to outliers: the pixel measurements $\MZ$ in eq.~\eqref{eq:basisgenmodel} are typically produced by learning-based or handcrafted detectors~\cite{Zhou17pami-shapeEstimationConvex}, which might produce largely incorrect measurements (\eg due to wrong data association $\vz_i \leftrightarrow \MB_{ki}$), which in turn leads to poor shape reconstruction results. 
This section shows how to regain robustness by iteratively solving the weighted least squares 
problem~\eqref{eq:weightedleastsquares} and adjusting the weights $w_i$ to reject outliers. 

The key insight is to substitute the least square penalty in~\eqref{eq:weightedleastsquares} 
with a robust cost function, namely the \emph{truncated least squares (TLS)} cost~\cite{Yang19iccv-QUASAR-duplicate,Yang19rss-teaser-duplicate,Lajoie19ral-DCGM-duplicate,Yang20arXiv-teaser}. 
Hence, we propose the following \emph{TLS shape reconstruction} formulation:
\bea \label{eq:TLSShapeRecon}
\hspace{-5mm} 
\!\!\!\min_{\substack{c_k\geq 0, \\ k=1,\dots,K \\ \vt \in \Real{2}, \MR\in \SOthree} }
 &
\hspace{-3mm} 
\displaystyle \sumfeatures 
\rho_{\barc}\left( r_i(c_k,\MR,\vt) \right) + \alpha \!\sumbasis \! c_k 
\eea
where 
$r_i(c_k,\MR,\vt) := \left\| \vz_i \!-\! \Pi \MR \left( \sumbasis c_k \MB_{ki} \right) \!-\! \vt \right\|$ 
(introduced for notational convenience), 
and
$\rho_{\barc}(r) = \min(r^2, \barcsq)$ implements a truncated least squares cost, which 
is quadratic for small residuals and saturates to a constant value for residuals larger than a
maximum error $\barc$.

Our second insight is that $\rho_{\barc}(r)$ can be written as 
$\rho_{\barc}(r)= \min_{w\in\{0,1\}} w r^2 + (1-w)\barcsq$, 
by introducing extra slack binary variables $w\in\{0,1\}$.
Therefore, we can write problem~\eqref{eq:TLSShapeRecon} equivalently as:
\bea \label{eq:TLSShapeRecon2}
\hspace{-7mm} 
\!\!\!\min_{\substack{
c_k \geq 0,  k=1,\dots,K,\\
w_i \in \{0,1\},  i=1,\dots,N  
\\ \vt \in \Real{2}, \MR\in \SOthree} 
}
 &
\hspace{-3mm} 
\displaystyle 
\sumfeatures 
w_i \left( r_i(c_k,\MR,\vt) \right) 
\!+\!(1\!-\!w_i)\barcsq  \!\!+\! \alpha \!\sumbasis \!\! c_k 
\eea

The final insight is that now we can minimize~\eqref{eq:TLSShapeRecon2} by iteratively
minimizing (i) over $c_k,\MR,\vt$ (with fixed weights $w_i$), and (ii) over the weights $w_i$ 
(with fixed $c_k,\MR,\vt$). The rationale for this approach is that step (i) can be implemented using 
\name (since in this case the weights are fixed), and step (ii) can be implemented in closed-form.
To improve convergence of this iterative algorithm, we adopt graduated non-convexity~\cite{Blake1987book-visualReconstruction,Yang20ral-GNC}, which starts with a convex approximation of problem~\eqref{eq:TLSShapeRecon2} 
and uses a control parameter $\mu$ to gradually increase the amount of non-convexity, till (for large $\mu$) 
one solves~\eqref{eq:TLSShapeRecon2}. 
The resulting algorithm named \namerobust is given in Algorithm~\ref{alg:shaperobust}.
We refer the reader to the \supp and~\cite{Yang20ral-GNC} for a complete derivation of Algorithm~\ref{alg:shaperobust} and 
for the closed-form expression of the weight update in line~\ref{line:weightUpdate} of the algorithm.

\input{robustAlgorithm.tex}

\namerobust is deterministic and does not require an initial guess. 
We remark that the graduated non-convexity scheme in \namerobust (contrarily to \name) is not guaranteed to converge to an 
optimal solution of~\eqref{eq:TLSShapeRecon2}, but we show in the next section that it is empirically robust to $70\%$ outliers.

%% file: robustAlgorithm.tex
\setlength{\textfloatsep}{2pt}
\begin{algorithm}[h]
\label{alg:two-stage}
\footnotesize
\SetKwInOut{Input}{input}
\SetKwInOut{Output}{output}
 \Input{
 measurements $\vz_i, \iOneToN$, \\
 basis shapes $\MB_{k}, k=1,\ldots,M$ \\
 maximum error $\barc$, regularization constant $\alpha$}
 \Output{ shape reconstruction: $c_k^\star,\MR^\star,\vt^\star$ }
 \BlankLine
 \tcc{Initialization}
$w_i\at{0}=1, \quad i=1, \dots,N$ \label{line:init1} \\
$\mu\at{0} = 10^{-4}$ \label{line:init2}  \\
 \tcc{Iterations}
 \For{ $\tau = 1:\text{maxIter}$ \label{line:maxIters}}{ 
 		\tcc{Variable update}
        $c_k\at{\tau},\MR\at{\tau},\vt\at{\tau} = \text{\name}(\vz_i,\MB_{k},\alpha,w_i\at{\tau-1}, \mu\at{\tau-1})$ \label{line:variableUpdate}\\
		\tcc{Compute residual errors}
		$r_i\at{\tau} = \| \vz_i - \Pi \MR^{(\tau)} \left( \sumbasis c_k^{(\tau)} \MB_{ki}\right) - \vt^{(\tau)} \|$ \label{line:computeResidual}\\
		\tcc{Weight update}
		$w_i^{(\tau)} = \text{weightUpdate}(r_i\at{\tau},\barc,\mu\at{\tau-1})$ \label{line:weightUpdate} \\
		\tcc{Compute objective function}
		$f\at{\tau} = \text{computeObjective}(r_i\at{\tau},w_i^{(\tau)},\mu\at{\tau-1},\alpha,\barc)$ \\
		\tcc{Check convergence ($\tau>1$)}
		\If{ $|f\at{\tau} - f\at{\tau-1}| < 10^{-10}$ \label{line:checkConvergence}}{
			break
		}
		\tcc{Update control parameter $\mu$}
		$\mu^{(\tau)} = 2 \cdot \mu^{(\tau-1)}$ \label{line:updateControlPara}
 } 
  return $c_k\at{\tau},\MR\at{\tau},\vt\at{\tau}$.
 \caption{Robust Shape Reconstruction. \label{alg:shaperobust}}
\end{algorithm}


%% file: experiments.tex
\section{Experiments}
\label{sec:experiments}
{\bf Implementation details.} Both \name and \namerobust are implemented in Matlab, with both SOS relaxations~\eqref{eq:naiveSOS} and~\eqref{eq:SOSBasisReduction} implemented using YALMIP~\cite{Lofberg04cacsd-yalmip} and the resulting SDPs solved using MOSEK~\cite{mosek}.

\subsection{Efficiency Improvement by Basis Reduction}
\label{sec:efficiencyImproveBasisReduction}
We first evaluate the efficiency improvement due to basis reduction in simulation. We fix the number of correspondences $N=100$, and increase the number of basis shapes $K=5,10,20$. At each $K$, we first randomly generate $K$ basis shapes $\MB_1,\dots,\MB_K \in \Real{3 \times N}$, with entries sampled independently from a Normal distribution $\calN(0,1)$. Then $K$ linear coefficients $\vc = [c_1,\dots,c_K]\tran$ are uniformly  sampled from the interval $[0,1]$, and a rotation matrix $\MR$ is randomly chosen. The 2D measurements $\MZ$ are computed from the generative model~\eqref{eq:basisgenmodel} with $\vt = \Zero$, $s_x = s_y =1$ for $\Pi$, and additive noise $\vepsilon_i \sim \calN(\Zero,0.01^2)$. For shape reconstruction, we feed the noisy $\MZ$ and bases $\MB_k$ to (i) the SOS relaxation~\eqref{eq:naiveSOS} without basis reduction, and (ii) the SOS relaxation~\eqref{eq:SOSBasisReduction} with basis reduction, both at relaxation order $\rorder=2$ and with no Lasso regularization ($\alpha=0$). 

To evaluate the effects of introducing basis reduction, we compute the following statistics for each choice of $K$: (i) solution time for the SDP; (ii) tightness of the SOS relaxation, including $\corank(\MS_0^{2\star})$ and relative duality gap $\eta_2$; (iii) accuracy of reconstruction, including the coefficients estimation error ($\ell_2$ norm of the difference between estimated and ground-truth coefficients) and the rotation estimation error (the geodesic distance between estimated and ground-truth rotation). Table~\ref{tab:improveBasisReduction} shows the resulting statistics. We see that the SOS relaxation without basis reduction quickly becomes intractable at $K=20$ (mean solution time is 2440 seconds), while the relaxation with basis reduction can still be solved in a reasonable amount of time (107 seconds)\footnote{Our basis reduction can potentially be \emph{combined} with other scalability improvement techniques reviewed in~\cite{majumdar19arXiv-surveySDPScalability}, such as low-rank SDP solvers.}. In addition, from the co-rank of $(\MS_0^{2\star})$ and the relative duality gap $\eta_2$, we see basis reduction has no negative impact on the quality of the relaxation, which remains tight at order $\rorder=2$. This observation is further validated by the identical accuracy of $\vc$ and $\MR$ estimation before and after basis reduction 
(last two rows of Table~\ref{tab:improveBasisReduction}).
\input{tab-improveBasisReduction}

\subsection{\name for Outlier-Free Reconstruction}
In this section, we compare the performance of \name against state-of-art optimization techniques for shape reconstruction. We follow the same protocol as in Section~\ref{sec:efficiencyImproveBasisReduction}, but only generate 2D measurements from a \emph{sparse} set of $p=2$ basis shapes. This is done by only sampling $p$ out of $K$ \emph{nonzero} shape coefficients, \ie, $c_{p+1},\dots,c_{K}=0$. 
We then compare the performance of \name, setting $\alpha=0.01$ to encourage sparseness, against three state-of-the-art optimization techniques: (i) the projective matching pursuit  method~\cite{Ramakrishna12eccv-humanPose} (label:~\PMP), which uses principal component analysis to first obtain a set of orthogonal bases from $\{\MB_k\}$ and then locally optimizes the shape parameters and camera pose using the mean shape as an initial guess; (ii) the alternative  optimization method~\cite{Zhou15cvpr} (label:~\altern), which locally optimizes problem~\eqref{eq:weightedleastsquares} by alternatively updating $\vc$ and $\MR$, initialized at the mean shape; and (iii) the convex relaxation with refinement proposed in~\cite{Zhou17pami-shapeEstimationConvex} (label:~\convexrefine), which uses a convex relaxation and then refines the solution to obtain $\vc$ and $\MR$. 
Fig.~\ref{fig:BM_SIM_OF_nrB} shows the boxplots of the 3D shape estimation error (mean $\ell_2$ distance between the reconstructed shape and the ground-truth shape) and the rotation estimation error for $K=5,10,20$ basis shapes and 20 Monte Carlo runs. We observe that \name has the highest accuracy in estimating the 3D shape and camera pose, though the other three methods also perform quite well. In all the Monte Carlo runs, \name achieves $\corank(\MS_0^{2\star}) = 1$ and mean relative duality gap $\eta_2 = 6.3\mathrm{e}{-5}$, indicating that \name was able to obtain an 
optimal solution.

\input{fig-BM_SIM_OF_nrB}

\subsection{\namerobust for Robust Reconstruction}
This section  shows that \namerobust achieves state-of-the-art performance  on the \FGCar~\cite{Lin14eccv-modelFitting} dataset. The \FGCar dataset contains 300 car images with ground-truth 2D landmarks $\MZ \in \Real{2\times N},N=256$. It also contains $K=15$ 3D mesh models of different cars $\{\MB_k\}_{k=1}^{K}$. To generate outliers, we randomly change $10\%-70\%$ of the $N$ ground-truth 2D landmarks $\MZ$ to be arbitrary positions inside the image. We then evaluate the robustness of \namerobust compared with two other robust methods based on the assumption of sparse outliers in~\cite{Zhou17pami-shapeEstimationConvex}: (i) robustified alternative optimization (label:~\alternrobust) and (ii) robustified convex optimization (label:~\convexrobust). 
Fig.~\ref{fig:BM_FG3DCar} boxplots the shape estimation and rotation estimation error\footnote{Although there is no ground-truth reconstruction for each image, the original paper~\cite{Lin14eccv-modelFitting} uses local optimization (with full perspective camera model) to reconstruct high-quality 3D shapes for all images, and we use their reconstructions as ground-truth.} under increasing outlier rates computed over 40 randomly chosen images in the~\FGCar dataset. We can see that \namerobust is insensitive to $70\%$ outliers, while the accuracy of both \alternrobust and \convexrobust decreases with respect to higher outlier rates and they fail at $60\%$ outliers. Fig.~\ref{fig:FG3DCar_qualitative} shows two examples of qualitative results, where we see \namerobust gives high-quality model fitting at $70\%$ outliers, while the quality of \alternrobust and \convexrobust starts decreasing at $40\%$ outliers. More qualitative results are given in the \supp.

\input{fig-BM_FG3DCar}

\input{fig-FG3DCar_qualitative_small}

%% file: tab-improveBasisReduction.tex
\begin{table}
\smaller
\begin{center}
\begin{tabular}{cccc}
\hline
\# of Bases $K$ & $K=5$ & $K=10$ & $K=20$ \\
\hline
SDP Time [s] & $\bmatc{c} 3.52 \\ \bm{0.550} \ematc$ & $\bmatc{c} 47.0 \\ \bm{5.28} \ematc$ & $\bmatc{c} 2440 \\ \bm{107} \ematc$ \\
\hline
$\corank(\MS_0^{2\star})$ & $\bmatc{c} 1 \\ \bm{1} \ematc$ & $\bmatc{c} 1 \\ \bm{1} \ematc$ & $\bmatc{c} 1 \\ \bm{1} \ematc$\\
\hline
Duality Gap $\eta_2$ & $\bmatc{c} 5\mathrm{e}{-6} \\ \bm{1\mathrm{e}{-5}} \ematc$ & $\bmatc{c} 7\mathrm{e}{-6} \\ \bm{2\mathrm{e}{-5}} \ematc$ & $\bmatc{c} 4\mathrm{e}{-5} \\ \bm{1\mathrm{e}{-5}} \ematc$ \\
\hline
$\vc$ Error & $\bmatc{c} 1.3\mathrm{e}{-3} \\ \bm{1.3\mathrm{e}{-3}} \ematc$ & $\bmatc{c} 2.3\mathrm{e}{-3} \\ \bm{2.3\mathrm{e}{-3}} \ematc$ & $\bmatc{c} 3.2\mathrm{e}{-3} \\ \bm{3.2\mathrm{e}{-3}} \ematc$ \\
\hline
$\MR$ Error [deg] & $\bmatc{c} 0.0690 \\ \bm{0.0690} \ematc$ & $\bmatc{c} 0.0487 \\ \bm{0.0487} \ematc$ & $\bmatc{c} 0.0298 \\ \bm{0.0298} \ematc$ \\
\hline
\end{tabular}
\vspace{0.5mm}
\caption{Efficiency improvement by basis reduction. {\bf Bold} text represent mean values computed by solving the SOS relaxation {\bf \emph{with}} basis reduction~\eqref{eq:SOSBasisReduction}, while normal text represent mean values computed by solving the SOS relaxation \emph{without} basis reduction~\eqref{eq:naiveSOS}. Statistics are computed over 20 Monte Carlo runs.}
\label{tab:improveBasisReduction}
\vspace{-2mm}
\end{center}
\end{table}

%% file: fig-BM_SIM_OF_nrB.tex

\begin{figure}[h]
	\begin{center}
	\hspace{-4mm}
	\includegraphics[width=\columnwidth]{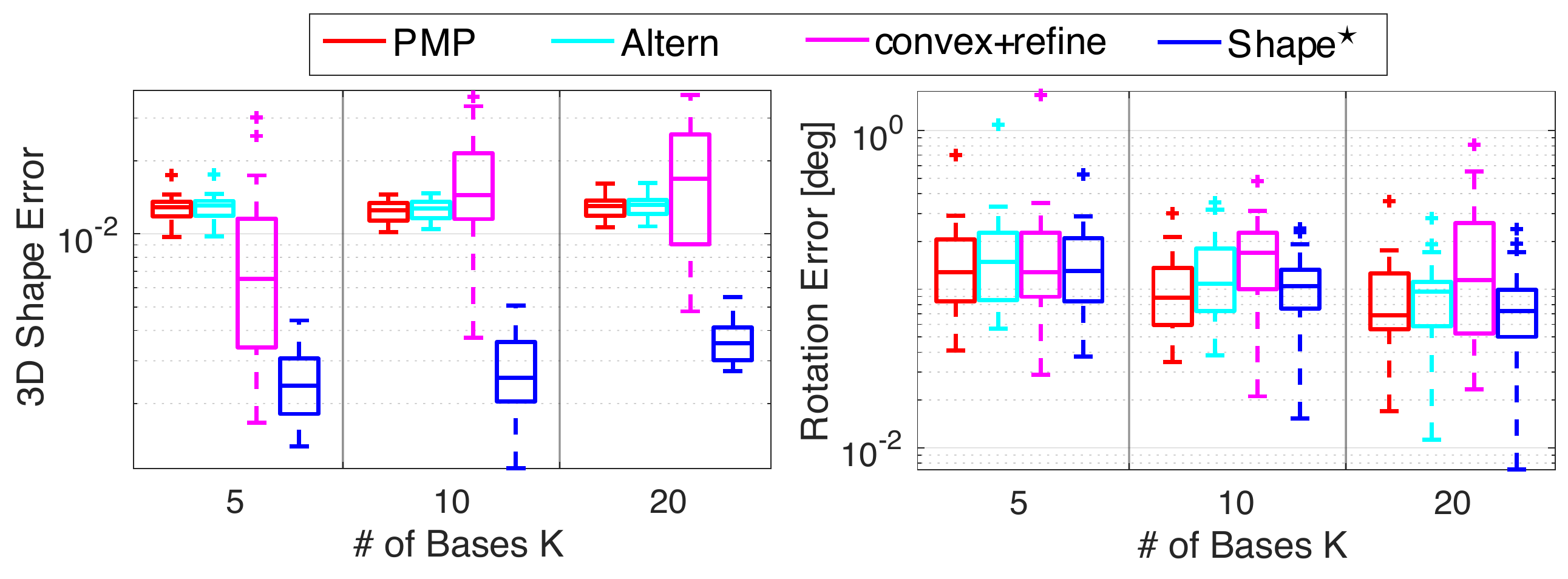} \\
	\caption{3D shape estimation error (left) and rotation estimation error (right) by \name compared with \PMP~\cite{Ramakrishna12eccv-humanPose}, \altern~\cite{Zhou15cvpr} and \convexrefine~\cite{Zhou17pami-shapeEstimationConvex}, for increasing basis shapes $K=5,10,20$.
	\label{fig:BM_SIM_OF_nrB}}
	\end{center}
	\vspace{-2mm}
\end{figure}

%% file: fig-BM_FG3DCar.tex

\begin{figure}[h]
	\begin{center}
	\begin{minipage}{\columnwidth}
	\begin{tabular}{c}%
		\hspace{-5mm}	
		\begin{minipage}{\columnwidth}%
		\centering%
		\includegraphics[width=\columnwidth]{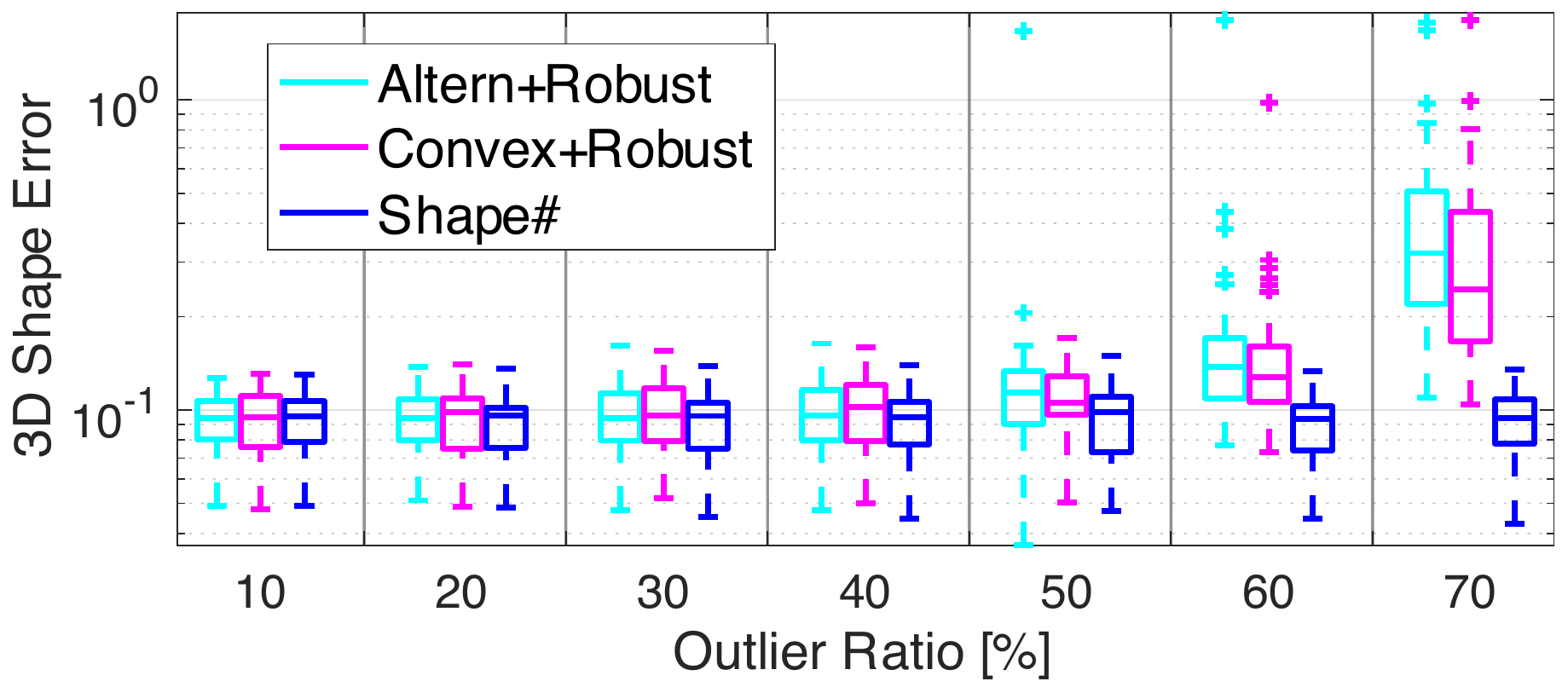} \\
		\end{minipage} 

		\\
		\hspace{-5mm}
		\begin{minipage}{\columnwidth}%
		\centering%
		\includegraphics[width=\columnwidth]{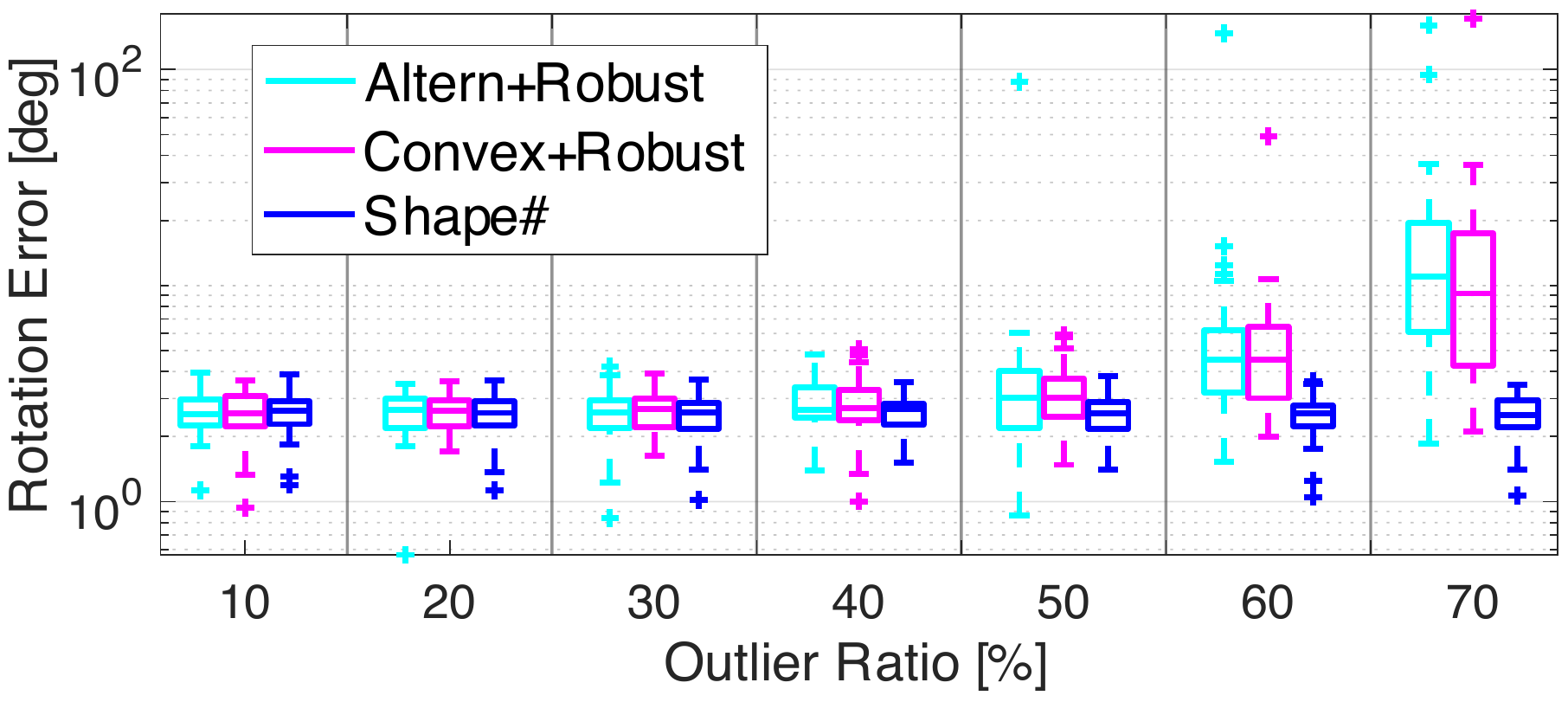} \\
		\end{minipage}
		\end{tabular}
	\end{minipage}
	\vspace{0.5mm}
	\caption{3D shape estimation error (top) and rotation estimation error (bottom) by \namerobust compared with \alternrobust~\cite{Zhou17pami-shapeEstimationConvex} and \convexrobust~\cite{Zhou17pami-shapeEstimationConvex} under increasing outlier rates.
	\label{fig:BM_FG3DCar}}
	\end{center}
\end{figure}

%% file: fig-FG3DCar_qualitative_small.tex
\newcommand{\mpwthree}{2.8cm}
\newcommand{\myhspace}{\hspace{-3.5mm}}
\newcommand{\myhspacehead}{\hspace{-3mm}}

\begin{figure}[h]
	\begin{center}
	\begin{minipage}{\textwidth}
	\begin{tabular}{ccc}%
		\myhspace \alternrobust & \myhspace \convexrobust & \myhspace \namerobust  \\
		\myhspacehead
			\begin{minipage}{\mpwthree}%
			\centering%
			\includegraphics[width=\columnwidth]{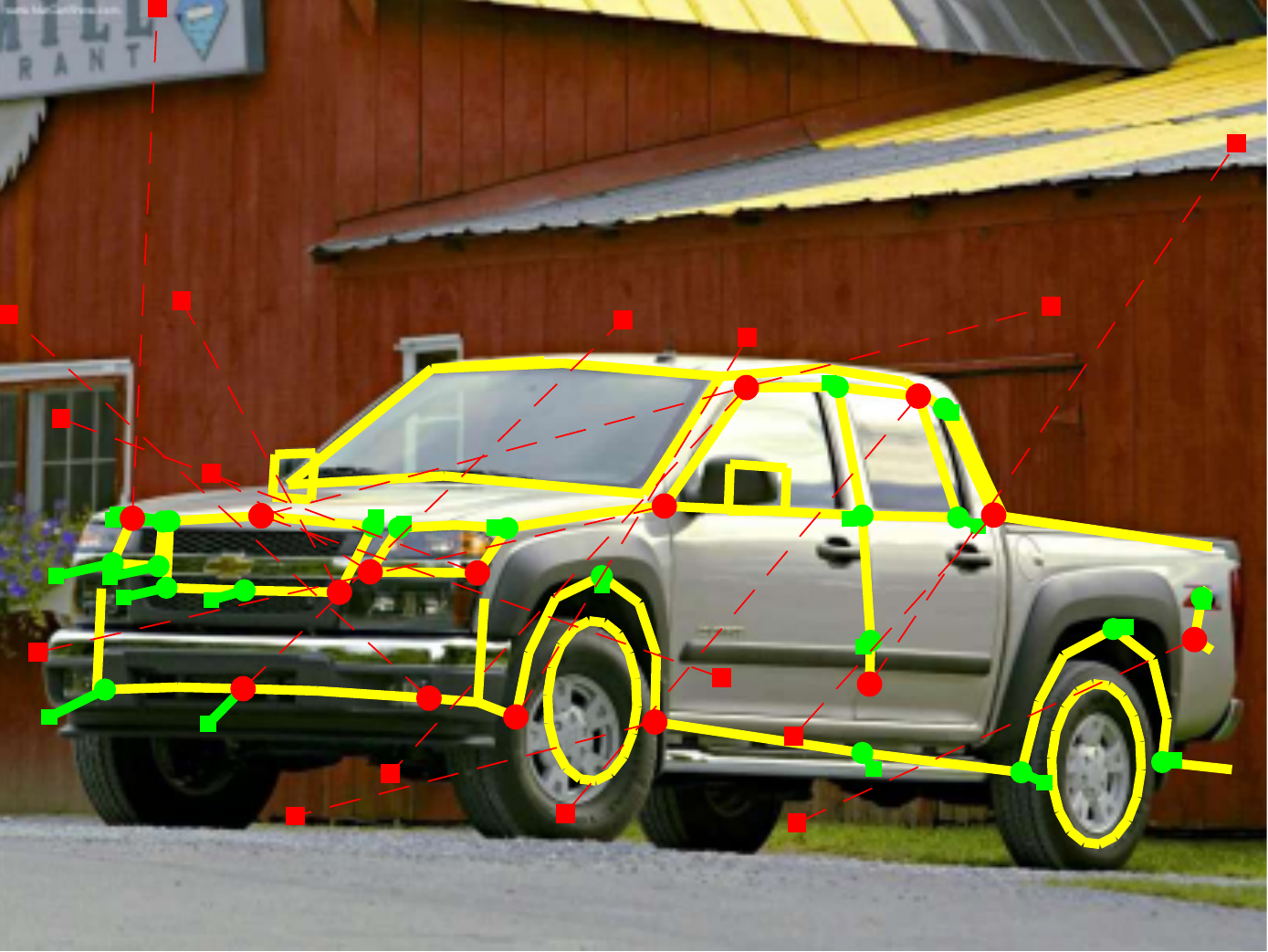} \\
			\vspace{1mm}
			\end{minipage}
		& \myhspace
			\begin{minipage}{\mpwthree}%
			\centering%
			\includegraphics[width=\columnwidth]{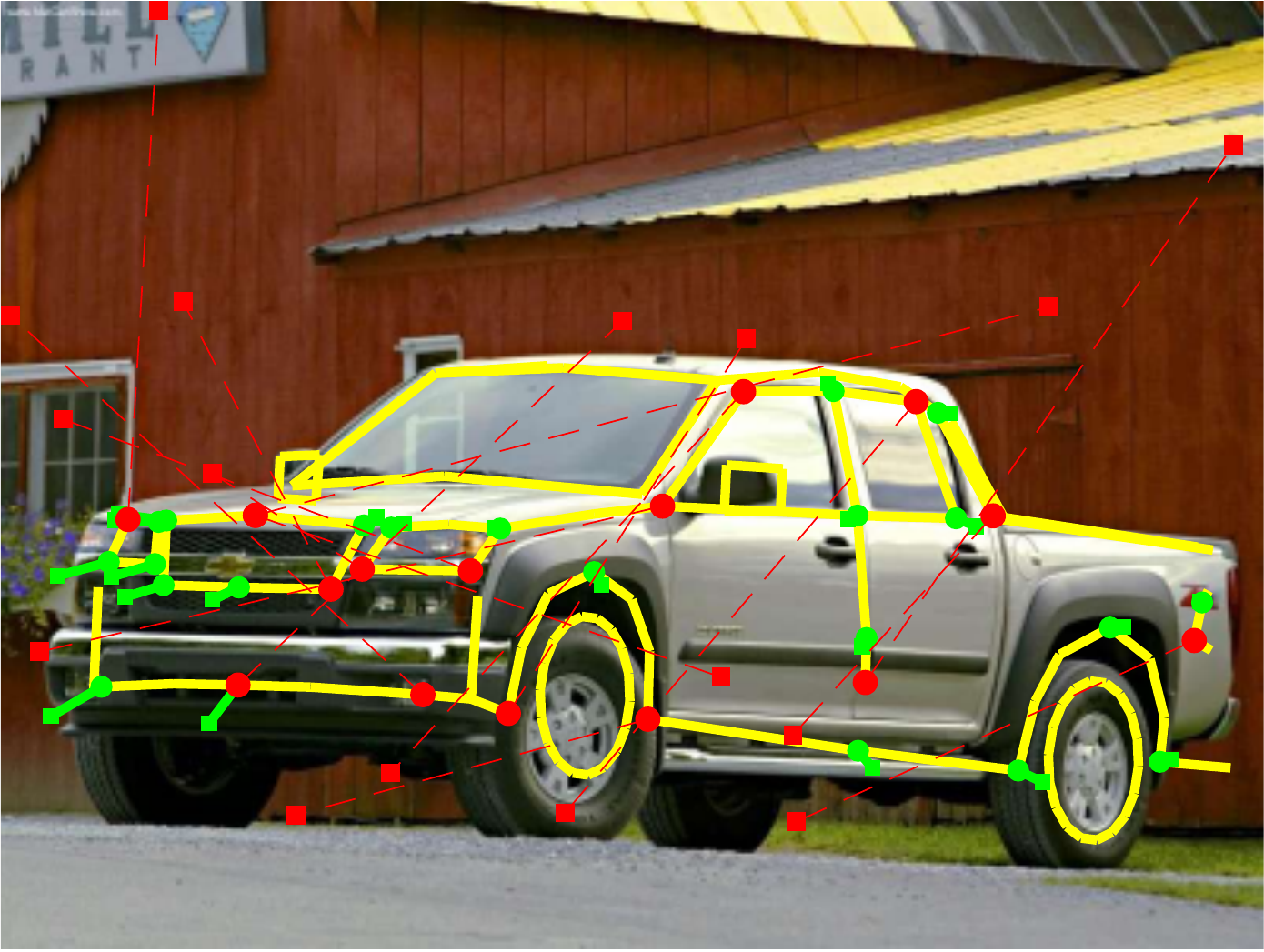} \\
			\vspace{1mm}
			\end{minipage}
		& \myhspace
			\begin{minipage}{\mpwthree}%
			\centering%
			\includegraphics[width=\columnwidth]{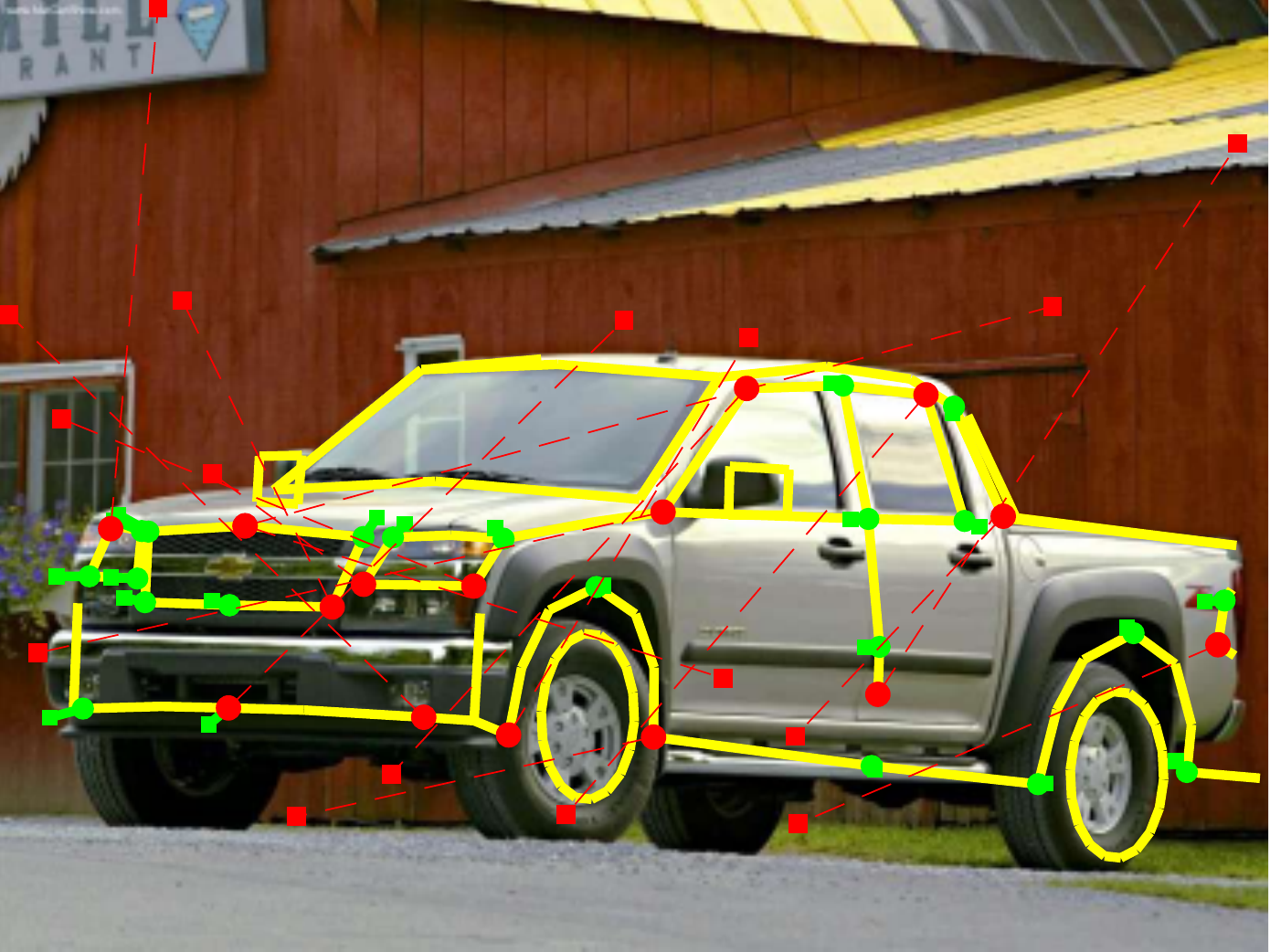} \\
			\vspace{1mm}
			\end{minipage} \\
		\multicolumn{3}{c}{(a) Chevrolet Colorado LS $40\%$ outliers. } \\
		\myhspacehead
			\begin{minipage}{\mpwthree}%
			\centering%
			\includegraphics[width=\columnwidth]{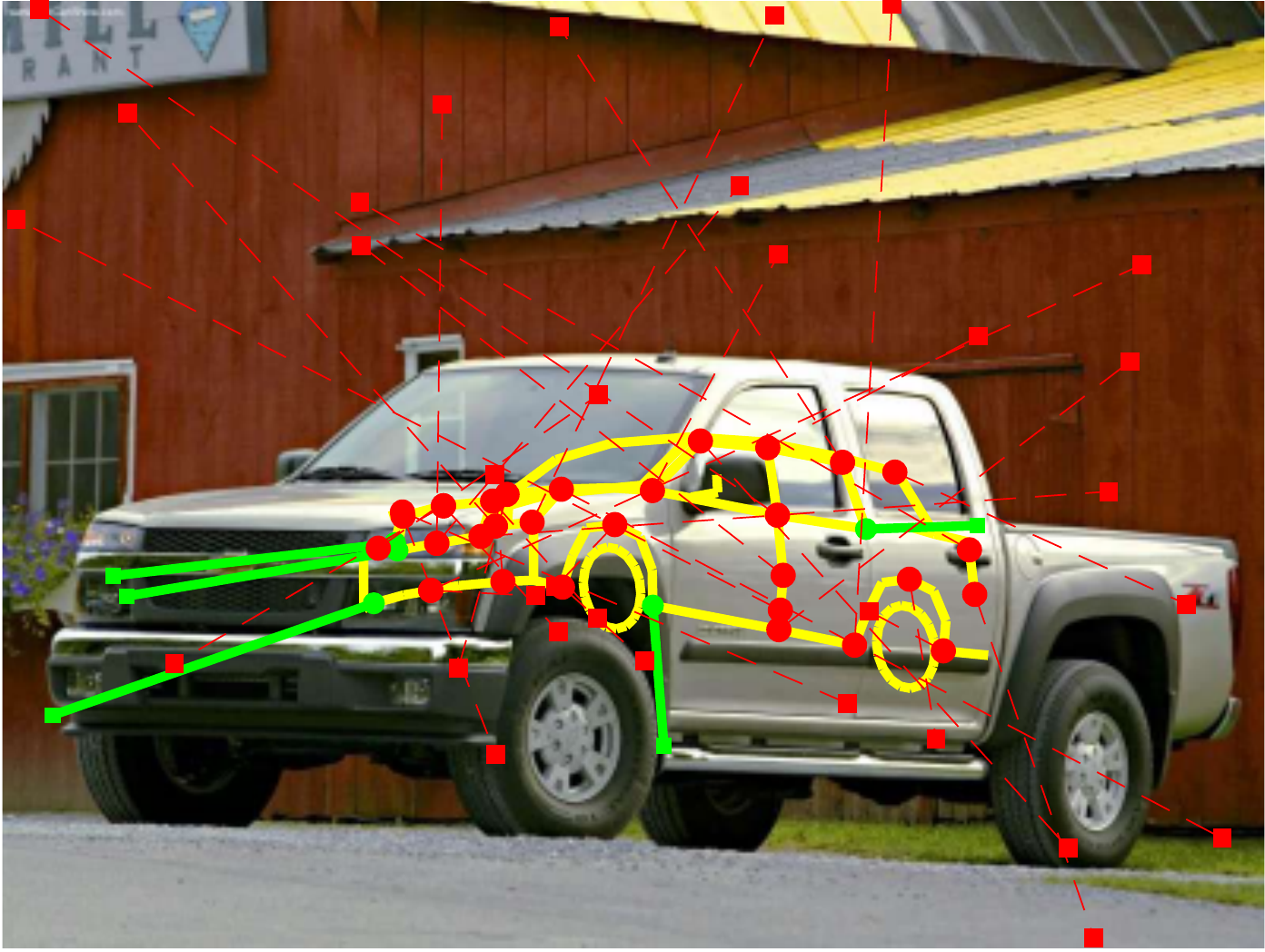} \\
			\vspace{1mm}
			\end{minipage}
		& \myhspace
			\begin{minipage}{\mpwthree}%
			\centering%
			\includegraphics[width=\columnwidth]{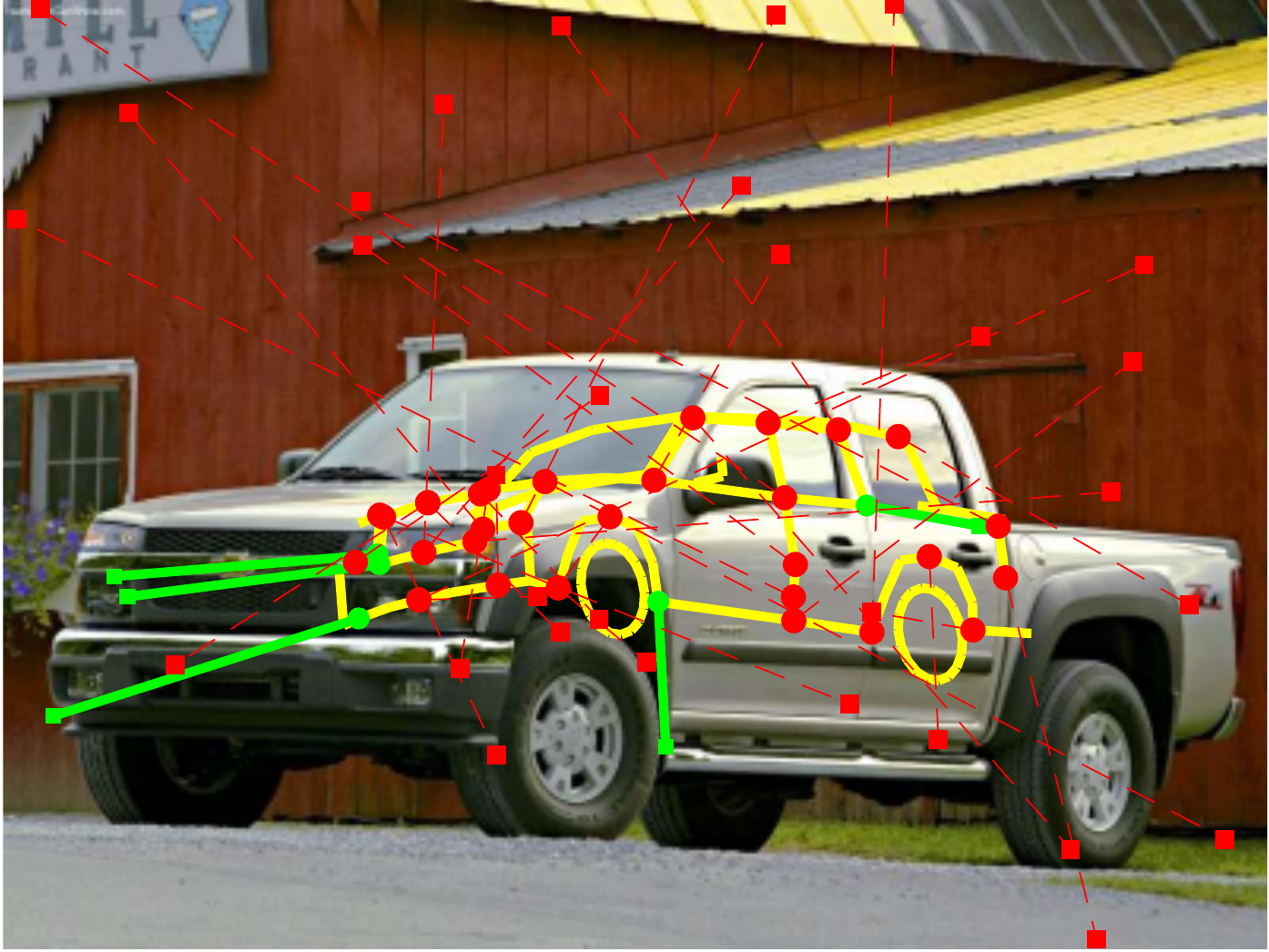} \\
			\vspace{1mm}
			\end{minipage}
		& \myhspace
			\begin{minipage}{\mpwthree}%
			\centering%
			\includegraphics[width=\columnwidth]{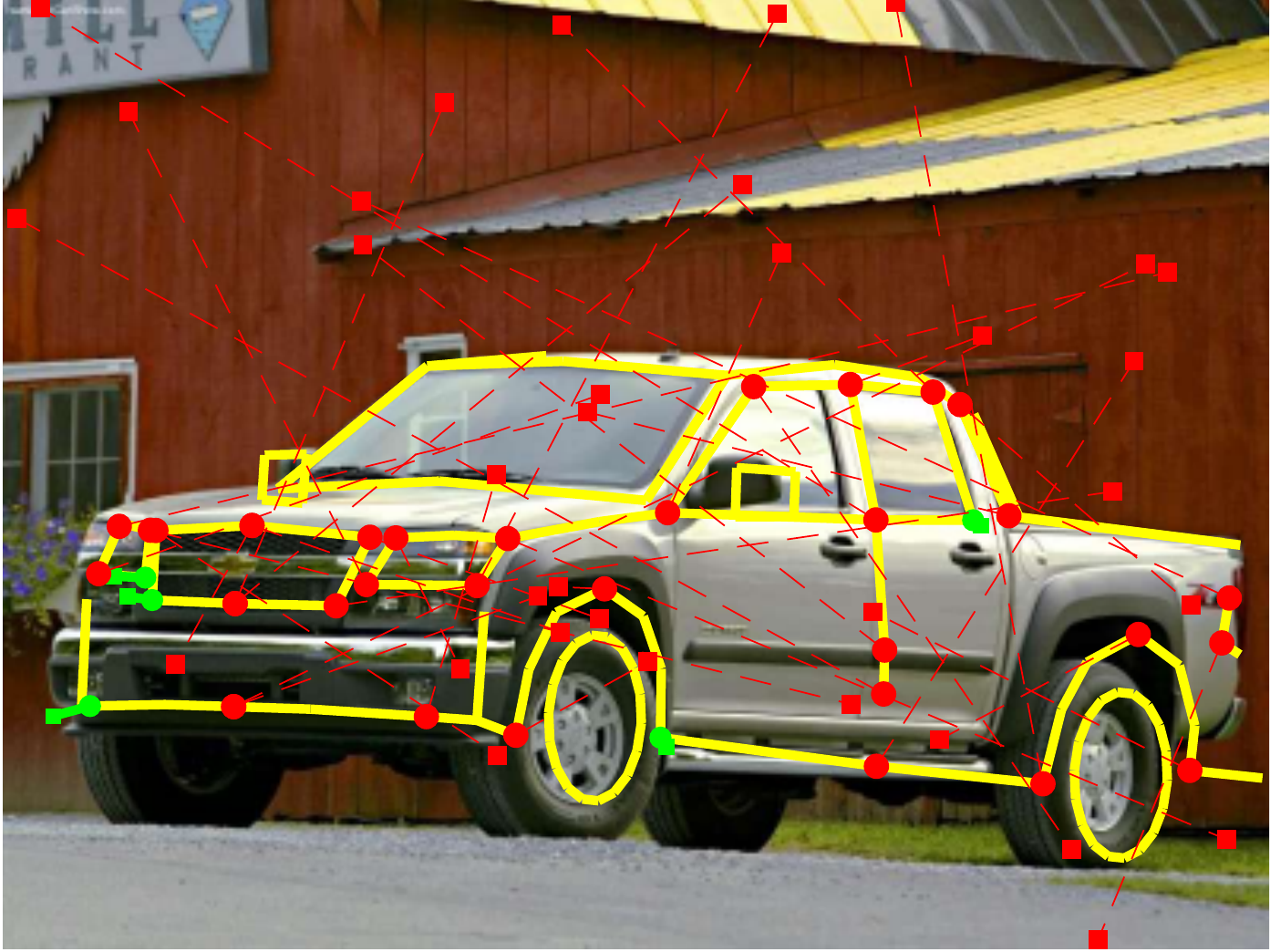} \\
			\vspace{1mm}
			\end{minipage} \\
		\multicolumn{3}{c}{(b) Chevrolet Colorado LS $70\%$ outliers. } \\

		\myhspacehead
			\begin{minipage}{\mpwthree}%
			\centering%
			\includegraphics[width=\columnwidth]{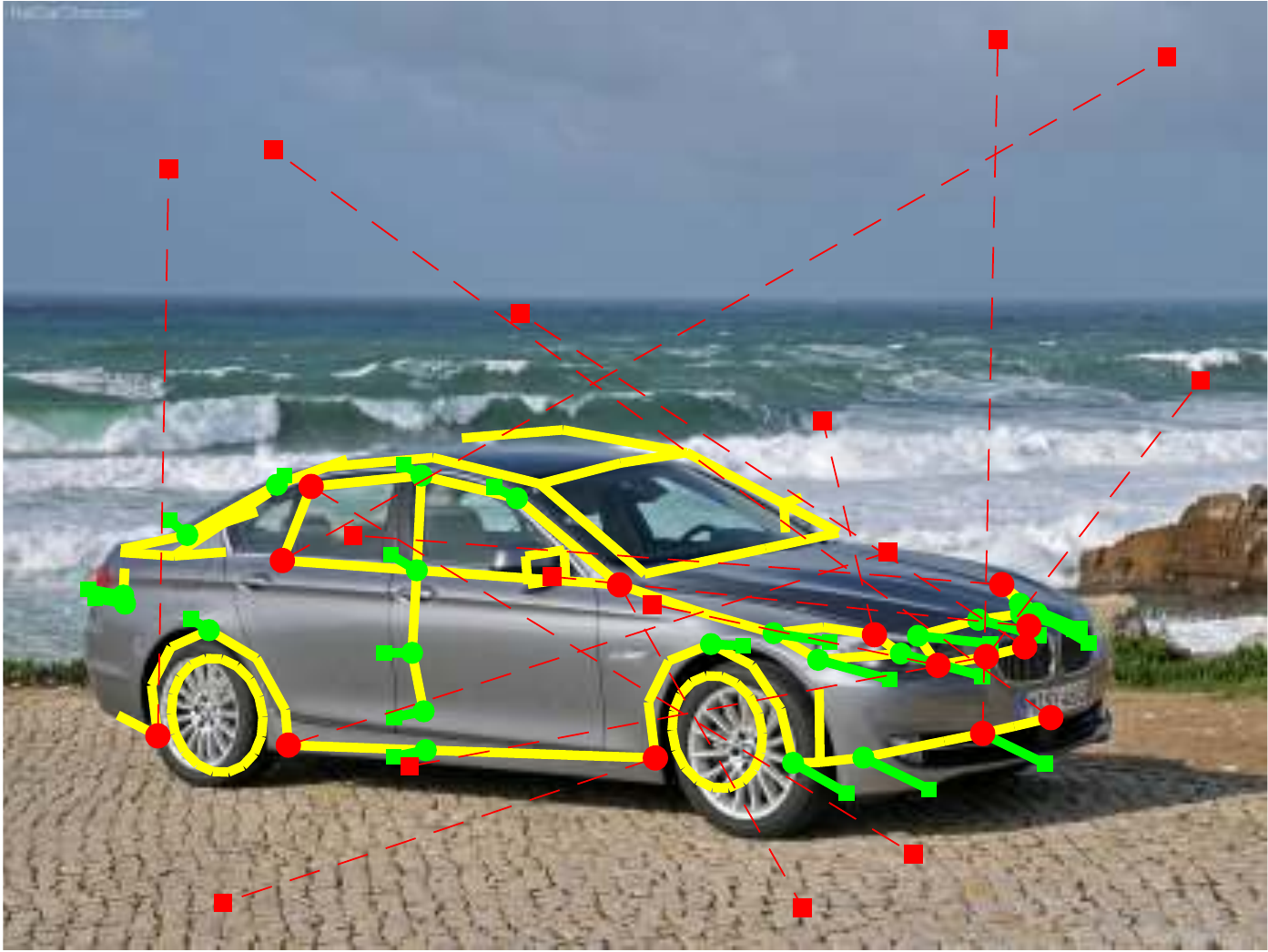} \\
			\vspace{1mm}
			\end{minipage}
		& \myhspace
			\begin{minipage}{\mpwthree}%
			\centering%
			\includegraphics[width=\columnwidth]{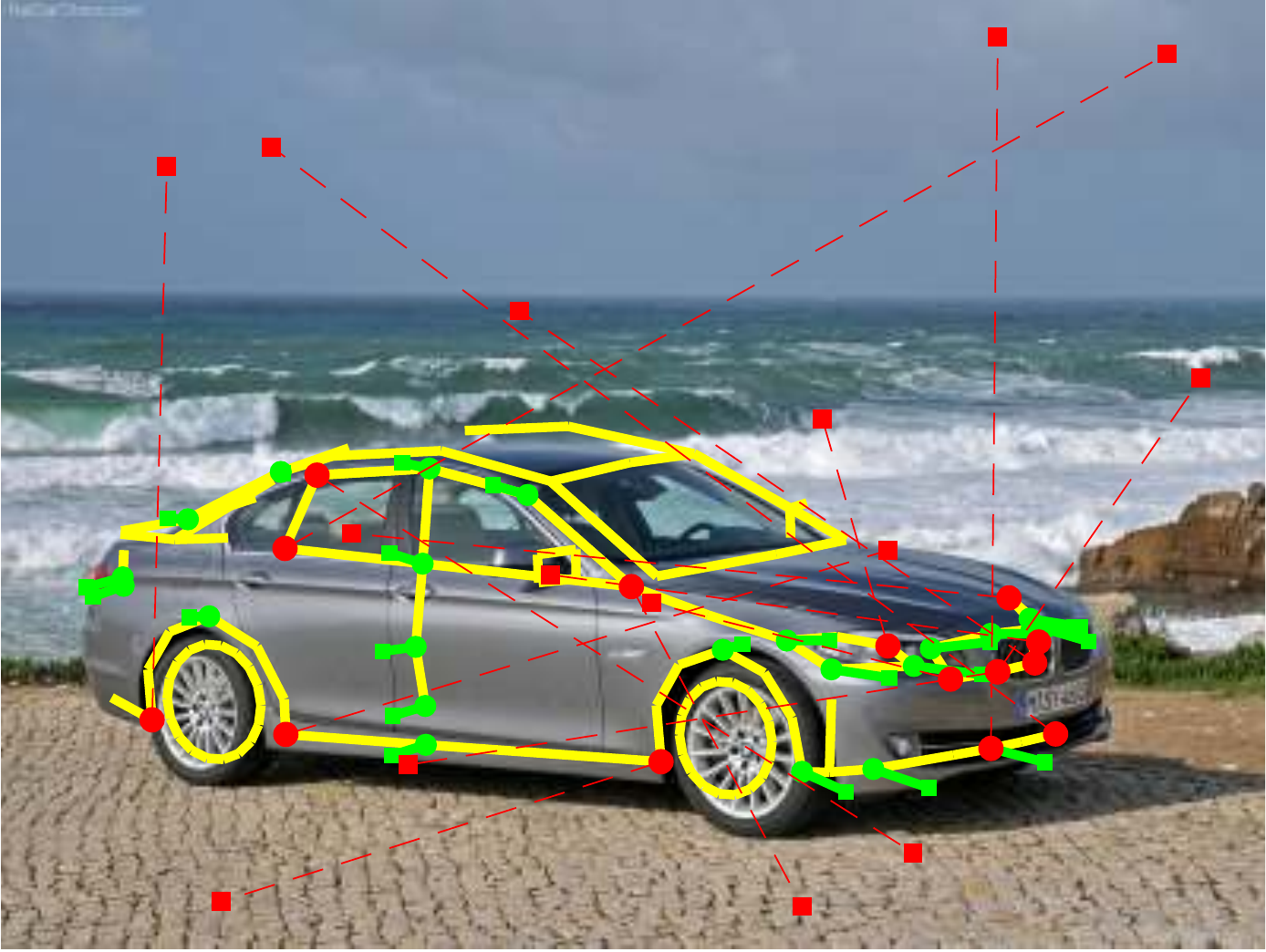} \\
			\vspace{1mm}
			\end{minipage}
		& \myhspace
			\begin{minipage}{\mpwthree}%
			\centering%
			\includegraphics[width=\columnwidth]{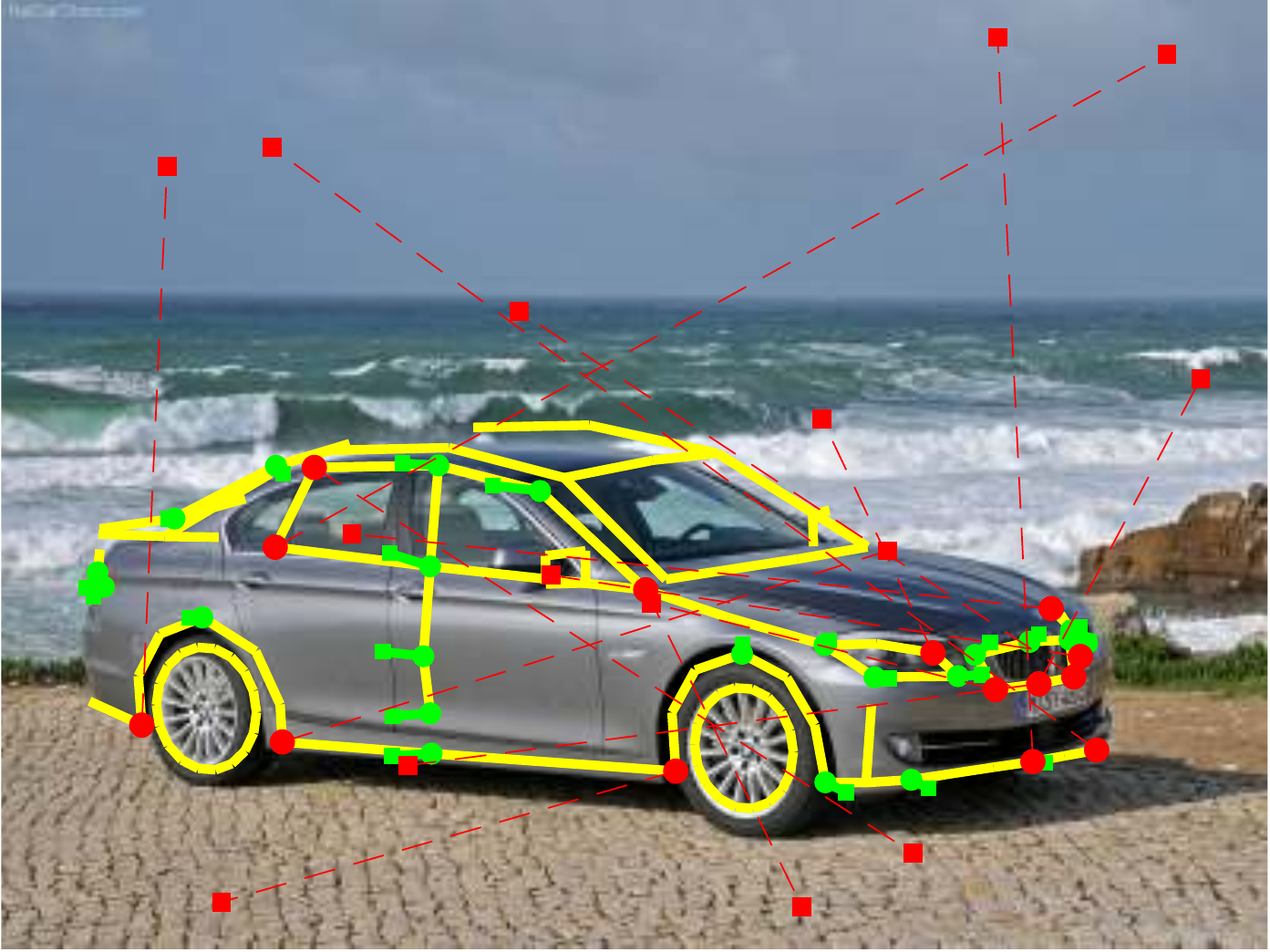} \\
			\vspace{1mm}
			\end{minipage} \\
		\multicolumn{3}{c}{(c) BMW 5-Series $40\%$ outliers.} \\

		\myhspacehead
			\begin{minipage}{\mpwthree}%
			\centering%
			\includegraphics[width=\columnwidth]{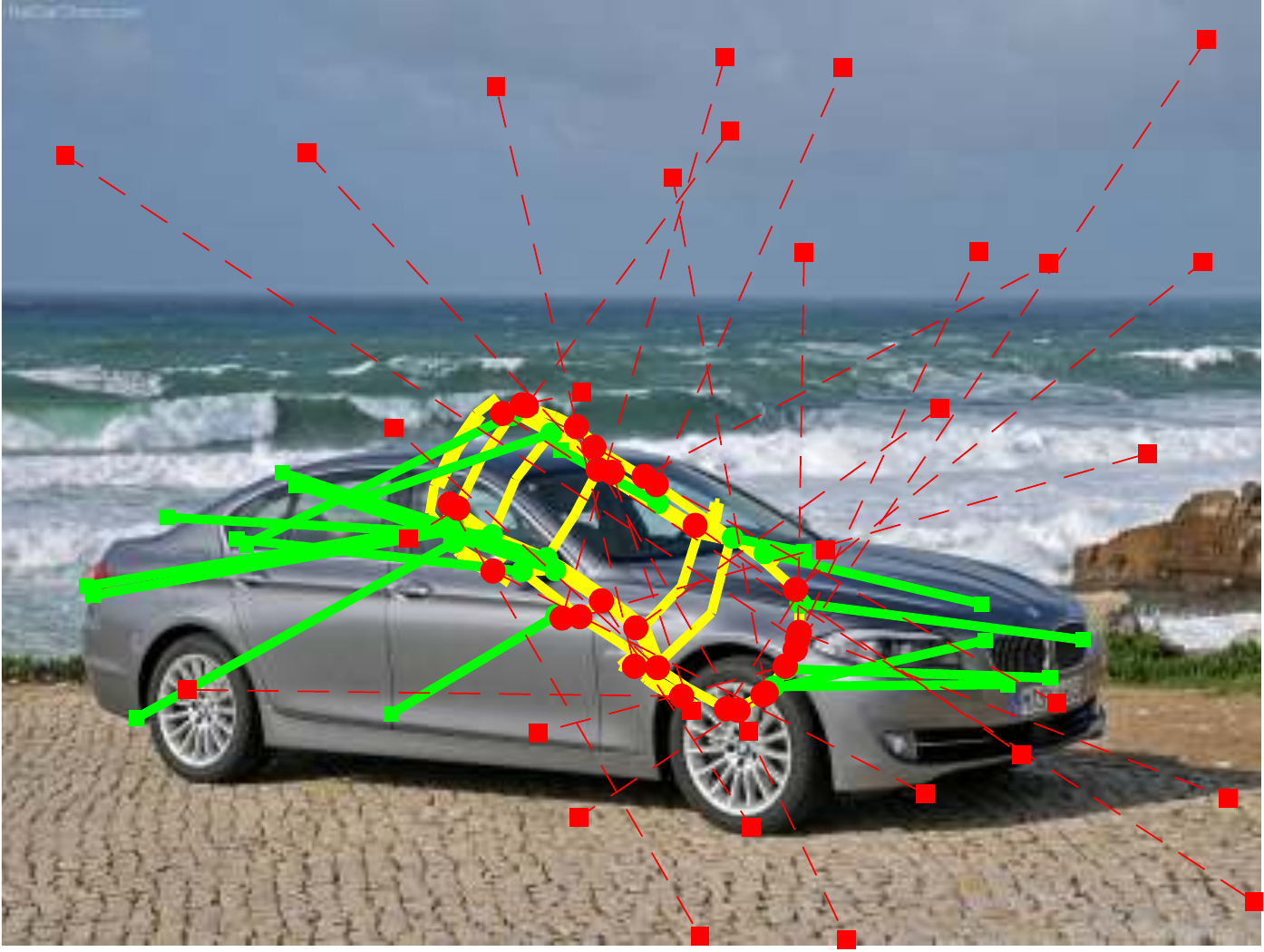} \\
			\vspace{1mm}
			\end{minipage}
		& \myhspace
			\begin{minipage}{\mpwthree}%
			\centering%
			\includegraphics[width=\columnwidth]{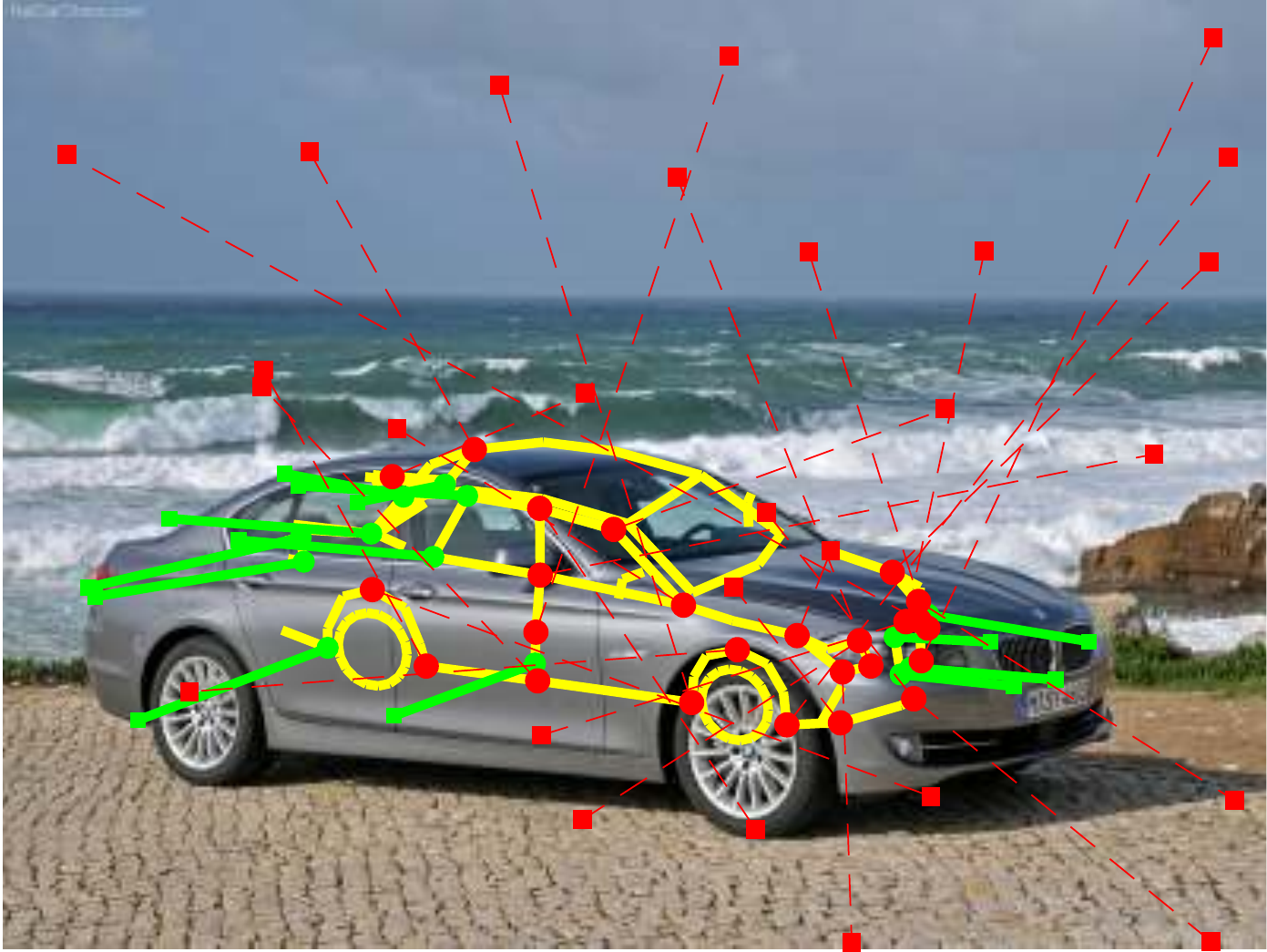} \\
			\vspace{1mm}
			\end{minipage}
		& \myhspace
			\begin{minipage}{\mpwthree}%
			\centering%
			\includegraphics[width=\columnwidth]{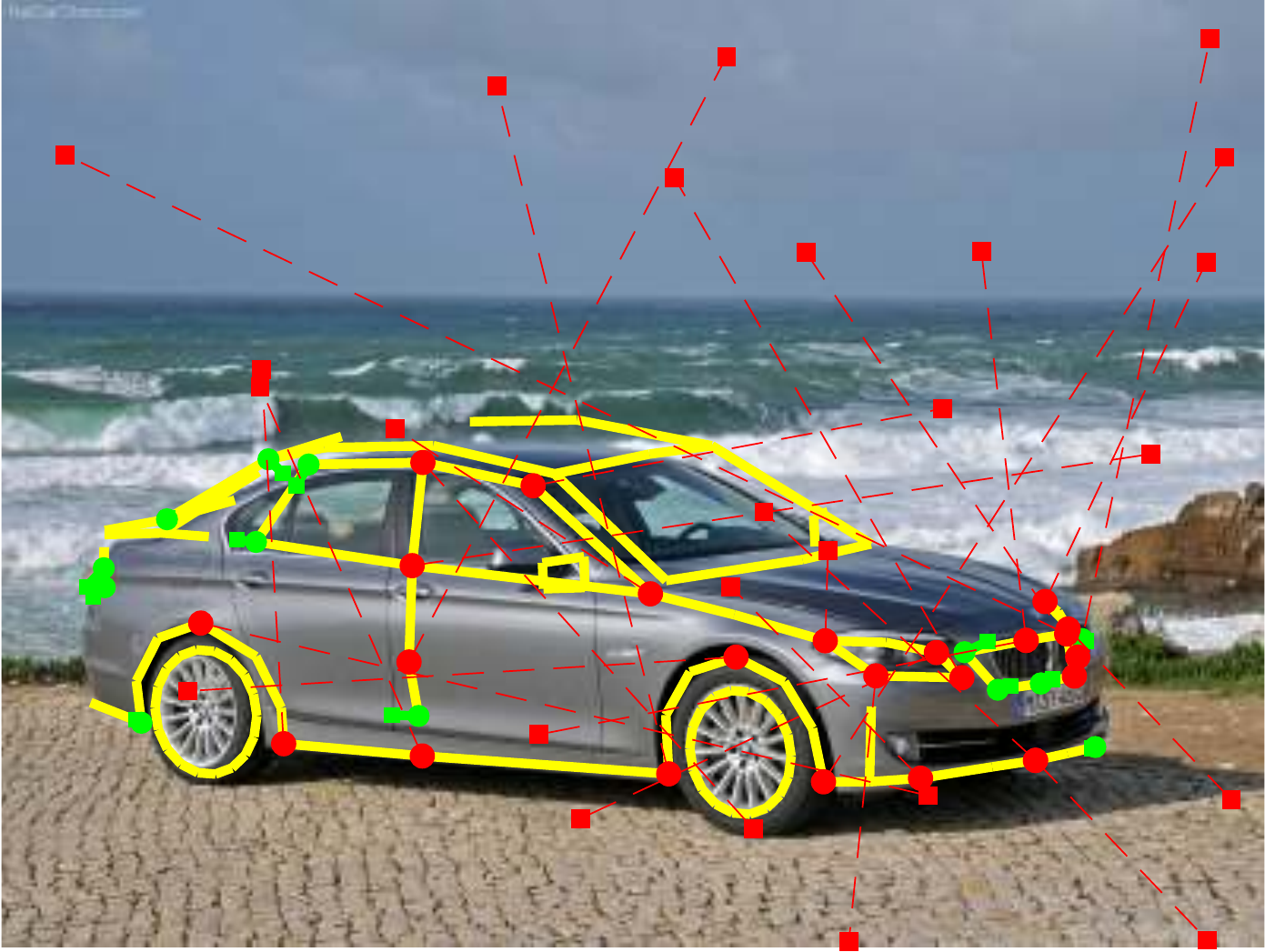} \\
			\vspace{1mm}
			\end{minipage} \\
		\multicolumn{3}{c}{(d) BMW 5-Series $70\%$ outliers.}

		\end{tabular}
	\end{minipage}
	\vspace{-2mm} 
	\caption{Selected qualitative results on the \FGCar dataset~\cite{Lin14eccv-modelFitting} under $40\%$ and $70\%$ outlier rates using \alternrobust~\cite{Zhou17pami-shapeEstimationConvex}, \convexrobust~\cite{Zhou17pami-shapeEstimationConvex}, and \namerobust. (a)-(b): results on the Chevrolet Colorado LS; (c)-(d): results on the BMW 5-Series. Green: inliers. Red: outliers. Circle: 3D landmark. Square: 2D landmark. [Best viewed electronically.]
	\label{fig:FG3DCar_qualitative}}
	\end{center}
\end{figure}

%% file: conclusions.tex

\section{Conclusions}
\label{sec:conclusions}

We presented~\name, the first certifiably optimal solver for 3D shape reconstruction from 2D landmarks in a single image. \name is developed by applying Lasserre's hierarchy of SOS relaxations combined with basis reduction to improve efficiency. Experimental results show that the SOS relaxation of order 2 always achieves global optimality. 
To handle  outlying measurements, we also proposed \namerobust, which solves a truncated least squares robust estimation problem by iteratively running \name~\revise{without the need for an initial guess}. 
We show that \namerobust achieves robustness against $70\%$ outliers on the \FGCar dataset and outperforms 
state-of-the-art solvers.

%% file: supp-proof-tran-free-shape-recon.tex
\section{Proof of Theorem~\ref{thm:trans-freeRecon}}
\label{sec:proof-tran-free-shape-recon}
\begin{proof}
Here we prove Theorem~\ref{thm:trans-freeRecon} in the main document. 
Recall the weighted least squares optimization for shape reconstruction in eq.~\eqref{eq:weightedleastsquares} and denote its objective function as $f(\vc,\MR,\vt)$, with $\vc = [c_1,\dots,c_k]\tran$:
\begin{multline} \label{eq:objectivef}
f(\vc,\MR,\vt) = \\ \displaystyle \sumfeatures w_i  \left\| \vz_i - \Pi \MR \left( \sumbasis c_k \MB_{ki} \right) - \vt \right\|^2 + \alpha \sumbasis \left| c_k \right|.
\end{multline}
In order to marginalize out the translation $\vt$, we compute the derivative of $f(\vc,\MR,\vt)$~\wrt~$\vt$:
\bea
\frac{\partial f}{\partial \vt} = 2\sumfeatures w_i \vt - 2\sumfeatures w_i \left( \vz_i - \Pi\MR \left( \sumbasis c_k \MB_{ki} \right) \right),
\eea
and set it to $\Zero$, which allows us to write $\vt^\star$ in closed form using $\MR^\star$ and $\vc^\star$:
\bea \label{eq:tofRc}
\vt^\star = \vzbar^w - \Pi \MR^\star \left( \sumbasis c_k^\star \vBbar^w_{k} \right),
\eea
with $\vzbar^w$ and $\vBbar^w_k,k=1,\dots,K,$ being the weighted centers of the 2D landmarks $\MZ$ and the 3D basis shapes $\MB_k$:
\bea
\vzbar^w = \frac{\sumfeatures w_i \vz_i}{\sumfeatures w_i}, \quad \vBbar^w_k = \frac{\sumfeatures \vBbar_{ki}}{\sumfeatures w_i}.
\eea
Then we can substitute the expression of $\vt^\star$ in~\eqref{eq:tofRc} back into the objective function in~\eqref{eq:objectivef} and obtain an objective function without translation:
\begin{multline} \label{eq:tran-free-obj}
f'(\vc,\MR) = \\ \sumfeatures w_i \left\|(\vz_i - \vzbar^w) - 
\Pi \MR \left( \sumbasis c_i (\MB_{ik} - \vBbar^w_k) \right)\right\|^2 + \\ \alpha \sumbasis | c_k |
\end{multline}
Lastly, by defining:
\bea
& \vztilde_i = \sqrt{w_i} (\vz_i - \vzbar^w), \\
& \vBtilde_{ki} = \sqrt{w_i} (\MB_{ki} - \vBbar^w), k=1,\dots,K
\eea
we can see the equivalence between the objective function in eq.~\eqref{eq:tran-free-obj} and the objective function in eq.~\eqref{eq:trans-freeshaperecon} of Theorem~\ref{thm:trans-freeRecon}. The constraints remain unchanged because we only marginalize out the unconstrained variable $\vt$. Therefore, the shape reconstruction problem~\eqref{eq:weightedleastsquares} is equivalent to the translation-free problem~\eqref{eq:trans-freeshaperecon}, and the optimal translation can be recovered using eq.~\eqref{eq:tofRc}.
\end{proof}

%% file: supp-proof-SOS_relaxation.tex
\section{Proof of Proposition~\ref{prop:naiveSOSRelax}}
\label{sec:proof-prop-SOS_relaxation}

\begin{proof}
Here we prove the SOS relaxation of order $\rorder$ ($\rorder \geq 2$) for the translation-free shape reconstruction problem~\eqref{eq:trans-freeshaperecon} is the semidefinite program in~\eqref{eq:naiveSOS}. First, let us rewrite the general form of Lasserre's hierarchy of order $\rorder$ in eq.~\eqref{eq:LasserreHierarchyordert} in Theorem~\ref{thm:lasserreHierarchy} as the following:
\bea \label{eq:lasserreHierarchyRewrite}
\max & \gamma \\
\subject & f(\vxx) - \gamma = h + g, \nonumber \\
& h \in \langle \vh \rangle_{2\rorder} , \nonumber \\
& g \in Q_\rorder(\vg) . \nonumber 
\eea
In words, the constraints of~\eqref{eq:lasserreHierarchyRewrite} ask the polynomial $f(\vxx) - \gamma$ to be written as a sum of two polynomials $h$ and $g$, with $h$ in the $2\rorder$-th truncated ideal of $\vh$, and $g$ in the $\rorder$-th truncated quadratic module of $\vg$. 

Next, we use the definition of the $2\rorder$-th truncated ideal and the $\rorder$-th truncated quadratic module to explicitly represent $h$ and $g$. First recall the definition of the $2\rorder$-th truncated ideal in eq.~\eqref{eq:2tideal}, which states that $h$ must be written as a sum of polynomial products between the equality constraints $h_i$'s and the polynomial multipliers $\lambda_i$'s:
\bea \label{eq:hassum}
h = \sum_{i=1}^{15} \lambda_i h_i,
\eea
and the degree of each polynomial product $\lambda_i h_i$ must be no greater than $2\rorder$,~\ie, $\deg(\lambda_i h_i) \leq 2\rorder$. In the translation-free shape reconstruction problem, because all the 15 equality constraints in eq.~\eqref{eq:polyConstraintsSOthree} (arising from $\MR \in \SOthree$) have degree 2, the degree of the polynomial multipliers must be at most $2\rorder-2$,~\ie, $\deg(\lambda_i) \leq 2\rorder-2$. Therefore, we can parametrize each $\lambda_i$ using $[\vxx]_{2\rorder-2}$, the vector of monomials up to degree $2\rorder-2$:
\bea \label{eq:paramlambda}
\lambda_i = \vlambda_i\tran [\vxx]_{2\rorder-2}, \quad \vlambda_i \in \Real{N_{\lambda}}, N_{\lambda} = \nchoosek{K+7+2\rorder}{2\rorder-2},
\eea
with $\vlambda_i$ being the vector of unknown coefficients associated with the monomial basis $[\vxx]_{2\rorder-2}$. The size of $\vlambda_i$ is equal to the length of $[\vxx]_{2\rorder-2}$, which can be computed by $\nchoosek{n+d}{d}$, with $n=K+9$ being the number of variables in $\vxx$, and $d=2\rorder-2$ being the maximum degree of the monomial basis. Similarly for $g$, we recall the definition of the $\rorder$-th truncated quadratic module in eq.~\eqref{eq:tqmodule}, which states that $g$ must be written as a sum of polynomial products between the inequality constraints $g_k$'s and the SOS polynomial multipliers $s_k$'s:
\bea \label{eq:gassum}
g = \sum_{k=0}^{2K} s_k g_k,
\eea
and the degree of each polynomial product $s_k g_k$ must be no greater than $2\rorder$,~\ie, $\deg(s_k g_k) \leq 2\rorder$. For our specific shape reconstruction problem, we have $g_0:=1$, $g_k = c_k, k=1,\dots,K$, and $g_{K+k} = 1-c_k^2, k=1,\dots,K$. Since $g_0$ has degree 0, $s_0$ can have degree up to $2\rorder$. All $g_k,k=1,\dots,K$, have degree 1, so $s_k,k=1,\dots,K$, can have degree up to $2\rorder-1$. However, because SOS polynomials can only have even degree, $s_k,k=1,\dots,K$ can only have degree up to $2\rorder-2$. For $g_{K+k},k=1,\dots,K$, they have degree 2, so their corresponding SOS polynomial multipliers $s_{K+k},k=1,\dots,K$ can have degree up to $2\rorder-2$. Now for each SOS polynomial $s_k,k=0,\dots,2K$, from the Gram matrix representation in eq.~\eqref{eq:PSDDescriptionSOS}, we can associate a PSD matrix $\MS_k$ with it using corresponding monomial bases:
\bea \label{eq:paramSOS}
\hspace{-6mm}
\begin{cases}
s_0 = [\vxx]_{\rorder}\tran \MS_0 [\vxx]_{\rorder} & \MS_0 \in \calS_+^{N_0}, N_0 = \nchoosek{K+9+\rorder}{\rorder} \\
s_{k\neq0} = [\vxx]_{\rorder-1}\tran \MS_k [\vxx]_{\rorder-1} & \MS_k \in \calS_+^{N_\lambda}, N_\lambda = \nchoosek{K+8+\rorder}{\rorder-1}
\end{cases}
\eea

Finally, by inserting the expressions of $s_k$ in~\eqref{eq:paramSOS} back to the expression of $g$ in~\eqref{eq:gassum}, and inserting the expression of $\lambda_i$ in~\eqref{eq:paramlambda} back to the expression of $h$ in~\eqref{eq:hassum}, we can convert the SOS relaxation of general form~\eqref{eq:lasserreHierarchyRewrite} to the semidefinite program~\eqref{eq:naiveSOS}.
\end{proof}

%% file: supp-proof-certificate_global_optimality.tex
\section{Proof of Theorem~\ref{thm:certificateGlobalOptimality}}
\label{sec:proof-certificate_global_optimality}
\begin{proof}
According to~\cite{lasserre10book-momentsOpt}, the dual SDP of~\eqref{eq:naiveSOS} is the following SDP:
\bea
\min_{\vy} & L_{\vy}(f) \\
\subject & \MM_\rorder(\vy) \succeq 0,\\
		 & \MM_{\rorder - v_k}(g_k \vy) \succeq 0, \\
		 & \MM_{\rorder - u_i}(h_i \vy) = \zero, \\
		 & y_{\zero} = 1,
\eea
where $\vy \in \Real{N_{2\rorder}}$, $N_{2\rorder} = \nchoosek{K+9+2\rorder}{2\rorder}$ is a vector of \emph{moments} for a probability measure supported on $\calX$ defined by the equalities $h_i$ and inequalities $g_k$; $L_{\vy}(f) = \sum_{\valpha} f_{\valpha}y_{\valpha}$ is a linear function of $\vy$, where $f_{\valpha}$ is the coefficient of $f(\vxx)$ associated with monomial $\vxx^{\valpha}$, and $y^{\valpha}$ is the moment of the monomial $\vxx^{\valpha}$~\wrt~the probability measure; $\MM_\rorder(\vy) \in \Real{N_\rorder}, N_\rorder = \nchoosek{K+9+\beta}{\beta}$ is the moment matrix of degree $\beta$ that assembles all the moments in $\vy$; $\MM_{\rorder - v_k} (g_k \vy), v_k = \lceil \deg (g_k)/2 \rceil$, is the \emph{localizing} matrix that takes some moments from the moment matrix $\MM_\rorder(\vy)$ and entry-wise multiply them with the inequality $g_k$ (\cf~\cite{lasserre10book-momentsOpt} for more details); $\MM_{\rorder - u_i}(h_i \vy), u_i = \lceil \deg(h_i)/2 \rceil$ is the localizing matrix that takes some moments from the moment matrix and entry-wise multiply them with the equality $h_i$. Due to strong duality of the primal-dual SDP, we have complementary slackness:
\bea \label{eq:complementarySlackness}
\MS_0^{\rorder\star} \MM_{\rorder}^{\star}  = \zero,
\eea
at global optimality of the SDP pair. Since $\corank(\MS_0^{\rorder\star}) = 1$, then according to Theorem 5.7 of~\cite{lasserre10book-momentsOpt}, we have $\rank{\MM_{\rorder}^{\star}} = 1$ and $f_{\rorder}^\star$ is the global minimum of the original shape reconstruction problem~\eqref{eq:trans-freeshaperecon}. Further, as $\rank{\MM_{\rorder}^{\star}} = 1$, $\MM_{\rorder}^{\star} = \vv^\star (\vv^\star)\tran $ where $\vv^\star = [\vxx^\star]_{\rorder}$ and $\vxx^\star$ is the unique global minimizer of the original problem~\eqref{eq:trans-freeshaperecon}. However, the fact that $\MS_0^{\rorder\star} \MM_{\rorder}^{\star} = \zero$ and $\MM_{\rorder}^{\star} = \vv^\star (\vv^\star)\tran$ implies:
\bea
\MS_0^{\rorder\star} \vv^\star = \zero,
\eea
and $\vv^\star$ is in the null-space of $\MS_0^{\rorder\star}$. Therefore, the solution extracted using Proposition~\ref{prop:extractSolutions} is also the unique global minimizer of problem~\eqref{eq:trans-freeshaperecon}.
\end{proof}

%% file: supp-derivation-basis_reduction.tex
\section{Derivation of Proposition~\ref{prop:sosBasisReduction}}
\label{sec:derivation-basis_reduction}
Here we show the intuition for using the basis reduction in Proposition~\ref{prop:sosBasisReduction}. In the original SOS relaxation~\eqref{eq:naiveSOS}, the parametrization of the SOS polynomial multipliers $s_k,k=0,\dots,2K$, and the polynomial multipliers $\lambda_i,i=1,\dots,15$, uses the vector of \emph{all} monomials up to their corresponding degrees (\cf~\eqref{eq:paramlambda} and~\eqref{eq:paramSOS}), which leads to an SDP of size $N_0 = \nchoosek{K+9+\rorder}{\rorder}$ that grows quadratically with the number of basis shapes $K$. In basis reduction, we do not limit ourselves to the vector of full monomials, but rather parametrize $s_0$, $s_k$ and $\lambda_i$ with unknown monomials bases $v_0[\vxx]$, $v_s[\vxx]$ and $v_\lambda[\vxx]$, which allows us to rewrite~\eqref{eq:naiveSOSLinearConstraints} as: 
\begin{multline} \label{eq:generalEqCons}
f(\vxx) - \gamma = \overbrace{ v_0[\vxx]\tran \MS_0 v_0[\vxx] }^{s_0} + \\
\sum_{k=1}^{2K} \overbrace{ \left( v_s[\vxx]\tran \MS_k v_k[\vxx] \right) }^{s_k} g_k(\vxx) +  \\
\sum_{i=1}^{15} \overbrace{ \left(\vlambda_i\tran v_{\lambda}[\vxx]\right) }^{\lambda_i} h_i(\vxx),
\end{multline}
with the hope that $v_0[\vxx] \subseteq [\vxx]_2$, $v_s[\vxx] \subseteq [\vxx]_1$ and $v_\lambda[\vxx] \subseteq [\vxx]_2$
 have much smaller sizes (we limit ourselves to the case of $\rorder=2$, at which level the relaxation is empirically tight).

As described, one can see that the problem of finding smaller $v_0[\vxx]$, $v_s[\vxx]$ and $v_\lambda[\vxx]$, while keeping the relaxation empirically tight, is highly combinatorial in general. Therefore, our strategy is to only consider the following case:
\begin{enumerate}[label=(\roman*)]
\item \label{im:expressive} Expressive: choose $v_0[\vxx]$ such that $s_0$ contains all the monomials in $f(\vxx) - \gamma$,
\item \label{im:balanced} Balanced: choose $v_s[\vxx]$ and $v_\lambda[\vxx]$ such that the sum $s_0 + \sum s_k g_k + \sum \lambda_i h_i$ can only have monomials from $f(\vxx) -\gamma$.
\end{enumerate}
In words, condition~\ref{im:expressive} ensures that the right-hand side (RHS) of~\eqref{eq:generalEqCons} contains all the monomials of the left-hand side (LHS). Condition~\ref{im:balanced} asks the three terms of the RHS,~\ie, $s_0$, $\sum s_k g_k$ and $\sum \lambda_i h_i$, to be self-balanced in the types of monomials. For example, if $s_0$ contains extra monomials that are not in the LHS, then those extra monomials better appear also in $\sum s_k g_k$ and/or $\sum \lambda_i h_i$ so that they could be canceled by summation. Under these two conditions, it is \emph{possible} to have equation~\eqref{eq:generalEqCons} hold\footnote{Whether or not these are sufficient or necessary conditions remains open. However, leveraging Theorem~\ref{thm:certificateGlobalOptimality} we can still check optimality a posteriori.}. 

The choices in both conditions depend on analyzing the monomials in $f(\vxx) - \gamma$. Recall the expression of $f(\vxx)$ in~\eqref{eq:trans-freeshaperecon} and the expression $q_i(\vxx)$ in~\eqref{eq:costqi} for each term inside the summation, it can be seen that $f(\vxx)$ only contains the following types of monomials:
\bea \label{eq:monomialsoff}
& 1 \\
& c_k, 1\leq k \leq K \\
& c_k r_j, 1 \leq k \leq K, 1 \leq j \leq 9 \\
& \hspace{-6mm} c_{k_1} c_{k_2} r_{j_1} r_{j_2}, 1\leq k_1 \leq k_2 \leq K, 1 \leq j_1 \leq j_2 \leq 9
\eea
and the key observation is that $f(\vxx)$ does \emph{not} contain degree-4 monomials \emph{purely} in $\vc$ or $\vr$,~\ie, $c_{k_1}c_{k_2}c_{k_3}c_{k_4}$ and $r_{j_1}r_{j_2}r_{j_3}r_{j_4}$, or any degree-3 monomials in $\vc$ and $\vr$. Therefore, when choosing $v_0[\vxx]$, we can exclude degree-2 monomials purely in $\vc$ and $\vr$ from $[\vxx]_2$, and set $v_0[\vxx] = m_2(\vxx) = [1, \vc\tran, \vr\tran, \vc\tran \kron \vr\tran]\tran$\footnote{A more rigorous analysis should follow the rules of Newton Polytope~\cite{Blekherman12Book-sdpandConvexAlgebraicGeometry}, but the intuition is the same as what we describe here.} as stated in Proposition~\ref{prop:sosBasisReduction}. This will satisfy the expressive condition~\ref{im:expressive}, because $s_0 = m_2[\vxx]\tran \MS_0 m_2[\vxx]$ can have the following monomials:
\bea 
& \label{eq:monomialsofs0_1} 1, \quad c_k, \quad c_k r_j, \quad c_{k_1} c_{k_2} r_{j_1} r_{j_2} \\
& \label{eq:monomialsofs0_2} r_j, \quad c_{k_1}c_{k_2}, \quad r_{j_1}r_{j_2}, \quad c_{k_1}c_{k_2} r_j, \quad c_k r_{j_1} r_{j_2}
\eea
and those in~\eqref{eq:monomialsofs0_1} cover the monomials in $f(\vxx)$. Replacing $[\vxx]_2$ with $v_0[\vxx] = m_2(\vxx)$ is the key step in reducing the size of the SDP, because it reduces the size of the SDP from $\nchoosek{K+11}{2}$ to $10K+10$,~\ie, from quadratic to linear in $K$.

In order to satisfy condition~\ref{im:balanced}, when choosing $v_s[\vxx]$ and $v_\lambda[\vxx]$, the goal is to have the product between $s_k$, $\lambda_i$ and $g_k$, $h_i$ result in monomials that appear in $f(x)-\gamma$, and ensure that monomials that do not appear in the 
latter can simplify our in the summation. For example, as stated in Proposition~\ref{prop:sosBasisReduction}, we choose $v_s[\vxx] = [\vr]_1 = [1, \vr\tran]\tran$ and $s_k$ will contain monomials 1, $r_j$ and $r_{j_1} r_{j_2}$.
Because $g_k$'s have monomials $1$, $c_k$ and $c_k^2$, we can see that $\sum s_k g_k$ will contain the following monomials:
\bea
& \label{eq:monomialsofsg_1} 1, \ c_k,\ r_{j_1}r_{j_2}c_k^2,\\
& \label{eq:monomialsofsg_2} c_k^2, \ r_j, \ r_j c_k^2,\ r_{j_1}r_{j_2},\ r_{j_1}r_{j_2}c_k,\ 
\eea
This still satisfies the balanced condition, because monomials of $\sum s_k g_k$ in~\eqref{eq:monomialsofsg_1} balance with monomials of $s_0$ in~\eqref{eq:monomialsofs0_1}, and monomials of $\sum s_k g_k$ in~\eqref{eq:monomialsofsg_1} balance with monomials of $s_0$ in~\eqref{eq:monomialsofs0_2}. Similarly, choosing $v_\lambda[\vxx] = [\vc]_2$ makes $\lambda_i$ have monomials $1$, $c_k$ and $c_{k_1} c_{k_2}$, and because $h_i$'s have monomials $1$, $r_j$ and $r_{j_1}r_{j_2}$, we see that $\sum \lambda_i h_i$ contains the following monomials:
\bea
& 1, \ c_k, \ c_k r_j, \ c_{k_1}c_{k_2}r_{j_1}r_{j_2}, \\
& r_j, \ r_{j_1}r_{j_2}, \ c_k r_{j_1}r_{j_2}, \ c_{k_1}c_{k_2},  \ c_{k_1}c_{k_2} r_j, 
\eea
which balance with monomials in $s_0$ from~\eqref{eq:monomialsofs0_1} and~\eqref{eq:monomialsofs0_2}.

We remark that 
we cannot guarantee that the SOS relaxation resulting from basis reduction can achieve the same performance as the original SOS relaxation and we cannot guarantee our choice of basis is ``optimal'' in any sense. Therefore, in practice, one needs to check the solution and compute $\corank(\MS_0^{2\star})$ and $\eta_2$ to check the optimality of the solution produced by~\eqref{eq:SOSBasisReduction}. Moreover, it remains an open problem to find a better set of monomials bases to achieve better reduction (\eg, knowing more about the algebraic geometry of $g_k$ and $h_i$ could possibly enable using the \emph{standard monomials} as a set of bases~\cite{Blekherman12Book-sdpandConvexAlgebraicGeometry}).

%% file: supp-derivation-robust_algorithm.tex
\section{Derivation of Algorithm~\ref{alg:shaperobust}}
\label{sec:derivation-robust_algorithm}
For a complete discussion of graduated non-convexity and its applications for robust spatial perception, please see~\cite{Yang20ral-GNC}.

In the main document, for robust shape reconstruction, we adopt the TLS shape reconstruction formulation:
\bea \label{eq:suppTLSShapeRecon}
\hspace{-5mm} 
\!\!\!\min_{\substack{c_k\geq 0, \\ k=1,\dots,K \\ \vt \in \Real{2}, \MR\in \SOthree} }
 &
\hspace{-3mm} 
\displaystyle \sumfeatures 
\rho_{\barc}\left( r_i(c_k,\MR,\vt) \right) + \alpha \!\sumbasis \! c_k 
\eea
where 
$r_i(c_k,\MR,\vt) := \left\| \vz_i \!-\! \Pi \MR \left( \sumbasis c_k \MB_{ki} \right) \!-\! \vt \right\|$ is called the \emph{residual}, 
and
$\rho_{\barc}(r) = \min(r^2, \barcsq)$ implements a truncated least squares cost. Recalling that $\rho_{\barc}(r) 
= \min(r^2, \barcsq)
= \min_{w\in\{0,1\}} w r^2 + (1-w)\barcsq$, we can rewrite the TLS shape reconstruction as a joint optimization of $(\vc,\MR,\vt)$ and the binary variables $w_i$'s, as in eq.~\eqref{eq:TLSShapeRecon2} in the main document. However, as hinted in the main document, due to the non-convexity of the TLS cost, directly solving the joint problem or alternating between solving for $(\vc,\MR,\vt)$ and binary variables $w_i$'s would require an initial guess and is prone to bad local optima. 

The idea of graduated non-convexity (\GNC)~\cite{Blake1987book-visualReconstruction} is to introduce a surrogate function $\rho^\mu_{\barc}(r)$, governed by a control parameter $\mu$, such that changing $\mu$ allows $\rho^\mu_{\barc}(r)$ to start from a convex proxy of $\rho_{\barc}(r)$, and gradually increase the amount of non-convexity till the  original TLS function $\rho_{\barc}(r)$ is recovered. The surrogate function for TLS is stated below.

\begin{proposition}[Truncated Least Squares (TLS) and \GNC] 
\label{prop:TLS}
The truncated least squares function is defined as:
\bea
\label{eq:formulationTLS}
\hspace{-2mm} \rho_{\barc}(r) = 
\begin{cases}
\scriptstyle{ r^2 } & {\footnotesize\text{ if }}  \scriptstyle{ r^2 \in \left[0, \barcsq \right] } \\
 \scriptstyle{ \barcsq } & {\footnotesize\text{ if }}  \scriptstyle{ r^2 \in \left[\barcsq, +\infty \right) }
\end{cases},
\eea
where $\barc$ is a given truncation threshold. 
The \GNC surrogate function with control parameter $\mu$ is:
\bea \label{eq:GNCsurrogateTLS}
\hspace{-3mm} \rho^\mu_{\barc}(r) = \begin{cases}
\scriptstyle{ r^2 } & {\footnotesize\text{ if }} \scriptstyle{ r^2 \in \left[0, \frac{\mu }{\mu + 1}\barcsq \right] } \\
  \scriptstyle{ 2\barc | r | \sqrt{\mu(\mu+1)} - \mu (\barcsq + r^2)  } & {\footnotesize\text{ if }} \scriptstyle{ r^2 \in \left[\frac{\mu }{\mu + 1} \barcsq, \frac{\mu+1}{\mu}\barcsq \right] }\\
 \scriptstyle{ \barcsq } & {\footnotesize\text{ if }} \scriptstyle{ r^2 \in \left[\frac{\mu+1}{\mu}\barcsq, +\infty \right) }
\end{cases}.
\eea
By inspection, one can verify $\rho^\mu_{\barc}(r)$ is convex for $\mu$ approaching zero (($\rho^{\mu}_{\barc}(r))'' = - 2\mu \rightarrow 0$) and retrieves $\rho_{\barc}(r)$  in~\eqref{eq:formulationTLS} for $\mu\rightarrow+\infty$.
An illustration of $\rho^\mu_{\barc}(r)$ 
is given in Fig.~\ref{fig:GNCplots}.
\end{proposition}

The nice property of the \GNC surrogate function is that when $\mu$ is close to zero, $\rho^\mu_{\barc}$ is convex, which means the only non-convexity of problem~\eqref{eq:suppTLSShapeRecon} comes from the constraints and can be relaxed using the SOS relaxations.

For the \GNC surrogate function $\rho^\mu_{\barc}$, the simple trick of introducing binary variables ($\rho_{\barc}(r)= \min_{w\in\{0,1\}} w r^2 + (1-w)\barcsq$) would not work. However, Black and Rangarajan~\cite{Black96ijcv-unification} showed that this idea of introducing an~\emph{outlier variable}\footnote{$w$ can be thought of an outlier variable: when $w=1$, the measurement is an inlier, when $w=1$, the measurement is an outlier.} can be generalized to many robust cost functions. In particular, for the \GNC surrogate function, we have the following.

\begin{theorem}[Black-Rangarajan Duality for \GNC surrogate TLS] \label{thm:brdualityGNCTLS}
The \GNC surrogate TLS shape reconstruction:
\bea \label{eq:GNCTLSshape}
\hspace{-5mm} 
\!\!\!\min_{\substack{c_k\geq 0, \\ k=1,\dots,K \\ \vt \in \Real{2}, \MR\in \SOthree} }
 &
\hspace{-3mm} 
\displaystyle \sumfeatures 
\rho_{\barc}^\mu\left( r_i(c_k,\MR,\vt) \right) + \alpha \!\sumbasis \! c_k 
\eea
with $\rho^\mu_{\barc}(r)$ defined in~\eqref{eq:GNCsurrogateTLS}, is equivalent to the following optimization with outlier variables $w_i$'s:
\bea \label{eq:GNCTLSshape_br}
\hspace{-5mm} 
\!\!\!\min_{\substack{c_k\geq 0, \\ k=1,\dots,K \\ \vt \in \Real{2}, \MR\in \SOthree \\ w_i \in [0,1],i=1,\dots,N} }
\hspace{-3mm} 
\displaystyle \sumfeatures \left[
w_i r_i^2(c_k,\MR,\vt) + \Phi^\mu_{\barc}(w_i) \right] + \alpha \!\sumbasis \! c_k
\eea
where $\Phi^\mu_{\barc}(w_i)$ is the following outlier process:
\bea \label{eq:outlierProcessGNCTLS}
\Phi^\mu_{\barc}(w_i) = \frac{\mu (1-w_i)}{\mu + w_i} \barcsq.
\eea
\end{theorem}
\begin{proof}
The derivation of $\Phi^\mu_{\barc}(w_i)$ in~\eqref{eq:outlierProcessGNCTLS} follows the Black-Rangarajan procedure in Fig.~10 of~\cite{Black96ijcv-unification}.
\end{proof}

In words, the Black-Rangarajan duality allows us to rewrite the non-convex shape reconstruction problem as a joint optimization in $(\vc,\MR,\vt)$ and outlier variables $w_i$'s.
The interested readers can find closed-form outlier processes for many other robust cost functions in the original paper~\cite{Black96ijcv-unification}.

Leveraging the Black-Rangarajan duality, for any given choice of the control parameter $\mu$, we can solve problem~\eqref{eq:GNCTLSshape_br} in two steps: first we solve $(\vc,\MR,\vt)$ using \name with fixed weights $w_i$'s, and then we update the weights with fixed $(\vc,\MR,\vt)$. In particular, at each iteration $\tau$ (corresponding to a given control parameter $\mu$), we perform the following:
\begin{enumerate}
\item {\bf Variable update}: minimize~\eqref{eq:GNCTLSshape_br} with respect to $(\vc,\MR,\vt)$, with fixed weights $w_i^{(\tau-1)}$:
\begin{multline}
\hspace{-10mm} c_k^{(\tau)}, \MR^{(\tau)}, \vt^{(\tau)} =  \\
\!\!\!\argmin_{\substack{c_k\geq 0, \\ k=1,\dots,K \\ \vt \in \Real{2}, \MR\in \SOthree} }
\hspace{-3mm} 
\displaystyle \sumfeatures 
w_i^{(\tau-1)} r_i^2(c_k,\MR,\vt)  + \alpha \!\sumbasis \! c_k,
\end{multline}
where we have dropped the term $\sumfeatures \Phi^\mu_{\barc}(w_i)$ because it is independent from $(\vc,\MR,\vt)$. This problem is exactly the weighted least squares problem~\eqref{eq:weightedleastsquares} and can be solved using \name (\cf~line~\ref{line:variableUpdate} in Algorithm~\ref{alg:shaperobust}). Using the solutions $(c_k^{(\tau)}, \MR^{(\tau)}, \vt^{(\tau)})$, we can compute the residuals $r_i^{(\tau)}$ (\cf~line~\ref{line:computeResidual} in Algorithm~\ref{alg:shaperobust}).

\item {\bf Weight update}: minimize~\eqref{eq:GNCTLSshape_br} with respect to $w_i$, with fixed residuals $r_i^{(\tau)}$:
\bea
\hspace{-5mm} w_i^{(\tau)} = \argmin_{w_i \in [0,1],i=1,\dots,N} \sumfeatures ( r_i^{(\tau)})^2 w_i + \Phi_{\barc}^\mu(w_i),
\eea 
where we have dropped $\sumbasis c_k^{(\tau)}$ because it is a constant for the optimization. This optimization, fortunately, can be solved in closed-form. We take the gradient of the objective function  with respect to $w_i$:
\bea
& \nabla_{w_i} = ( r_i^{(\tau)})^2 + \nabla_{w_i} \Phi_{\barc}^\mu(w_i) \nonumber \\
& = ( r_i^{(\tau)})^2  - \frac{\mu(\mu+1)}{(\mu + w_i)^2} \barcsq
\eea
and observe that $\nabla_{w_i} = ( r_i^{(\tau)})^2 - \frac{\mu+1}{\mu}\barcsq$ when $w_i=0$, and $\nabla_{w_i} = ( r_i^{(\tau)})^2 - \frac{\mu}{\mu+1}\barcsq$ when $w_i=1$. Therefore, the global minimizer $w_i^\star := w_i^{(\tau)}$ is:
\smaller
\bea \label{eq:dualUpdateTLS}
\hspace{-10mm}{}w_i^{(\tau)} = \begin{cases}
0 & \text{ if } ( r_i^{(\tau)})^2 \in \left[ \frac{\mu+1}{\mu}\barcsq, +\infty \right] \\
\frac{\barc}{r_i^{(\tau)}}\sqrt{\mu(\mu+1)} - \mu & \text{ if } ( r_i^{(\tau)})^2 \in \left[ \frac{\mu}{\mu+1}\barcsq ,\frac{\mu+1}{\mu}\barcsq \right] \\
1 & \text{ if } ( r_i^{(\tau)})^2 \in \left[ 0,\frac{\mu}{\mu+1}\barcsq \right]
\end{cases}.
\eea
\normalsize
and this is the weight update rule in line~\ref{line:weightUpdate} of Algorithm~\ref{alg:shaperobust}.
\end{enumerate}

After both the variables and weights are updated using the shaperobust approach described above, we increase the control parameter $\mu$ to increase the non-convexity of the surrogate function $\rho^\mu_{\barc}$ (\cf~line~\ref{line:updateControlPara} of Algorithm~\ref{alg:shaperobust}). At the next iteration $\tau+1$, the updated weights are used to perform the variable update. 
The iterations terminate when the change in the objective function becomes negligible (\cf~line~\ref{line:checkConvergence} of Algorithm~\ref{alg:shaperobust})
or after a maximum number of iterations (\cf~line~\ref{line:maxIters} of Algorithm~\ref{alg:shaperobust}). 
Note that all weights are initialized to $1$ (\cf~line~\ref{line:init1} in Algorithm~\ref{alg:shaperobust}), which means 
that initially all measurements are tentatively accepted as inliers, therefore no prior information about inlier/outlier is required.

\input{supp-fig-GNC_TLS}

%% file: supp-fig-GNC_TLS.tex

\newcommand{\mpw}{4.7cm}
\begin{figure}[t!]
	\begin{center}
	\begin{minipage}{\columnwidth}
	\hspace{-0.2cm}
			\centering%
			\includegraphics[width=0.7\columnwidth]{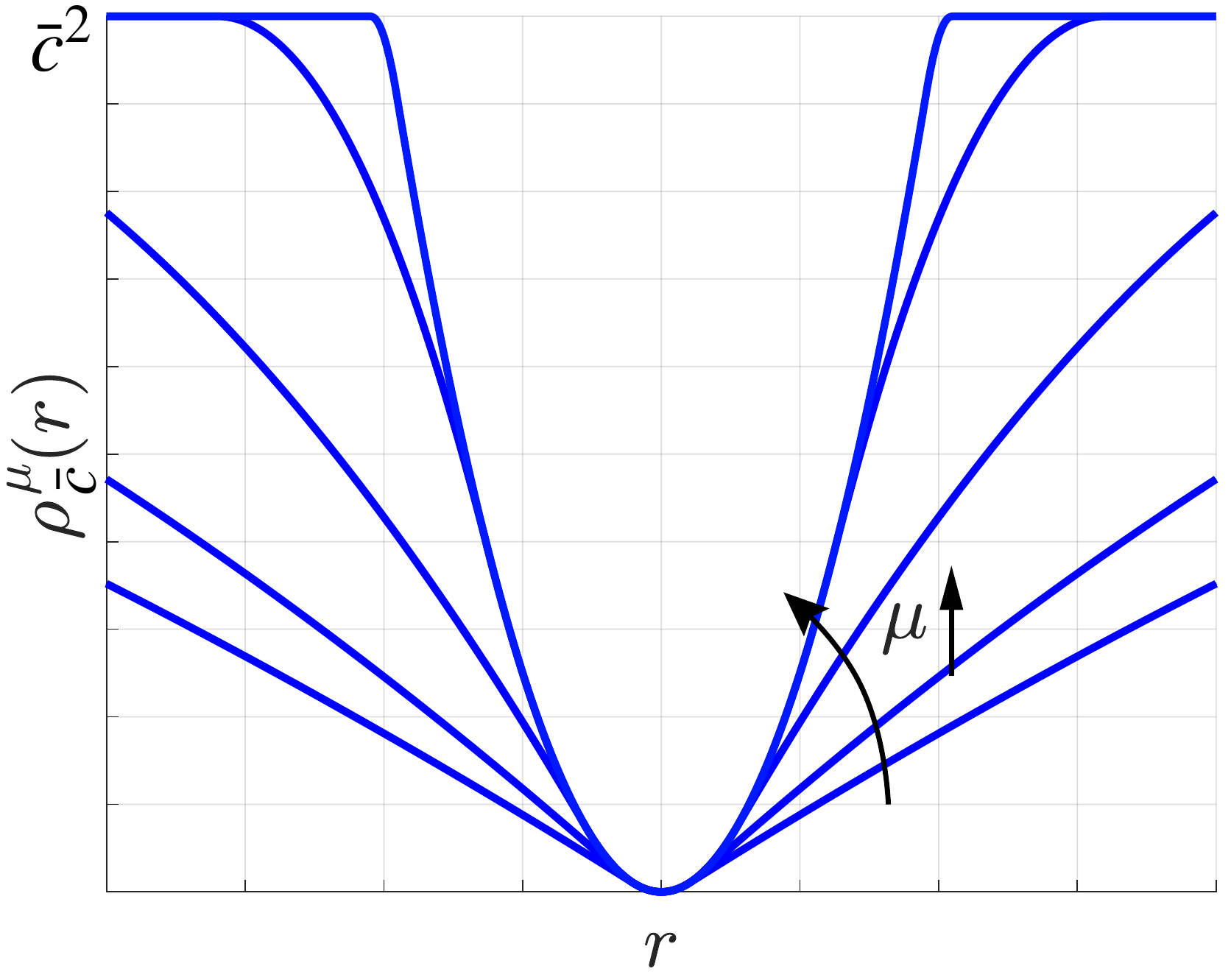}
	\end{minipage}
	\vspace{-4mm}
	\caption{Graduated Non-Convexity (\GNC) with control parameter $\mu$ for the Truncated Least Squares (TLS) cost.
	\label{fig:GNCplots}}
	\end{center}
\end{figure}

%% file: supp-experiments.tex

\section{\FGCar Qualitative Results}
\label{sec:supp_experiments}

Fig.~\ref{fig:supp_FG3DCar_qualitative} shows 9 full qualitative results comparing the performances of \alternrobust~\cite{Zhou17pami-shapeEstimationConvex}, \convexrobust~\cite{Zhou17pami-shapeEstimationConvex} and \namerobust on the \FGCar~\cite{Lin14eccv-modelFitting} dataset under $10\%$ to $70\%$ outlier rates. One can further see that the performance of~\namerobust is sensitive to $70\%$ outliers, while the performances of \alternrobust and \convexrobust gradually degrade and fail at $50\%$ to $60\%$ outliers.

\input{supp-fig-FG3DCar_qualitative}
\input{supp-fig-FG3DCar_qualitative-1}
\input{supp-fig-FG3DCar_qualitative-2}

%% file: supp-fig-FG3DCar_qualitative.tex

\renewcommand{\mpwthree}{2.6cm}
\renewcommand{\myhspace}{\hspace{-3.5mm}}

\begin{figure*}[h]
	\begin{center}{}
	\begin{minipage}{\textwidth}
	\begin{tabular}{cccccccc}%
		\myhspace \myhspace \hspace{-2mm} & \myhspace $10\%$ & \myhspace $20\%$ & \myhspace $30\%$ & \myhspace $40\%$ & \myhspace $50\%$ & \myhspace $60\%$ & \myhspace $70\%$ \\
		\input{056-Chevrolet_Colorado_LS}
		\input{035-BMW_5_Series}
		\input{169-MercedesBenz_C_600}
		\end{tabular}
	\end{minipage}
	\vspace{-2mm} 
	\caption{Qualitative results on the \FGCar dataset~\cite{Lin14eccv-modelFitting} under $10-70\%$ outlier rates using \alternrobust~\cite{Zhou17pami-shapeEstimationConvex}, \convexrobust~\cite{Zhou17pami-shapeEstimationConvex}, and \namerobust. Yellow: shape reconstruction result projected onto the image.
	Green: inliers. Red: outliers. Circle: 3D landmark. Square: 2D landmark. [Best viewed electronically.]
	\label{fig:supp_FG3DCar_qualitative}}
	\end{center}
\end{figure*}
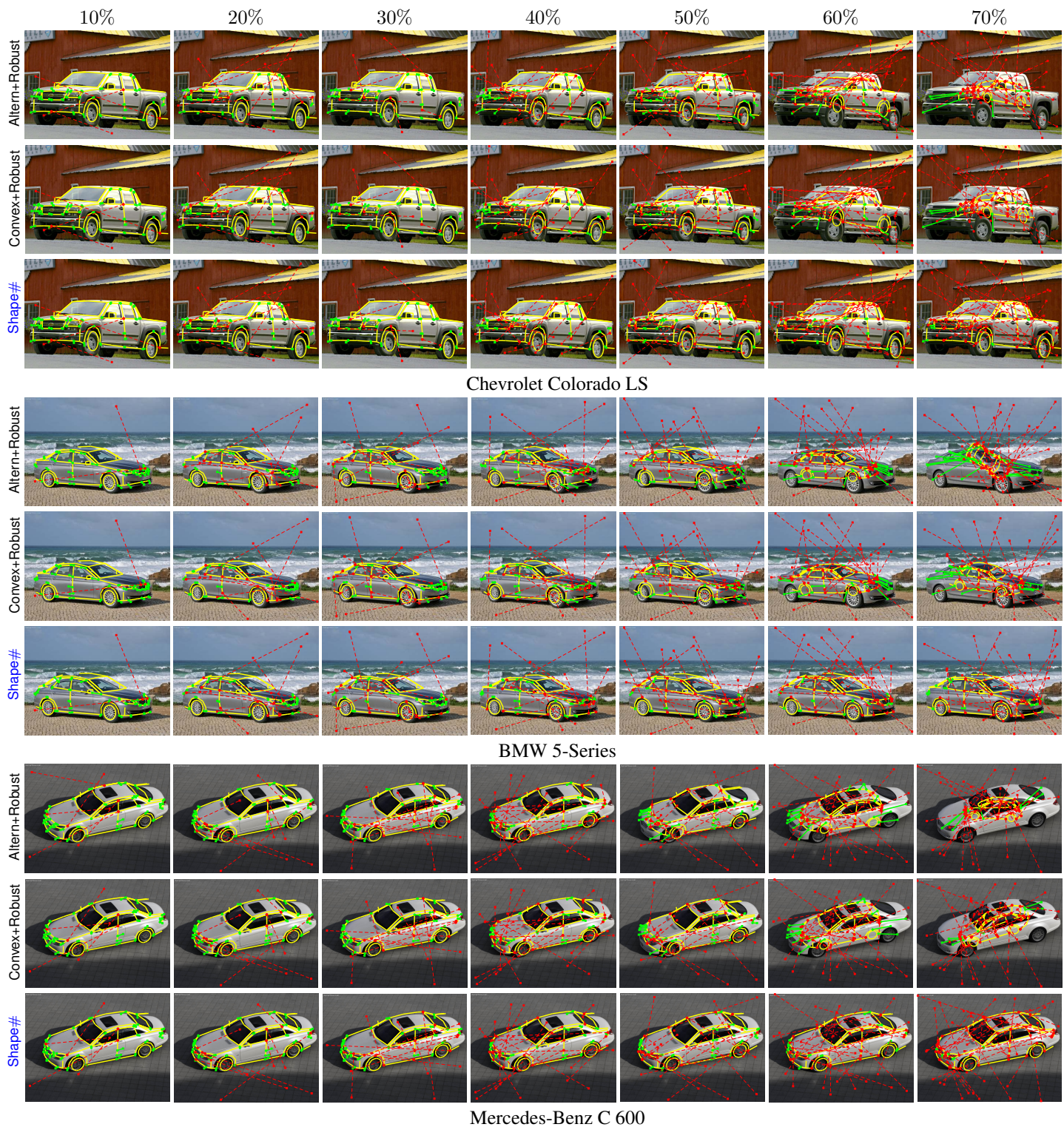

%% file: 056-Chevrolet_Colorado_LS.tex
\myhspace \myhspace \hspace{-2mm} \rotatebox{90}{\hspace{-7mm} {\smaller \alternrobust} } & 
\myhspace
	\begin{minipage}{\mpwthree}%
	\centering%
	\includegraphics[width=\columnwidth]{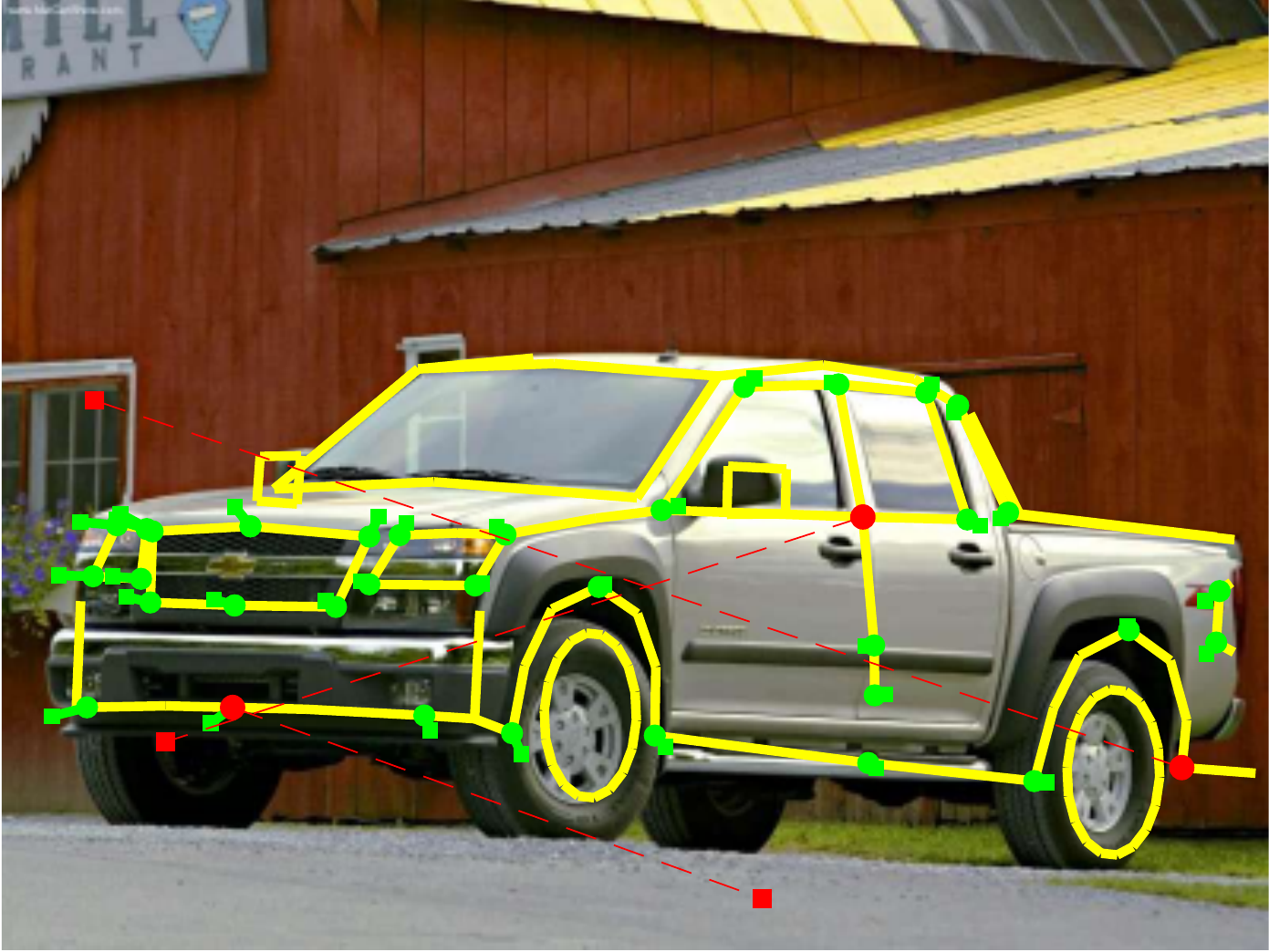} \\
	\vspace{1mm}
	\end{minipage}
& \myhspace
	\begin{minipage}{\mpwthree}%
	\centering%
	\includegraphics[width=\columnwidth]{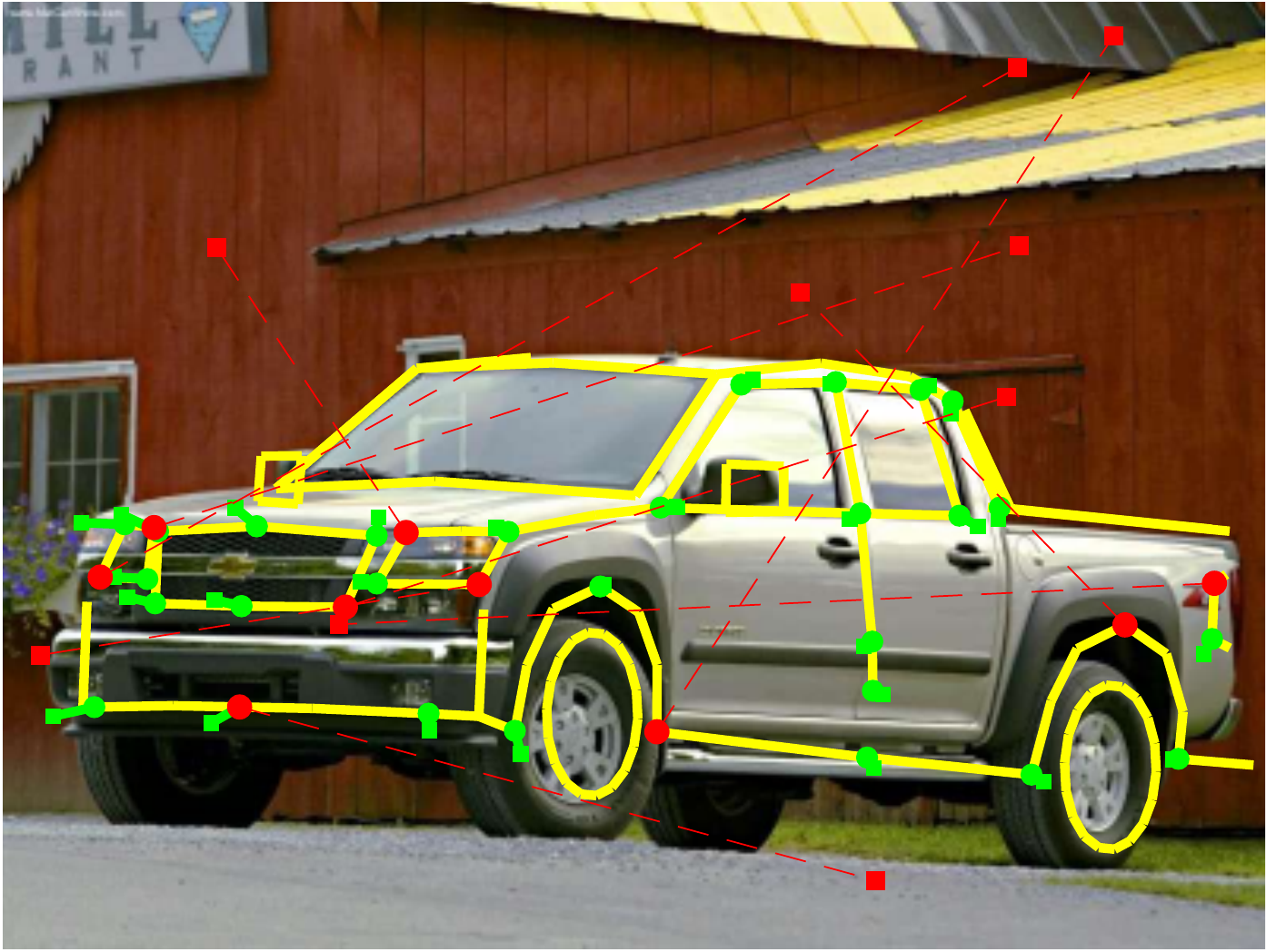} \\
	\vspace{1mm}
	\end{minipage}
& \myhspace
	\begin{minipage}{\mpwthree}%
	\centering%
	\includegraphics[width=\columnwidth]{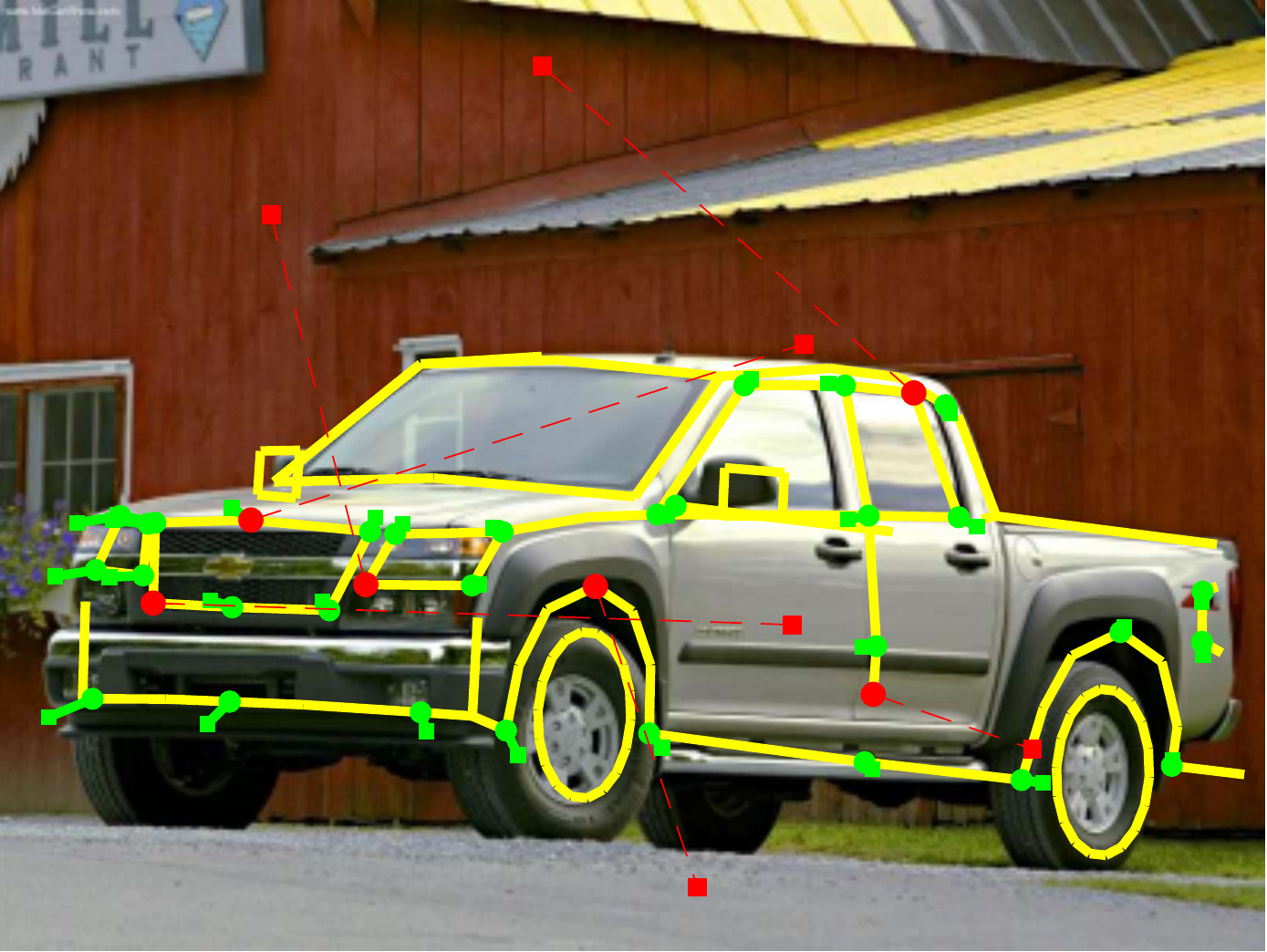} \\
	\vspace{1mm}
	\end{minipage} 
& \myhspace
	\begin{minipage}{\mpwthree}%
	\centering%
	\includegraphics[width=\columnwidth]{056_40_altern+robust_Chevrolet_Colorado-LS_2004_06.pdf} \\
	\vspace{1mm}
	\end{minipage}
& \myhspace
	\begin{minipage}{\mpwthree}%
	\centering%
	\includegraphics[width=\columnwidth]{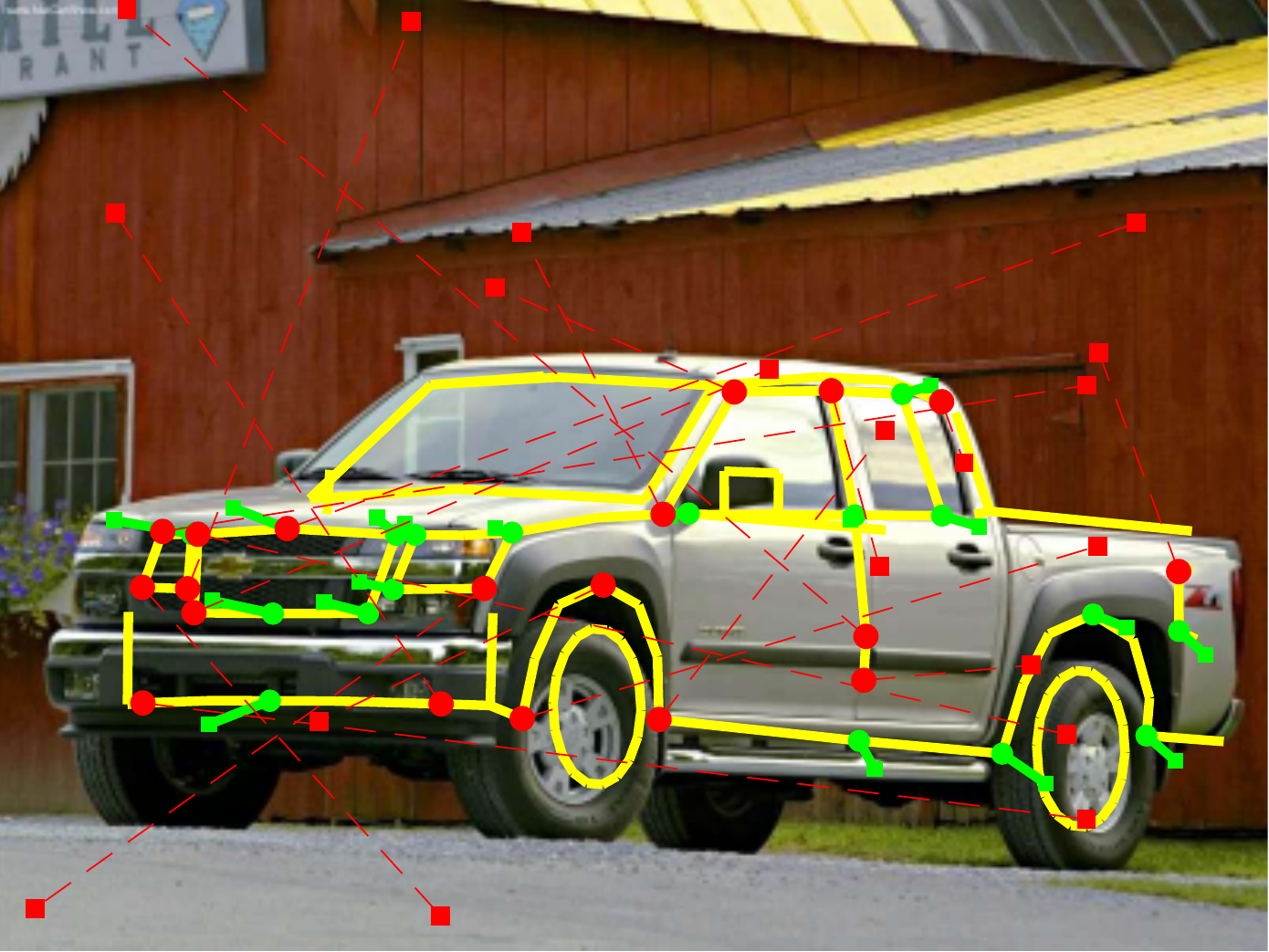} \\
	\vspace{1mm}
	\end{minipage}
&  \myhspace
	\begin{minipage}{\mpwthree}%
	\centering%
	\includegraphics[width=\columnwidth]{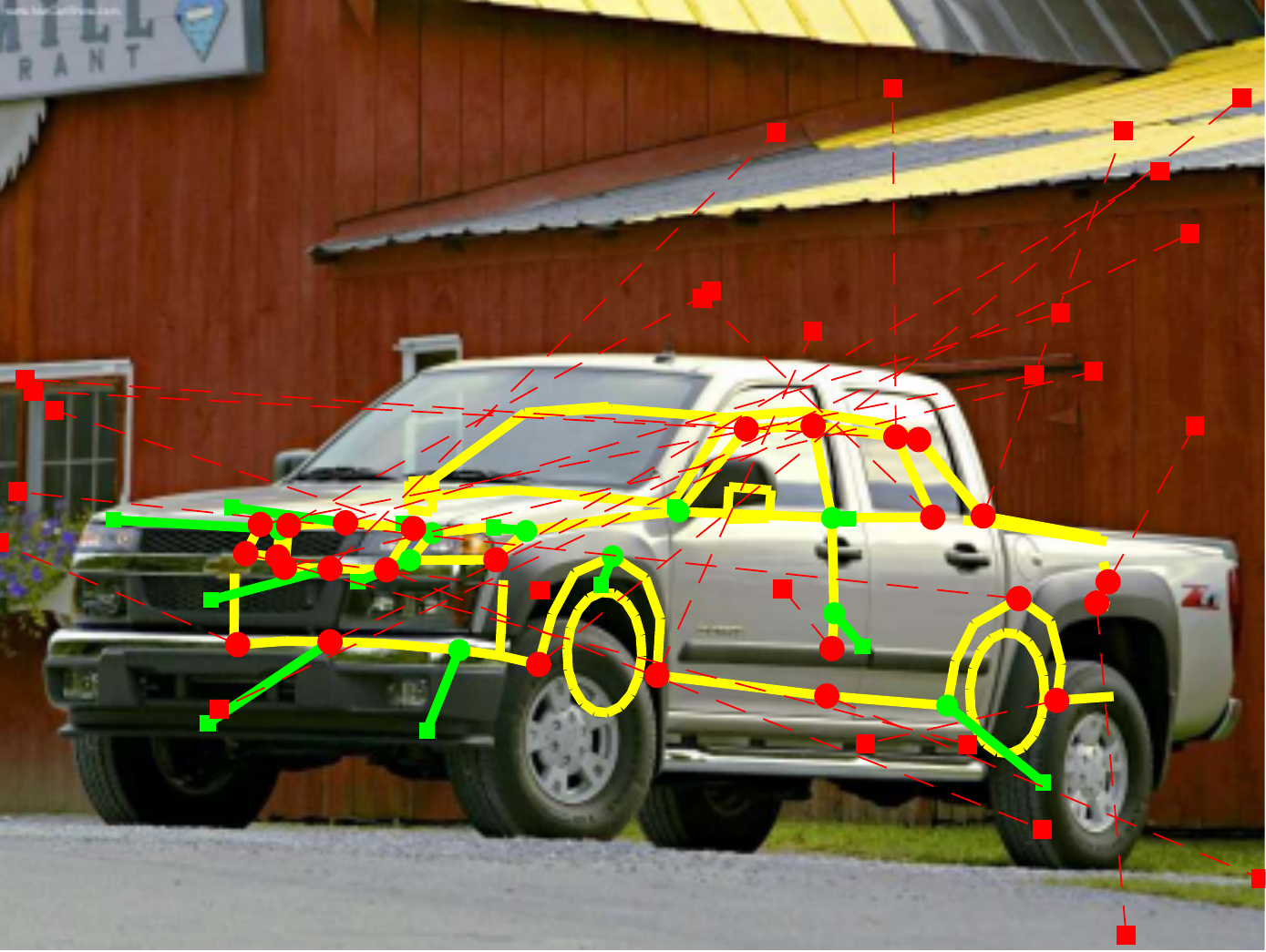} \\
	\vspace{1mm}
	\end{minipage} 
&  \myhspace
	\begin{minipage}{\mpwthree}%
	\centering%
	\includegraphics[width=\columnwidth]{056_70_altern+robust_Chevrolet_Colorado-LS_2004_06.pdf} \\
	\vspace{1mm}
	\end{minipage} \\
\myhspace \myhspace \hspace{-2mm} \rotatebox{90}{\hspace{-8mm} {\smaller \convexrobust} } & 
\myhspace
	\begin{minipage}{\mpwthree}%
	\centering%
	\includegraphics[width=\columnwidth]{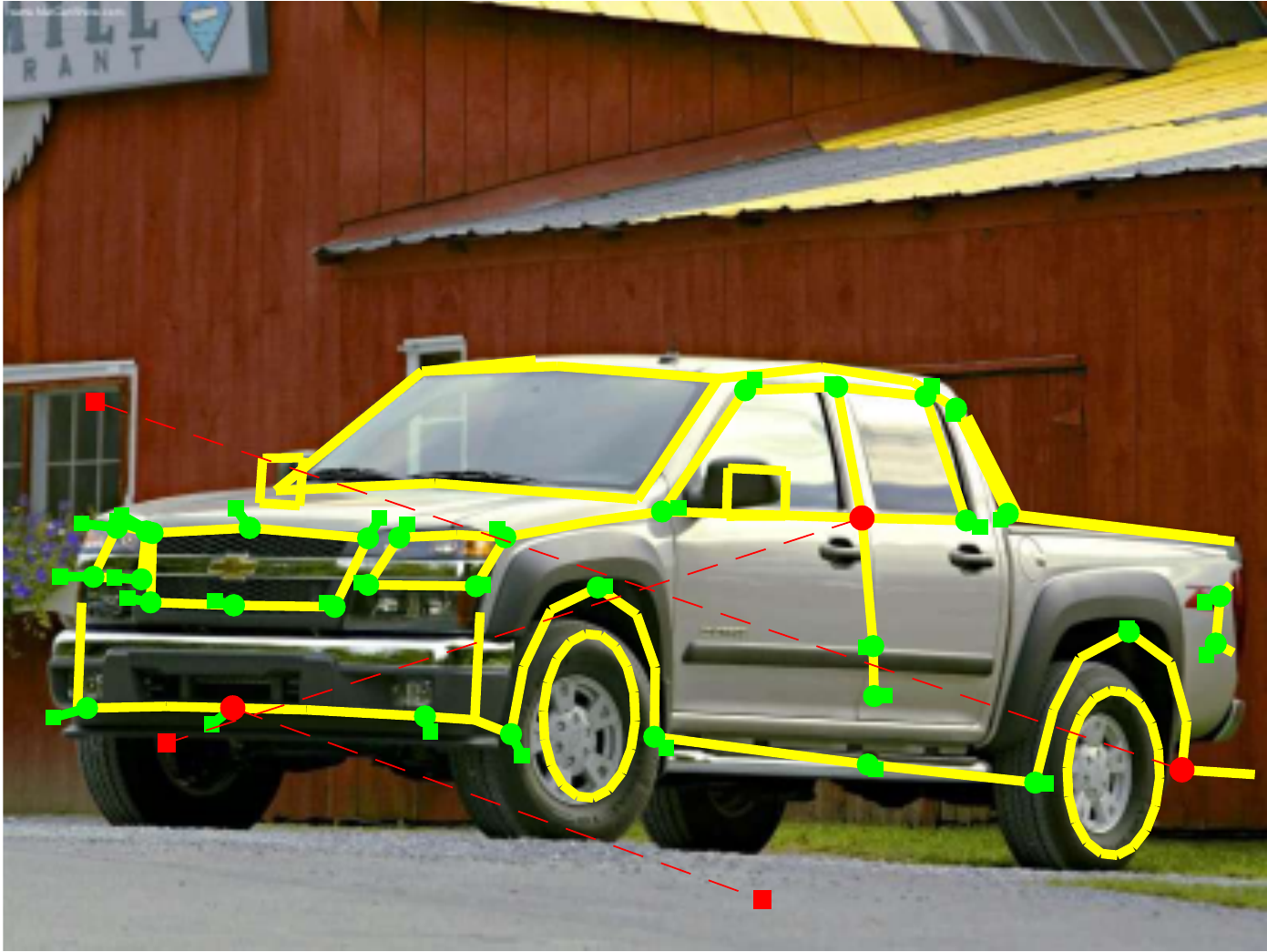} \\
	\vspace{1mm}
	\end{minipage}
& \myhspace
	\begin{minipage}{\mpwthree}%
	\centering%
	\includegraphics[width=\columnwidth]{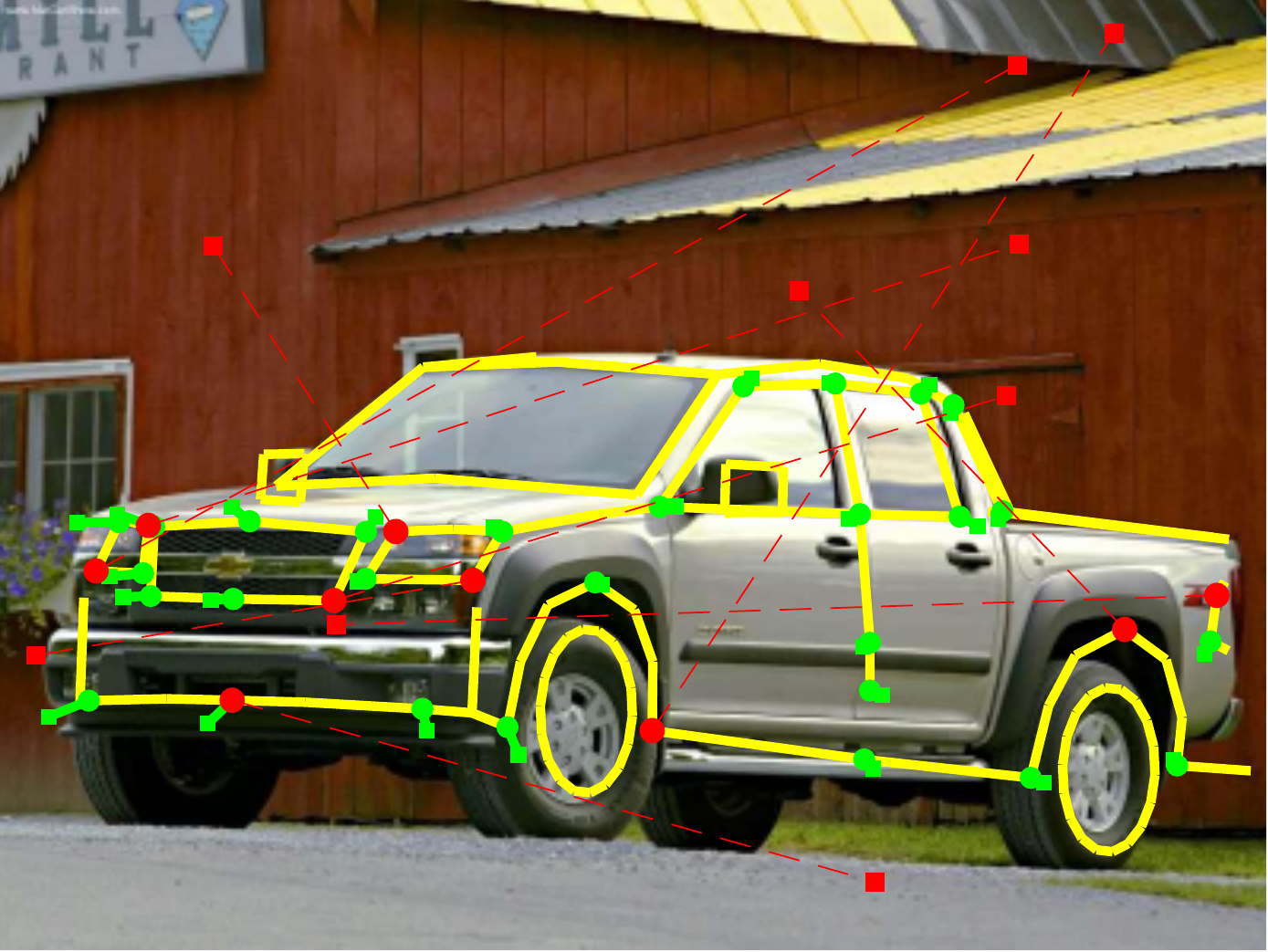} \\
	\vspace{1mm}
	\end{minipage}
& \myhspace
	\begin{minipage}{\mpwthree}%
	\centering%
	\includegraphics[width=\columnwidth]{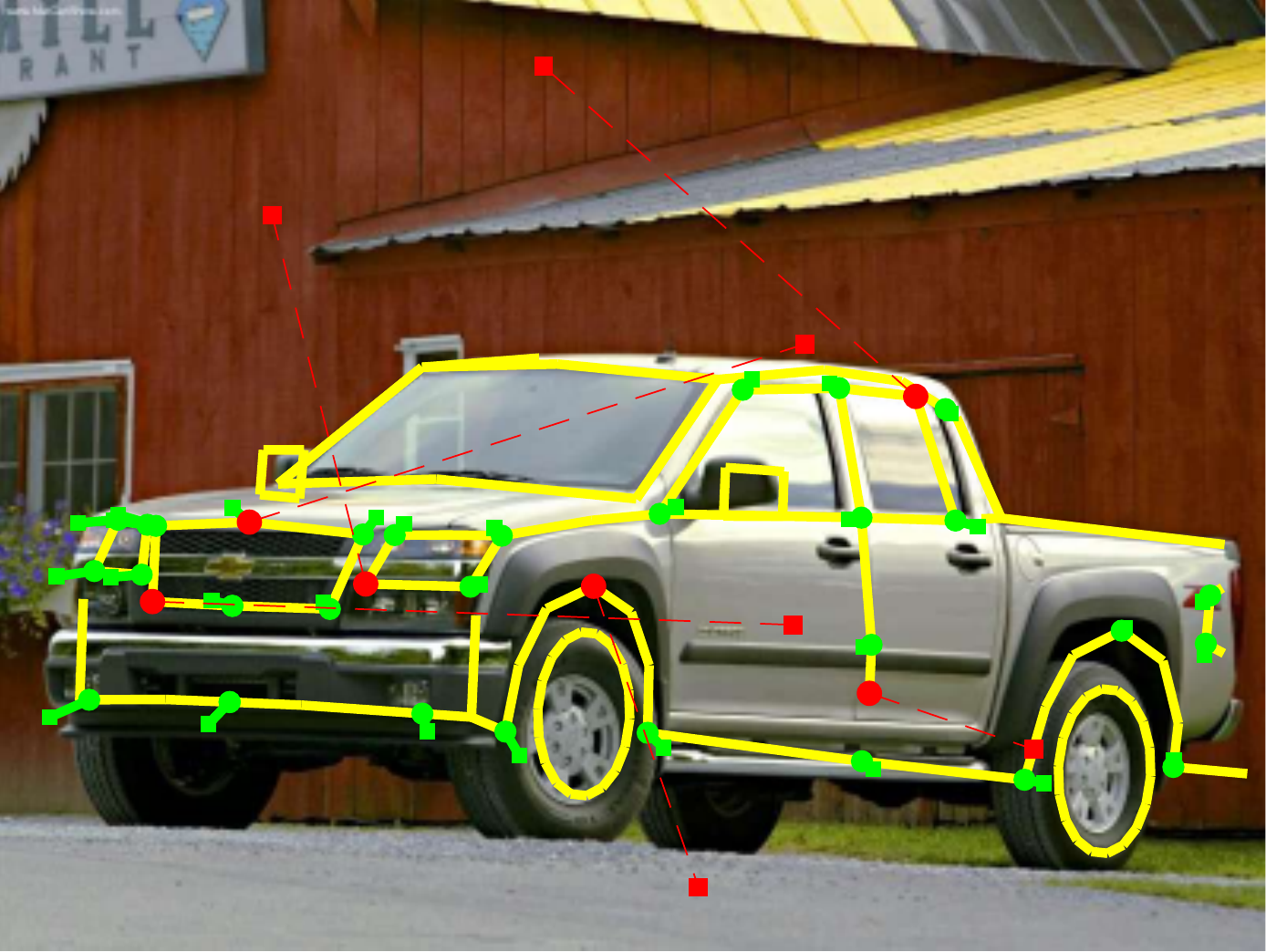} \\
	\vspace{1mm}
	\end{minipage} 
& \myhspace
	\begin{minipage}{\mpwthree}%
	\centering%
	\includegraphics[width=\columnwidth]{056_40_convex+robust+refine_Chevrolet_Colorado-LS_2004_06.pdf} \\
	\vspace{1mm}
	\end{minipage}
& \myhspace
	\begin{minipage}{\mpwthree}%
	\centering%
	\includegraphics[width=\columnwidth]{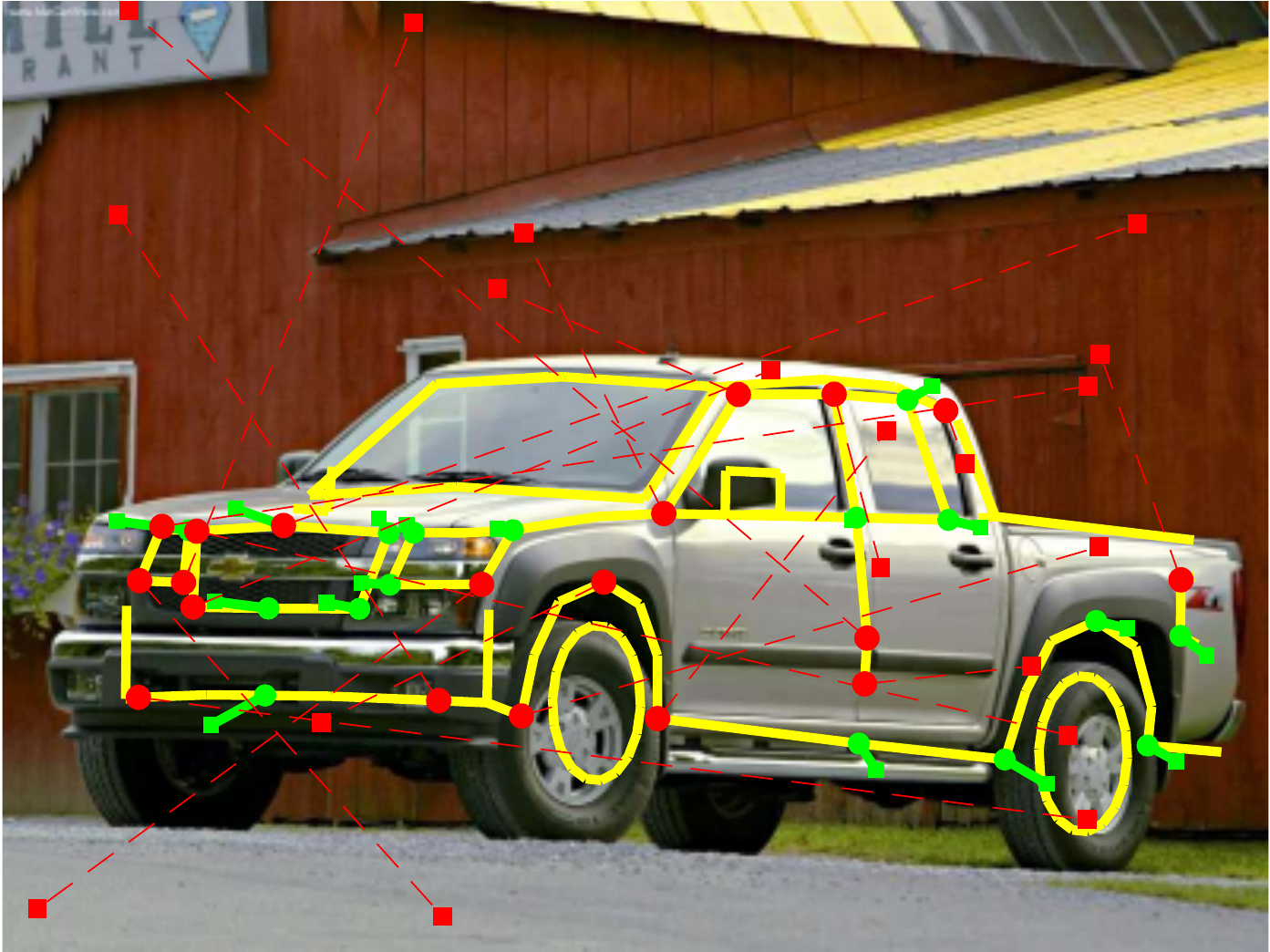} \\
	\vspace{1mm}
	\end{minipage}
&  \myhspace
	\begin{minipage}{\mpwthree}%
	\centering%
	\includegraphics[width=\columnwidth]{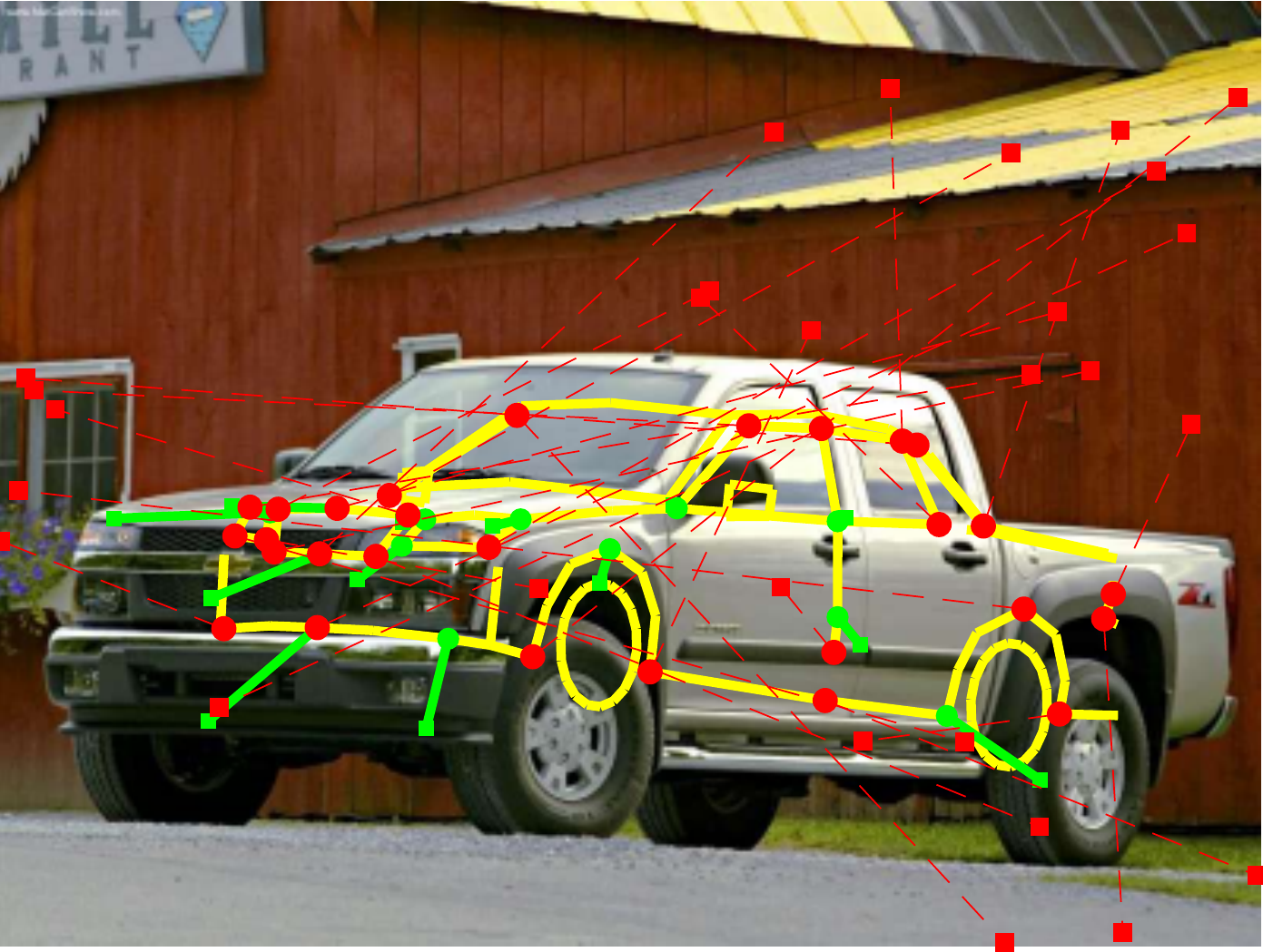} \\
	\vspace{1mm}
	\end{minipage} 
&  \myhspace
	\begin{minipage}{\mpwthree}%
	\centering%
	\includegraphics[width=\columnwidth]{056_70_convex+robust+refine_Chevrolet_Colorado-LS_2004_06.pdf} \\
	\vspace{1mm}
	\end{minipage} \\
\myhspace \myhspace \hspace{-2mm} \rotatebox{90}{\hspace{-3mm} {\smaller \blue{ \namerobust}} } & 
\myhspace
	\begin{minipage}{\mpwthree}%
	\centering%
	\includegraphics[width=\columnwidth]{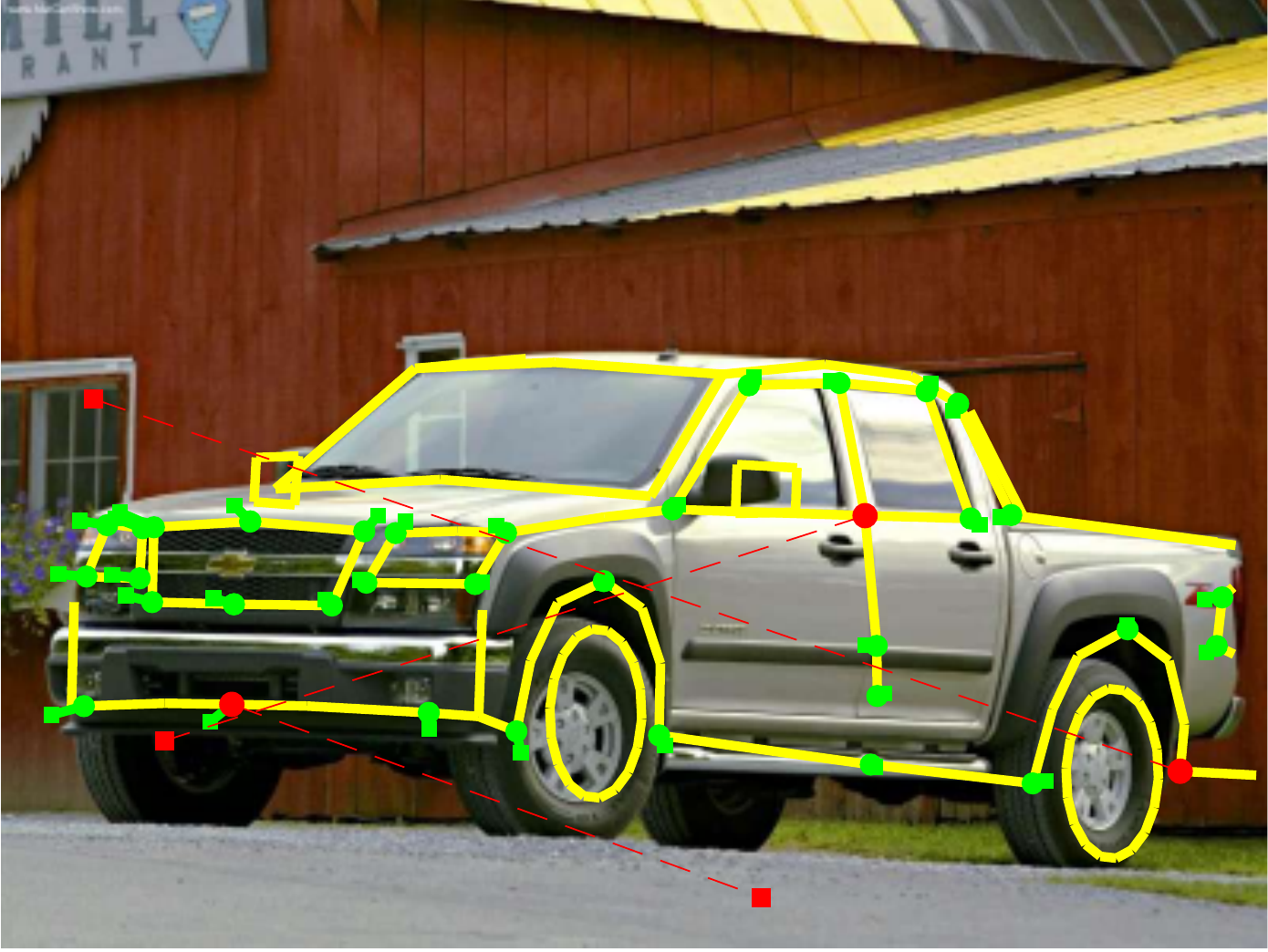} \\
	\vspace{1mm}
	\end{minipage}
& \myhspace
	\begin{minipage}{\mpwthree}%
	\centering%
	\includegraphics[width=\columnwidth]{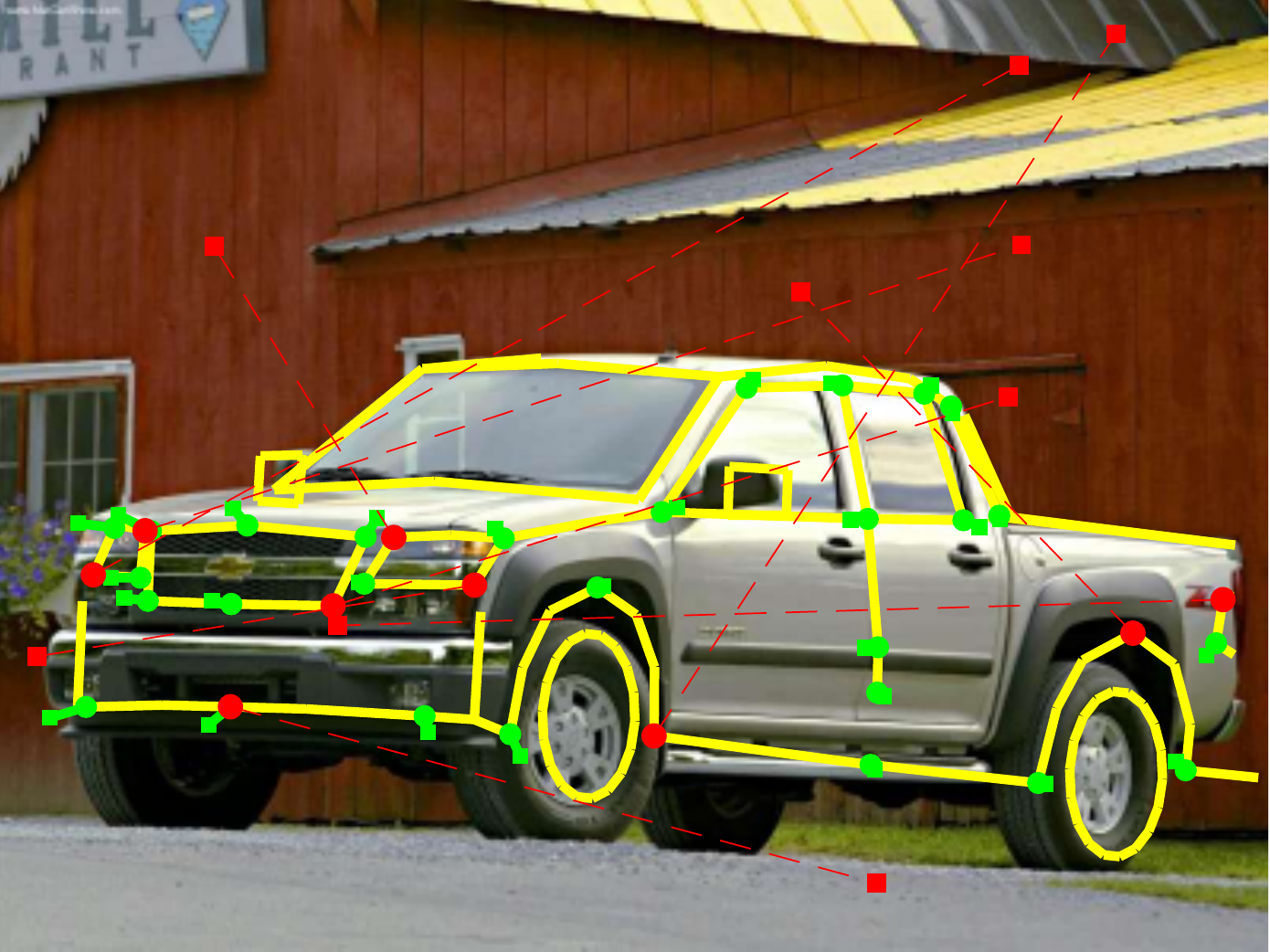} \\
	\vspace{1mm}
	\end{minipage}
& \myhspace
	\begin{minipage}{\mpwthree}%
	\centering%
	\includegraphics[width=\columnwidth]{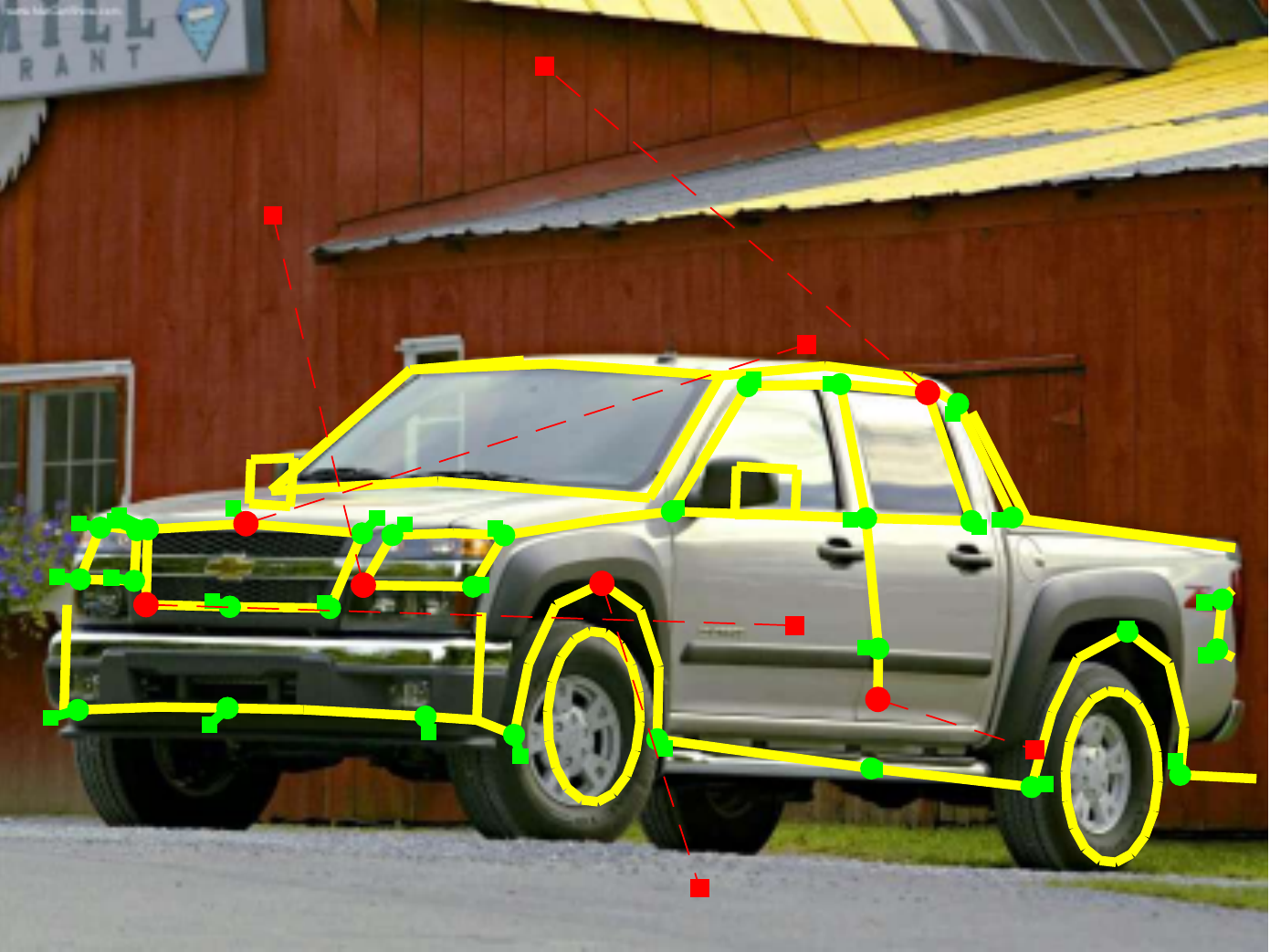} \\
	\vspace{1mm}
	\end{minipage} 
& \myhspace
	\begin{minipage}{\mpwthree}%
	\centering%
	\includegraphics[width=\columnwidth]{056_40_GNC-TLS_Chevrolet_Colorado-LS_2004_06.pdf} \\
	\vspace{1mm}
	\end{minipage}
& \myhspace
	\begin{minipage}{\mpwthree}%
	\centering%
	\includegraphics[width=\columnwidth]{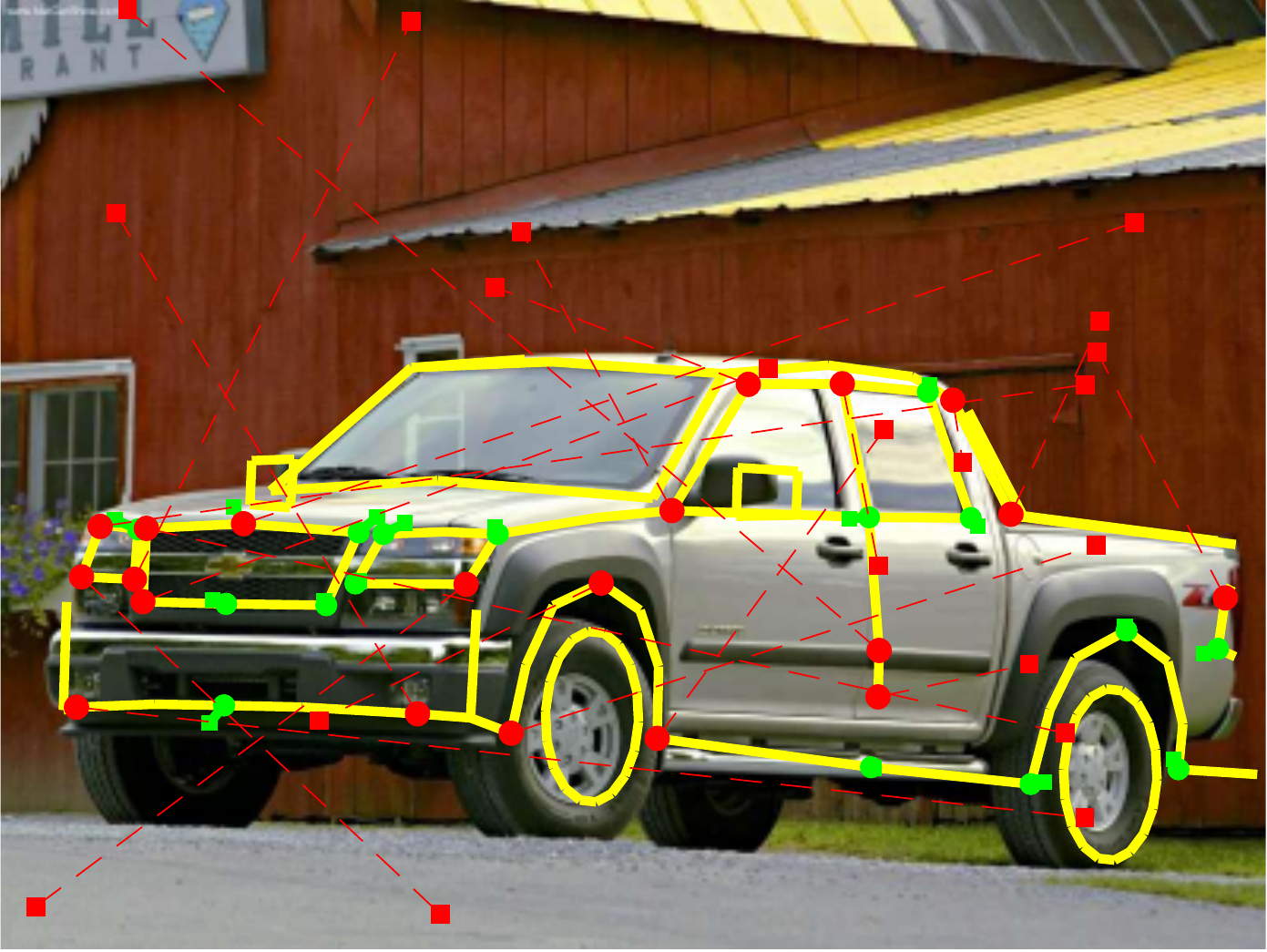} \\
	\vspace{1mm}
	\end{minipage}
&  \myhspace
	\begin{minipage}{\mpwthree}%
	\centering%
	\includegraphics[width=\columnwidth]{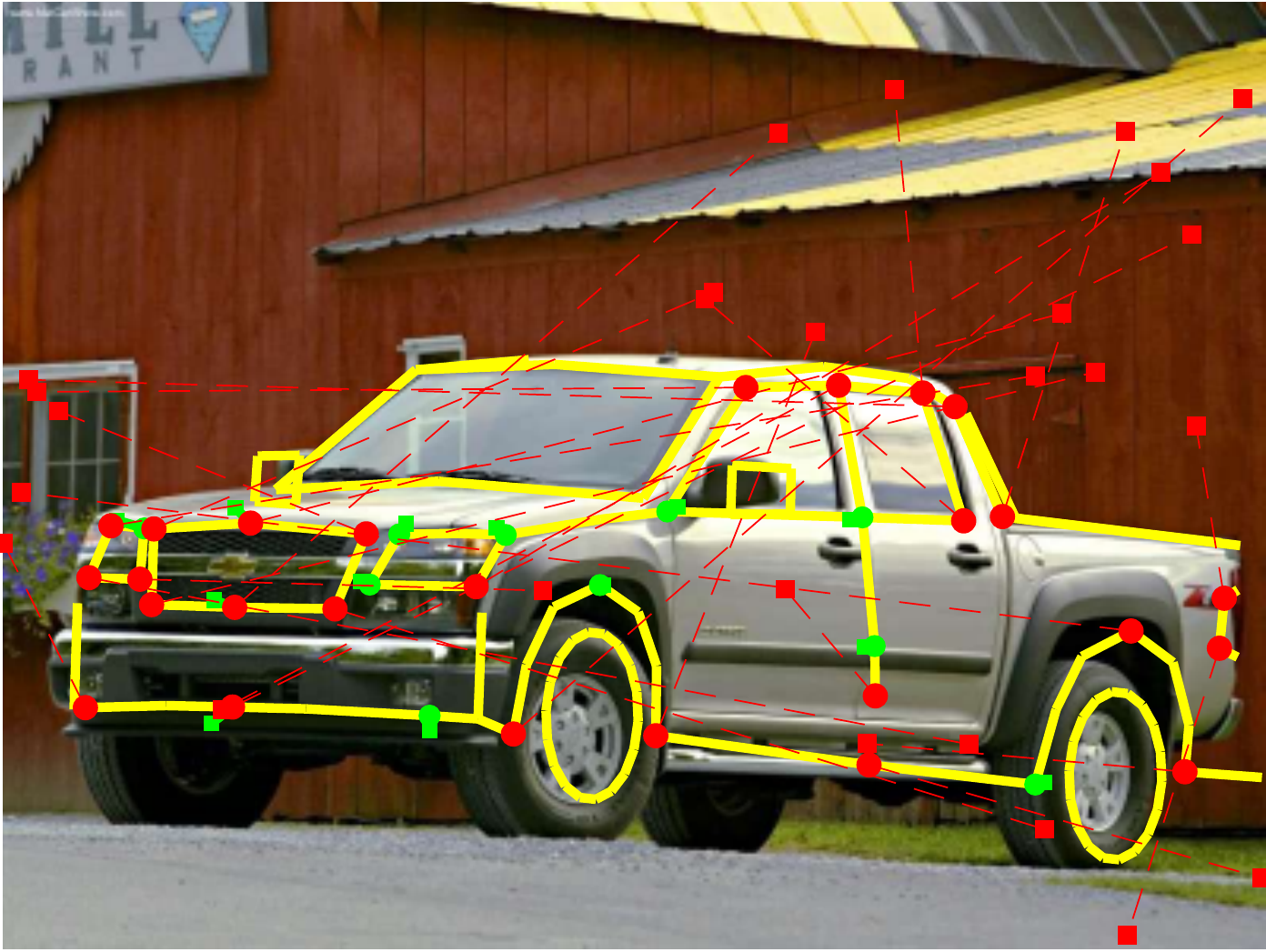} \\
	\vspace{1mm}
	\end{minipage} 
&  \myhspace
	\begin{minipage}{\mpwthree}%
	\centering%
	\includegraphics[width=\columnwidth]{056_70_GNC-TLS_Chevrolet_Colorado-LS_2004_06.pdf} \\
	\vspace{1mm}
	\end{minipage} \\
\multicolumn{8}{c}{Chevrolet Colorado LS } \\

%% file: 035-BMW_5_Series.tex
\myhspace \myhspace \hspace{-2mm} \rotatebox{90}{\hspace{-7mm} {\smaller \alternrobust} } & 
\myhspace
	\begin{minipage}{\mpwthree}%
	\centering%
	\includegraphics[width=\columnwidth]{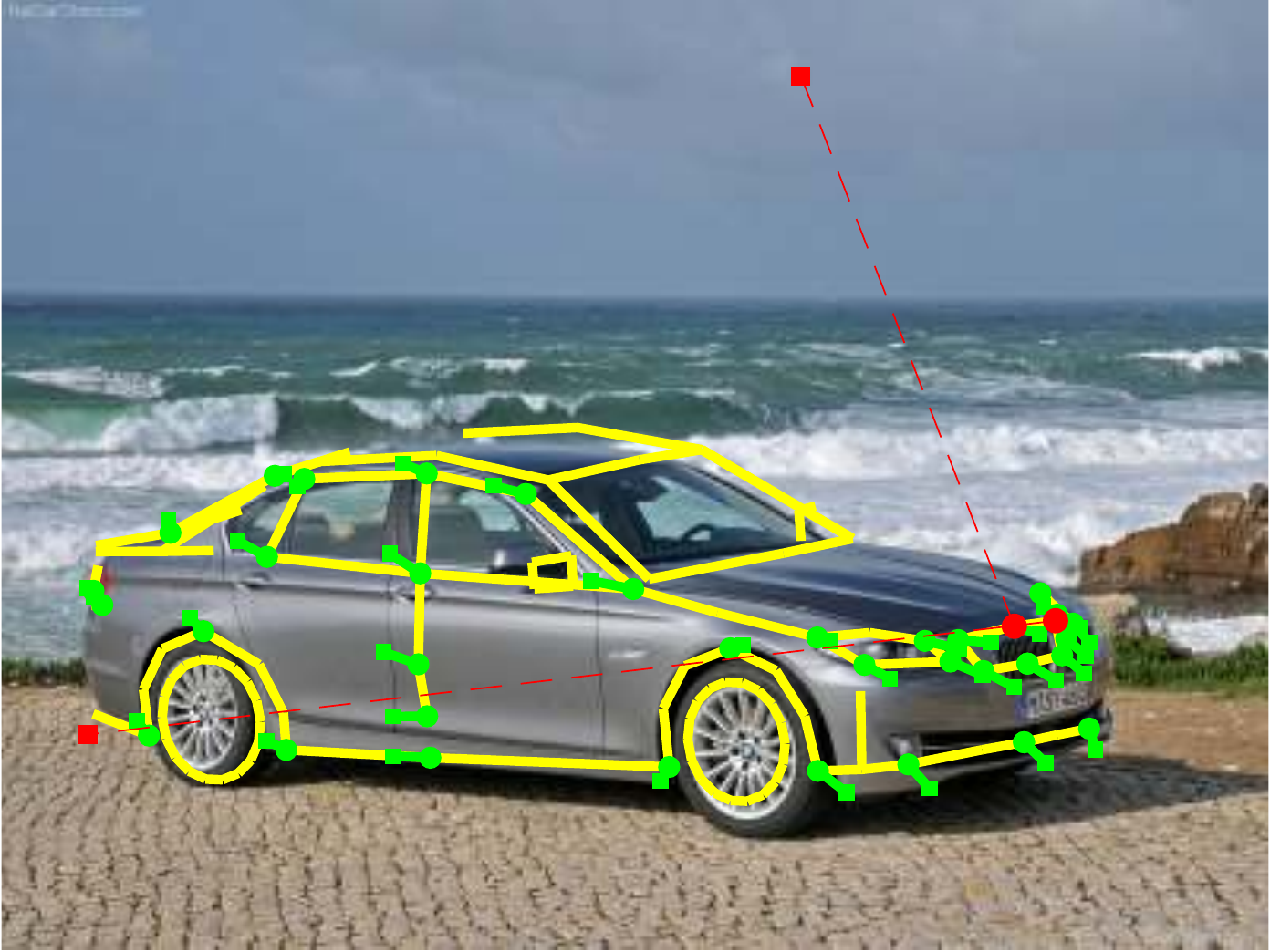} \\
	\vspace{1mm}
	\end{minipage}
& \myhspace
	\begin{minipage}{\mpwthree}%
	\centering%
	\includegraphics[width=\columnwidth]{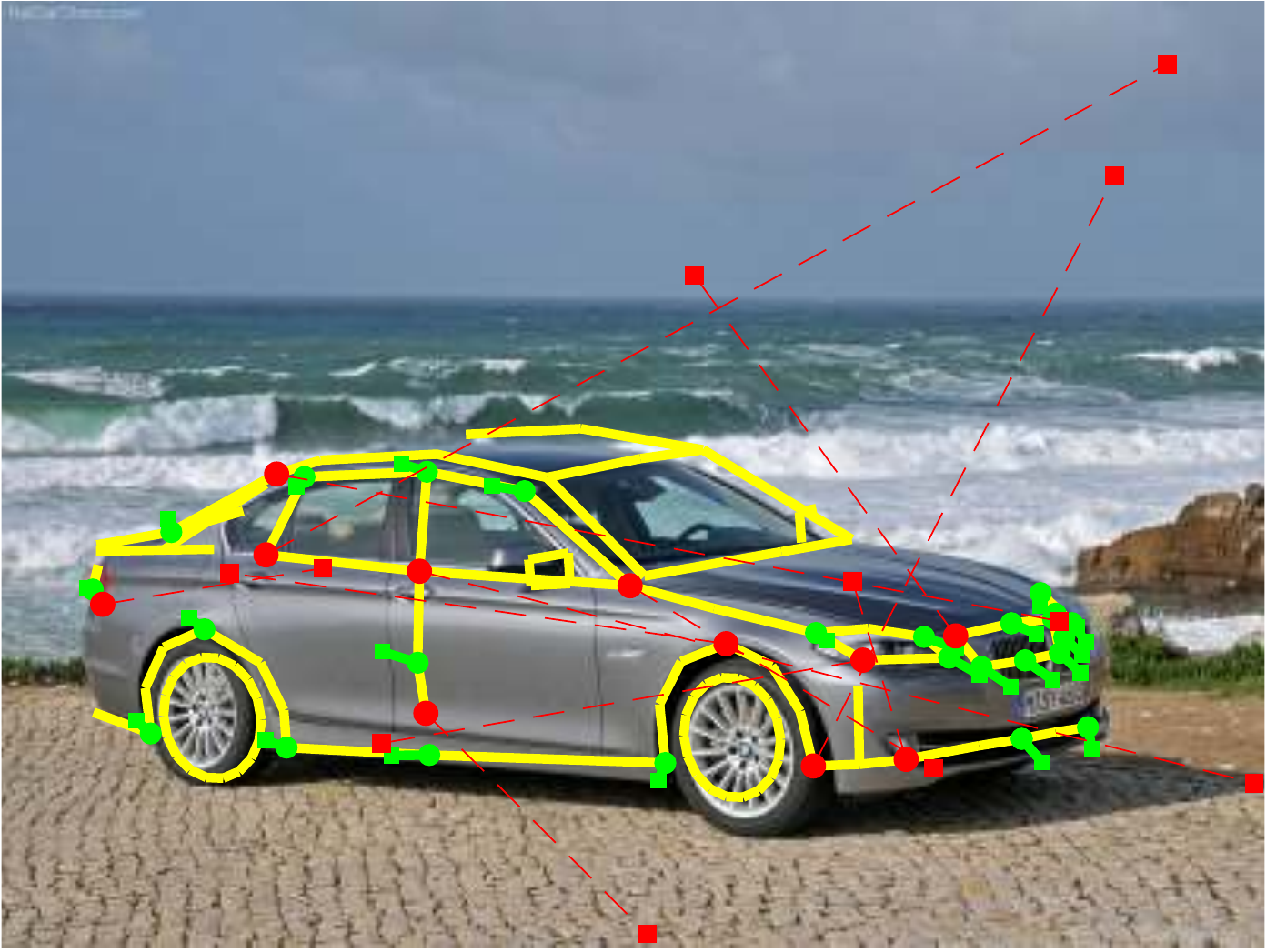} \\
	\vspace{1mm}
	\end{minipage}
& \myhspace
	\begin{minipage}{\mpwthree}%
	\centering%
	\includegraphics[width=\columnwidth]{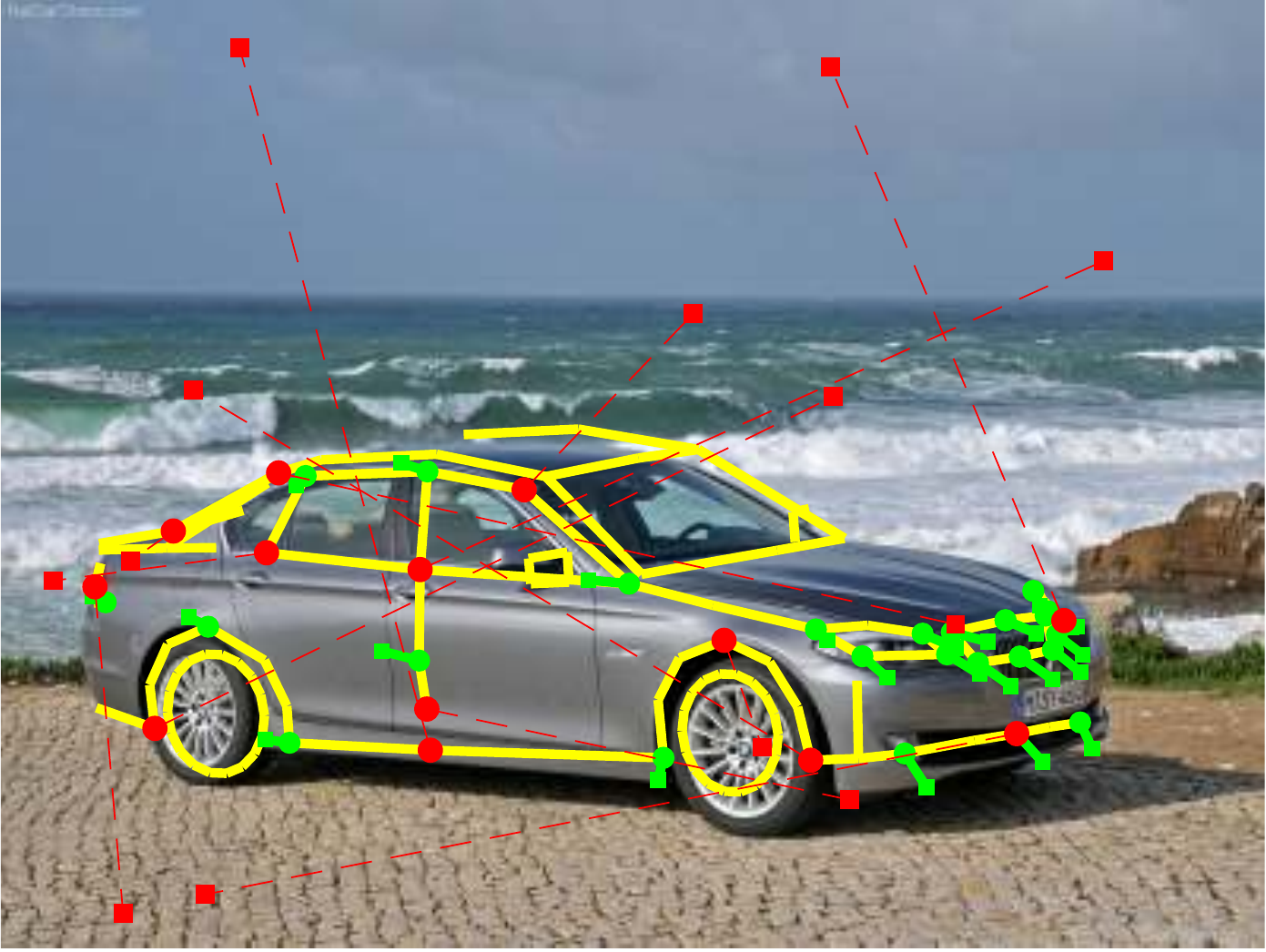} \\
	\vspace{1mm}
	\end{minipage} 
& \myhspace
	\begin{minipage}{\mpwthree}%
	\centering%
	\includegraphics[width=\columnwidth]{035_40_altern+robust_BMW_5-Series_2011_05.pdf} \\
	\vspace{1mm}
	\end{minipage}
& \myhspace
	\begin{minipage}{\mpwthree}%
	\centering%
	\includegraphics[width=\columnwidth]{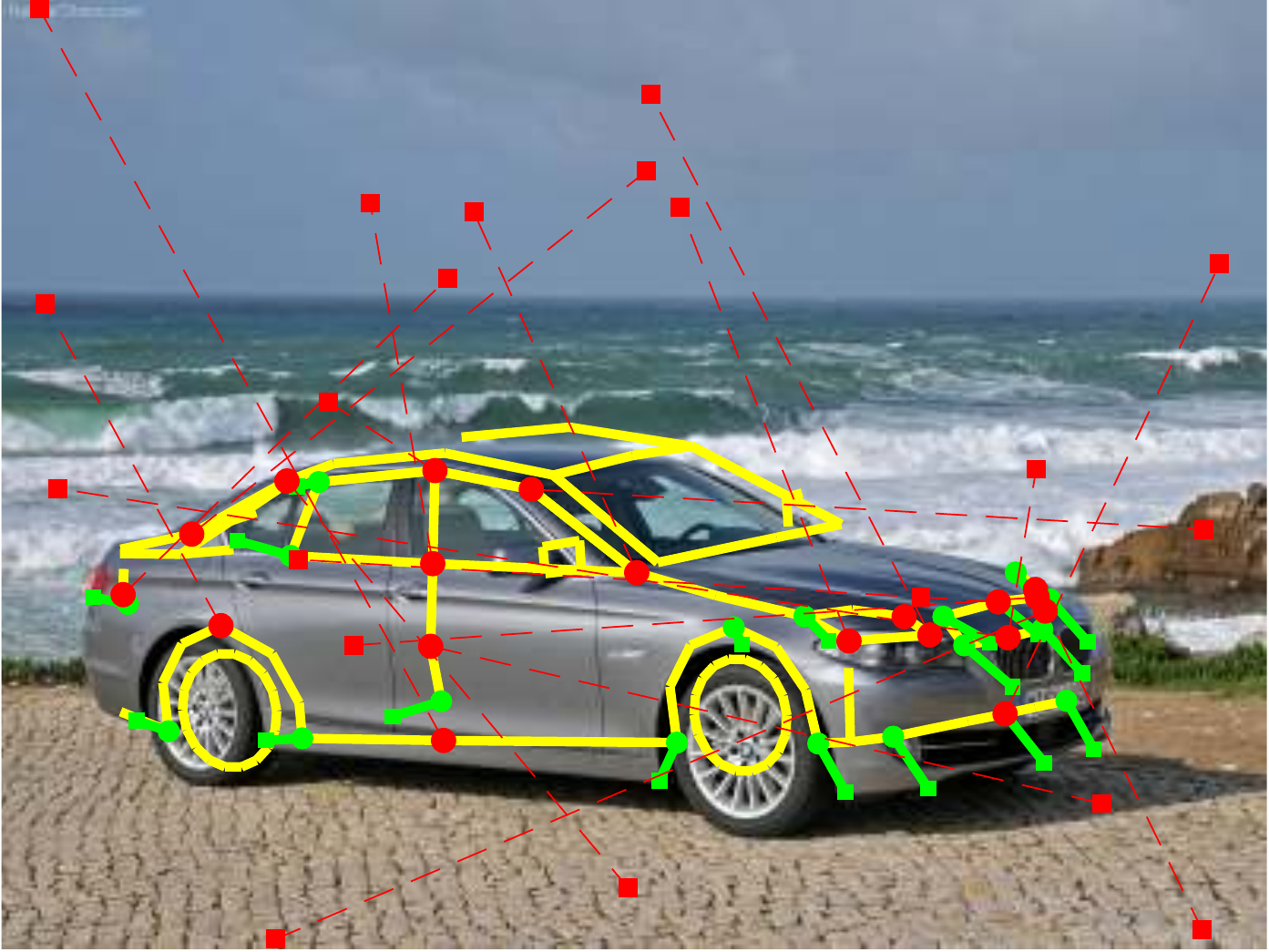} \\
	\vspace{1mm}
	\end{minipage}
&  \myhspace
	\begin{minipage}{\mpwthree}%
	\centering%
	\includegraphics[width=\columnwidth]{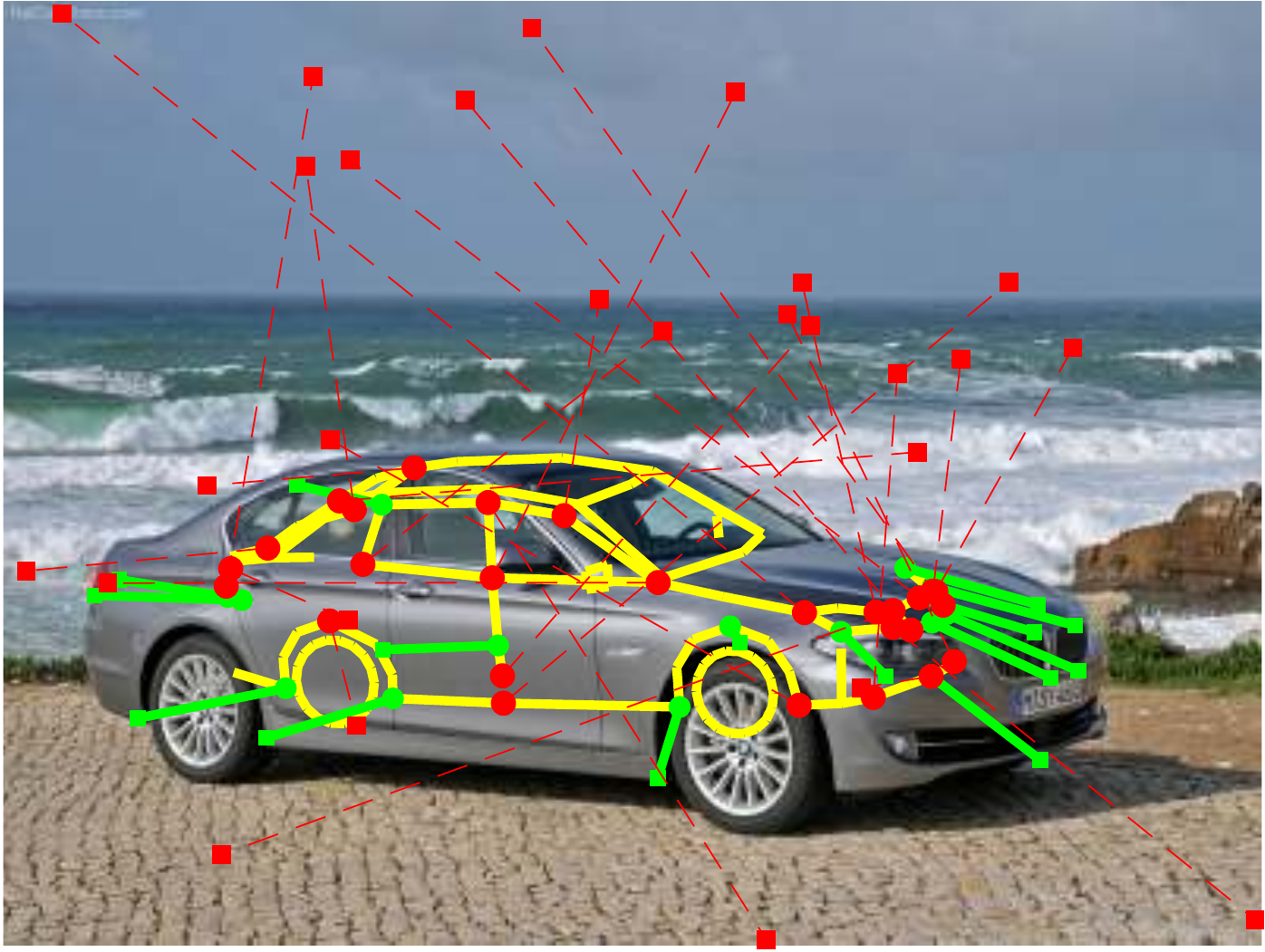} \\
	\vspace{1mm}
	\end{minipage} 
&  \myhspace
	\begin{minipage}{\mpwthree}%
	\centering%
	\includegraphics[width=\columnwidth]{035_70_altern+robust_BMW_5-Series_2011_05.pdf} \\
	\vspace{1mm}
	\end{minipage} \\
\myhspace \myhspace \hspace{-2mm} \rotatebox{90}{\hspace{-8mm} {\smaller \convexrobust} } & 
\myhspace
	\begin{minipage}{\mpwthree}%
	\centering%
	\includegraphics[width=\columnwidth]{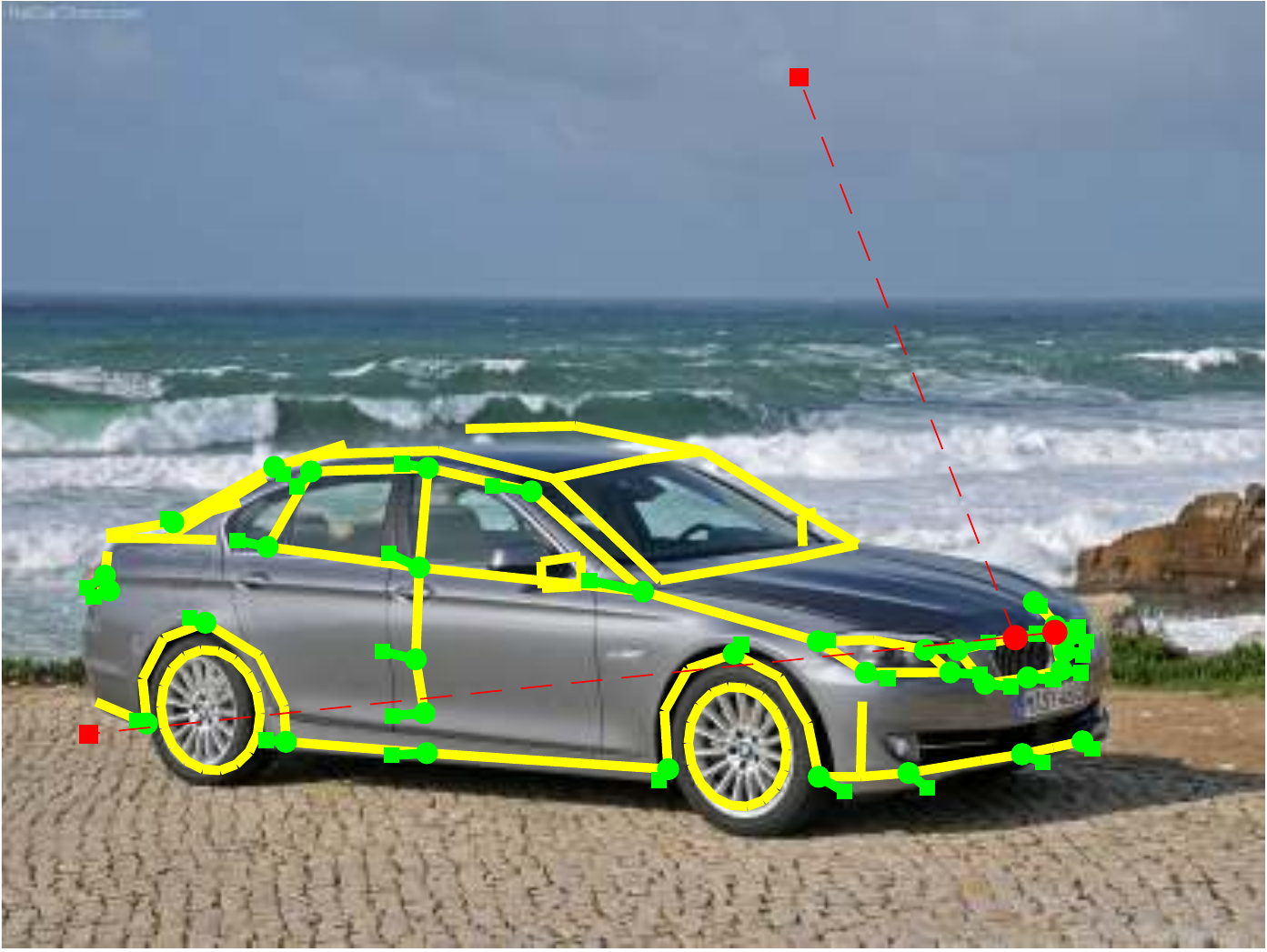} \\
	\vspace{1mm}
	\end{minipage}
& \myhspace
	\begin{minipage}{\mpwthree}%
	\centering%
	\includegraphics[width=\columnwidth]{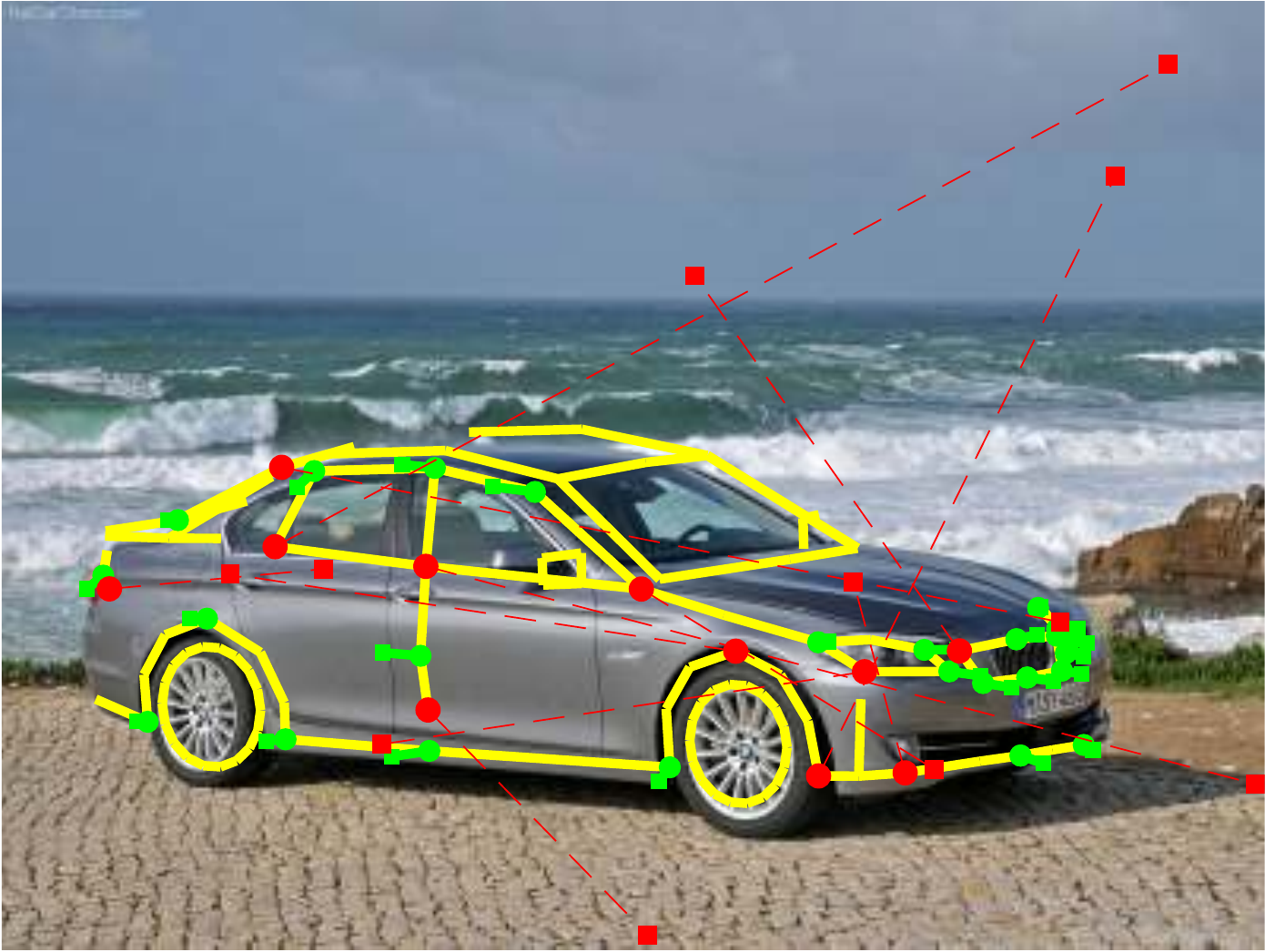} \\
	\vspace{1mm}
	\end{minipage}
& \myhspace
	\begin{minipage}{\mpwthree}%
	\centering%
	\includegraphics[width=\columnwidth]{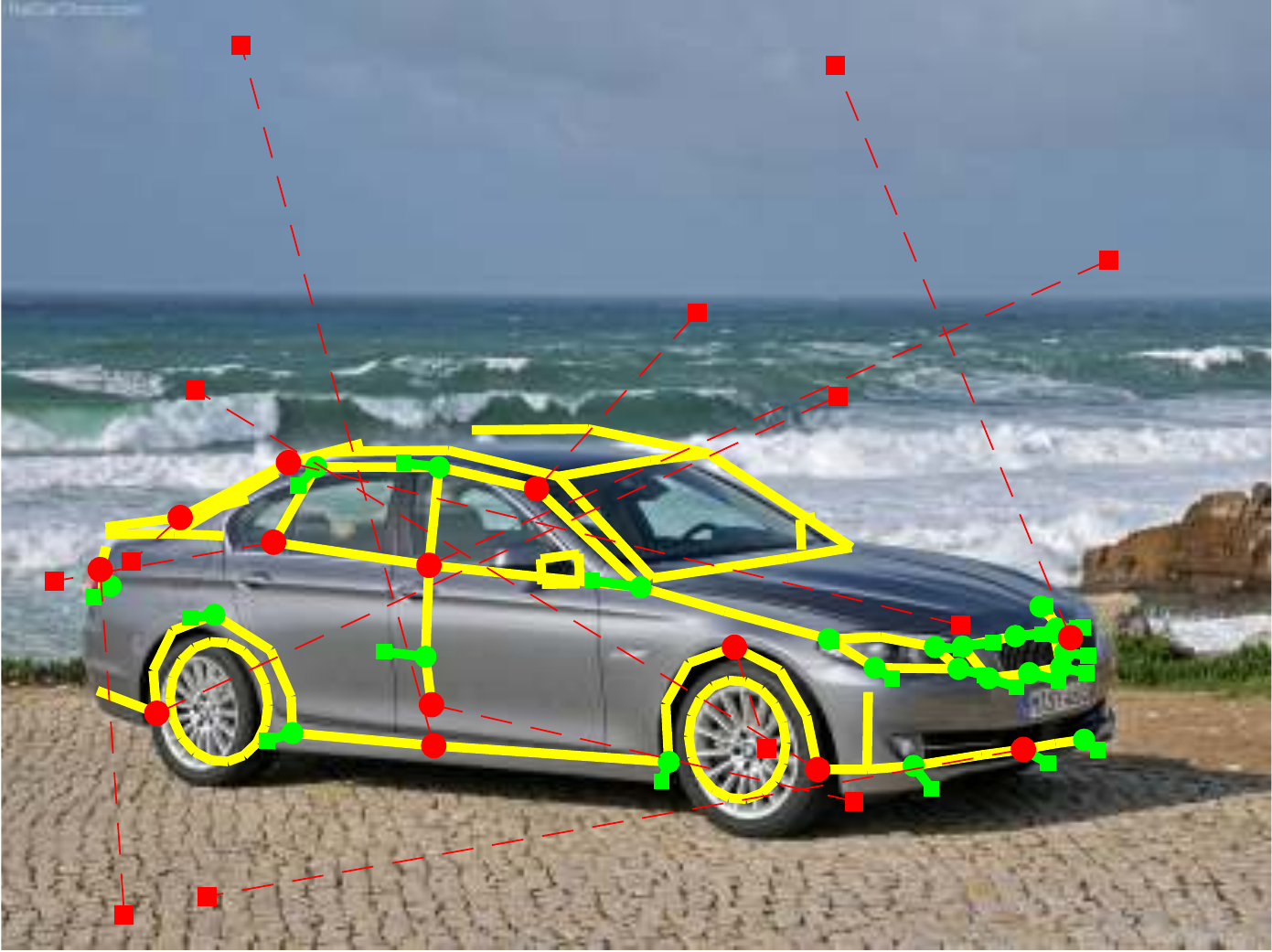} \\
	\vspace{1mm}
	\end{minipage} 
& \myhspace
	\begin{minipage}{\mpwthree}%
	\centering%
	\includegraphics[width=\columnwidth]{035_40_convex+robust+refine_BMW_5-Series_2011_05.pdf} \\
	\vspace{1mm}
	\end{minipage}
& \myhspace
	\begin{minipage}{\mpwthree}%
	\centering%
	\includegraphics[width=\columnwidth]{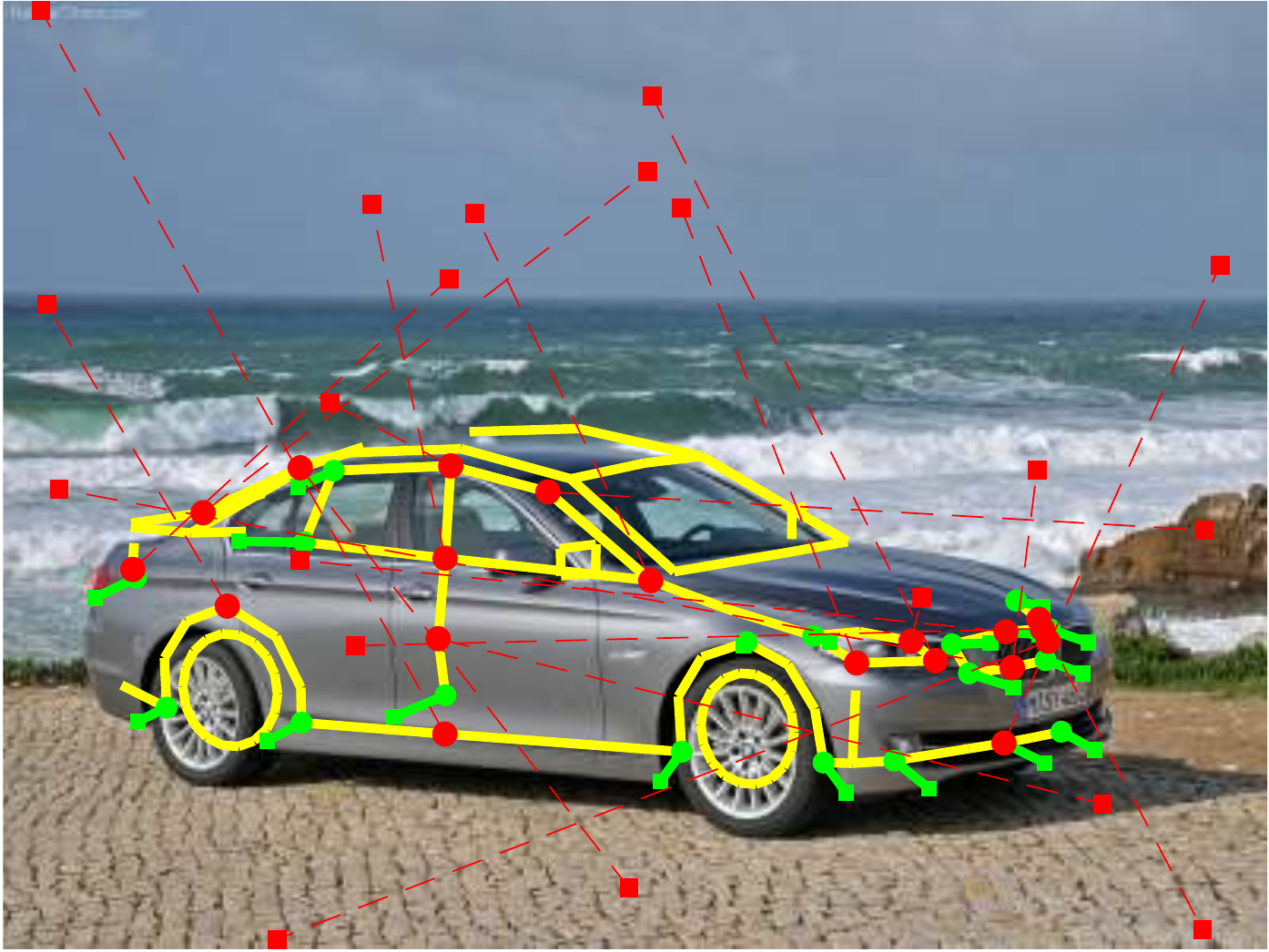} \\
	\vspace{1mm}
	\end{minipage}
&  \myhspace
	\begin{minipage}{\mpwthree}%
	\centering%
	\includegraphics[width=\columnwidth]{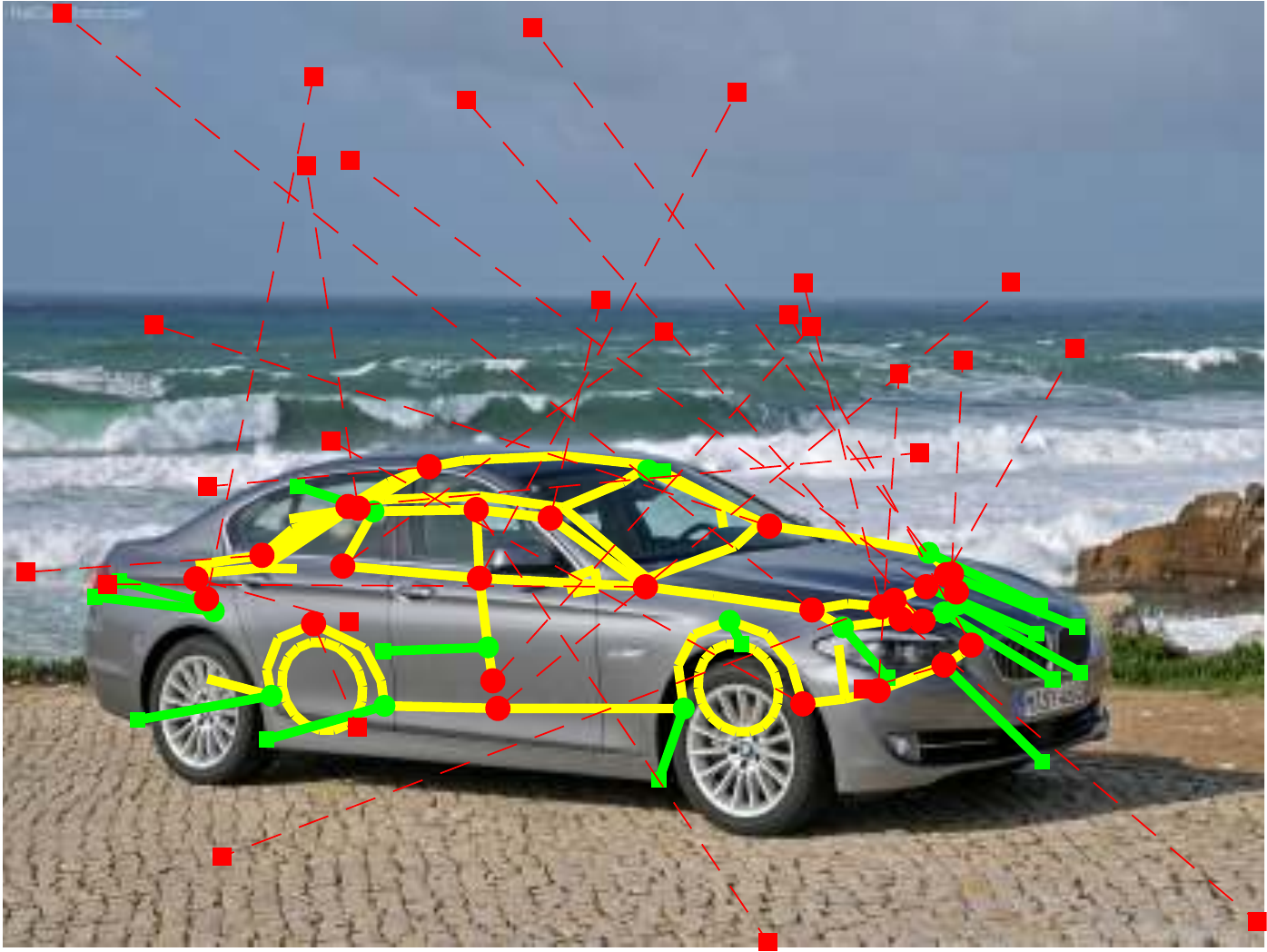} \\
	\vspace{1mm}
	\end{minipage} 
&  \myhspace
	\begin{minipage}{\mpwthree}%
	\centering%
	\includegraphics[width=\columnwidth]{035_70_convex+robust+refine_BMW_5-Series_2011_05.pdf} \\
	\vspace{1mm}
	\end{minipage} \\
\myhspace \myhspace \hspace{-2mm} \rotatebox{90}{\hspace{-3mm} {\smaller \blue{ \namerobust}} } & 
\myhspace
	\begin{minipage}{\mpwthree}%
	\centering%
	\includegraphics[width=\columnwidth]{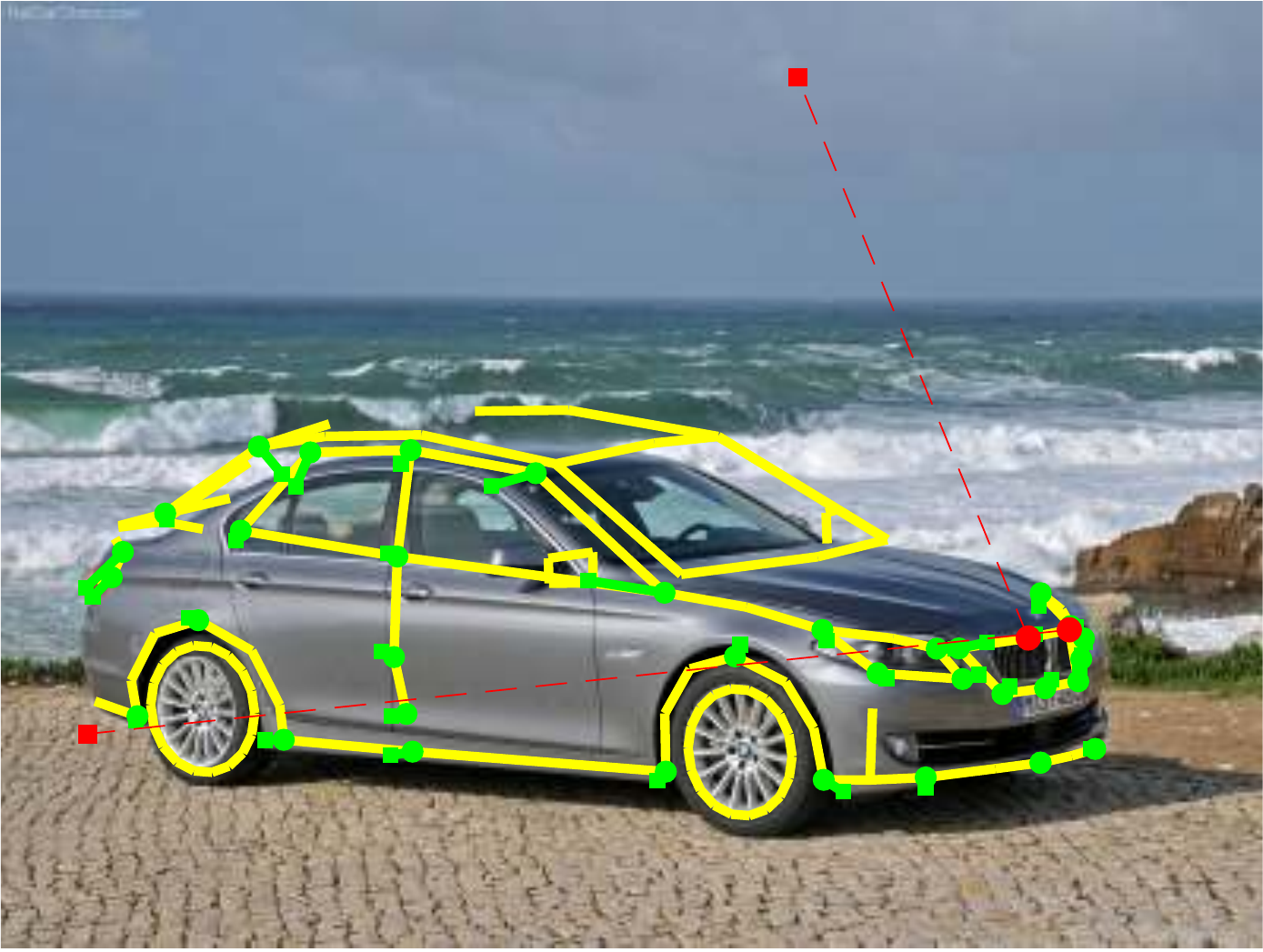} \\
	\vspace{1mm}
	\end{minipage}
& \myhspace
	\begin{minipage}{\mpwthree}%
	\centering%
	\includegraphics[width=\columnwidth]{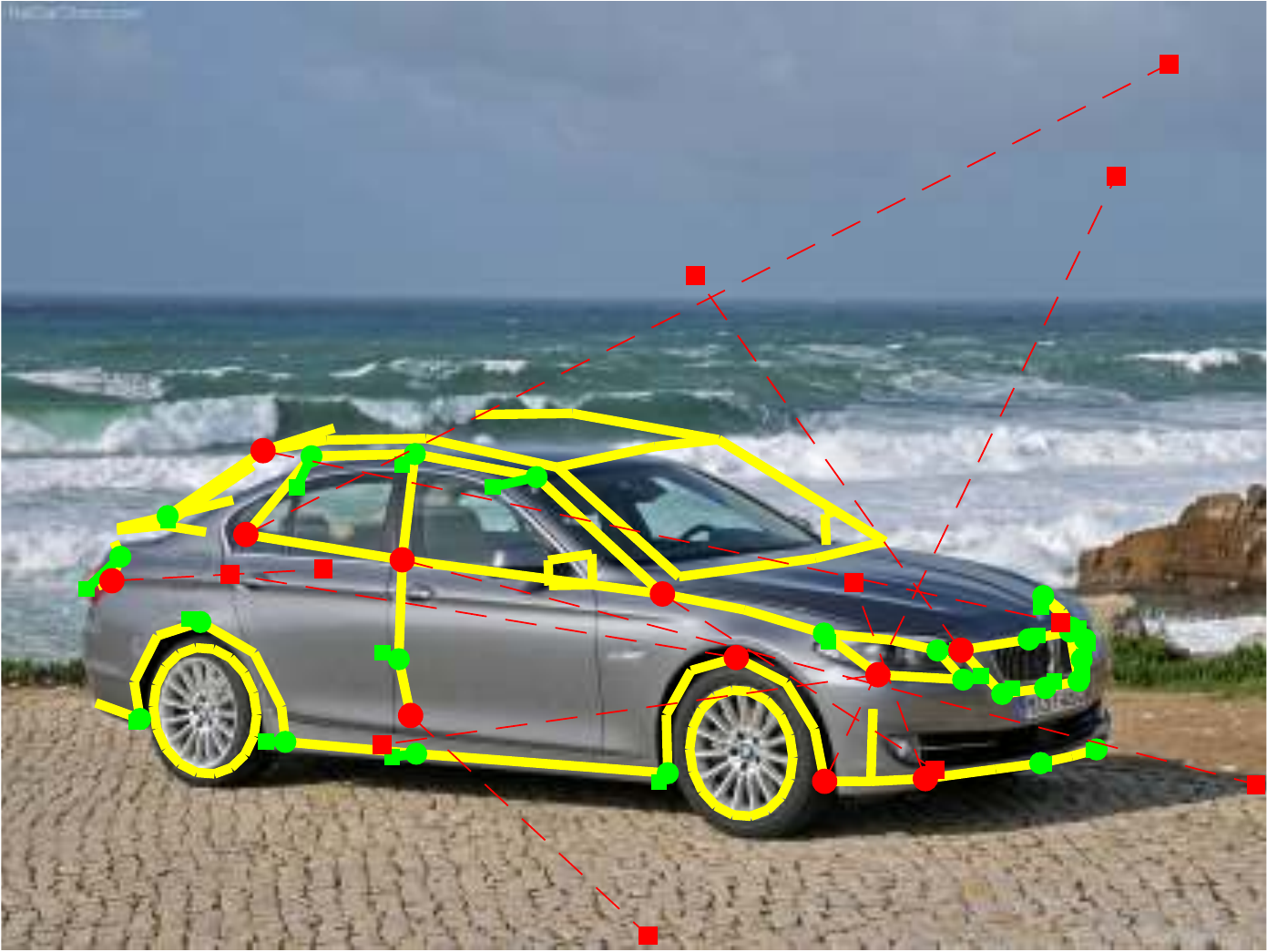} \\
	\vspace{1mm}
	\end{minipage}
& \myhspace
	\begin{minipage}{\mpwthree}%
	\centering%
	\includegraphics[width=\columnwidth]{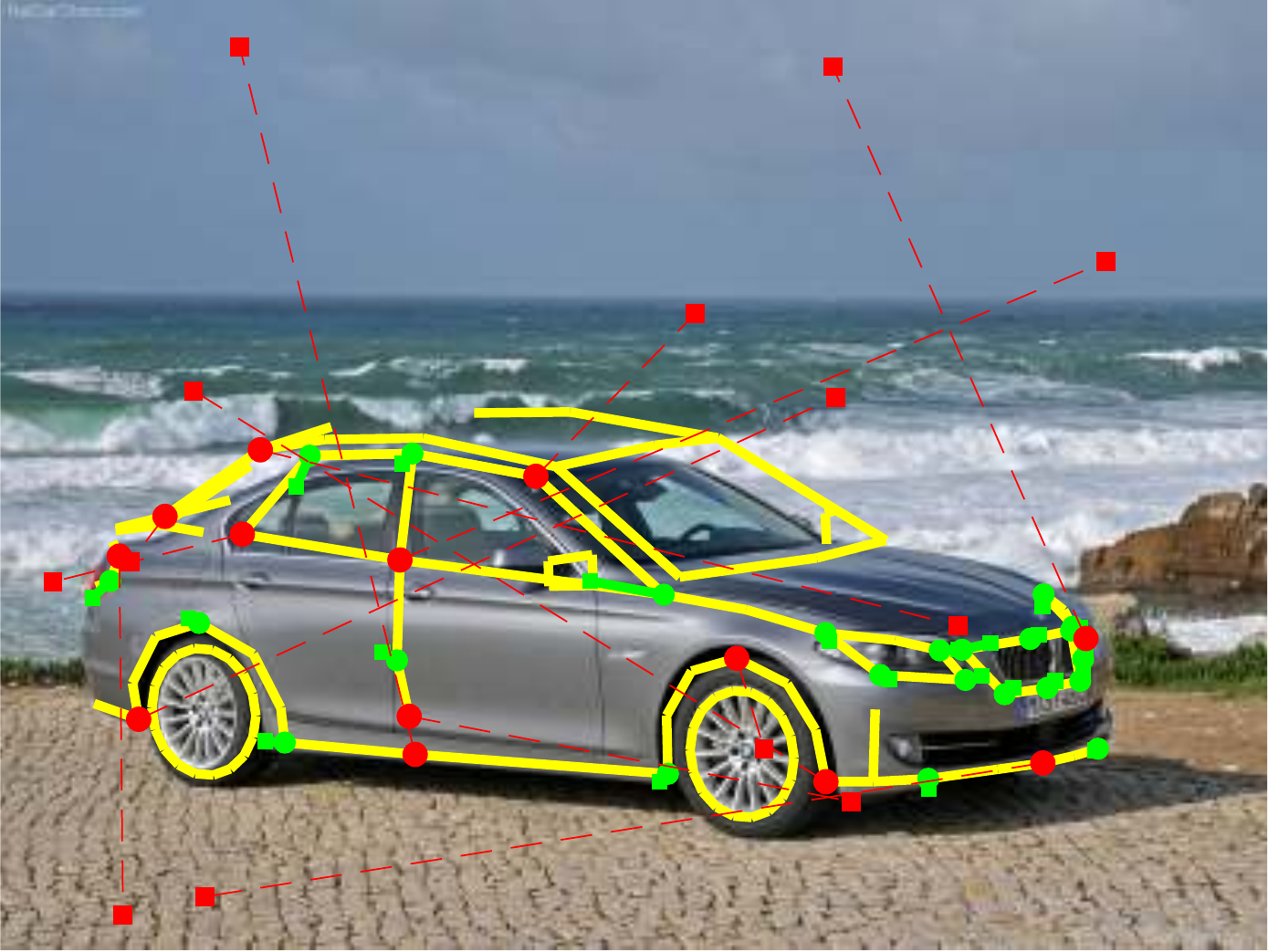} \\
	\vspace{1mm}
	\end{minipage} 
& \myhspace
	\begin{minipage}{\mpwthree}%
	\centering%
	\includegraphics[width=\columnwidth]{035_40_GNC-TLS_BMW_5-Series_2011_05.pdf} \\
	\vspace{1mm}
	\end{minipage}
& \myhspace
	\begin{minipage}{\mpwthree}%
	\centering%
	\includegraphics[width=\columnwidth]{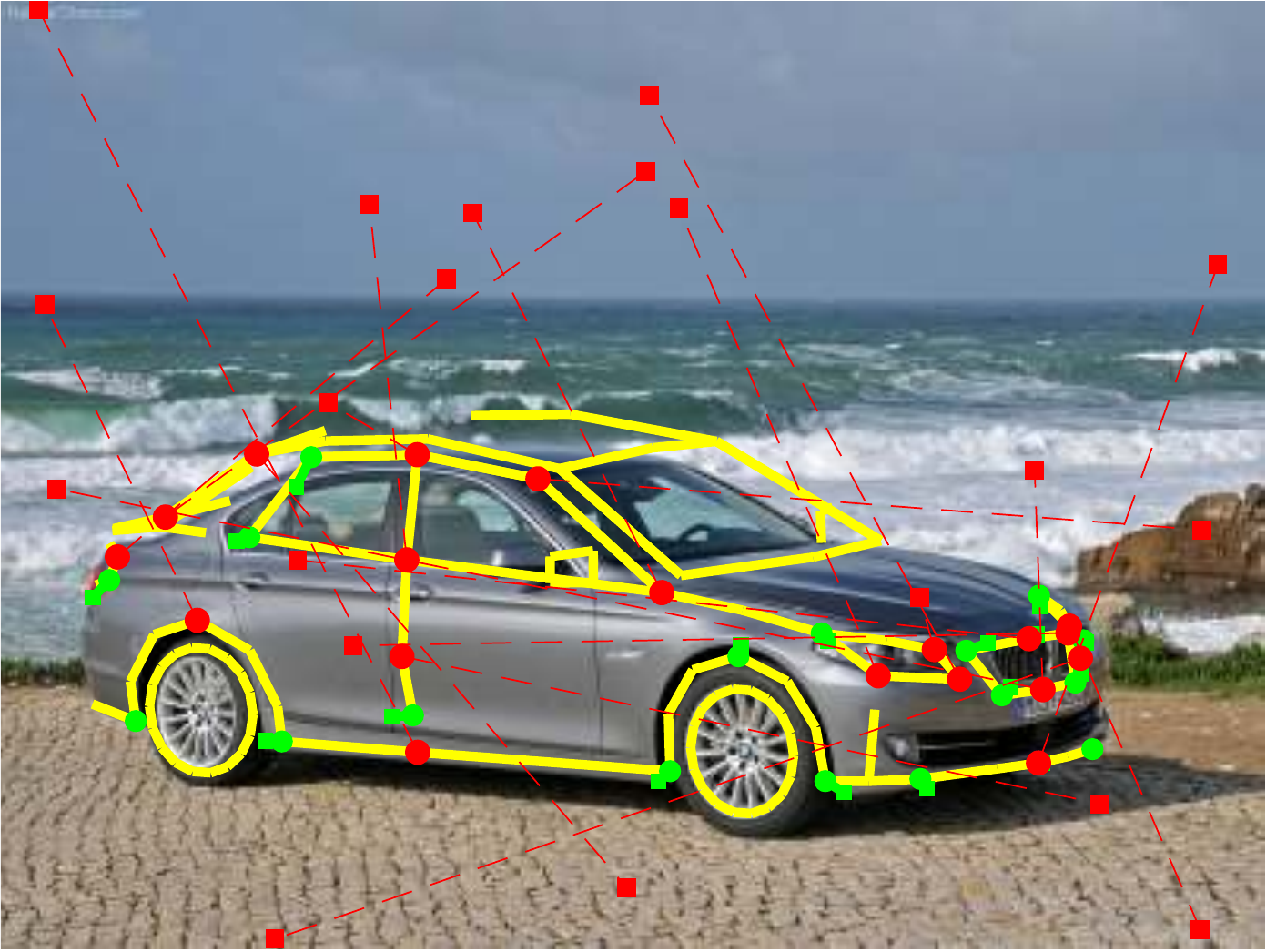} \\
	\vspace{1mm}
	\end{minipage}
&  \myhspace
	\begin{minipage}{\mpwthree}%
	\centering%
	\includegraphics[width=\columnwidth]{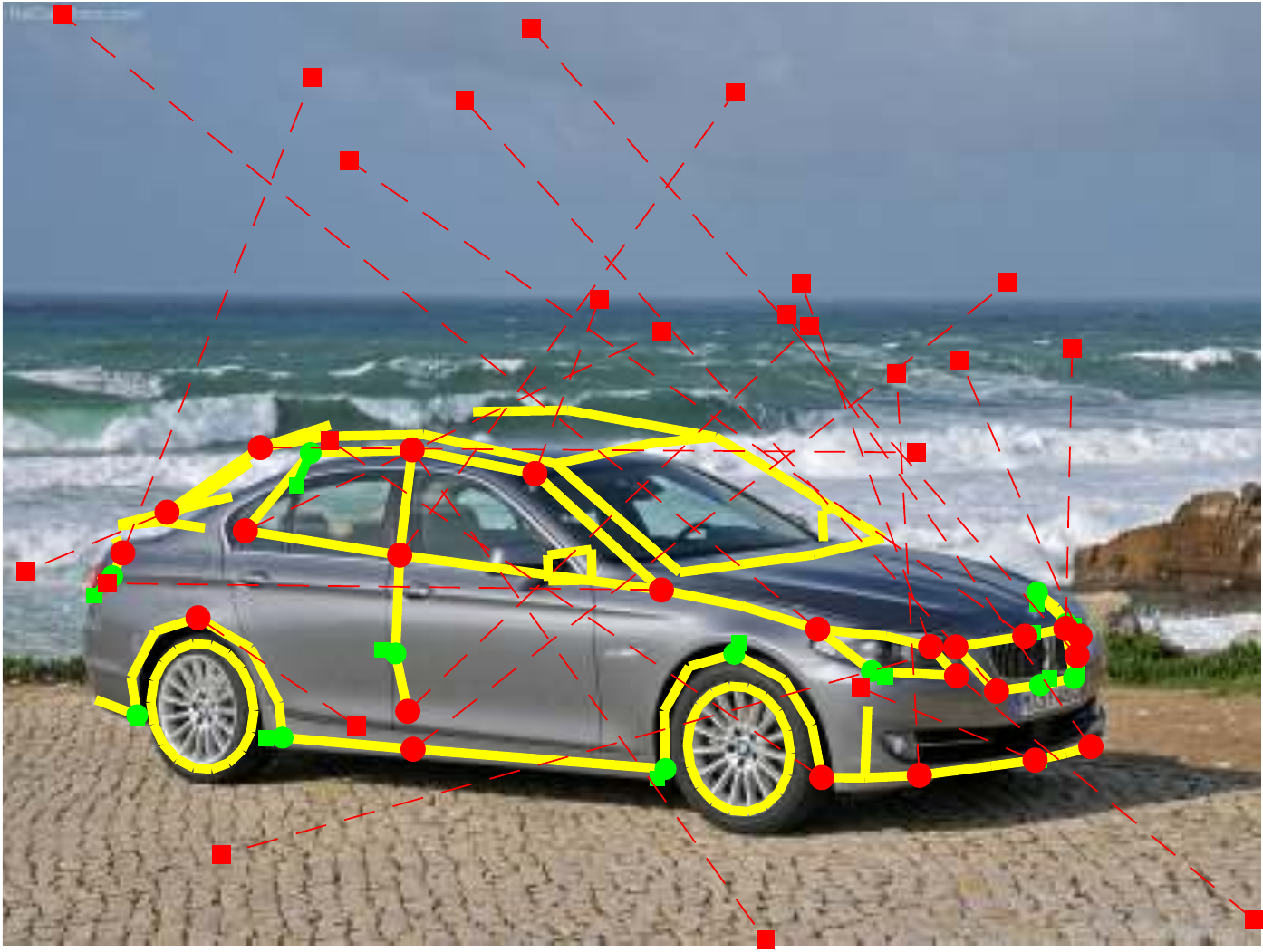} \\
	\vspace{1mm}
	\end{minipage} 
&  \myhspace
	\begin{minipage}{\mpwthree}%
	\centering%
	\includegraphics[width=\columnwidth]{035_70_GNC-TLS_BMW_5-Series_2011_05.pdf} \\
	\vspace{1mm}
	\end{minipage} \\
\multicolumn{8}{c}{BMW 5-Series } \\

%% file: 169-MercedesBenz_C_600.tex
\myhspace \myhspace \hspace{-2mm} \rotatebox{90}{\hspace{-7mm} {\smaller \alternrobust} } & 
\myhspace
	\begin{minipage}{\mpwthree}%
	\centering%
	\includegraphics[width=\columnwidth]{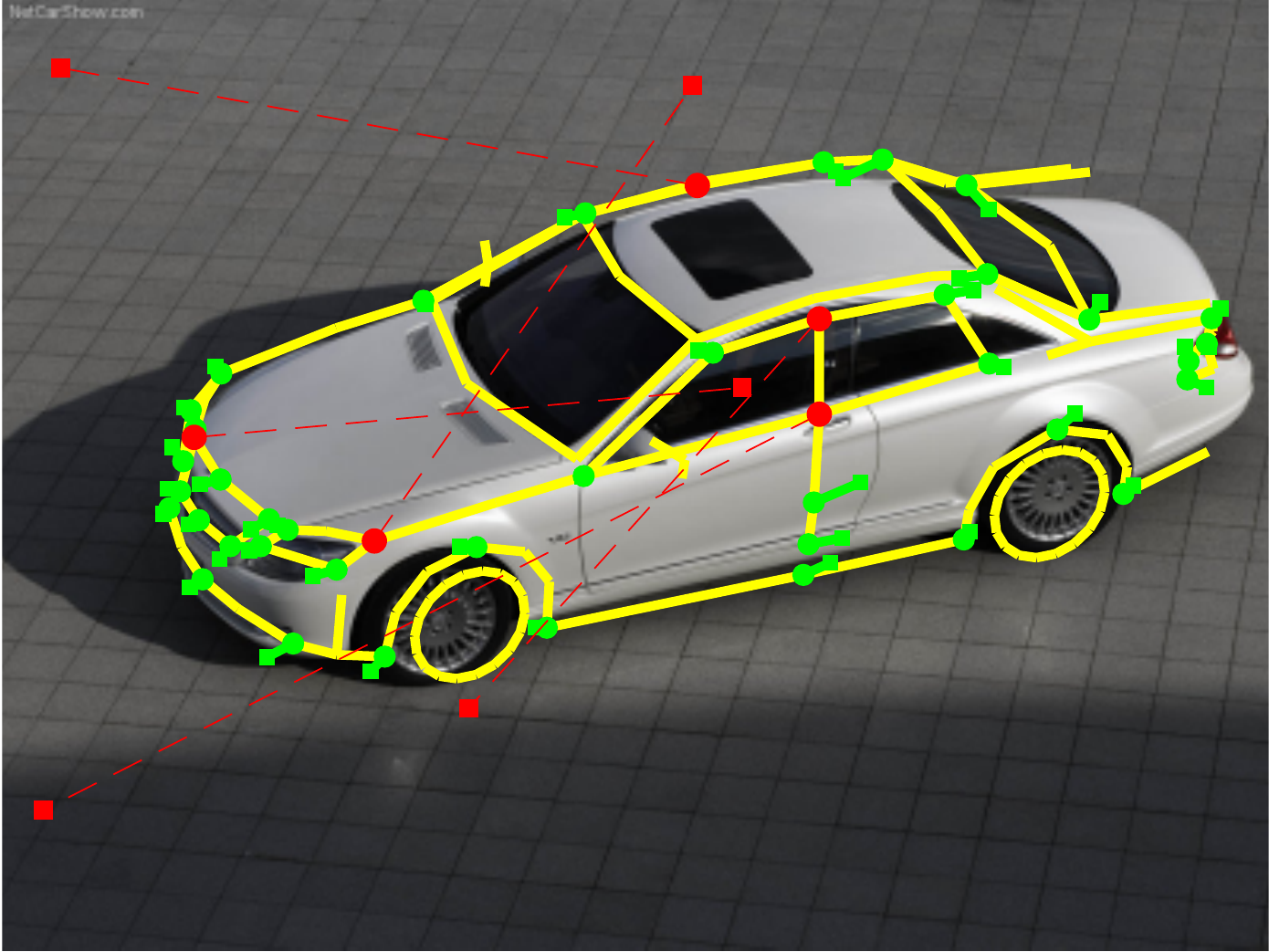} \\
	\vspace{1mm}
	\end{minipage}
& \myhspace
	\begin{minipage}{\mpwthree}%
	\centering%
	\includegraphics[width=\columnwidth]{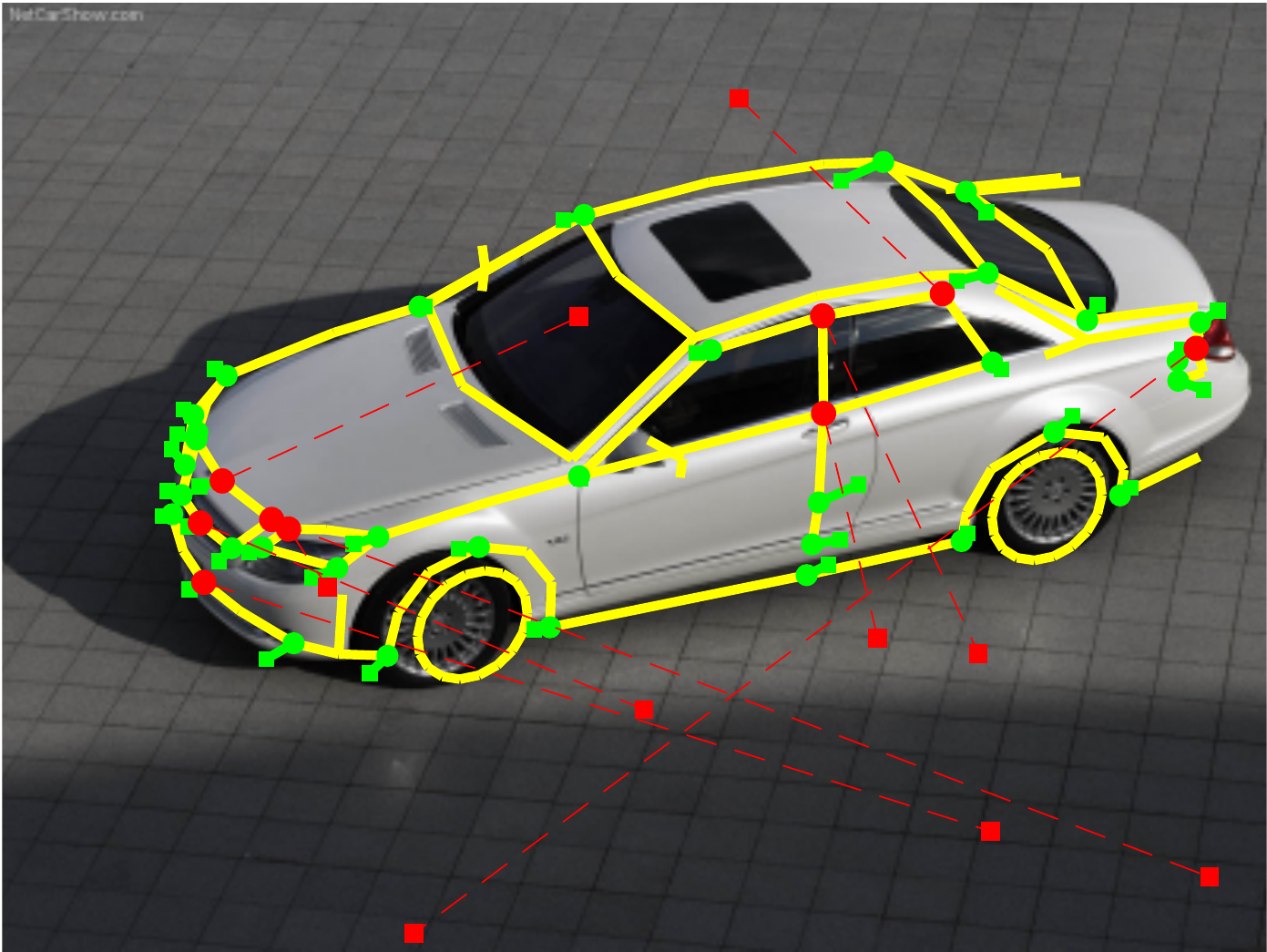} \\
	\vspace{1mm}
	\end{minipage}
& \myhspace
	\begin{minipage}{\mpwthree}%
	\centering%
	\includegraphics[width=\columnwidth]{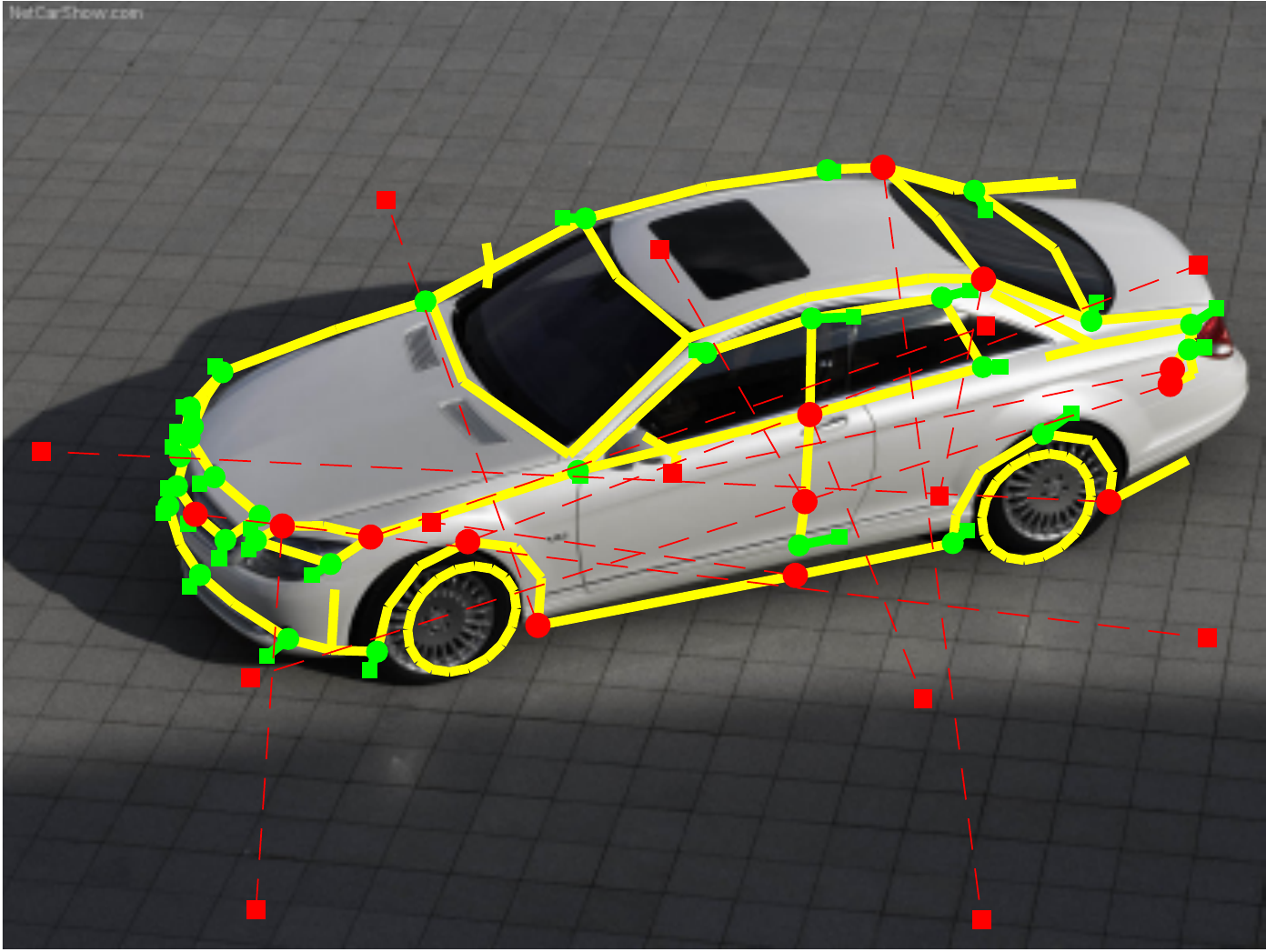} \\
	\vspace{1mm}
	\end{minipage} 
& \myhspace
	\begin{minipage}{\mpwthree}%
	\centering%
	\includegraphics[width=\columnwidth]{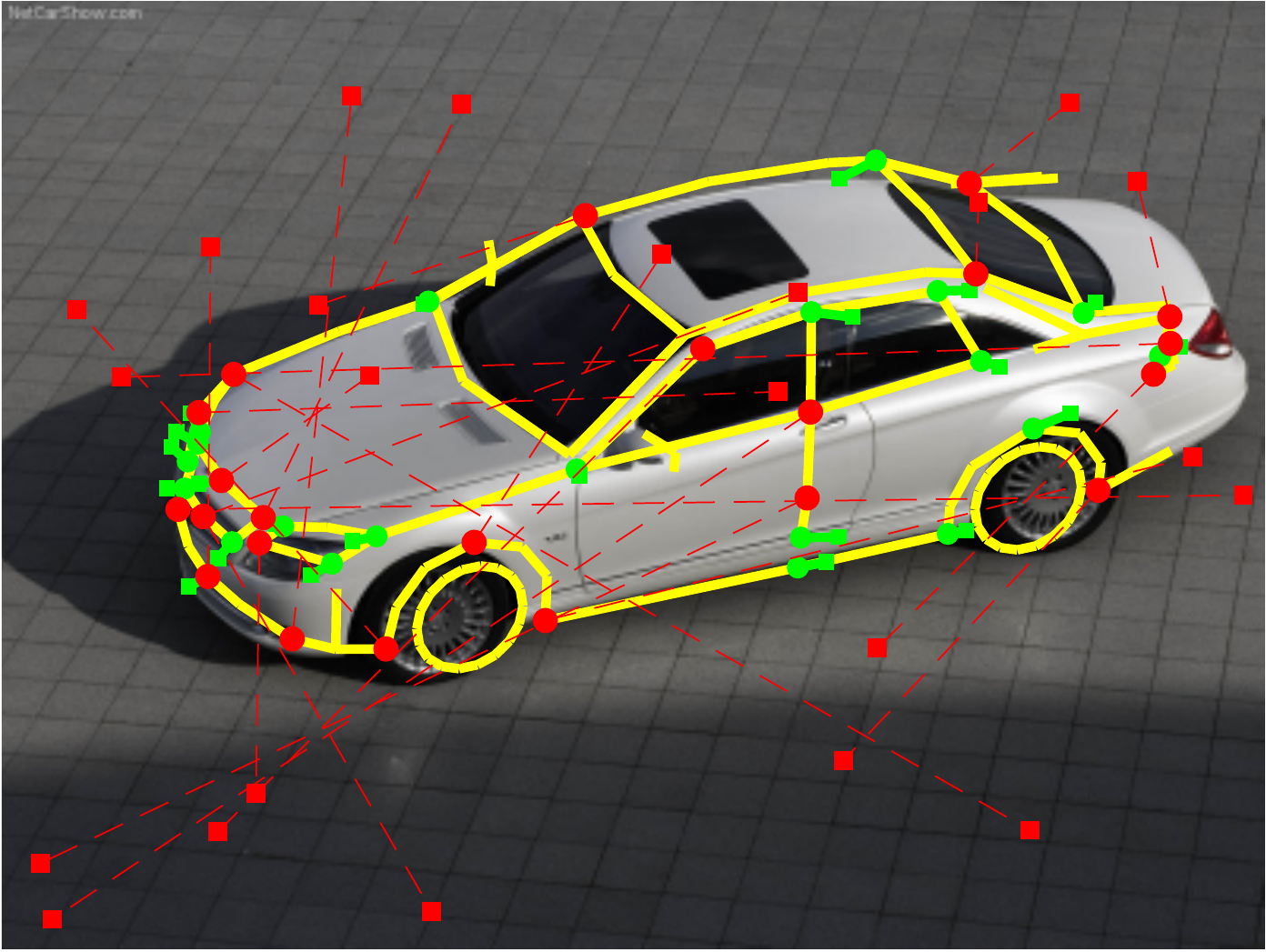} \\
	\vspace{1mm}
	\end{minipage}
& \myhspace
	\begin{minipage}{\mpwthree}%
	\centering%
	\includegraphics[width=\columnwidth]{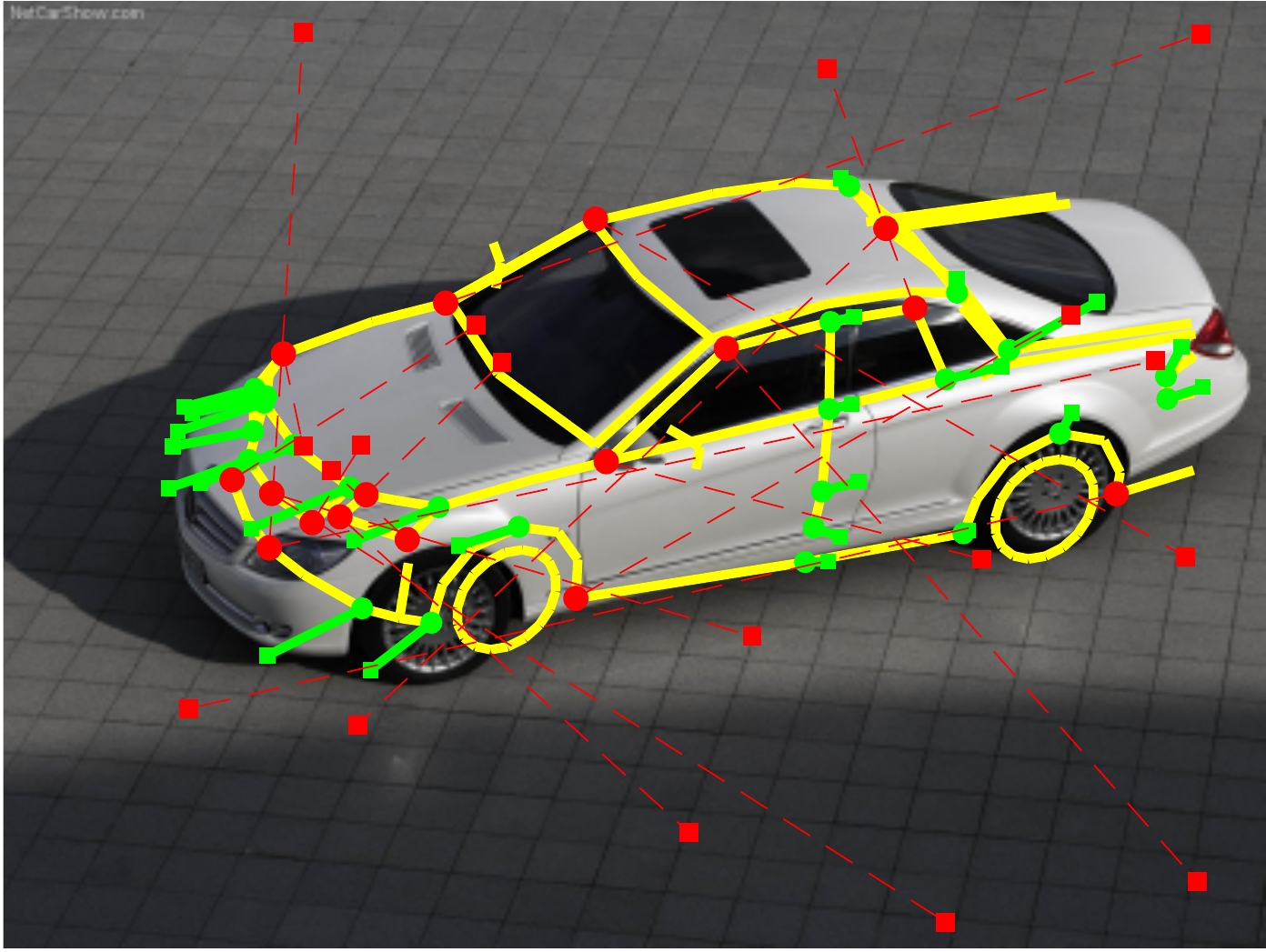} \\
	\vspace{1mm}
	\end{minipage}
&  \myhspace
	\begin{minipage}{\mpwthree}%
	\centering%
	\includegraphics[width=\columnwidth]{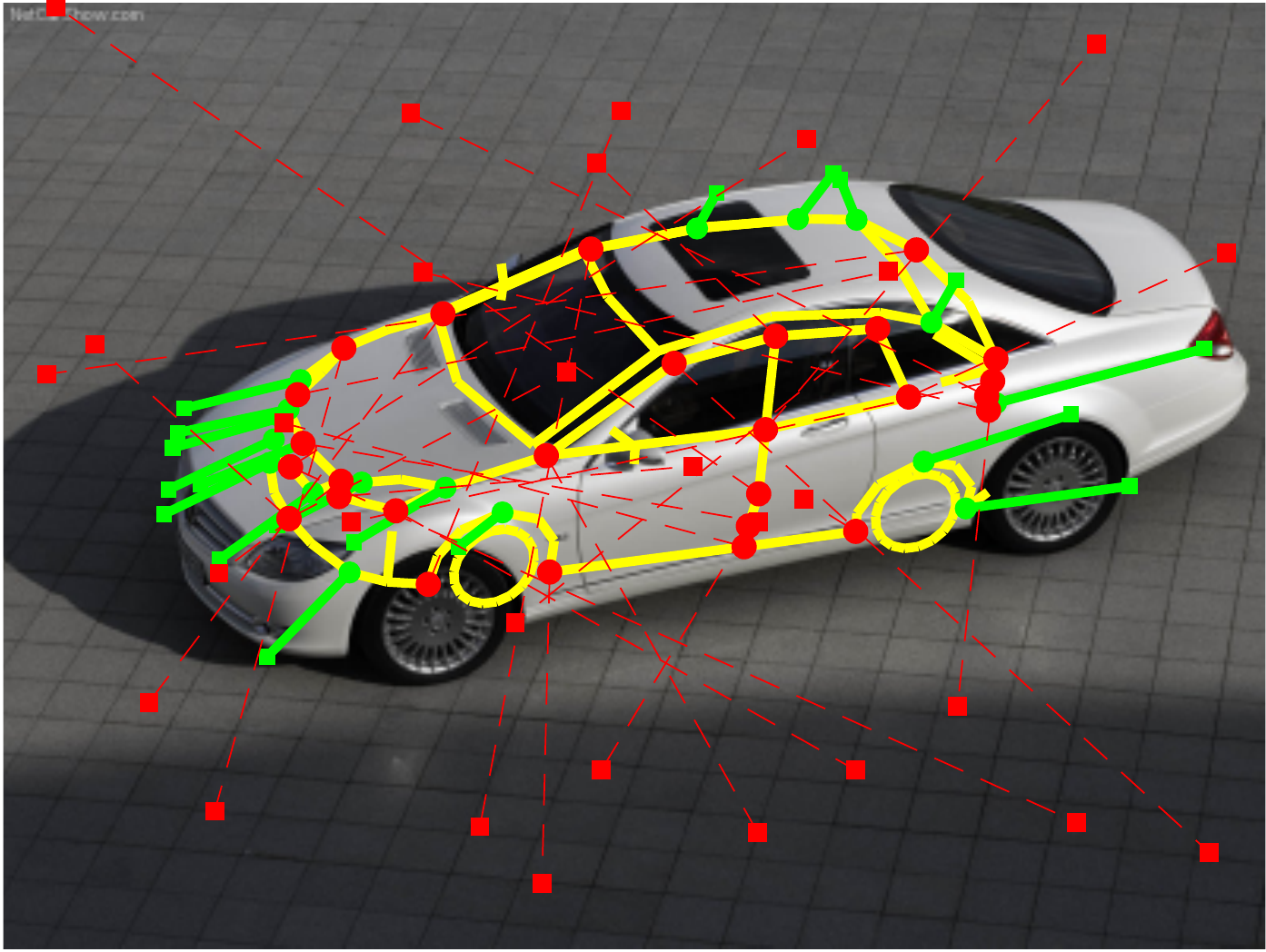} \\
	\vspace{1mm}
	\end{minipage} 
&  \myhspace
	\begin{minipage}{\mpwthree}%
	\centering%
	\includegraphics[width=\columnwidth]{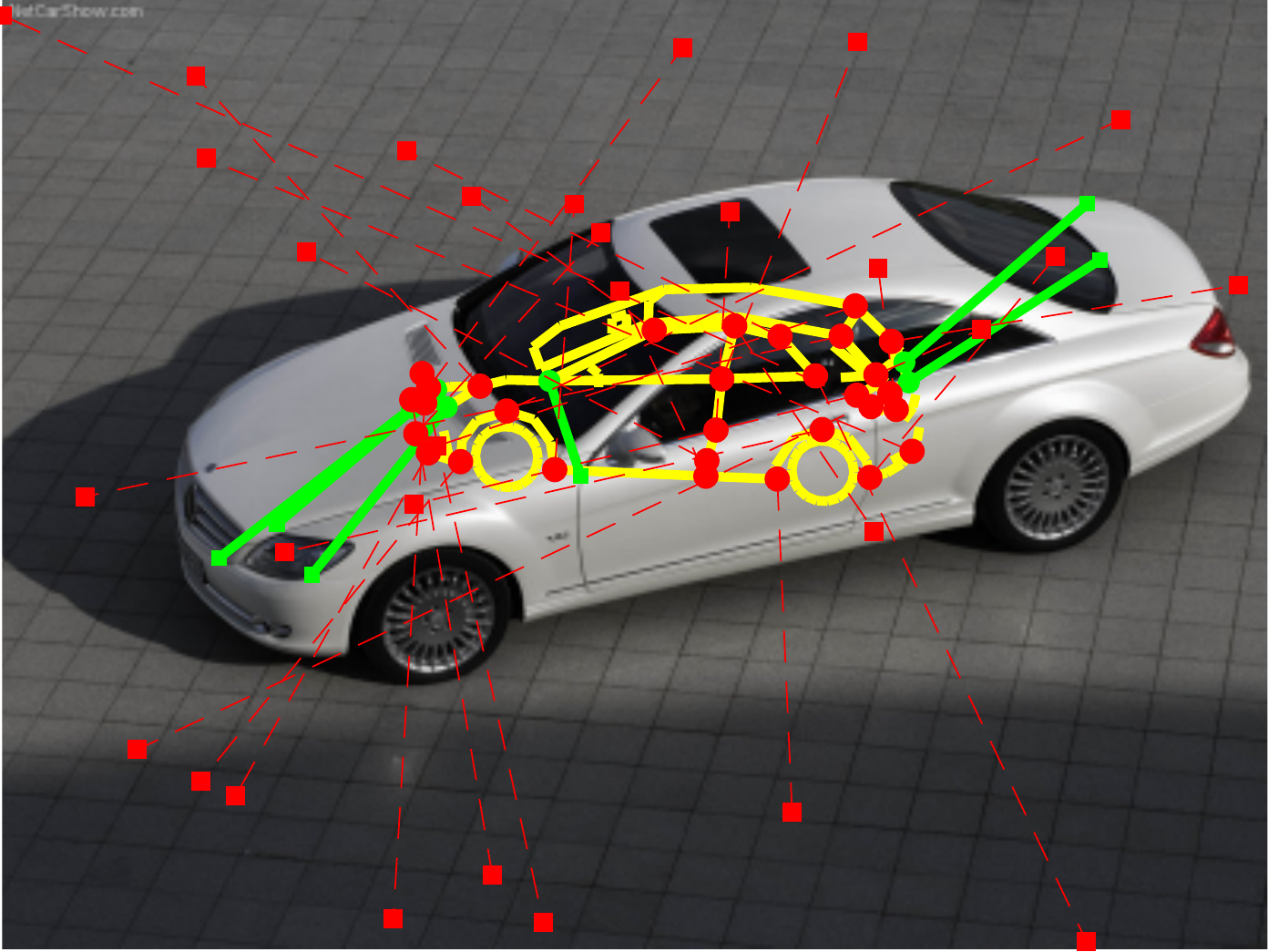} \\
	\vspace{1mm}
	\end{minipage} \\
\myhspace \myhspace \hspace{-2mm} \rotatebox{90}{\hspace{-8mm} {\smaller \convexrobust} } & 
\myhspace
	\begin{minipage}{\mpwthree}%
	\centering%
	\includegraphics[width=\columnwidth]{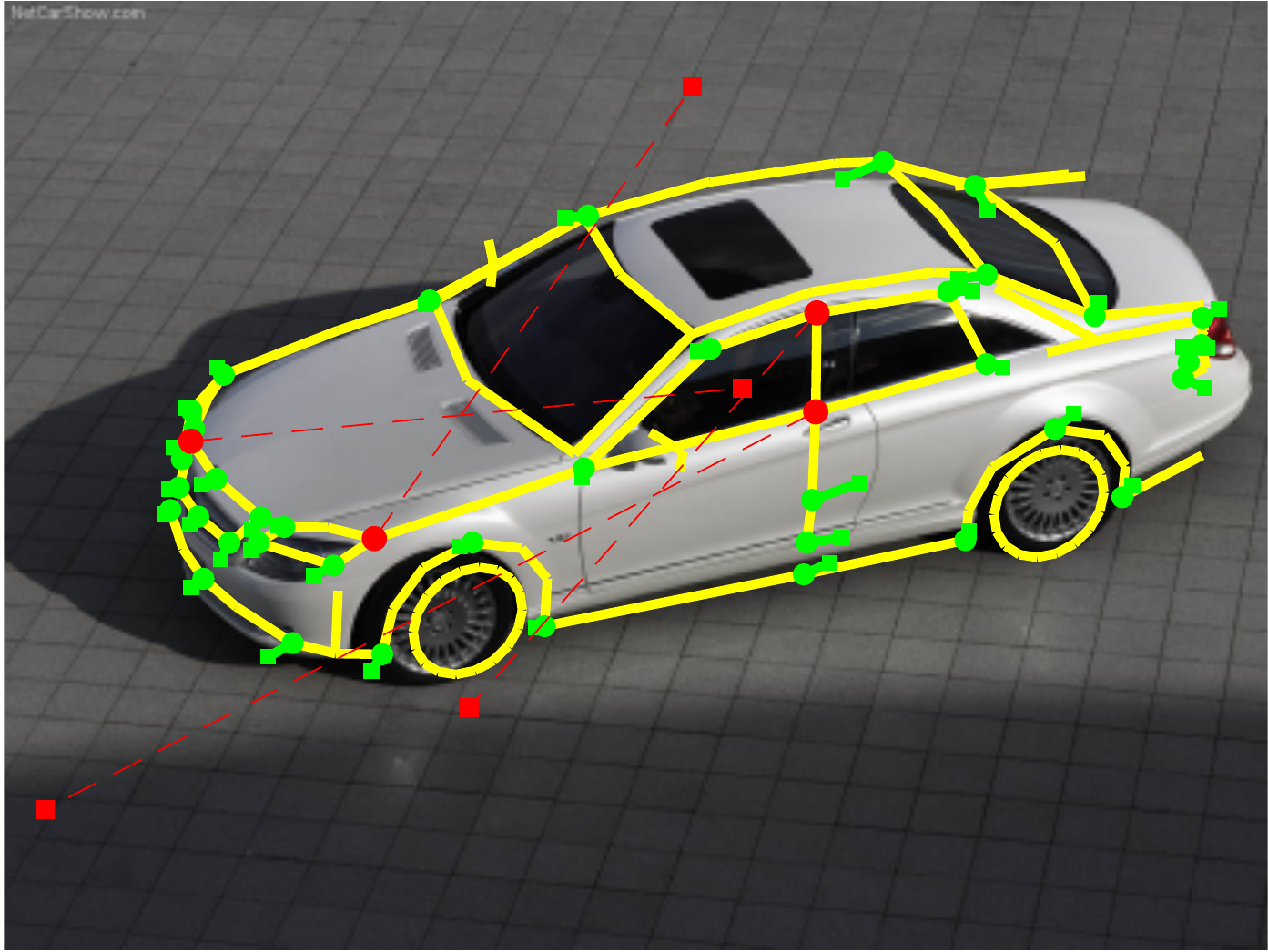} \\
	\vspace{1mm}
	\end{minipage}
& \myhspace
	\begin{minipage}{\mpwthree}%
	\centering%
	\includegraphics[width=\columnwidth]{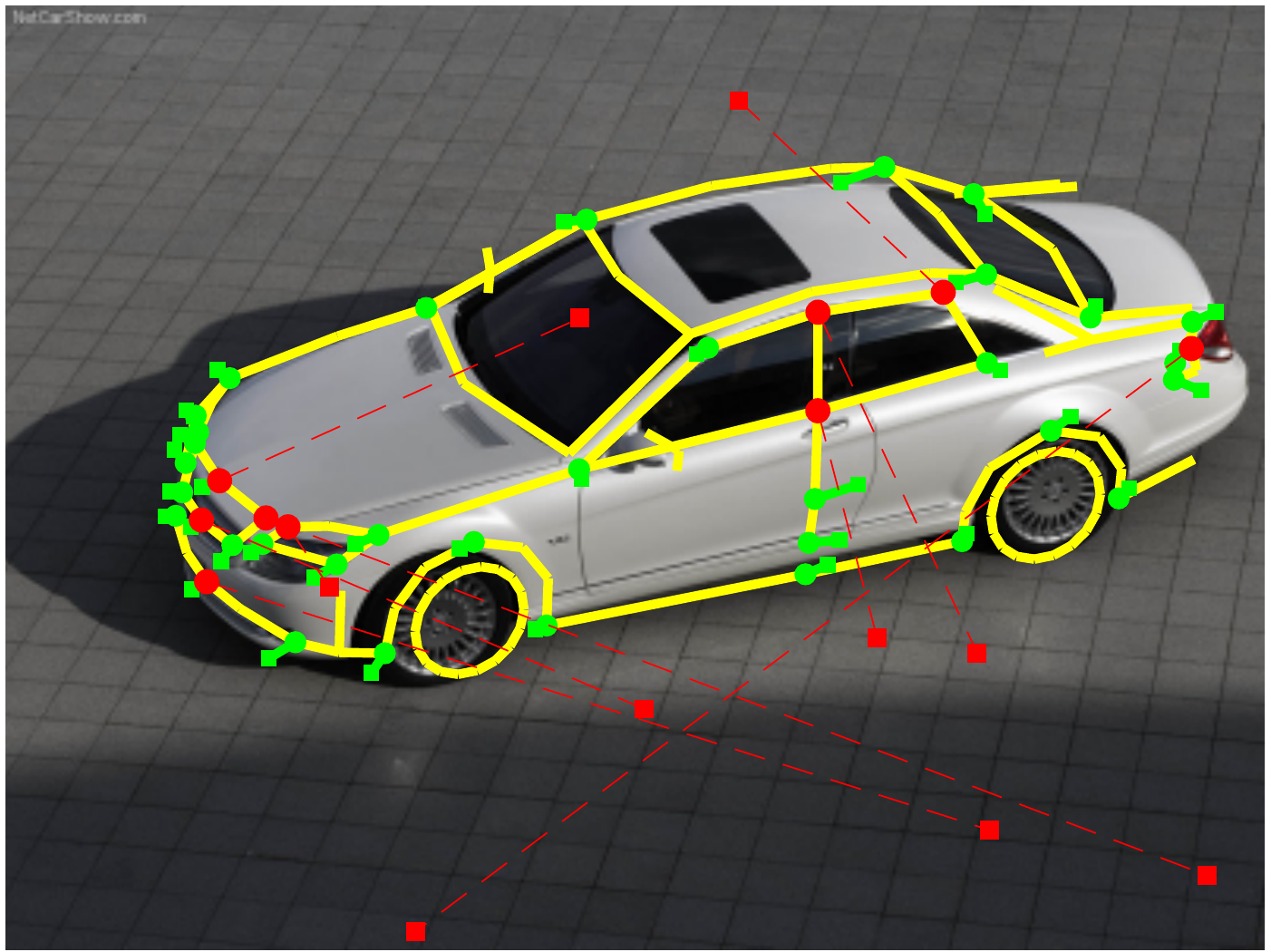} \\
	\vspace{1mm}
	\end{minipage}
& \myhspace
	\begin{minipage}{\mpwthree}%
	\centering%
	\includegraphics[width=\columnwidth]{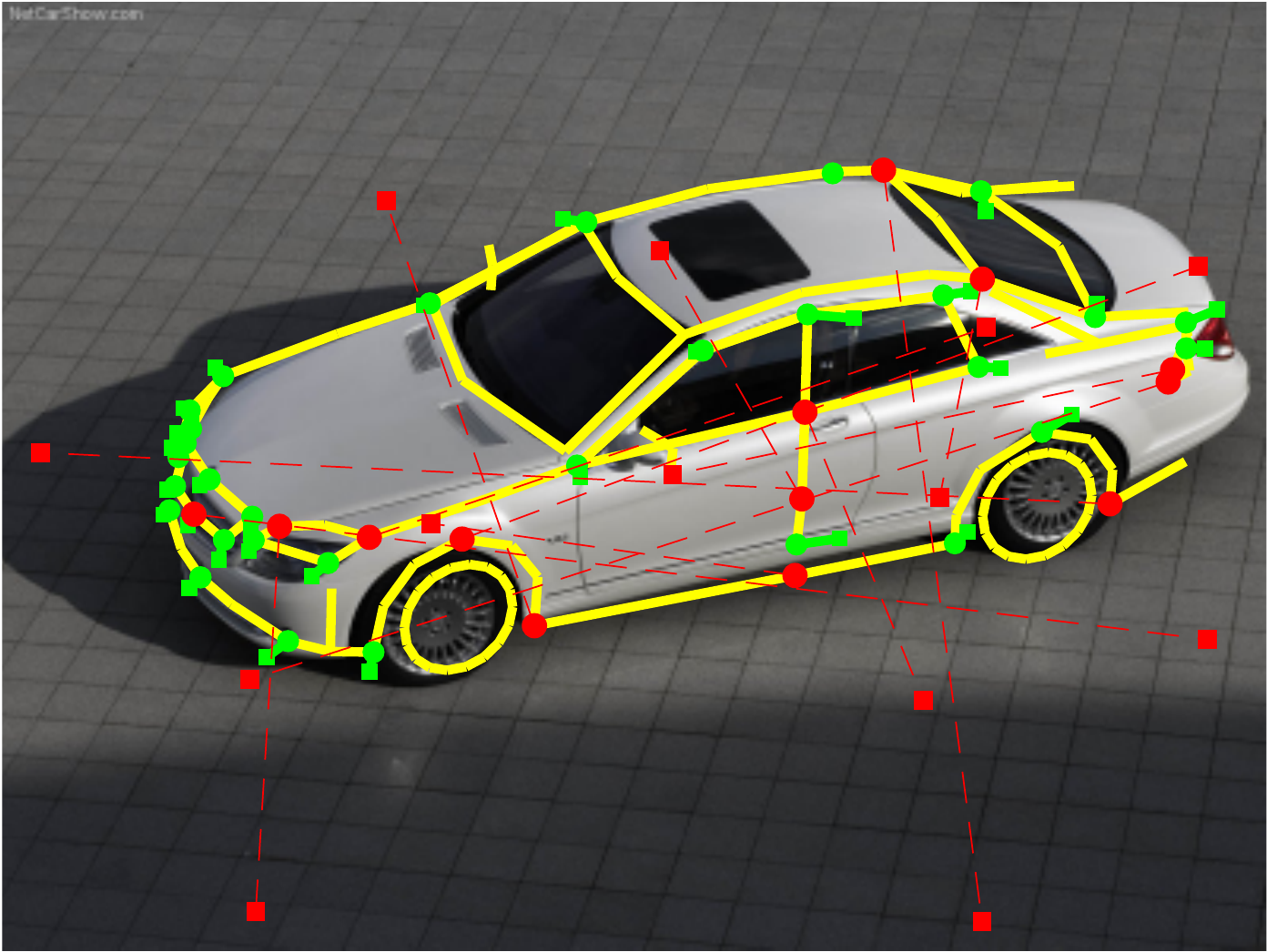} \\
	\vspace{1mm}
	\end{minipage} 
& \myhspace
	\begin{minipage}{\mpwthree}%
	\centering%
	\includegraphics[width=\columnwidth]{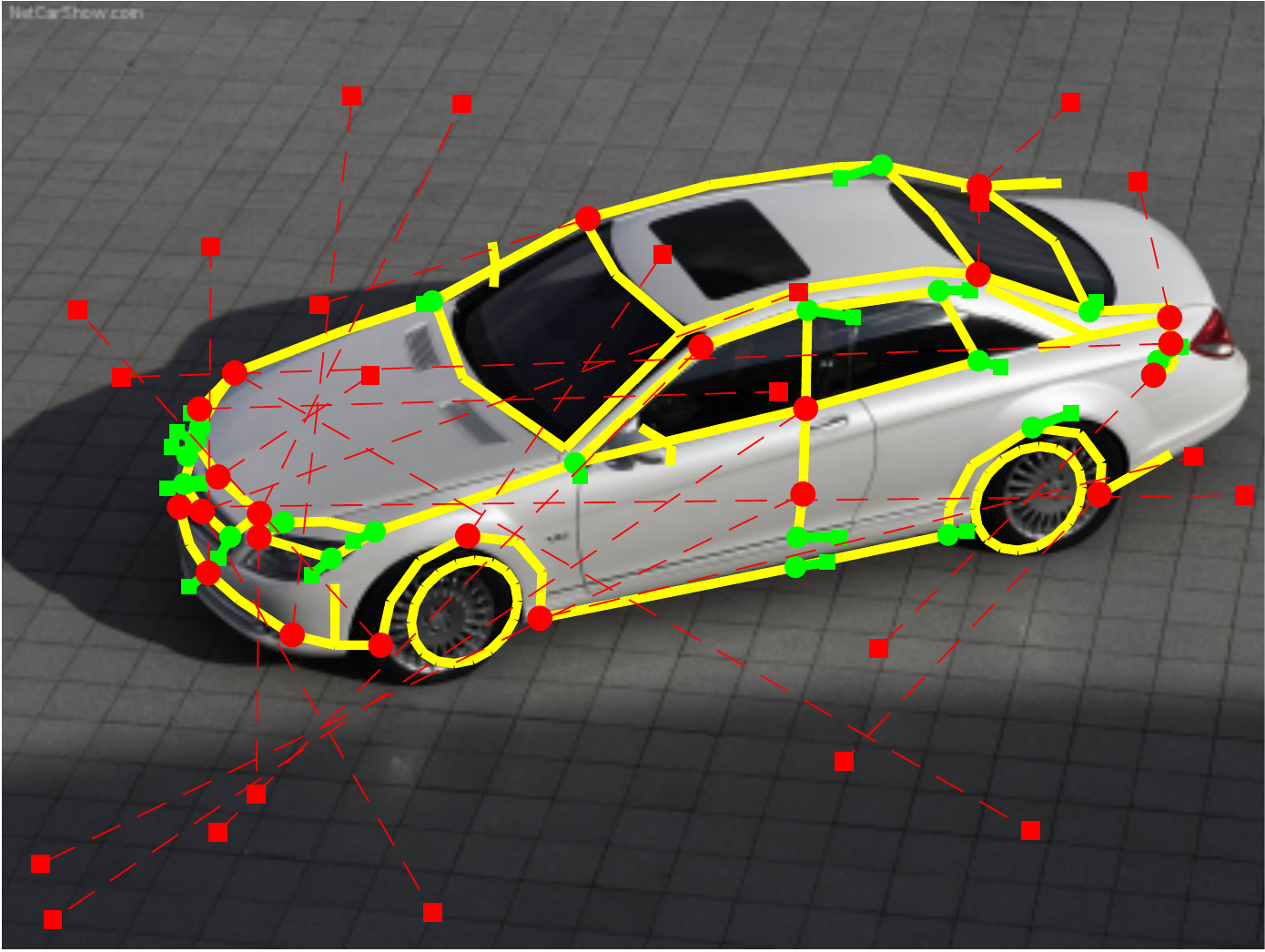} \\
	\vspace{1mm}
	\end{minipage}
& \myhspace
	\begin{minipage}{\mpwthree}%
	\centering%
	\includegraphics[width=\columnwidth]{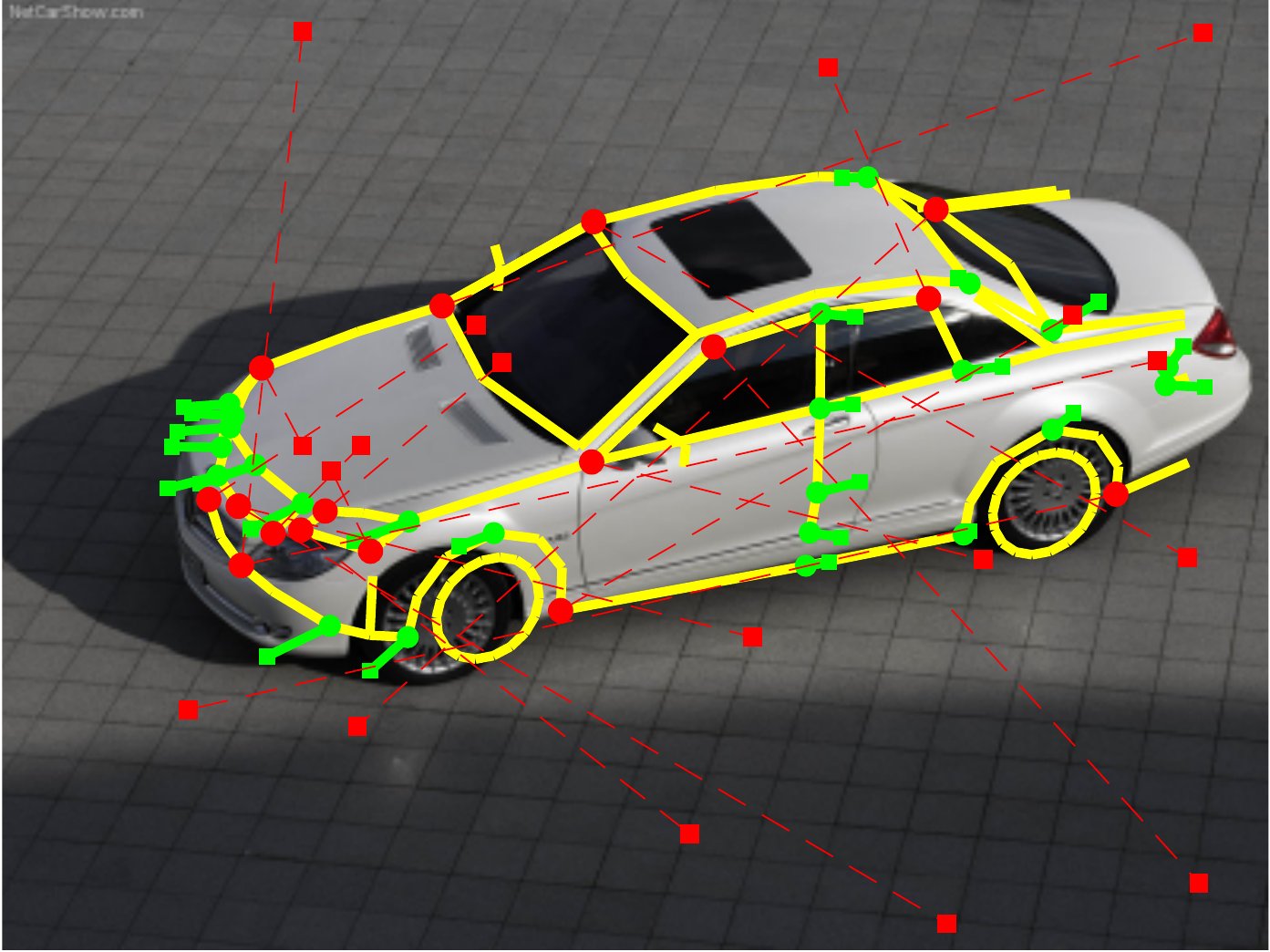} \\
	\vspace{1mm}
	\end{minipage}
&  \myhspace
	\begin{minipage}{\mpwthree}%
	\centering%
	\includegraphics[width=\columnwidth]{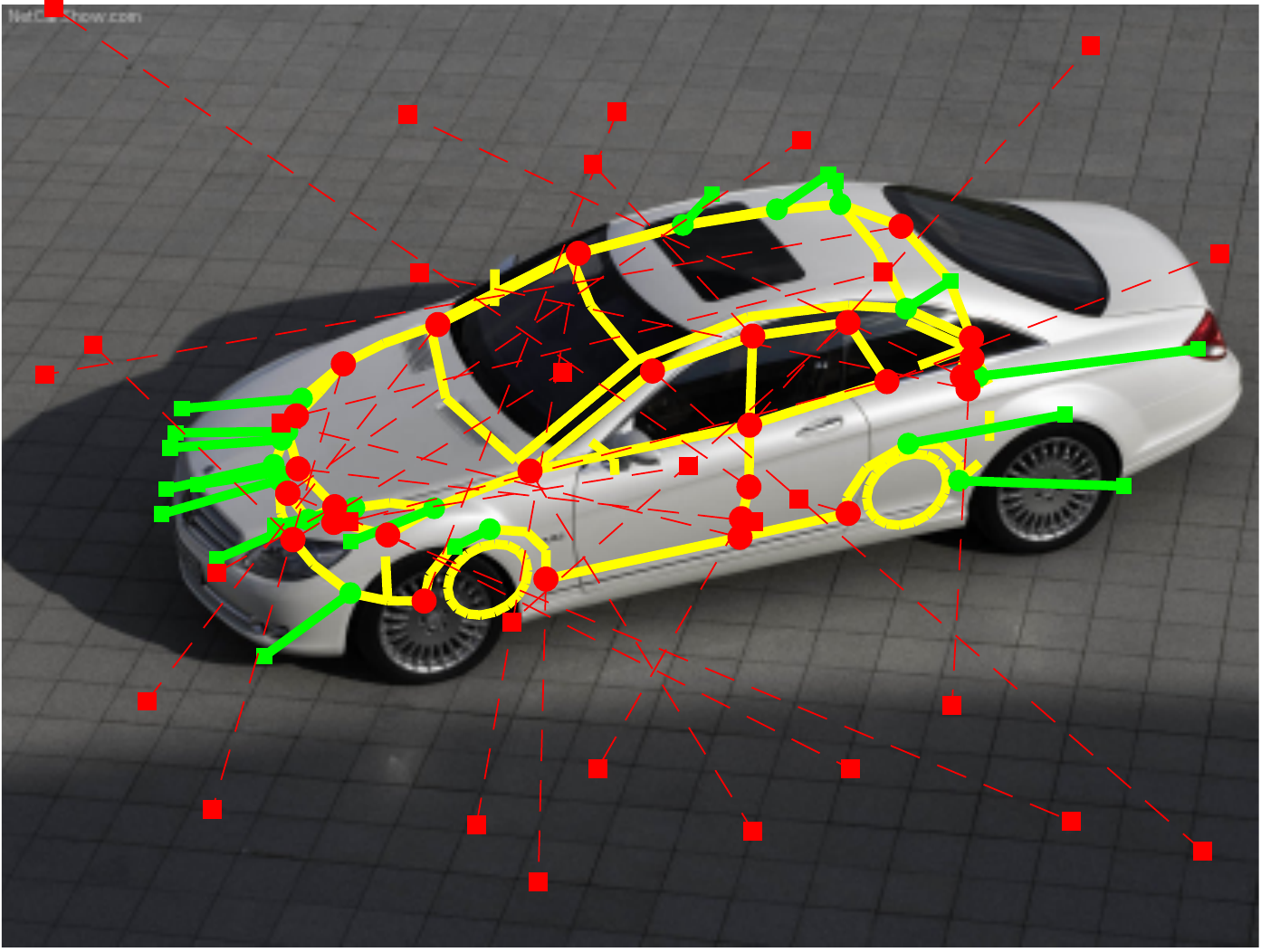} \\
	\vspace{1mm}
	\end{minipage} 
&  \myhspace
	\begin{minipage}{\mpwthree}%
	\centering%
	\includegraphics[width=\columnwidth]{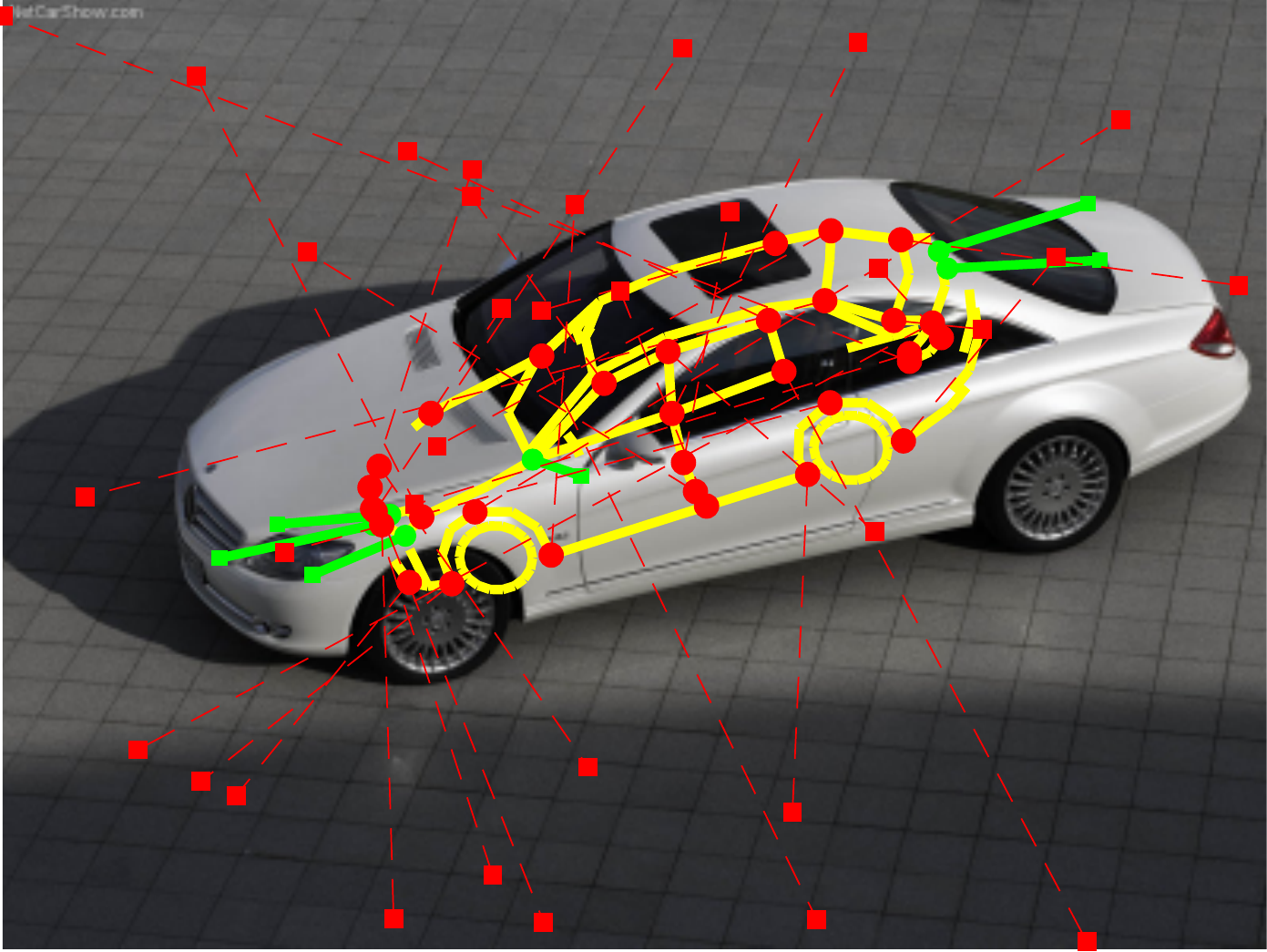} \\
	\vspace{1mm}
	\end{minipage} \\
\myhspace \myhspace \hspace{-2mm} \rotatebox{90}{\hspace{-3mm} {\smaller \blue{ \namerobust}} } & 
\myhspace
	\begin{minipage}{\mpwthree}%
	\centering%
	\includegraphics[width=\columnwidth]{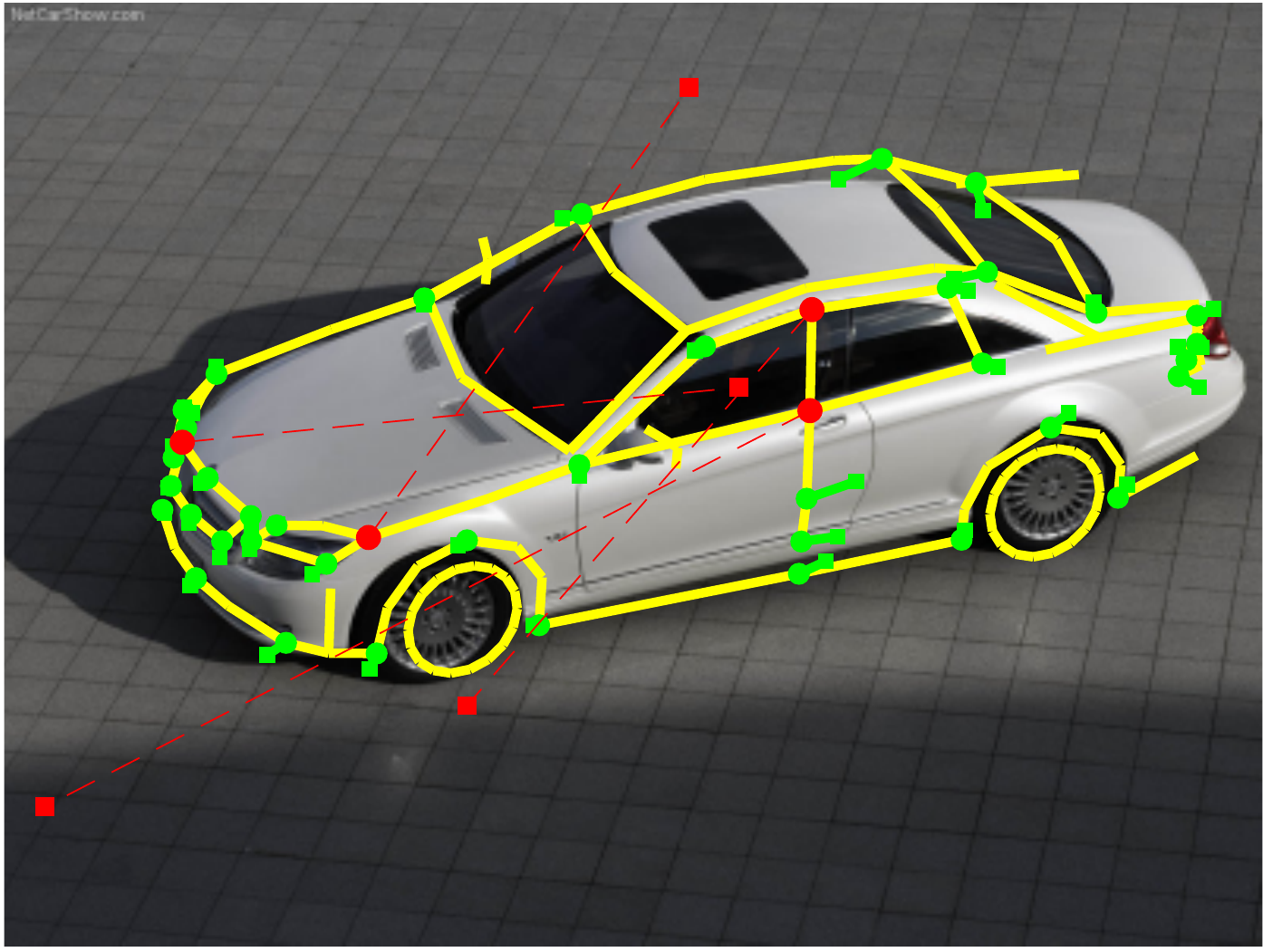} \\
	\vspace{1mm}
	\end{minipage}
& \myhspace
	\begin{minipage}{\mpwthree}%
	\centering%
	\includegraphics[width=\columnwidth]{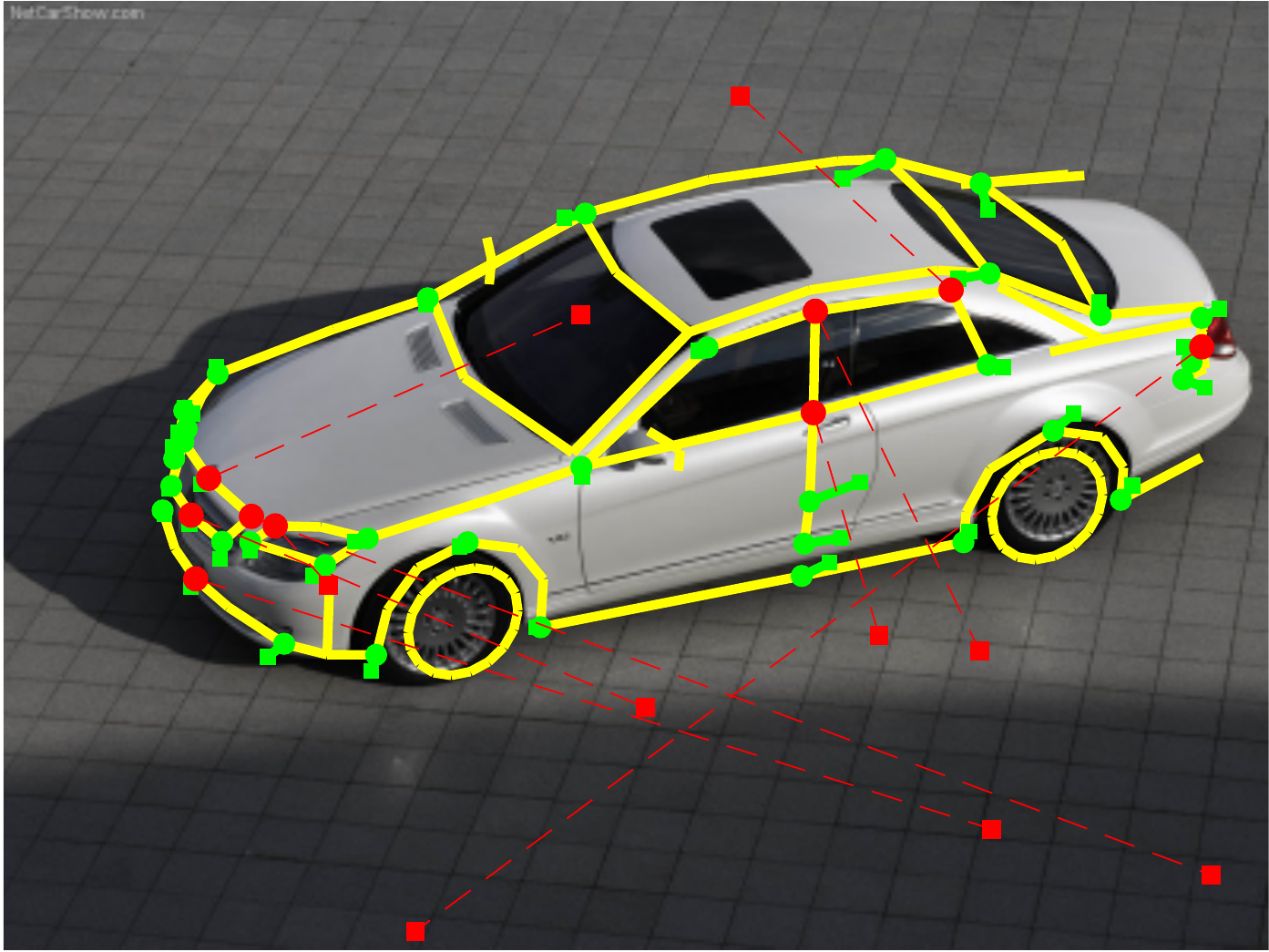} \\
	\vspace{1mm}
	\end{minipage}
& \myhspace
	\begin{minipage}{\mpwthree}%
	\centering%
	\includegraphics[width=\columnwidth]{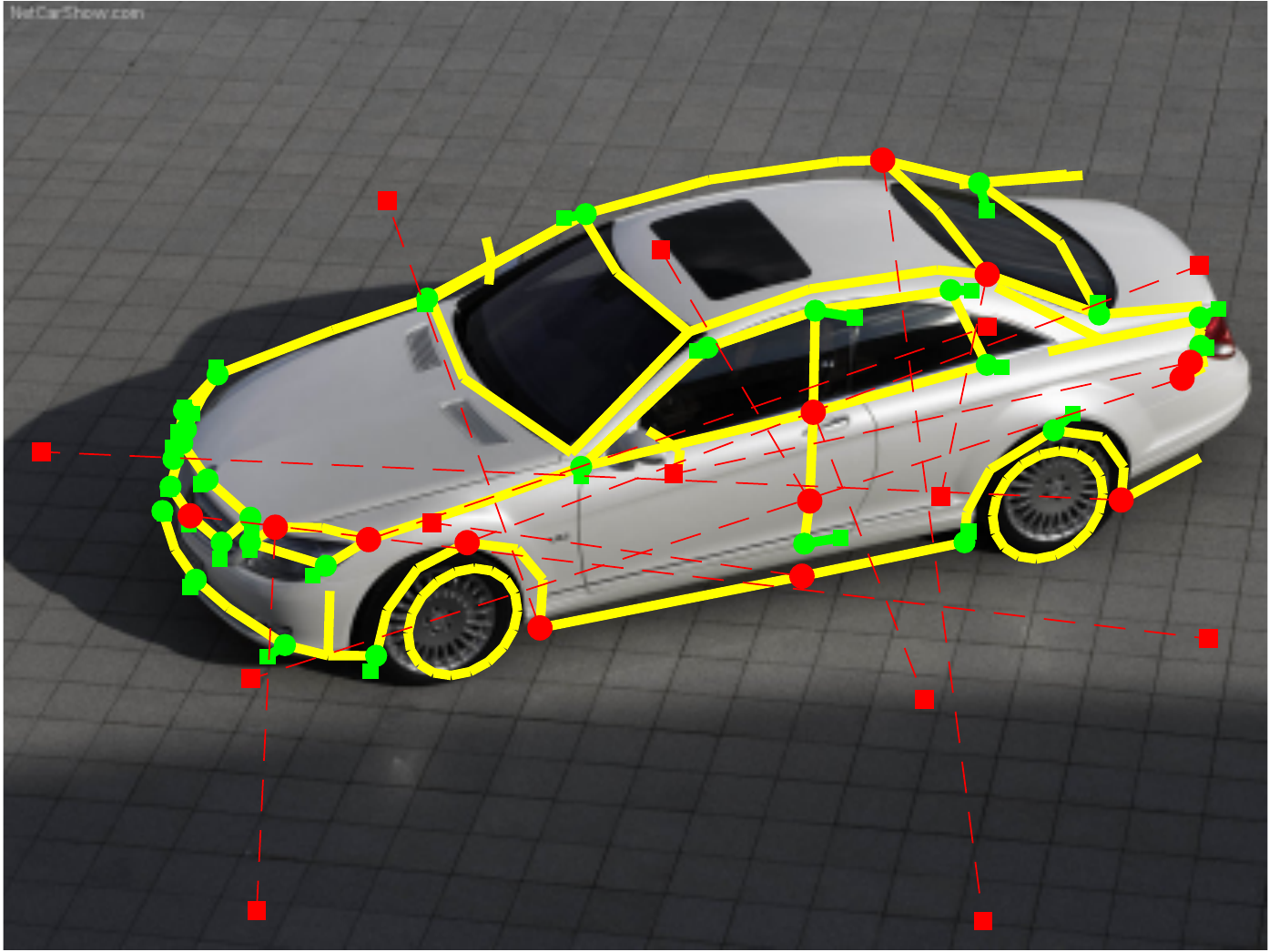} \\
	\vspace{1mm}
	\end{minipage} 
& \myhspace
	\begin{minipage}{\mpwthree}%
	\centering%
	\includegraphics[width=\columnwidth]{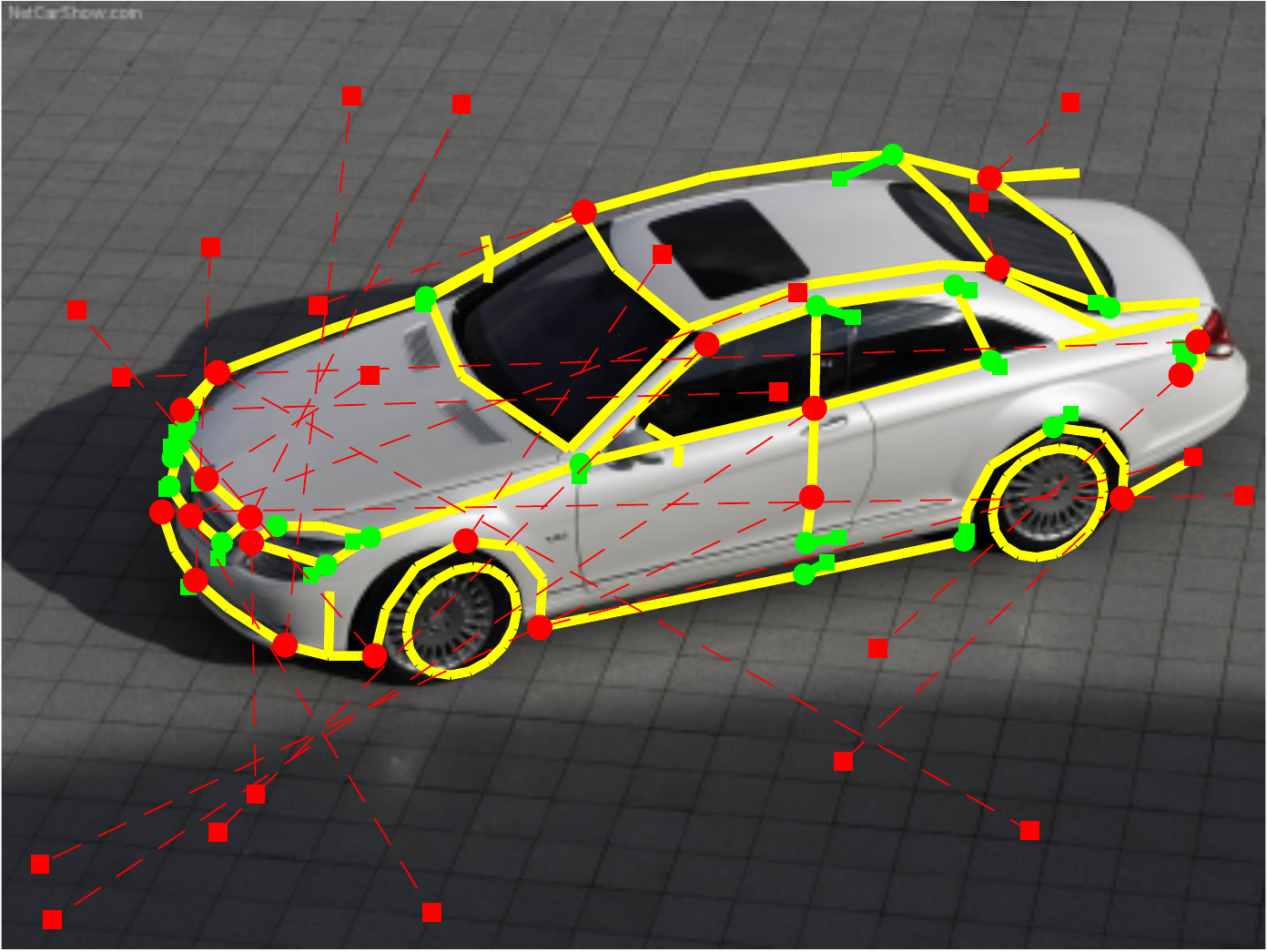} \\
	\vspace{1mm}
	\end{minipage}
& \myhspace
	\begin{minipage}{\mpwthree}%
	\centering%
	\includegraphics[width=\columnwidth]{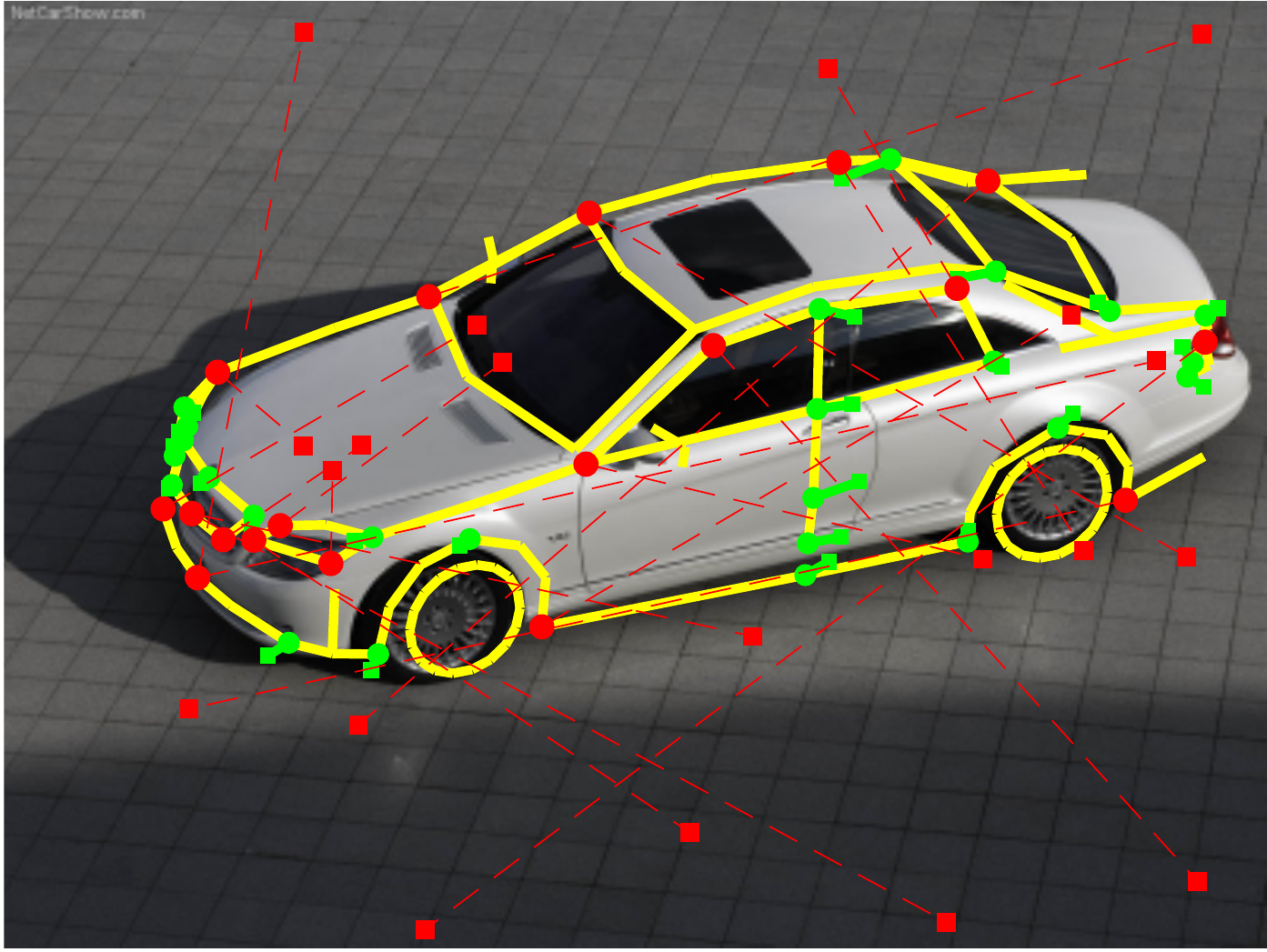} \\
	\vspace{1mm}
	\end{minipage}
&  \myhspace
	\begin{minipage}{\mpwthree}%
	\centering%
	\includegraphics[width=\columnwidth]{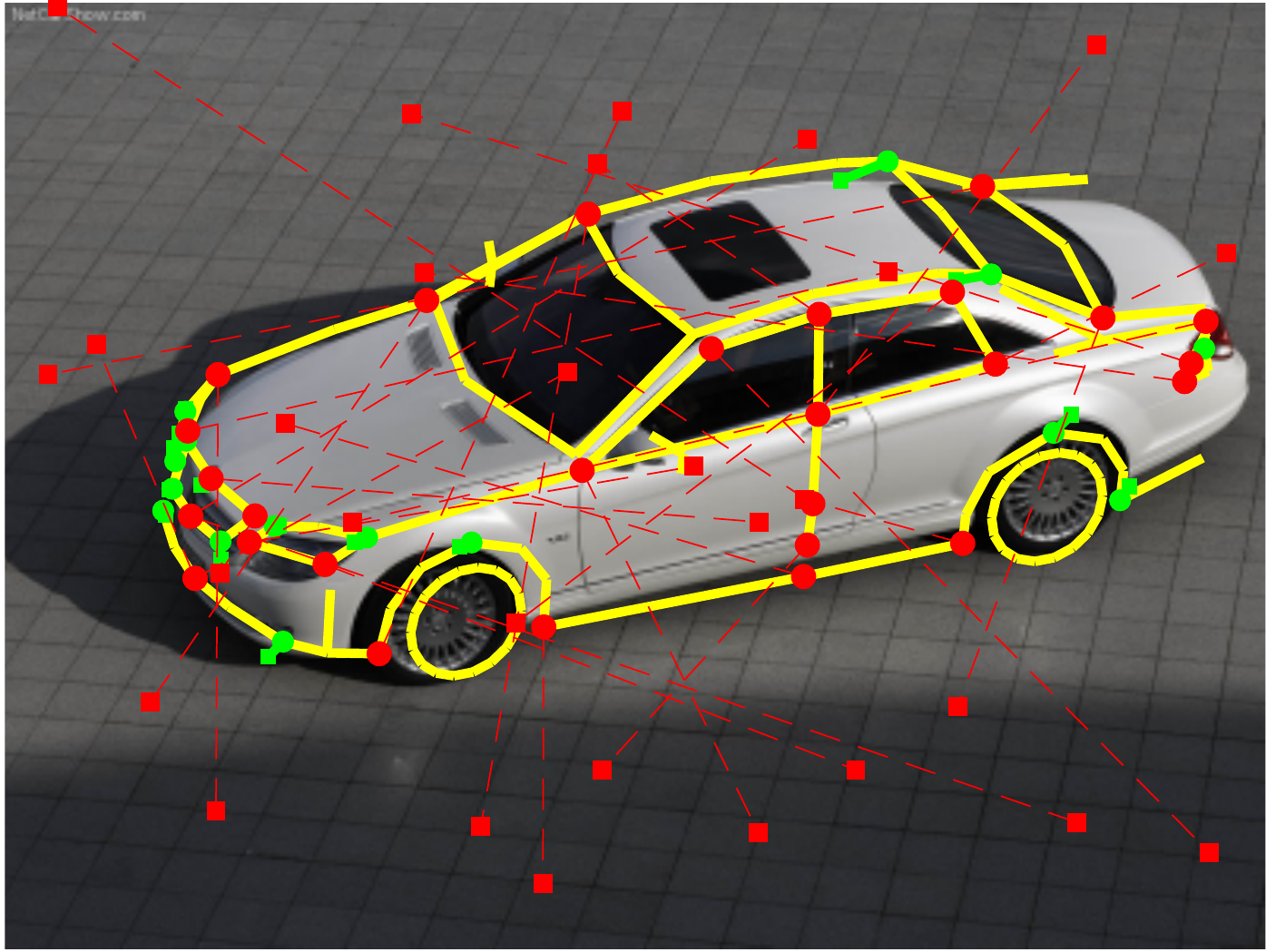} \\
	\vspace{1mm}
	\end{minipage} 
&  \myhspace
	\begin{minipage}{\mpwthree}%
	\centering%
	\includegraphics[width=\columnwidth]{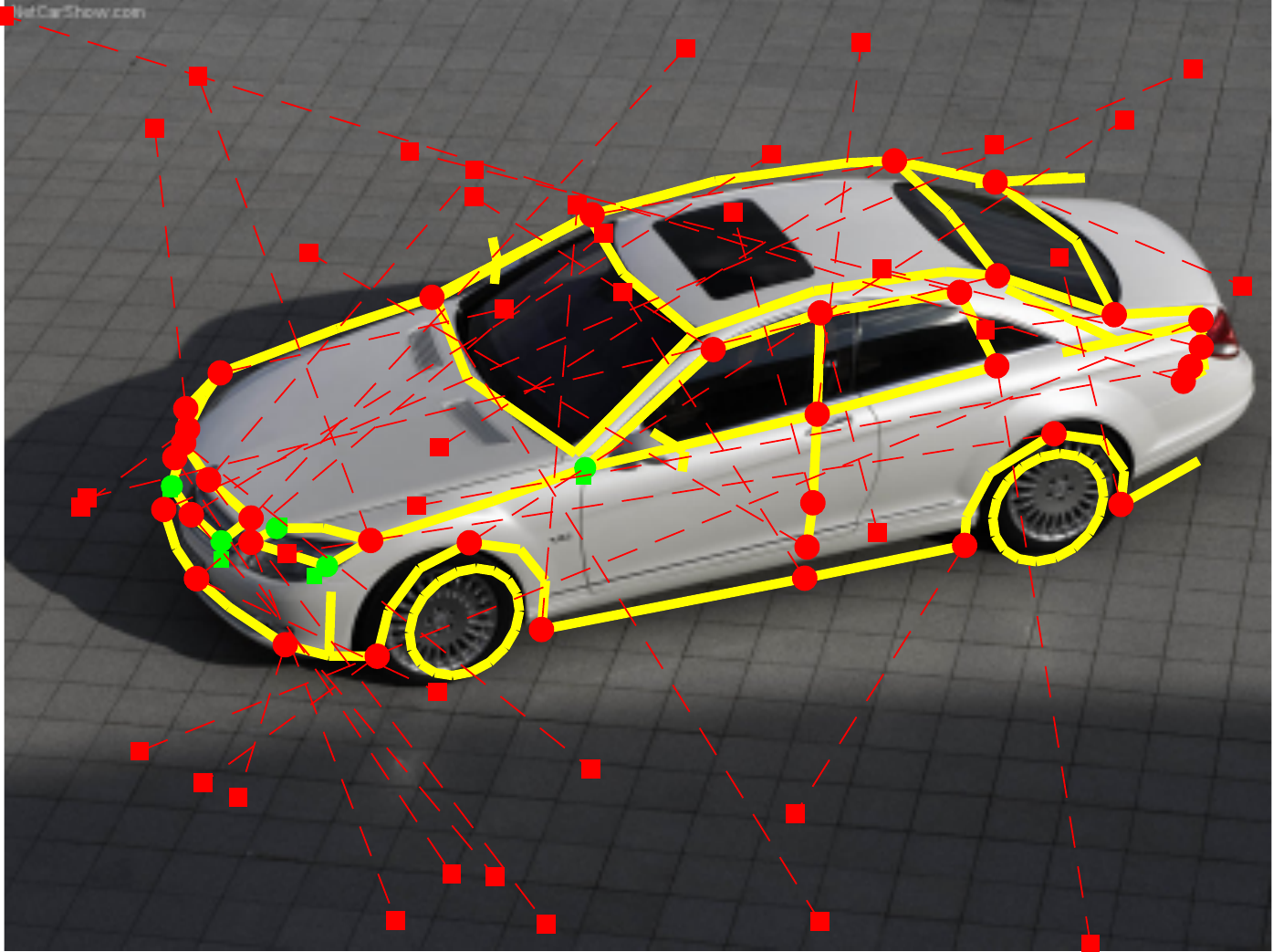} \\
	\vspace{1mm}
	\end{minipage} \\
\multicolumn{8}{c}{Mercedes-Benz C 600 } \\

%% file: supp-fig-FG3DCar_qualitative-1.tex

\renewcommand{\mpwthree}{2.6cm}
\renewcommand{\myhspace}{\hspace{-3.5mm}}

\begin{figure*}[h]\ContinuedFloat
	\begin{center}{}
	\begin{minipage}{\textwidth}
	\begin{tabular}{cccccccc}%
		\input{128-Jeep_Commander}
		\input{097-Honda_CRV}
		\input{173-Mercedes-Benz_GL450}

		\end{tabular}
	\end{minipage}
	\vspace{-2mm} 
	\caption{Qualitative results on the \FGCar dataset~\cite{Lin14eccv-modelFitting} under $10-70\%$ outlier rates using \alternrobust~\cite{Zhou17pami-shapeEstimationConvex}, \convexrobust~\cite{Zhou17pami-shapeEstimationConvex}, and \namerobust. 
	Yellow: shape reconstruction result projected onto the image.
	Green: inliers. Red: outliers. Circle: 3D landmark. Square: 2D landmark. (cont.) [Best viewed electronically.]
	\label{fig:supp_FG3DCar_qualitative}}
	\end{center}
\end{figure*}

%% file: 128-Jeep_Commander.tex
\myhspace \myhspace \hspace{-2mm} \rotatebox{90}{\hspace{-7mm} {\smaller \alternrobust} } & 
\myhspace
	\begin{minipage}{\mpwthree}%
	\centering%
	\includegraphics[width=\columnwidth]{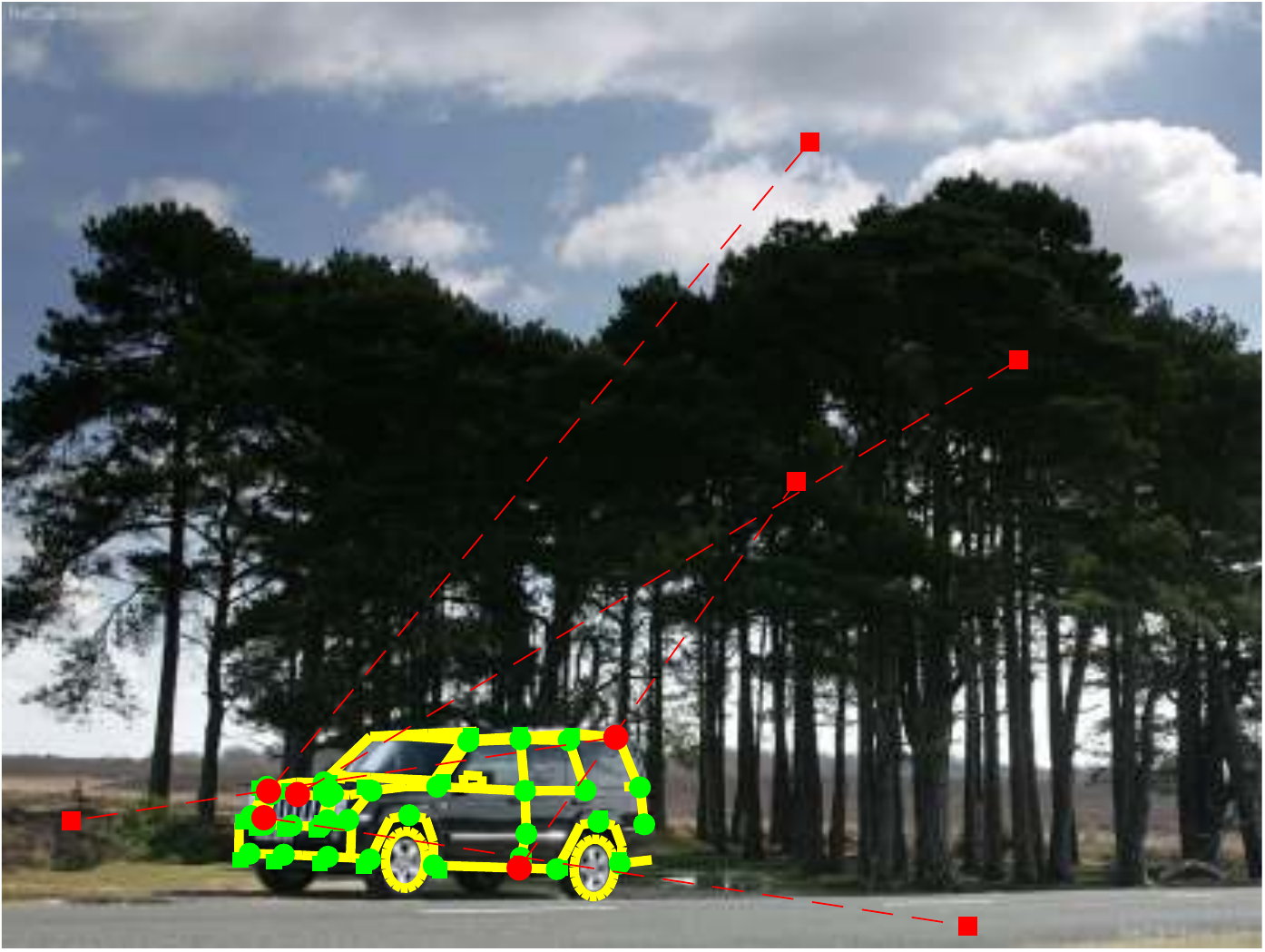} \\
	\vspace{1mm}
	\end{minipage}
& \myhspace
	\begin{minipage}{\mpwthree}%
	\centering%
	\includegraphics[width=\columnwidth]{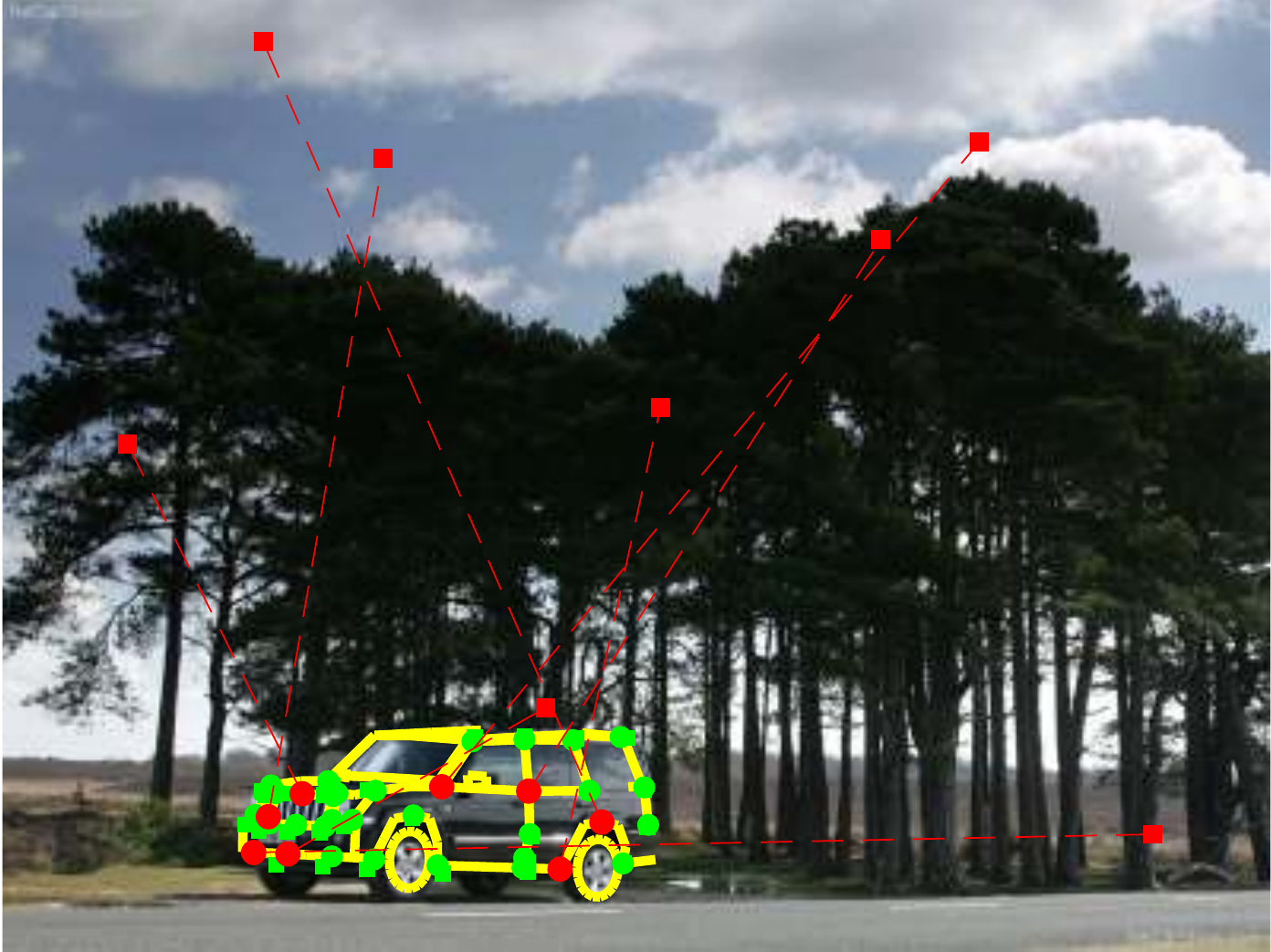} \\
	\vspace{1mm}
	\end{minipage}
& \myhspace
	\begin{minipage}{\mpwthree}%
	\centering%
	\includegraphics[width=\columnwidth]{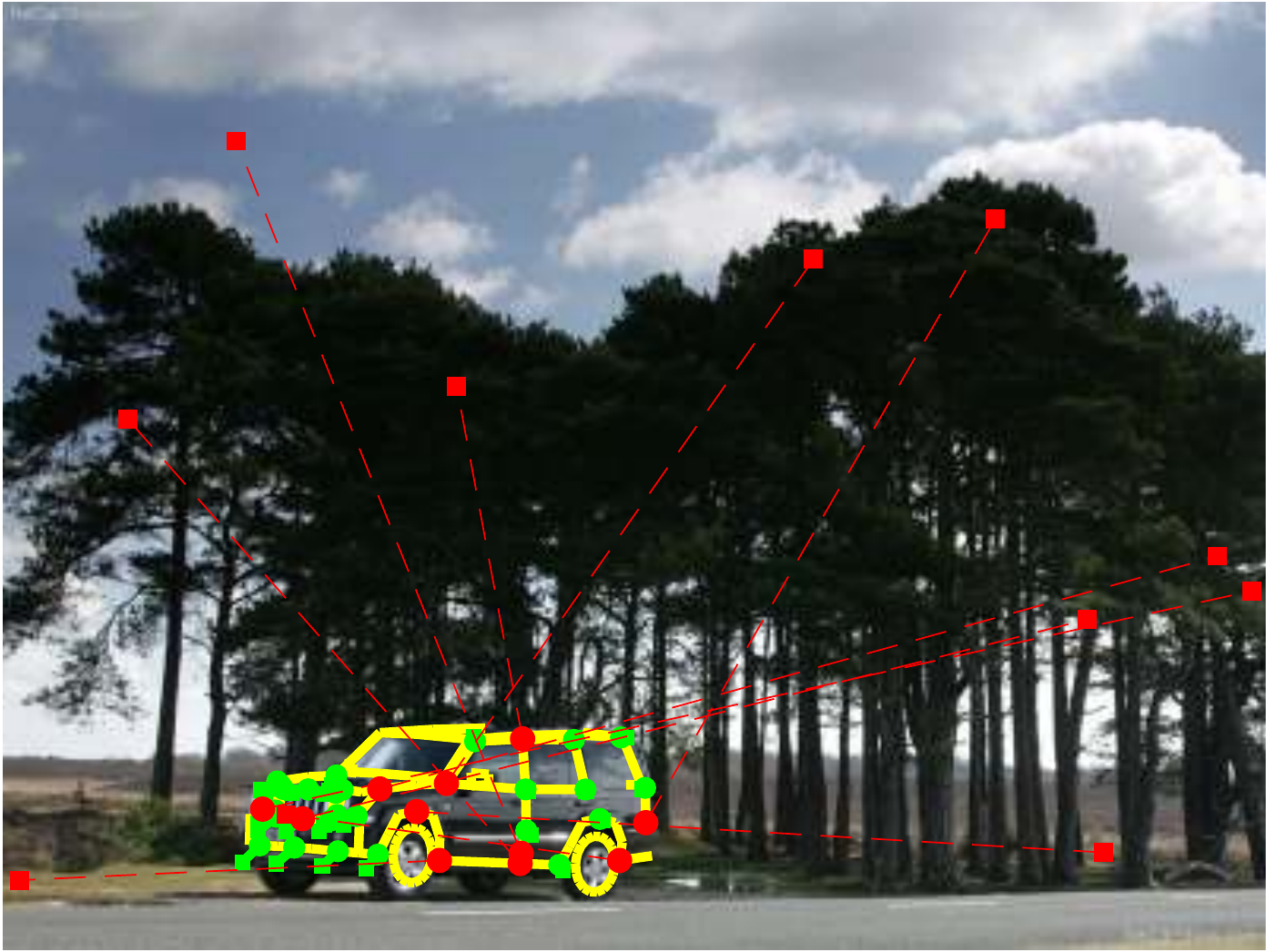} \\
	\vspace{1mm}
	\end{minipage} 
& \myhspace
	\begin{minipage}{\mpwthree}%
	\centering%
	\includegraphics[width=\columnwidth]{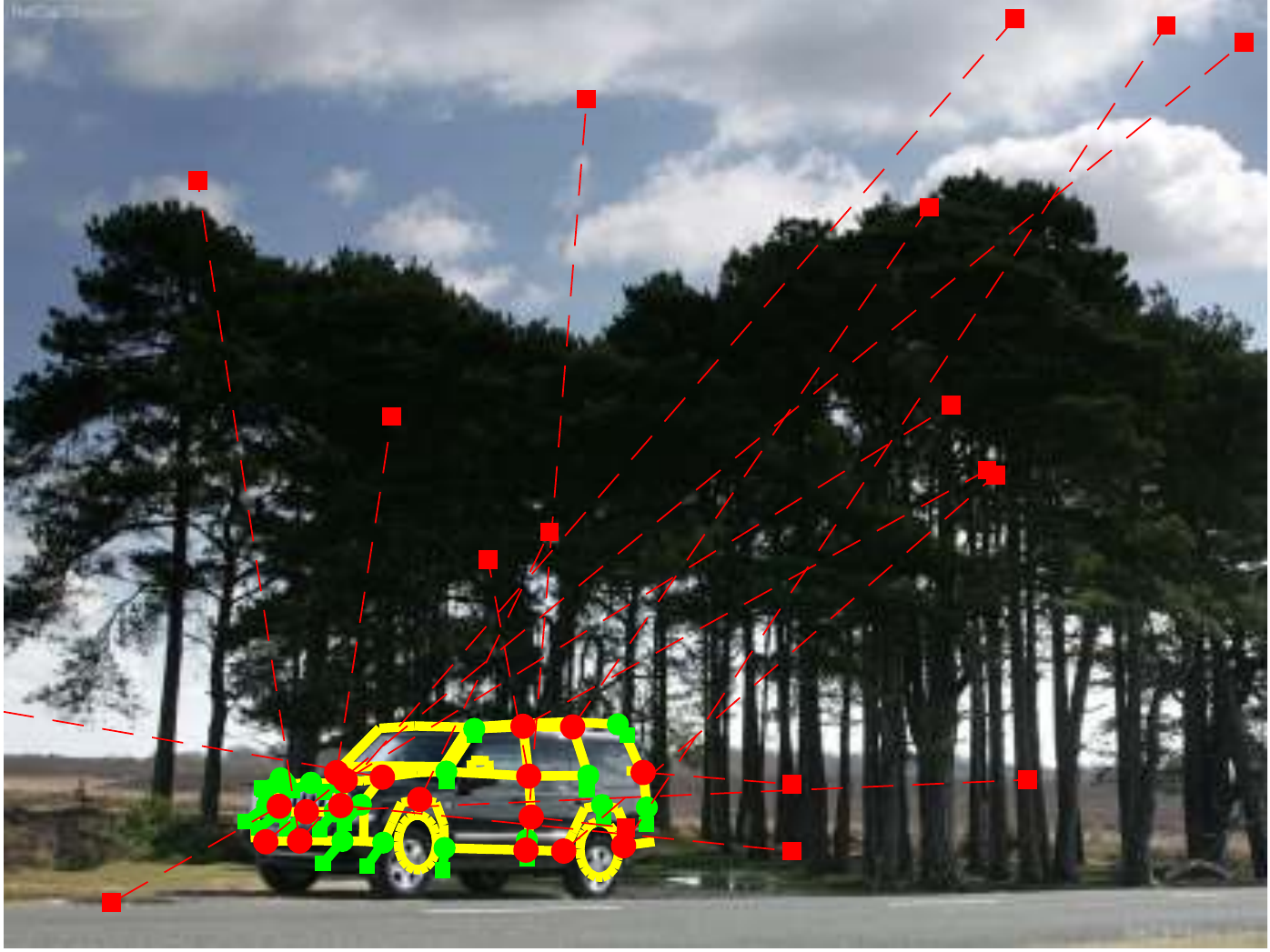} \\
	\vspace{1mm}
	\end{minipage}
& \myhspace
	\begin{minipage}{\mpwthree}%
	\centering%
	\includegraphics[width=\columnwidth]{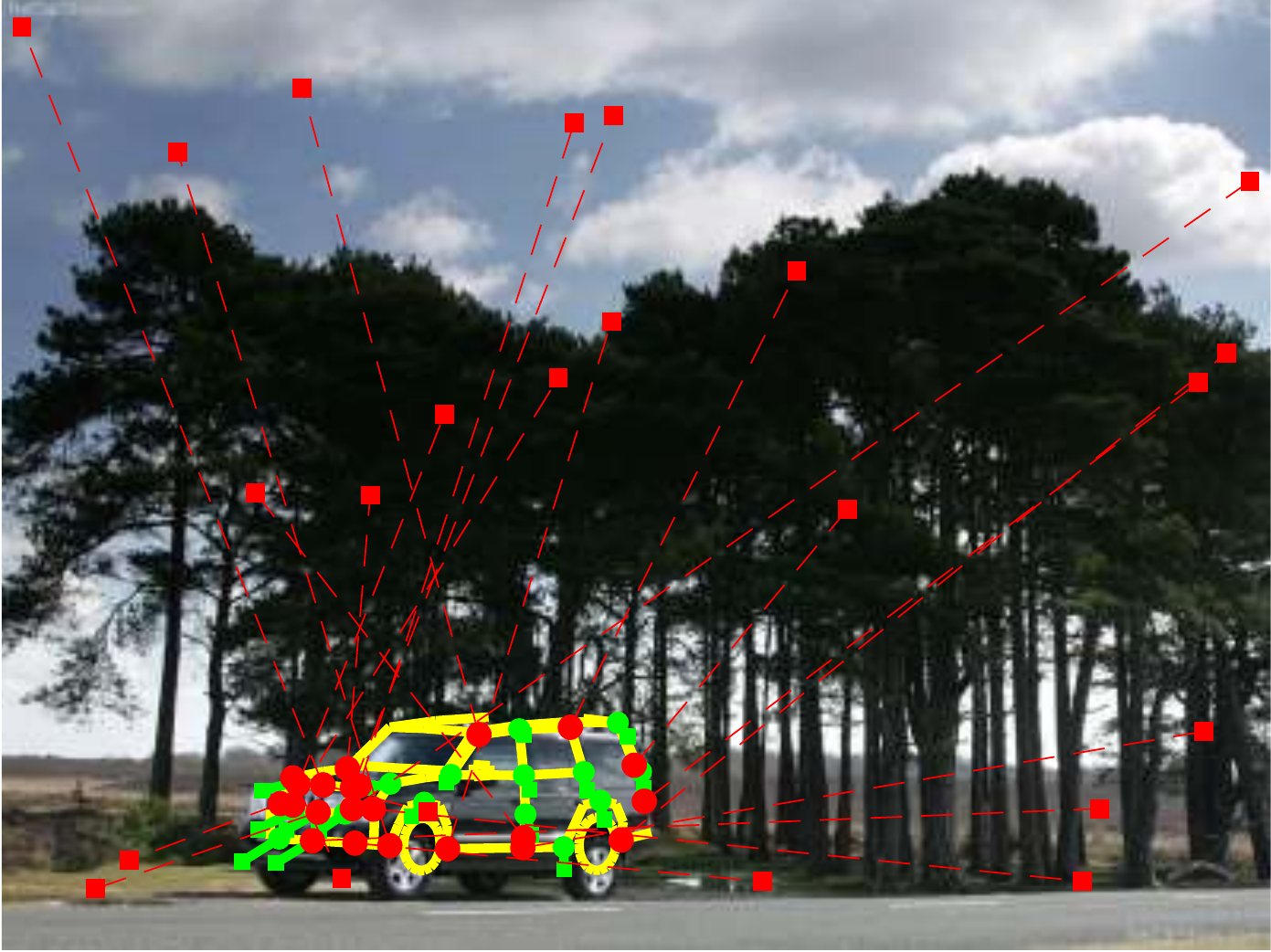} \\
	\vspace{1mm}
	\end{minipage}
&  \myhspace
	\begin{minipage}{\mpwthree}%
	\centering%
	\includegraphics[width=\columnwidth]{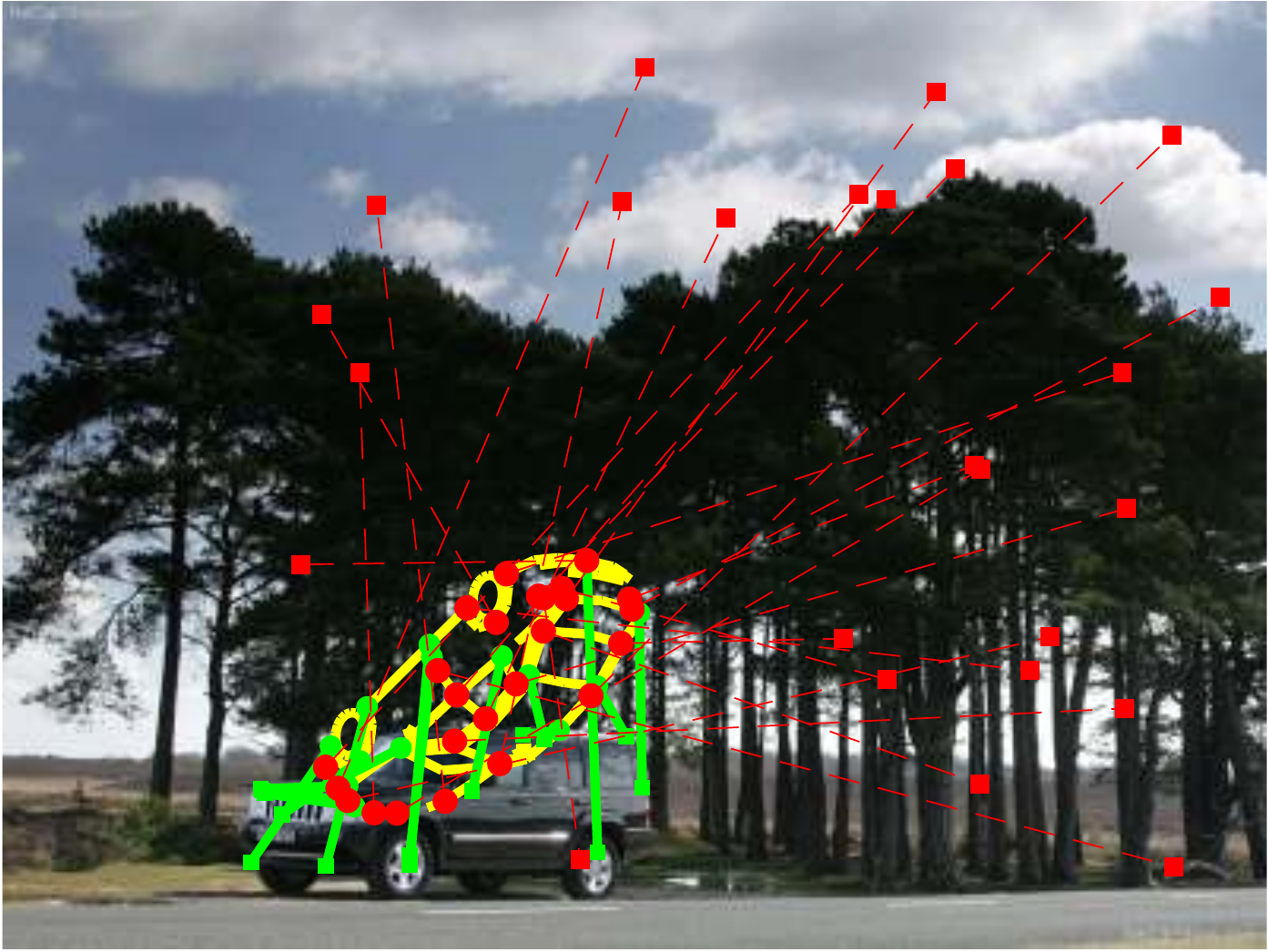} \\
	\vspace{1mm}
	\end{minipage} 
&  \myhspace
	\begin{minipage}{\mpwthree}%
	\centering%
	\includegraphics[width=\columnwidth]{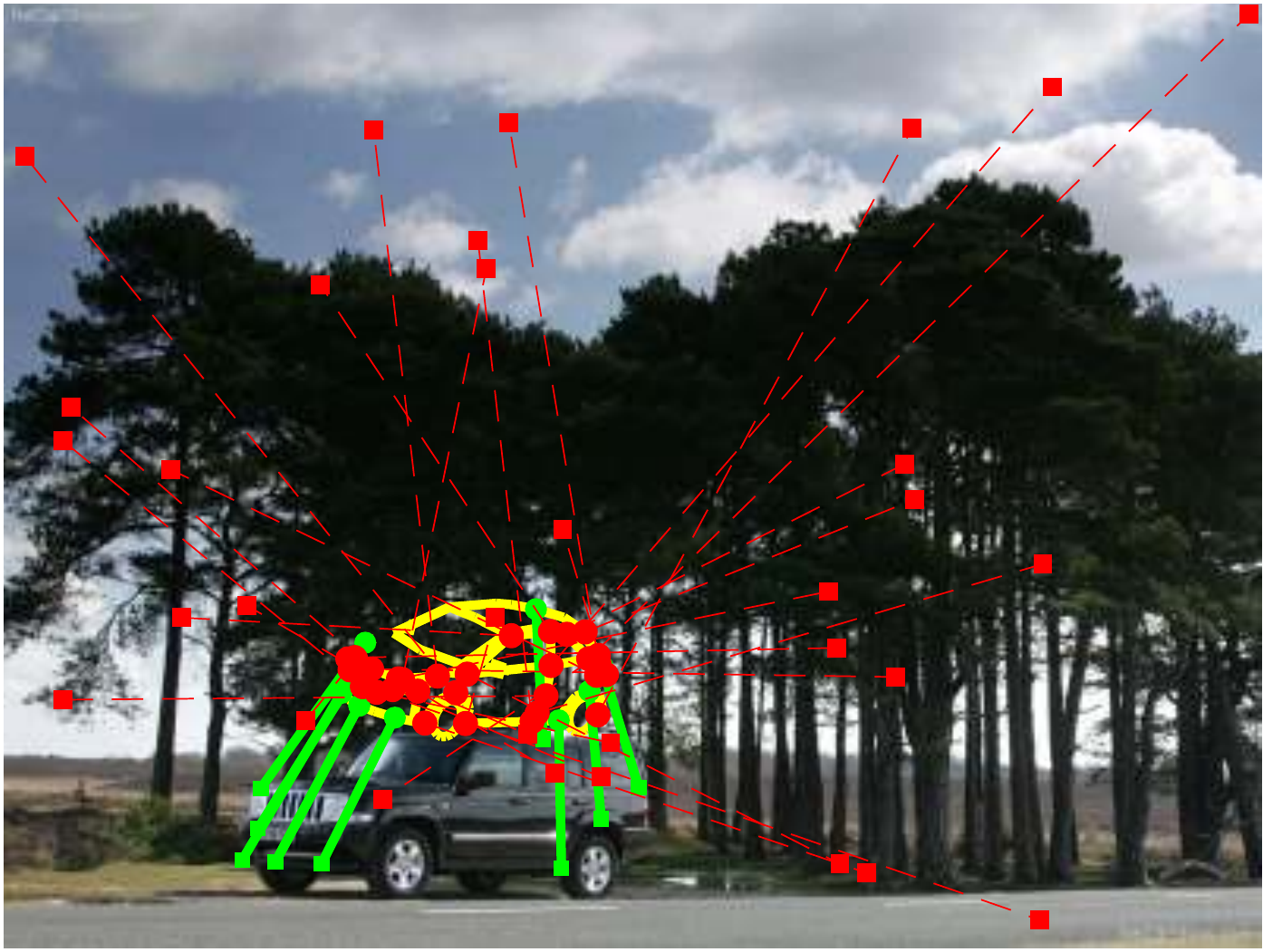} \\
	\vspace{1mm}
	\end{minipage} \\
\myhspace \myhspace \hspace{-2mm} \rotatebox{90}{\hspace{-8mm} {\smaller \convexrobust} } & 
\myhspace
	\begin{minipage}{\mpwthree}%
	\centering%
	\includegraphics[width=\columnwidth]{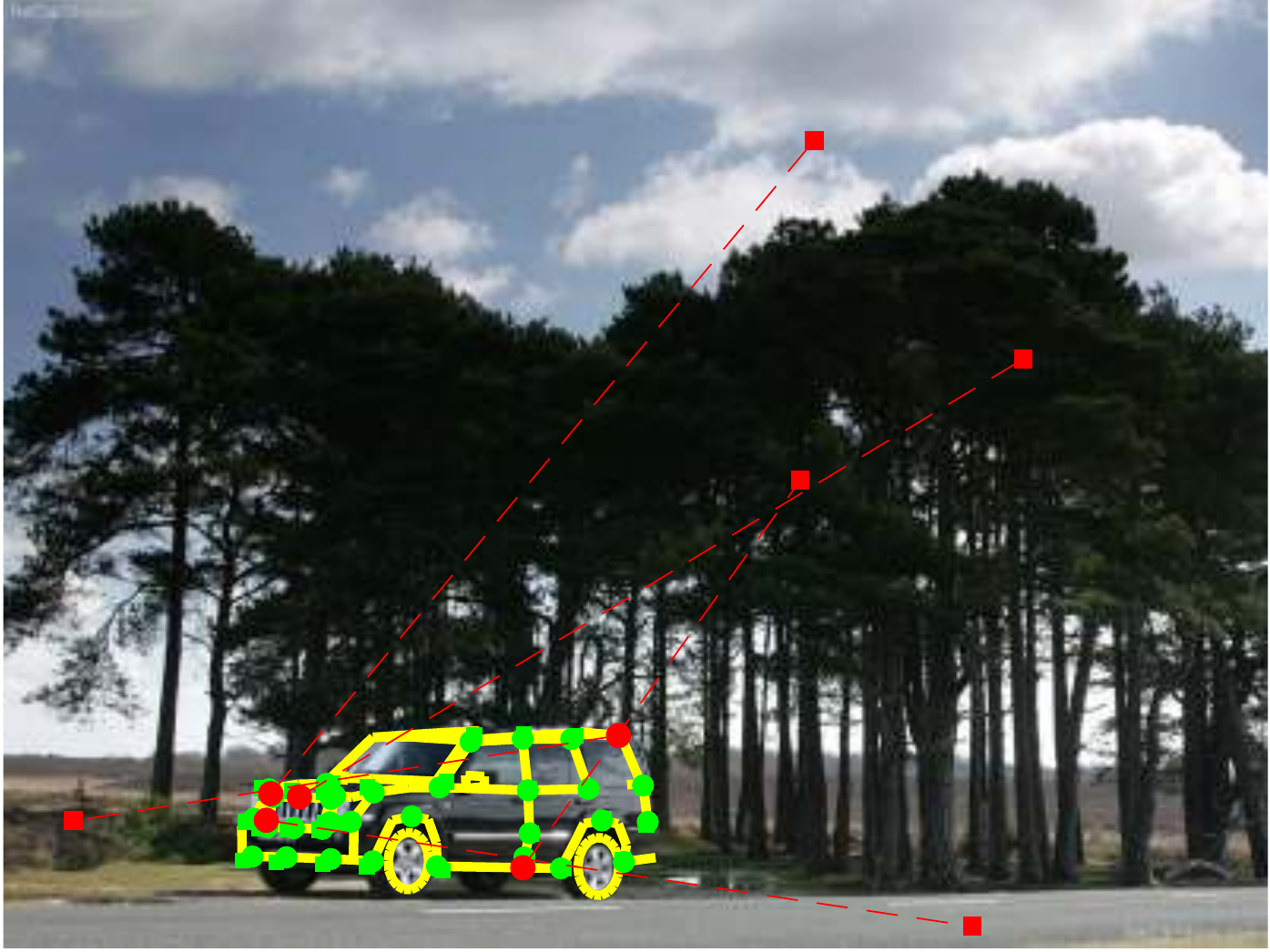} \\
	\vspace{1mm}
	\end{minipage}
& \myhspace
	\begin{minipage}{\mpwthree}%
	\centering%
	\includegraphics[width=\columnwidth]{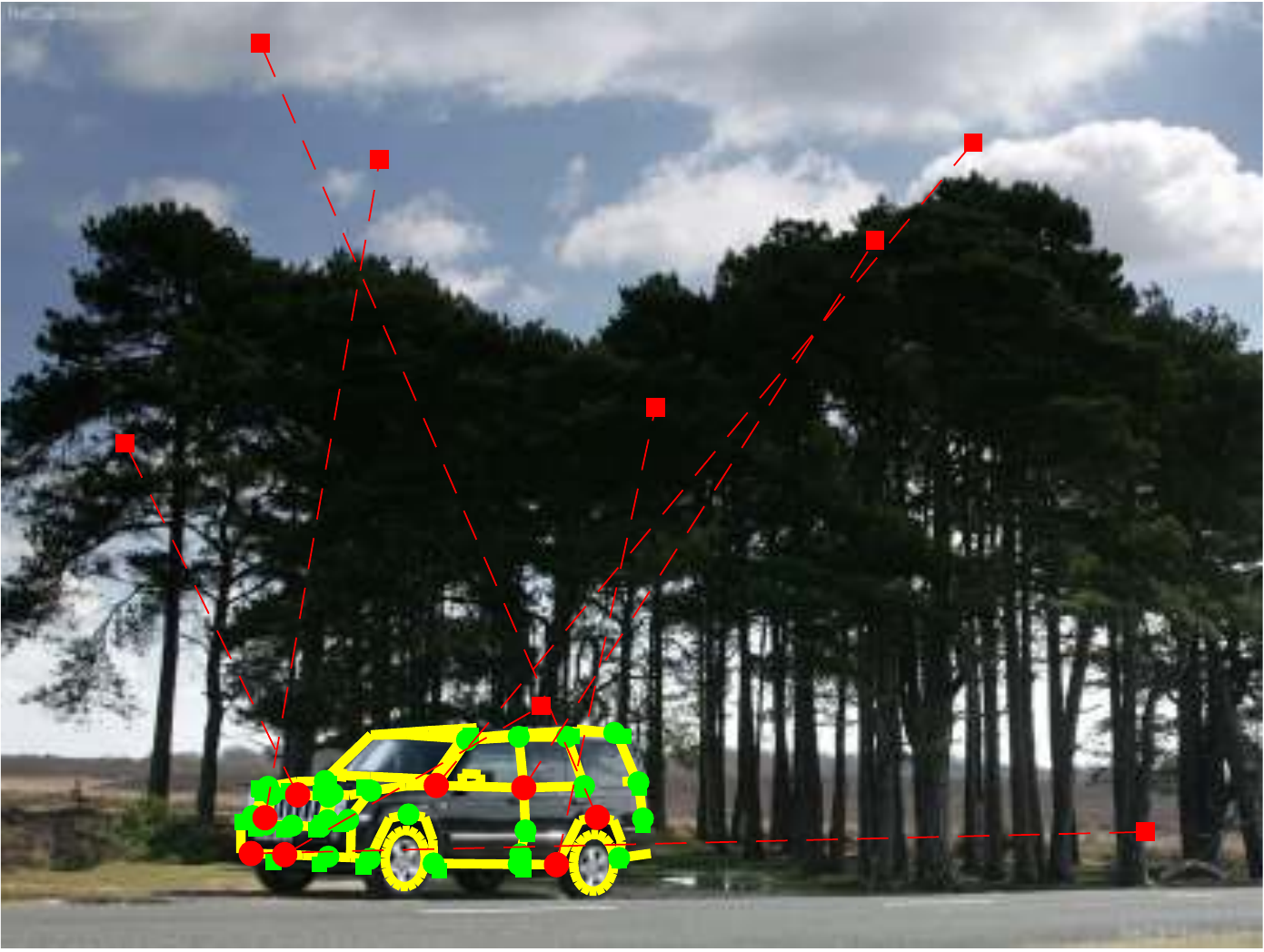} \\
	\vspace{1mm}
	\end{minipage}
& \myhspace
	\begin{minipage}{\mpwthree}%
	\centering%
	\includegraphics[width=\columnwidth]{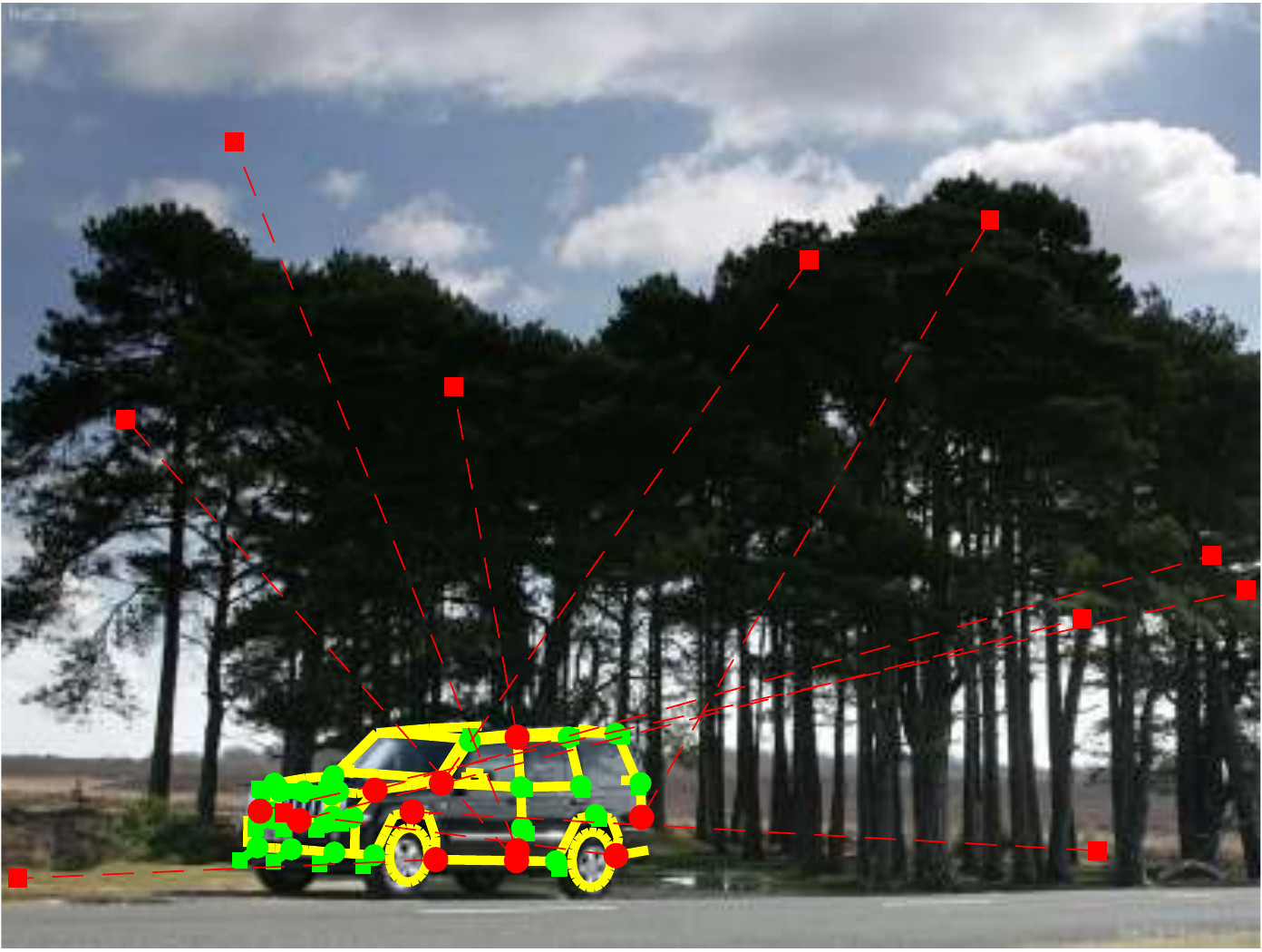} \\
	\vspace{1mm}
	\end{minipage} 
& \myhspace
	\begin{minipage}{\mpwthree}%
	\centering%
	\includegraphics[width=\columnwidth]{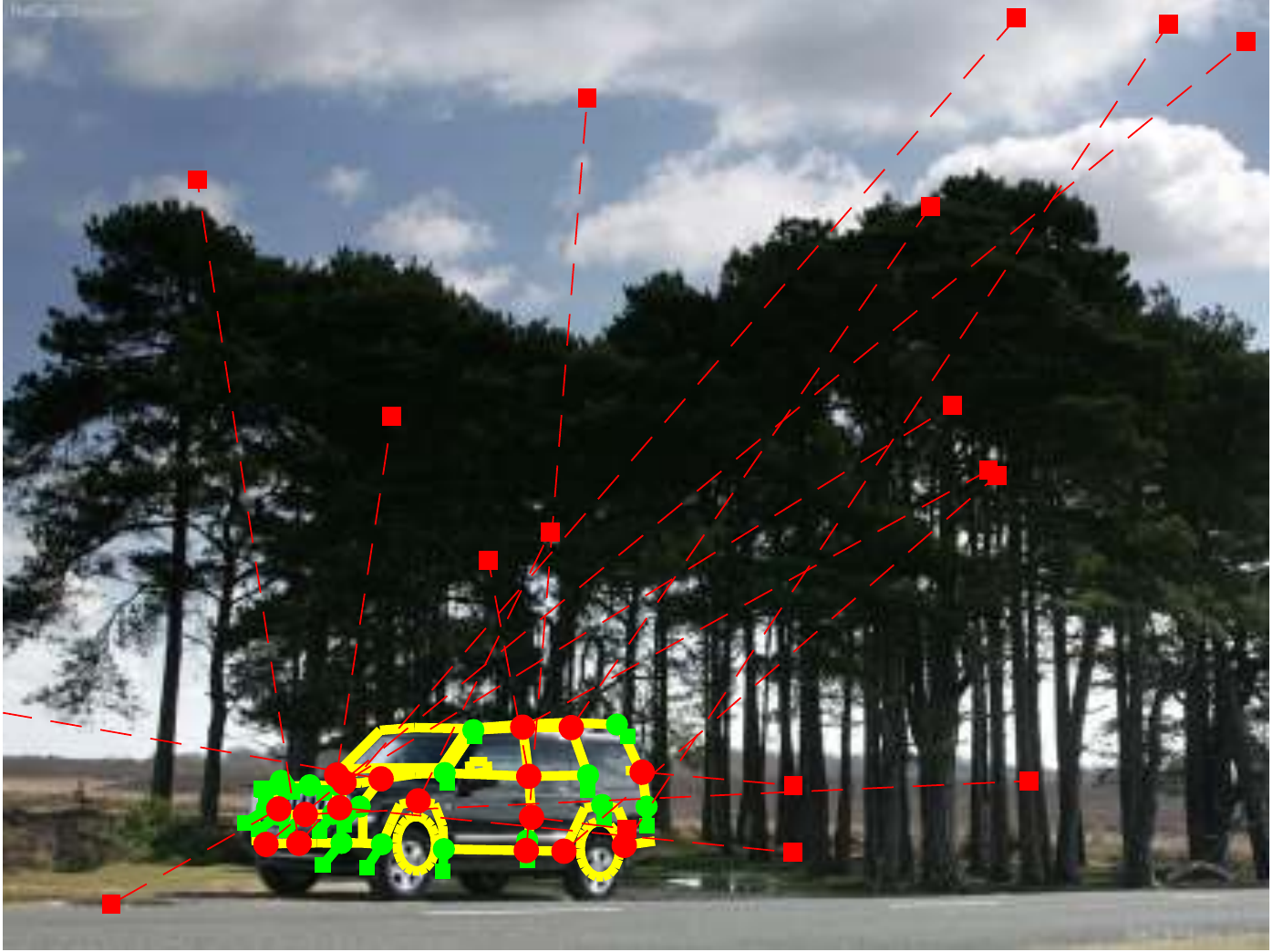} \\
	\vspace{1mm}
	\end{minipage}
& \myhspace
	\begin{minipage}{\mpwthree}%
	\centering%
	\includegraphics[width=\columnwidth]{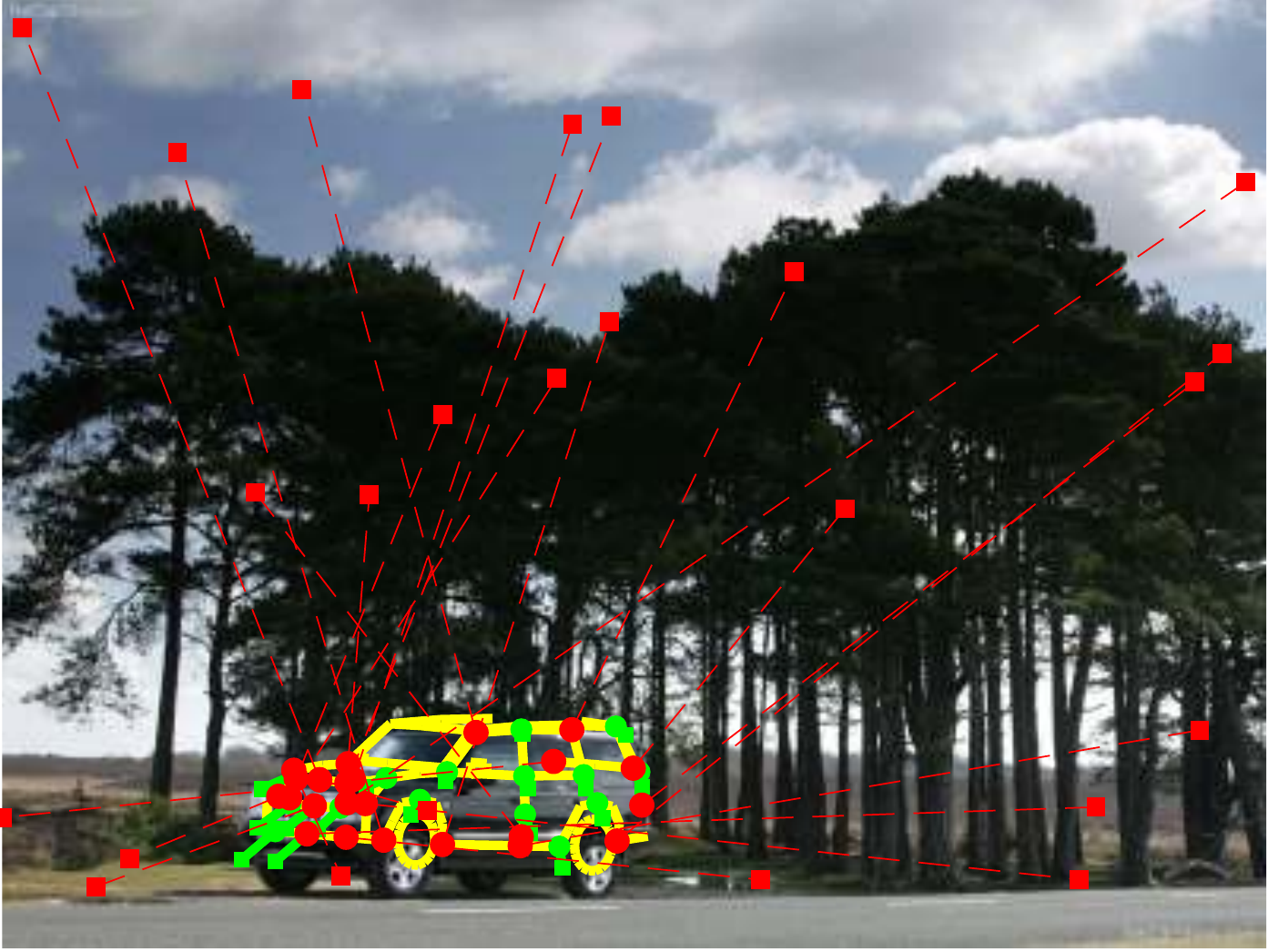} \\
	\vspace{1mm}
	\end{minipage}
&  \myhspace
	\begin{minipage}{\mpwthree}%
	\centering%
	\includegraphics[width=\columnwidth]{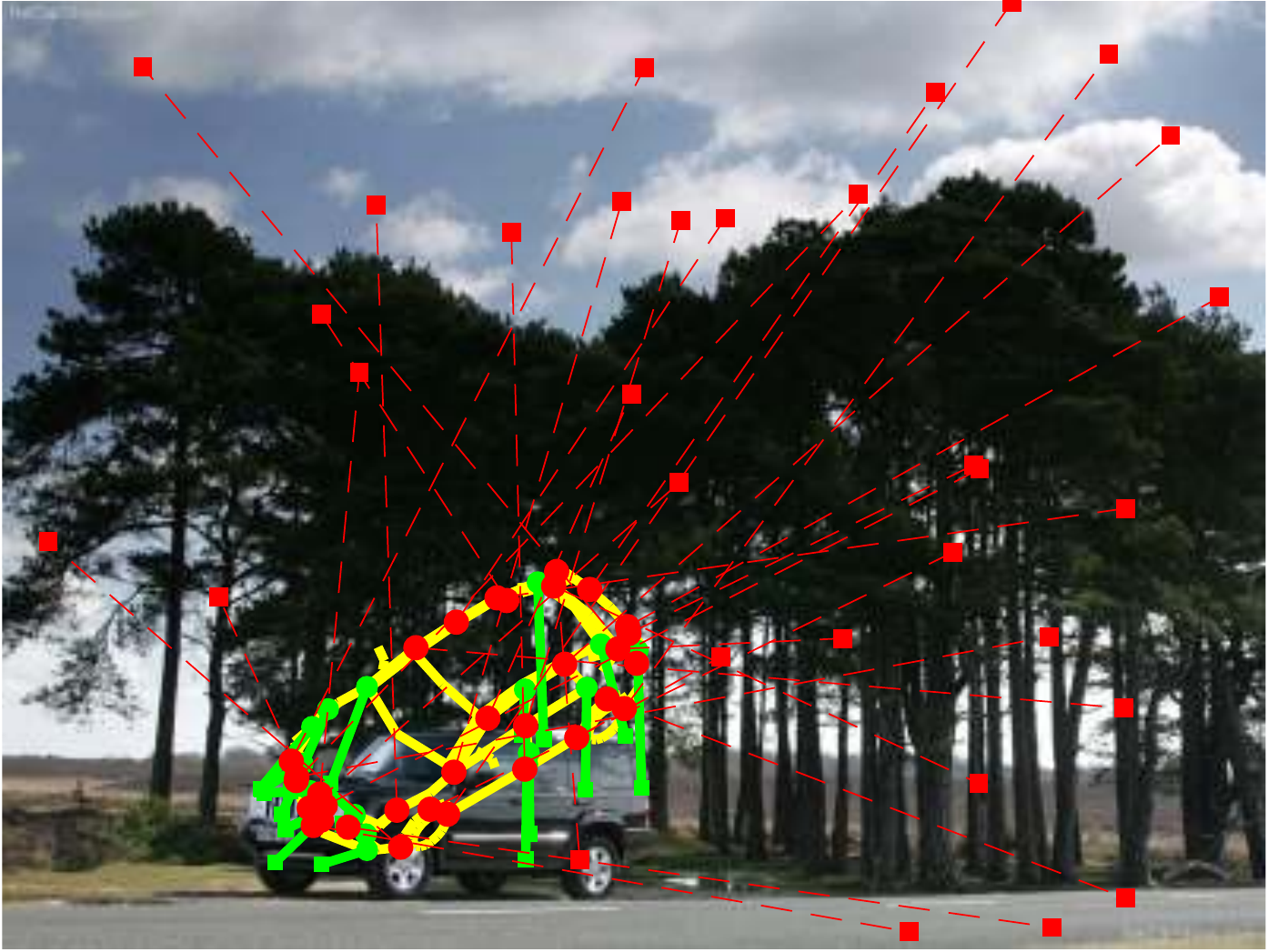} \\
	\vspace{1mm}
	\end{minipage} 
&  \myhspace
	\begin{minipage}{\mpwthree}%
	\centering%
	\includegraphics[width=\columnwidth]{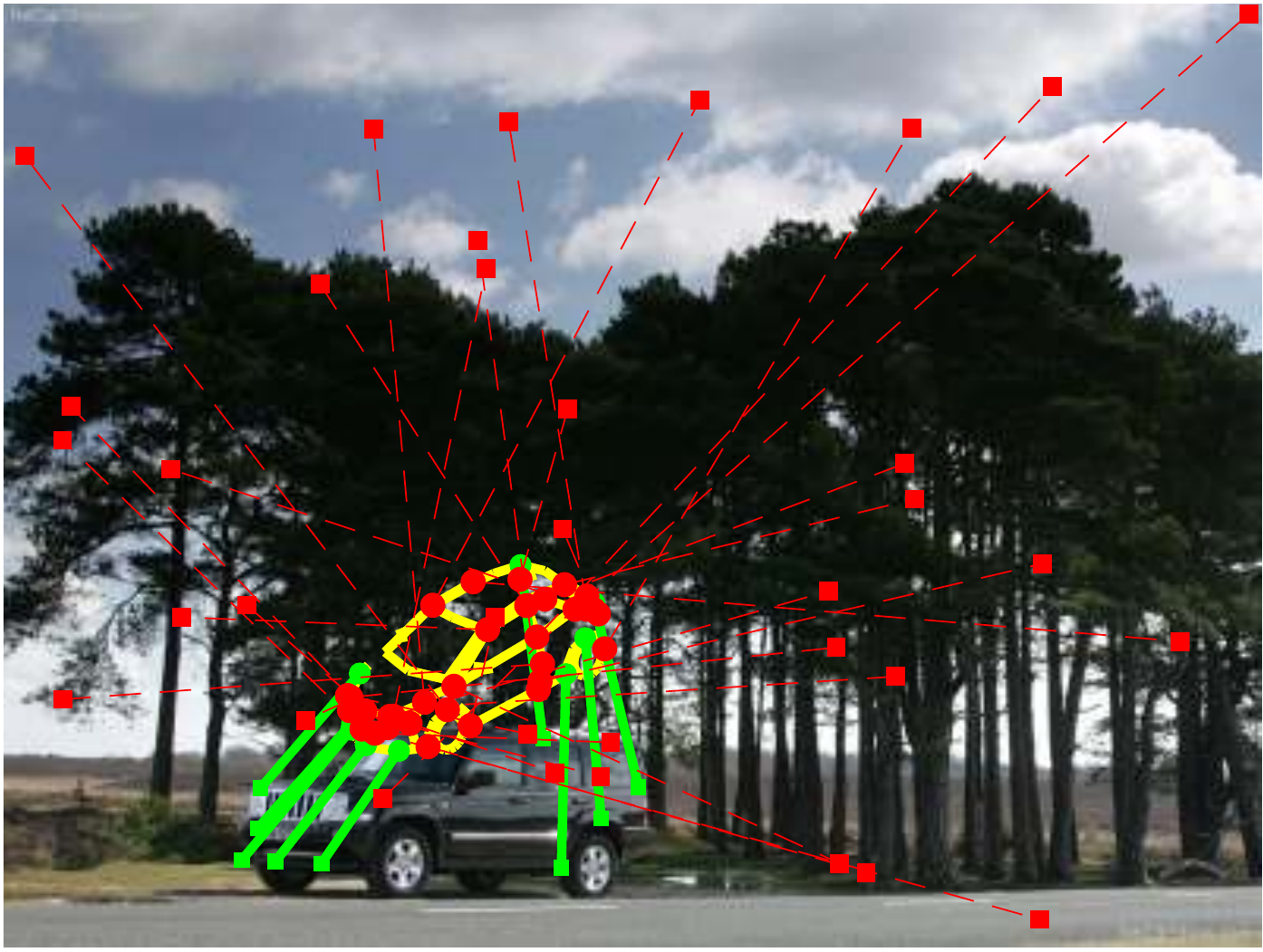} \\
	\vspace{1mm}
	\end{minipage} \\
\myhspace \myhspace \hspace{-2mm} \rotatebox{90}{\hspace{-3mm} {\smaller \blue{ \namerobust}} } & 
\myhspace
	\begin{minipage}{\mpwthree}%
	\centering%
	\includegraphics[width=\columnwidth]{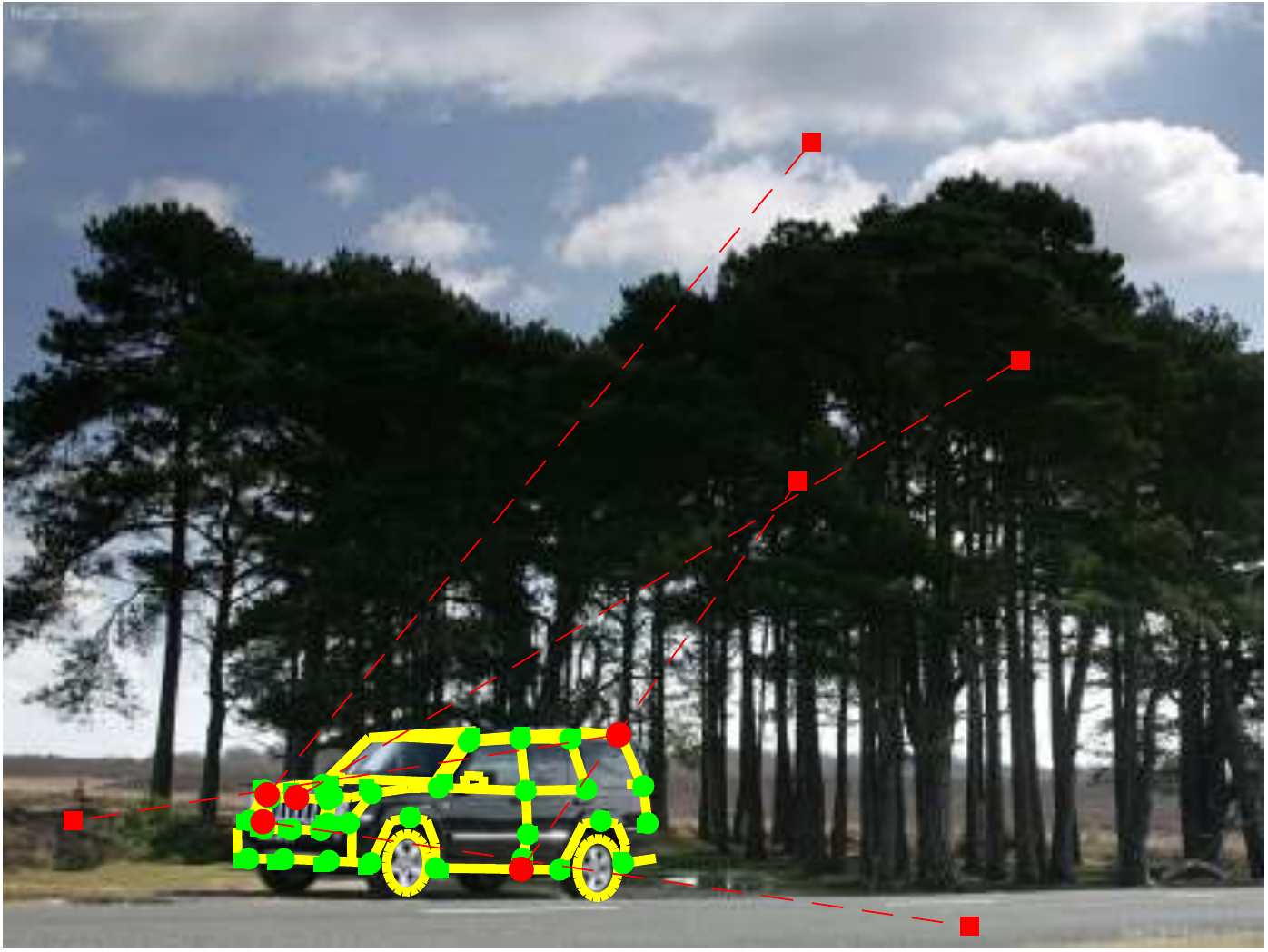} \\
	\vspace{1mm}
	\end{minipage}
& \myhspace
	\begin{minipage}{\mpwthree}%
	\centering%
	\includegraphics[width=\columnwidth]{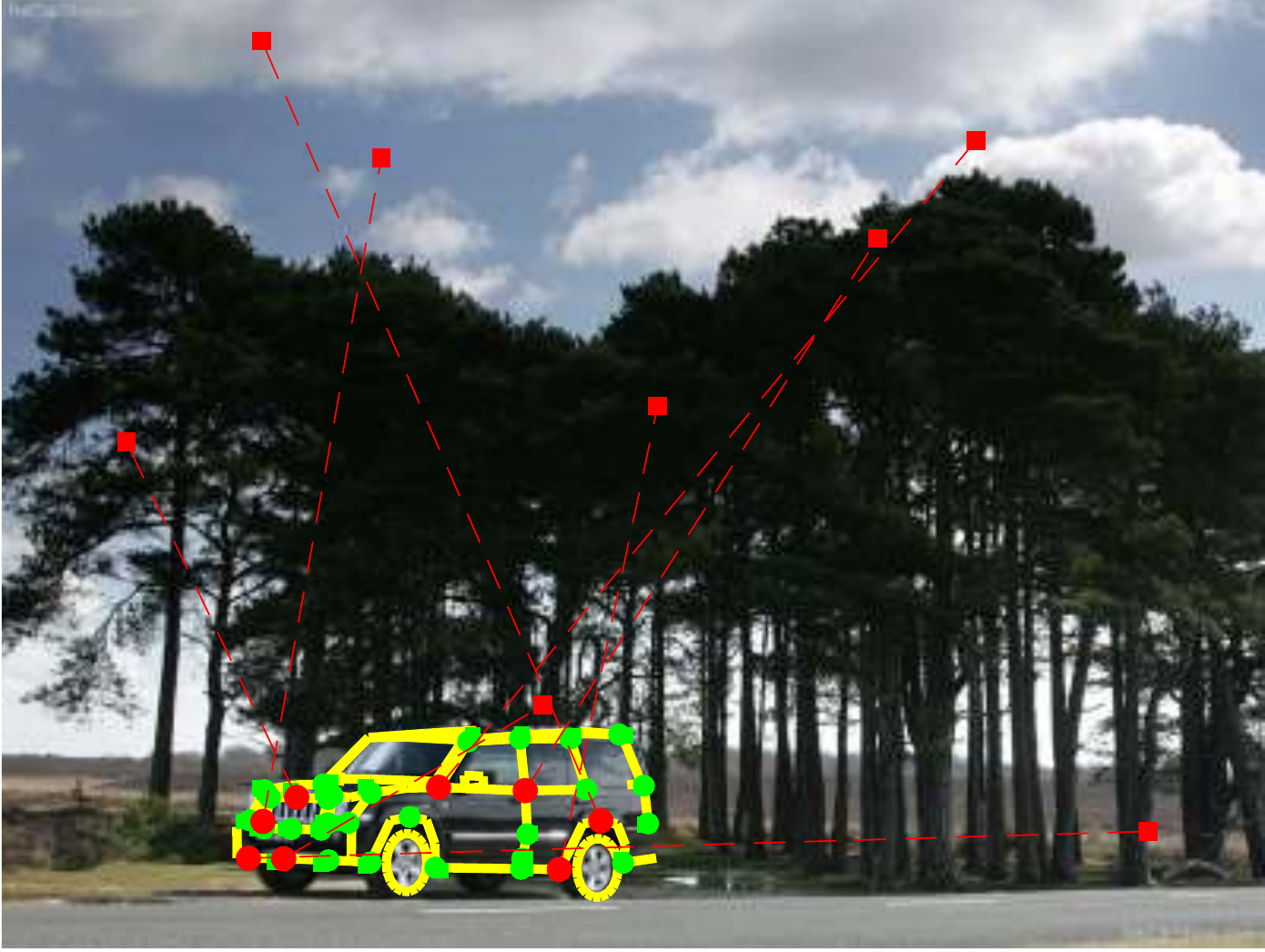} \\
	\vspace{1mm}
	\end{minipage}
& \myhspace
	\begin{minipage}{\mpwthree}%
	\centering%
	\includegraphics[width=\columnwidth]{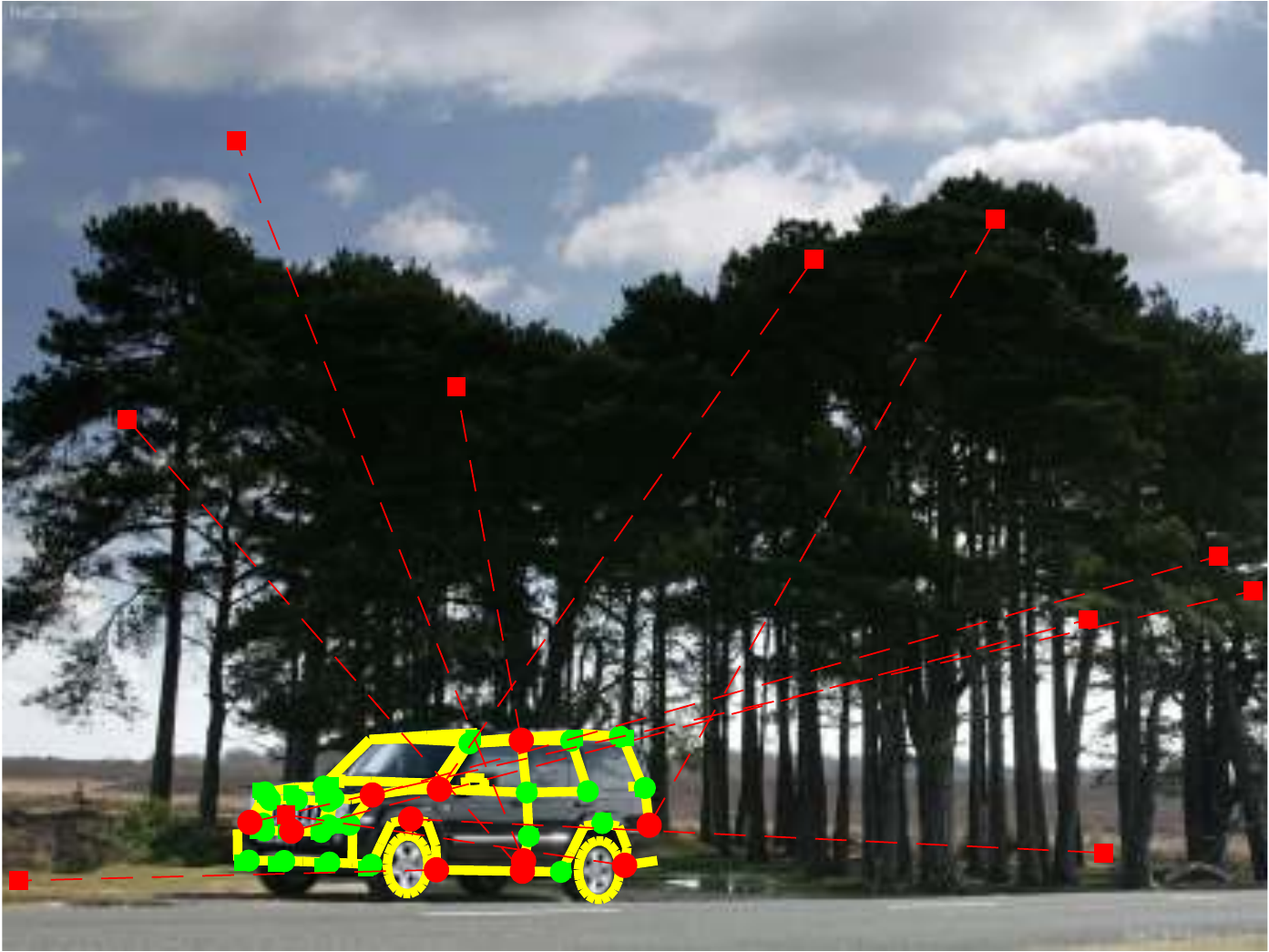} \\
	\vspace{1mm}
	\end{minipage} 
& \myhspace
	\begin{minipage}{\mpwthree}%
	\centering%
	\includegraphics[width=\columnwidth]{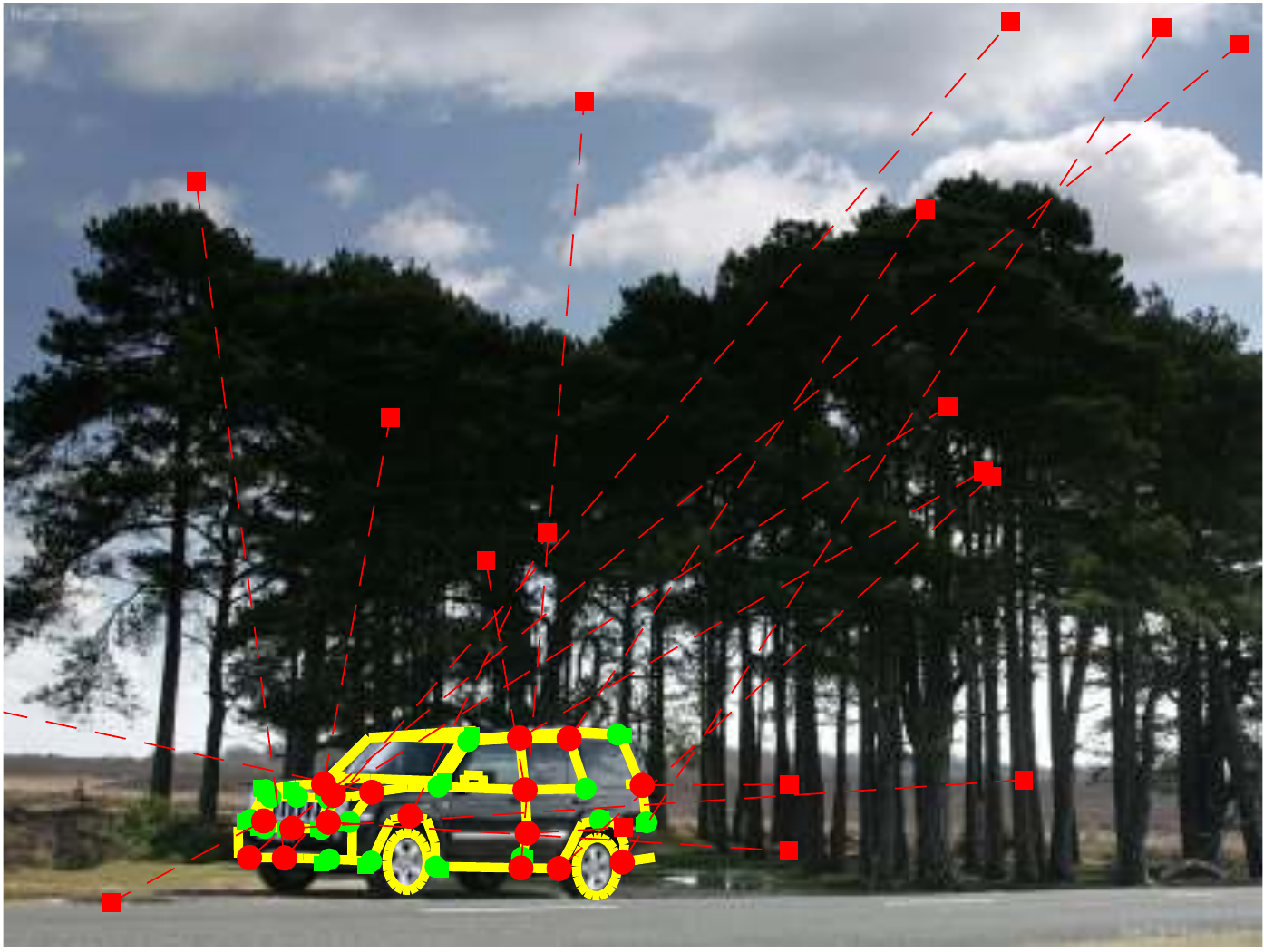} \\
	\vspace{1mm}
	\end{minipage}
& \myhspace
	\begin{minipage}{\mpwthree}%
	\centering%
	\includegraphics[width=\columnwidth]{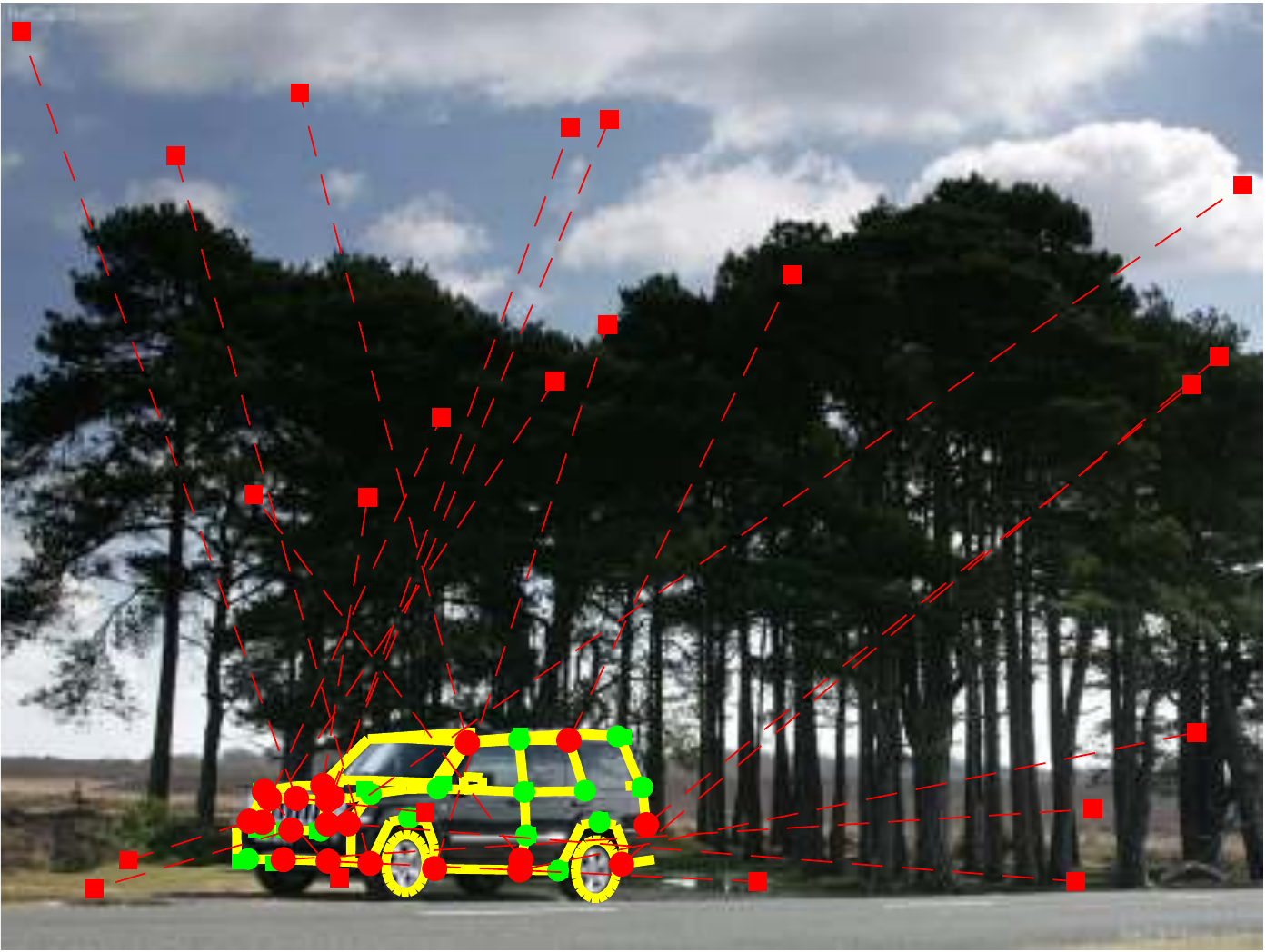} \\
	\vspace{1mm}
	\end{minipage}
&  \myhspace
	\begin{minipage}{\mpwthree}%
	\centering%
	\includegraphics[width=\columnwidth]{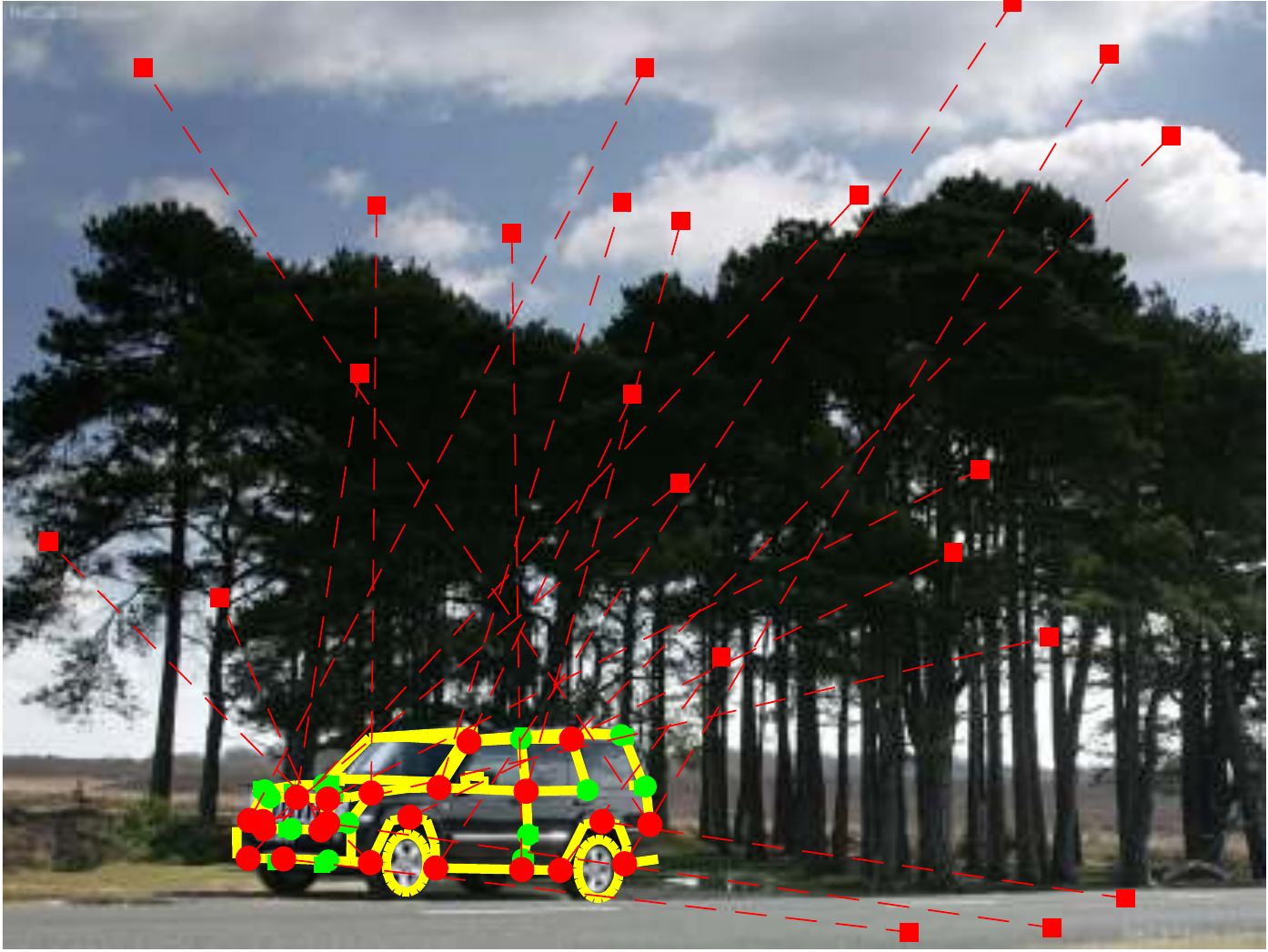} \\
	\vspace{1mm}
	\end{minipage} 
&  \myhspace
	\begin{minipage}{\mpwthree}%
	\centering%
	\includegraphics[width=\columnwidth]{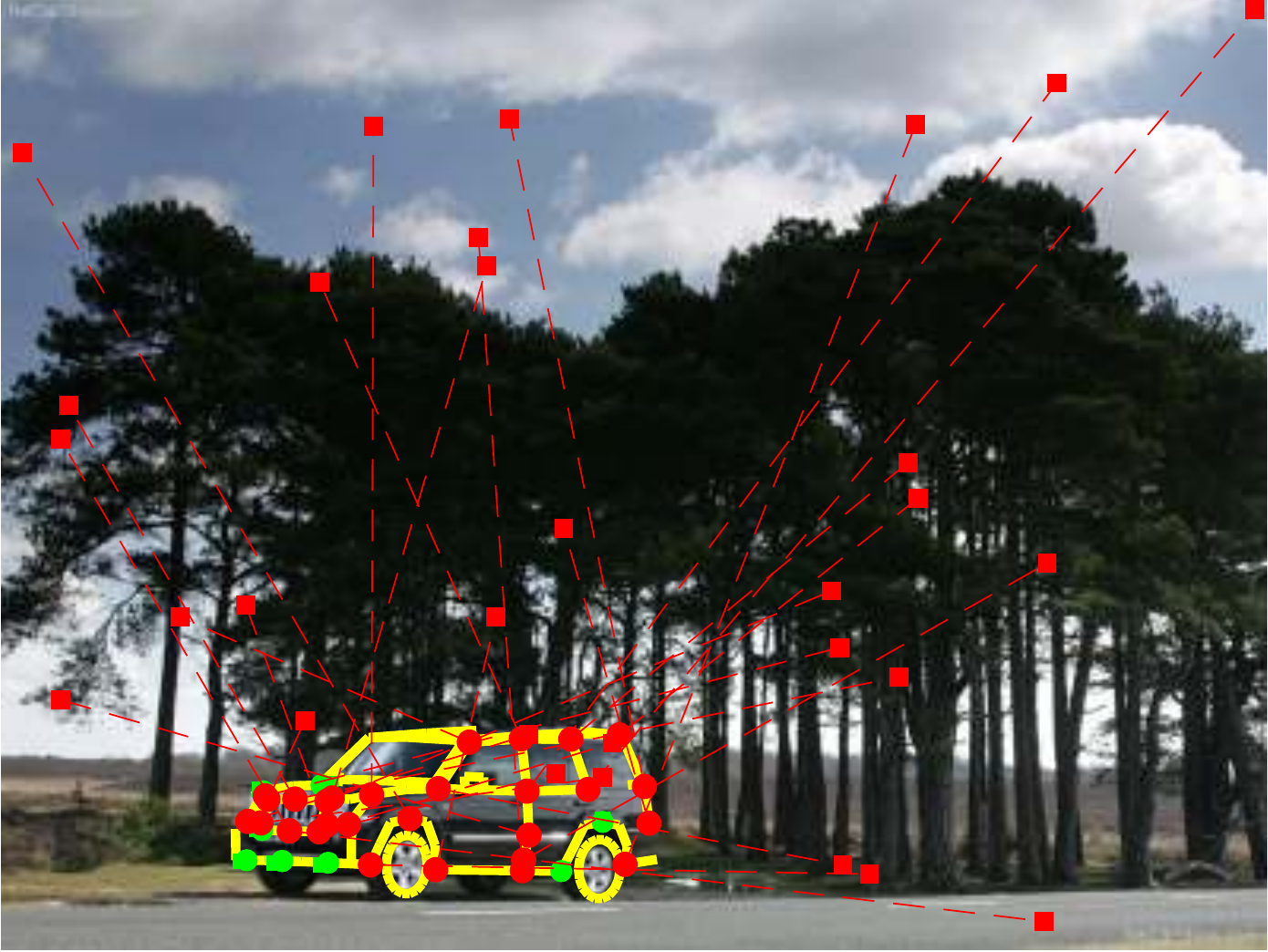} \\
	\vspace{1mm}
	\end{minipage} \\
\multicolumn{8}{c}{Jeep Commander } \\

%% file: 097-Honda_CRV.tex
\myhspace \myhspace \hspace{-2mm} \rotatebox{90}{\hspace{-7mm} {\smaller \alternrobust} } & 
\myhspace
	\begin{minipage}{\mpwthree}%
	\centering%
	\includegraphics[width=\columnwidth]{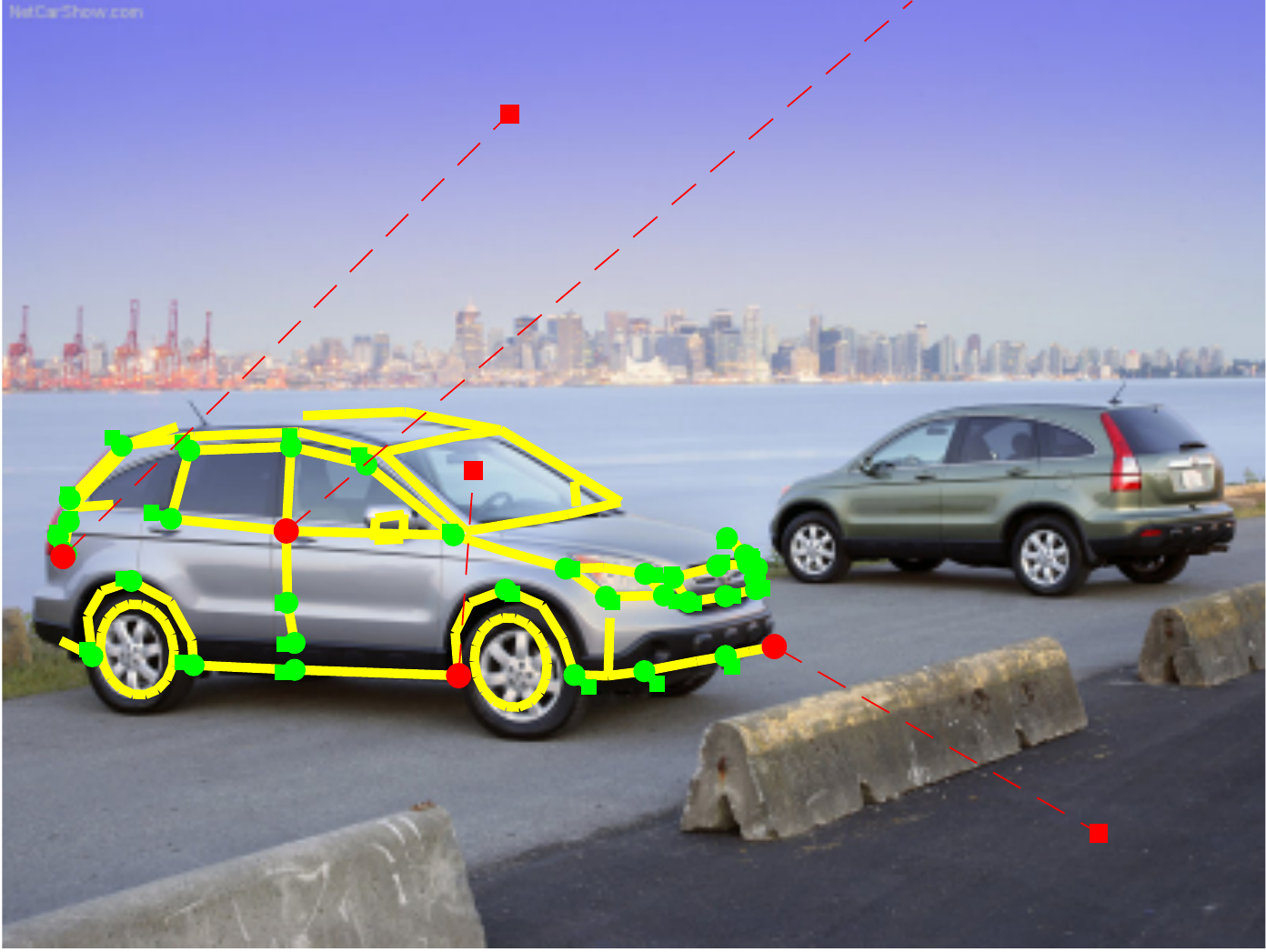} \\
	\vspace{1mm}
	\end{minipage}
& \myhspace
	\begin{minipage}{\mpwthree}%
	\centering%
	\includegraphics[width=\columnwidth]{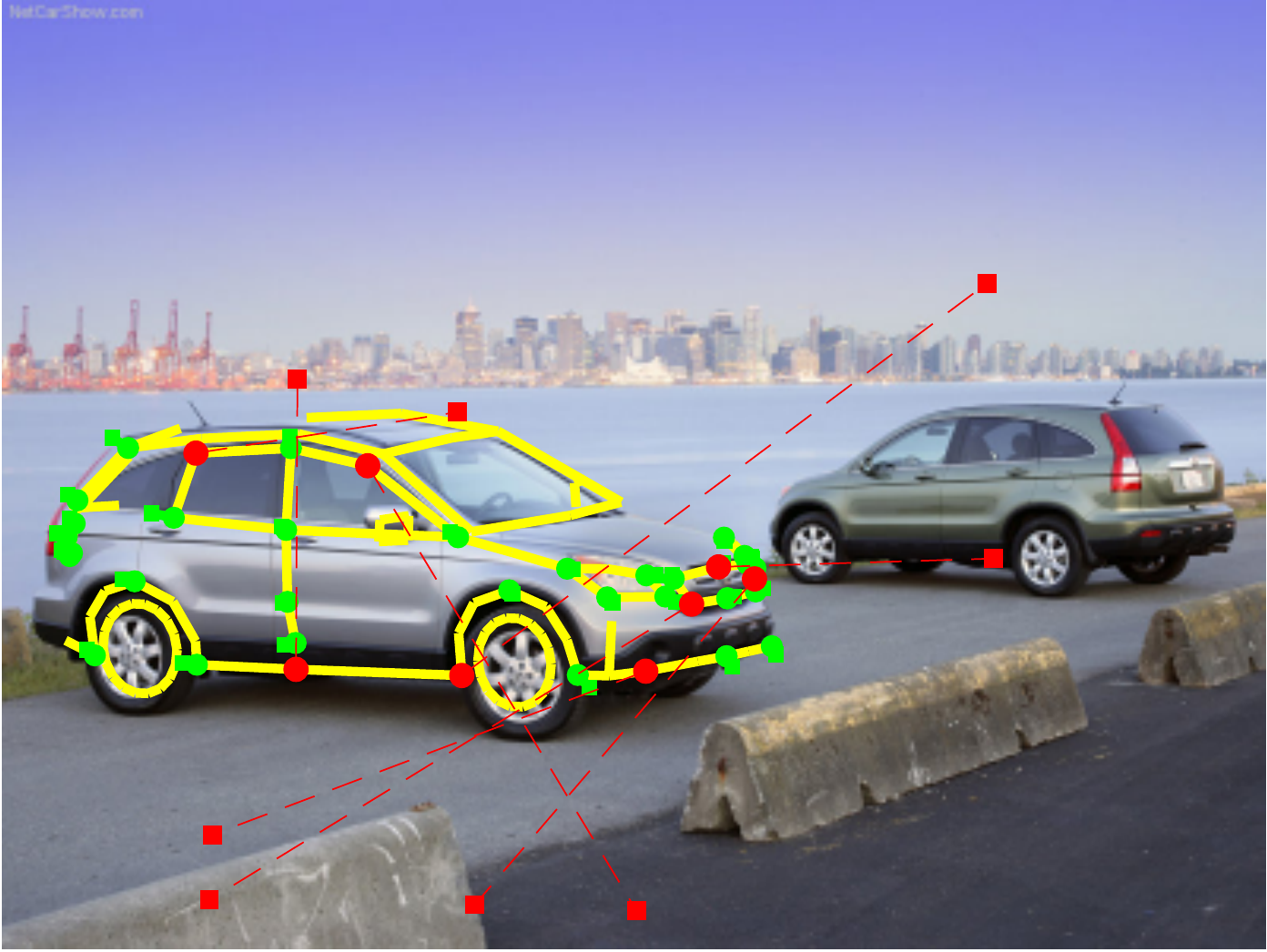} \\
	\vspace{1mm}
	\end{minipage}
& \myhspace
	\begin{minipage}{\mpwthree}%
	\centering%
	\includegraphics[width=\columnwidth]{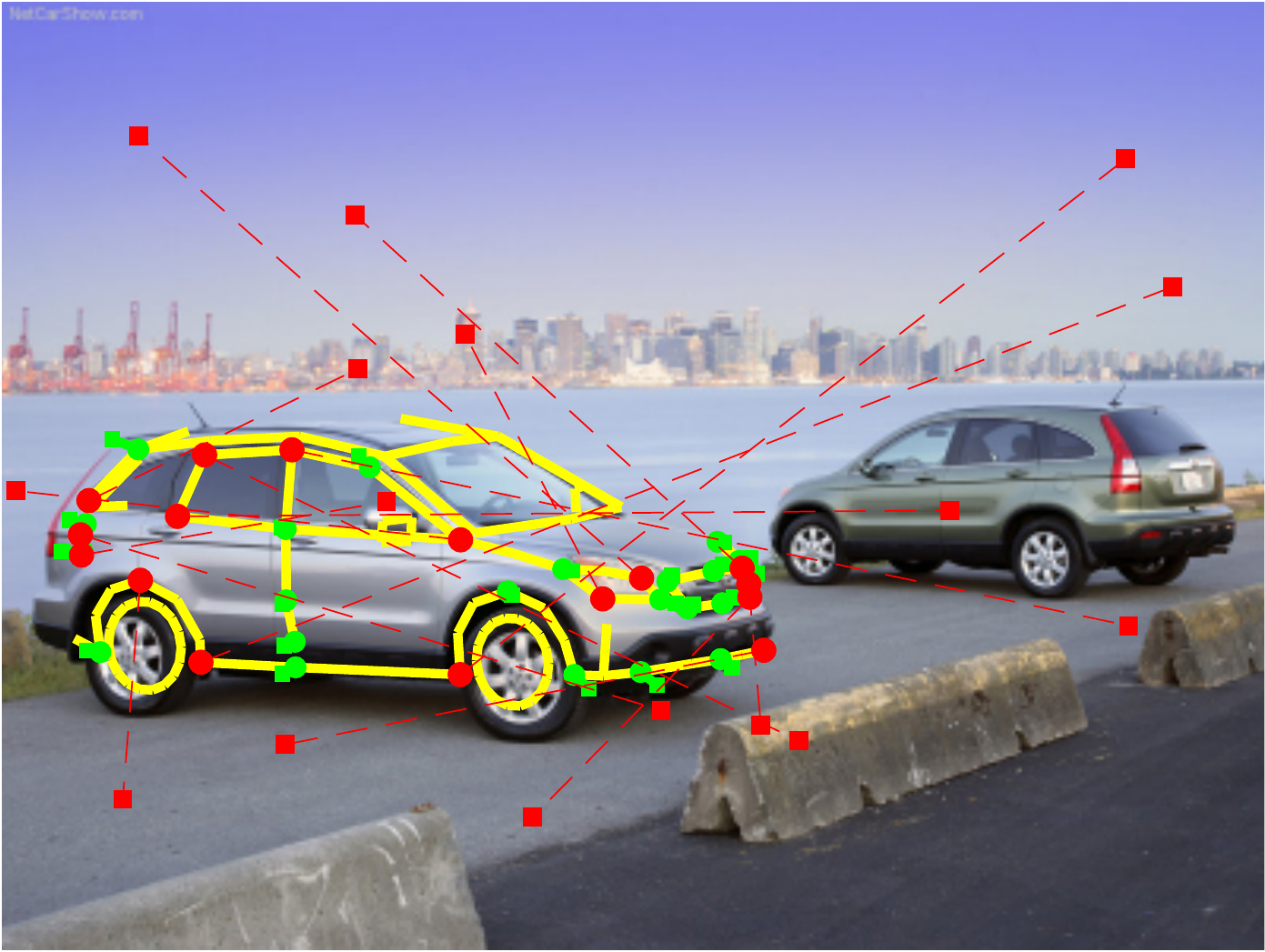} \\
	\vspace{1mm}
	\end{minipage} 
& \myhspace
	\begin{minipage}{\mpwthree}%
	\centering%
	\includegraphics[width=\columnwidth]{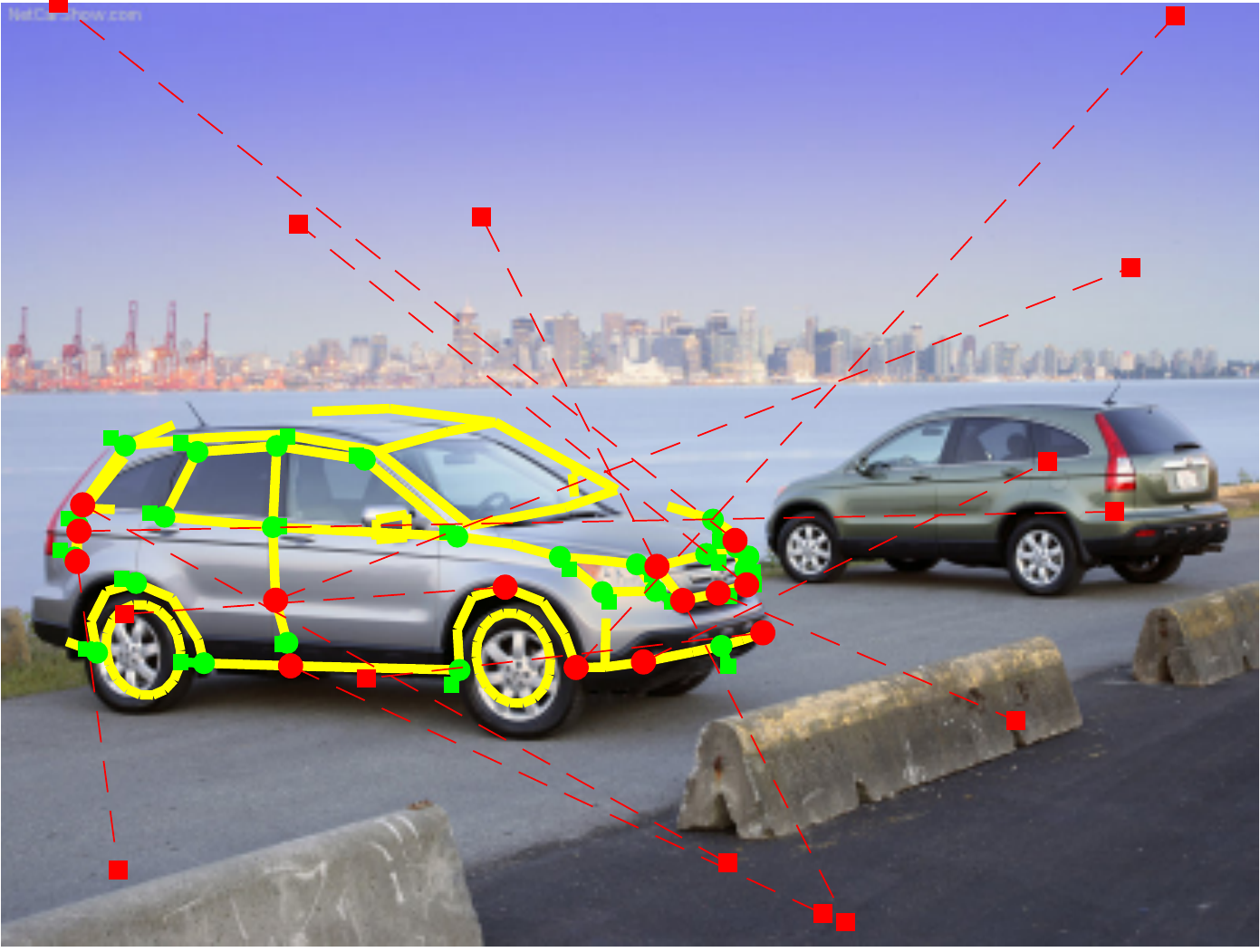} \\
	\vspace{1mm}
	\end{minipage}
& \myhspace
	\begin{minipage}{\mpwthree}%
	\centering%
	\includegraphics[width=\columnwidth]{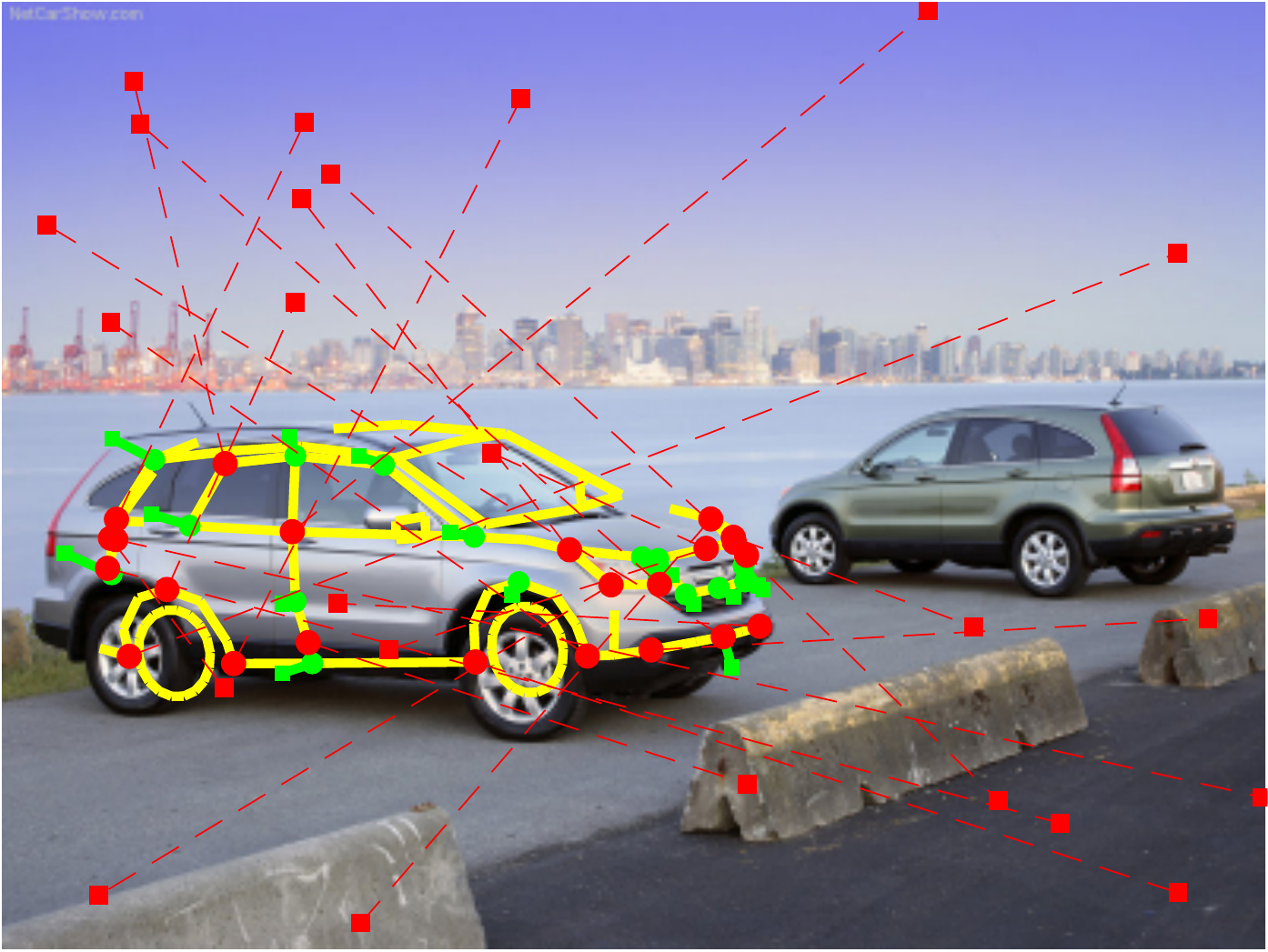} \\
	\vspace{1mm}
	\end{minipage}
&  \myhspace
	\begin{minipage}{\mpwthree}%
	\centering%
	\includegraphics[width=\columnwidth]{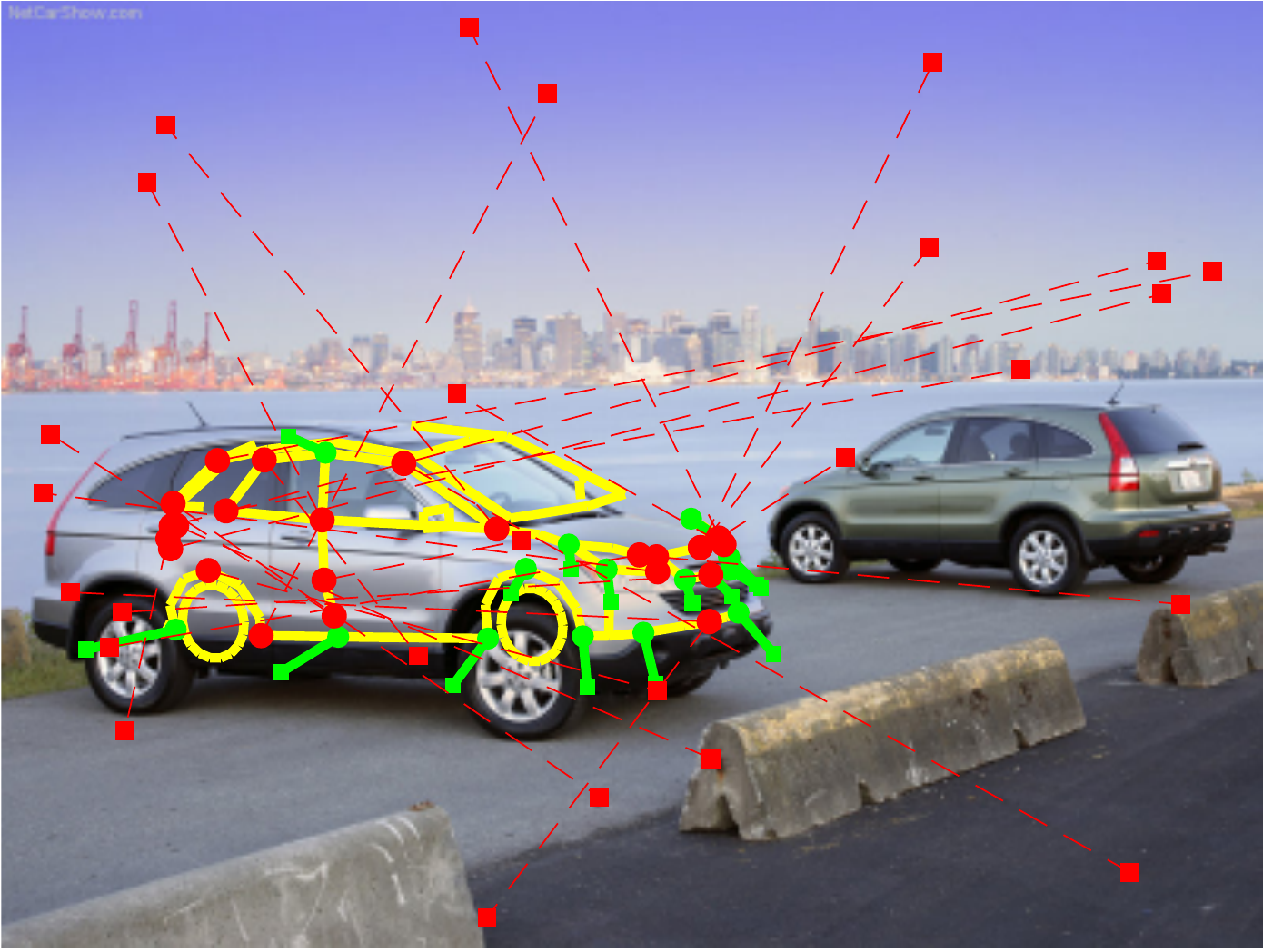} \\
	\vspace{1mm}
	\end{minipage} 
&  \myhspace
	\begin{minipage}{\mpwthree}%
	\centering%
	\includegraphics[width=\columnwidth]{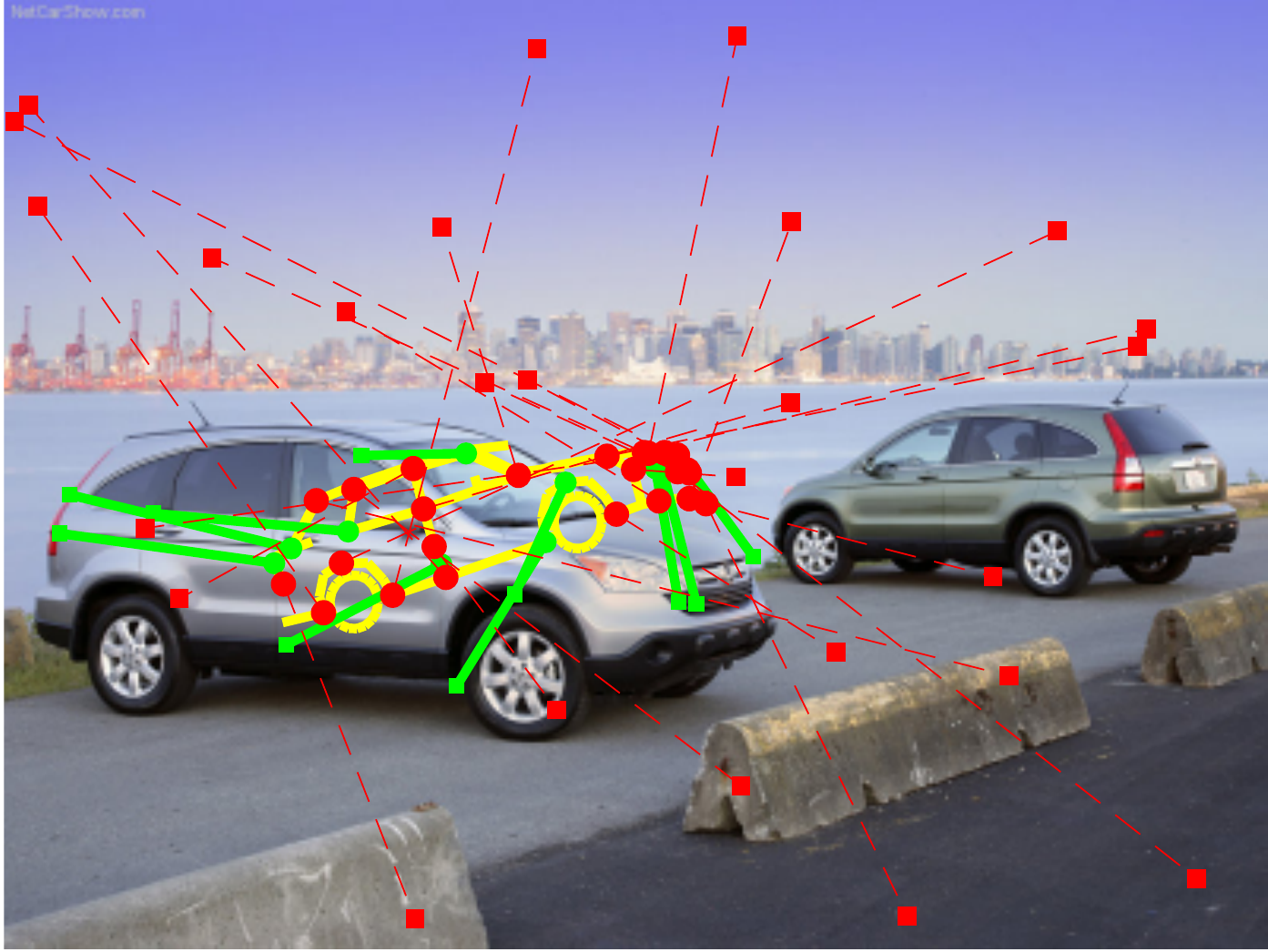} \\
	\vspace{1mm}
	\end{minipage} \\
\myhspace \myhspace \hspace{-2mm} \rotatebox{90}{\hspace{-8mm} {\smaller \convexrobust} } & 
\myhspace
	\begin{minipage}{\mpwthree}%
	\centering%
	\includegraphics[width=\columnwidth]{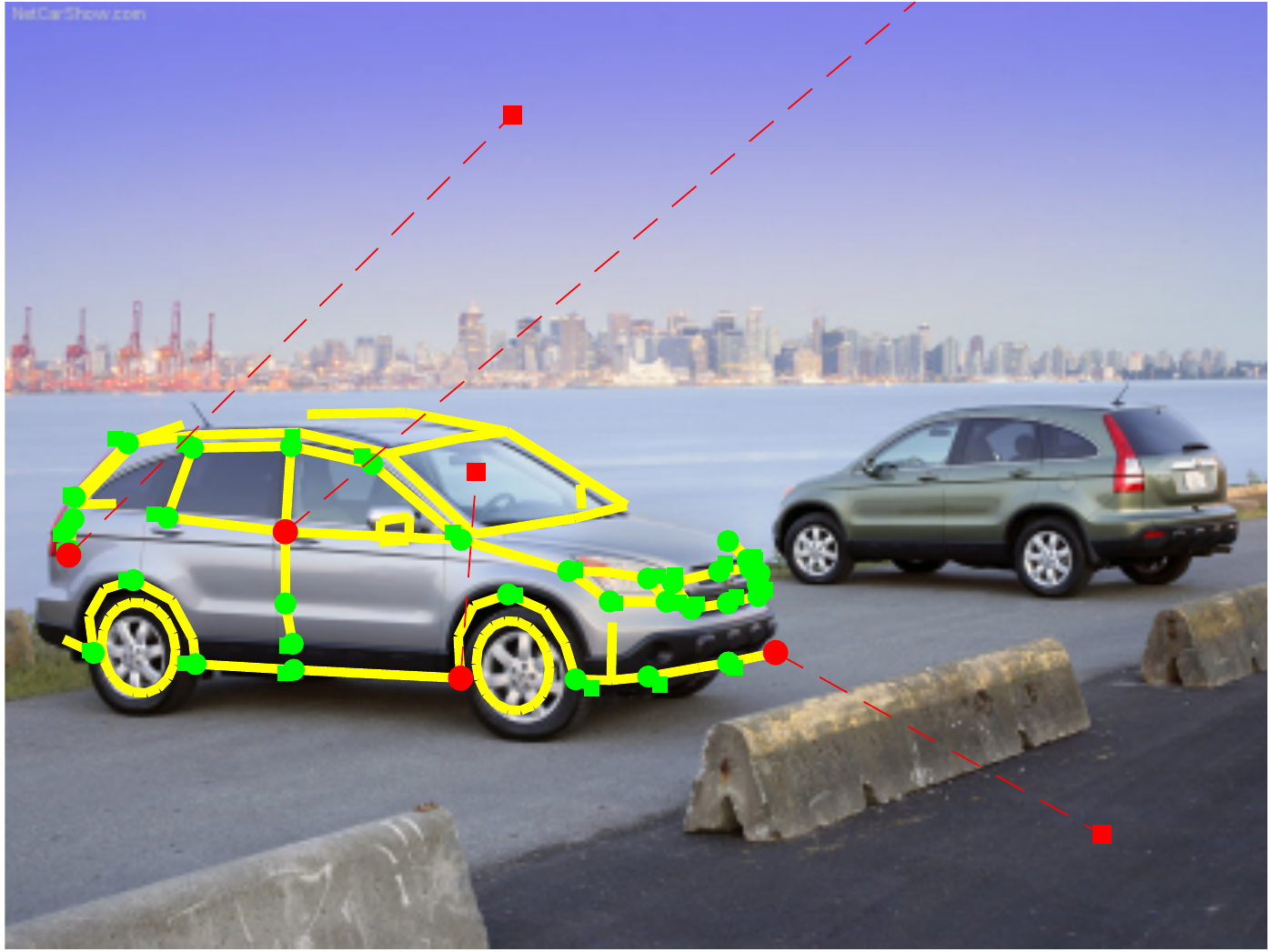} \\
	\vspace{1mm}
	\end{minipage}
& \myhspace
	\begin{minipage}{\mpwthree}%
	\centering%
	\includegraphics[width=\columnwidth]{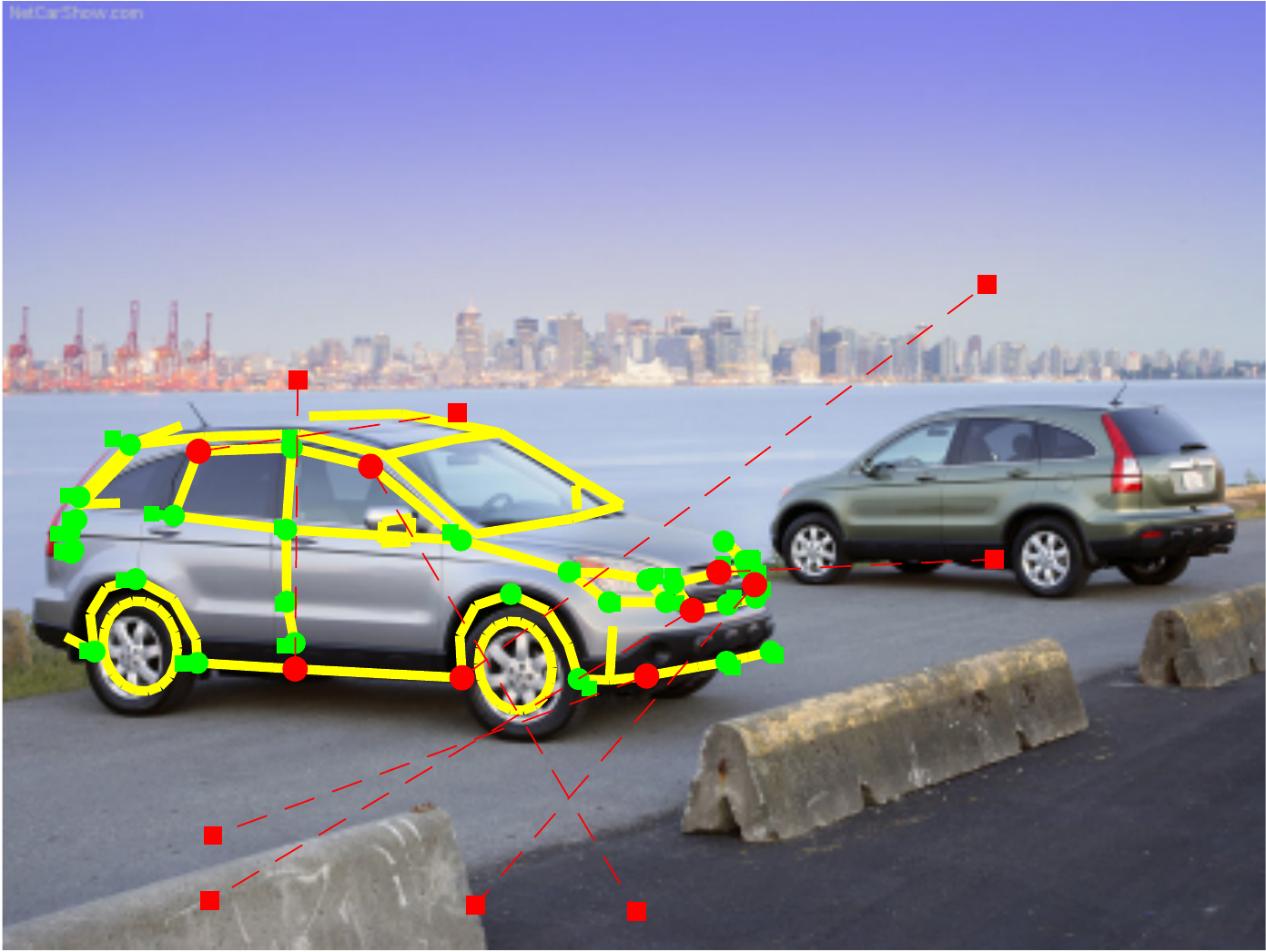} \\
	\vspace{1mm}
	\end{minipage}
& \myhspace
	\begin{minipage}{\mpwthree}%
	\centering%
	\includegraphics[width=\columnwidth]{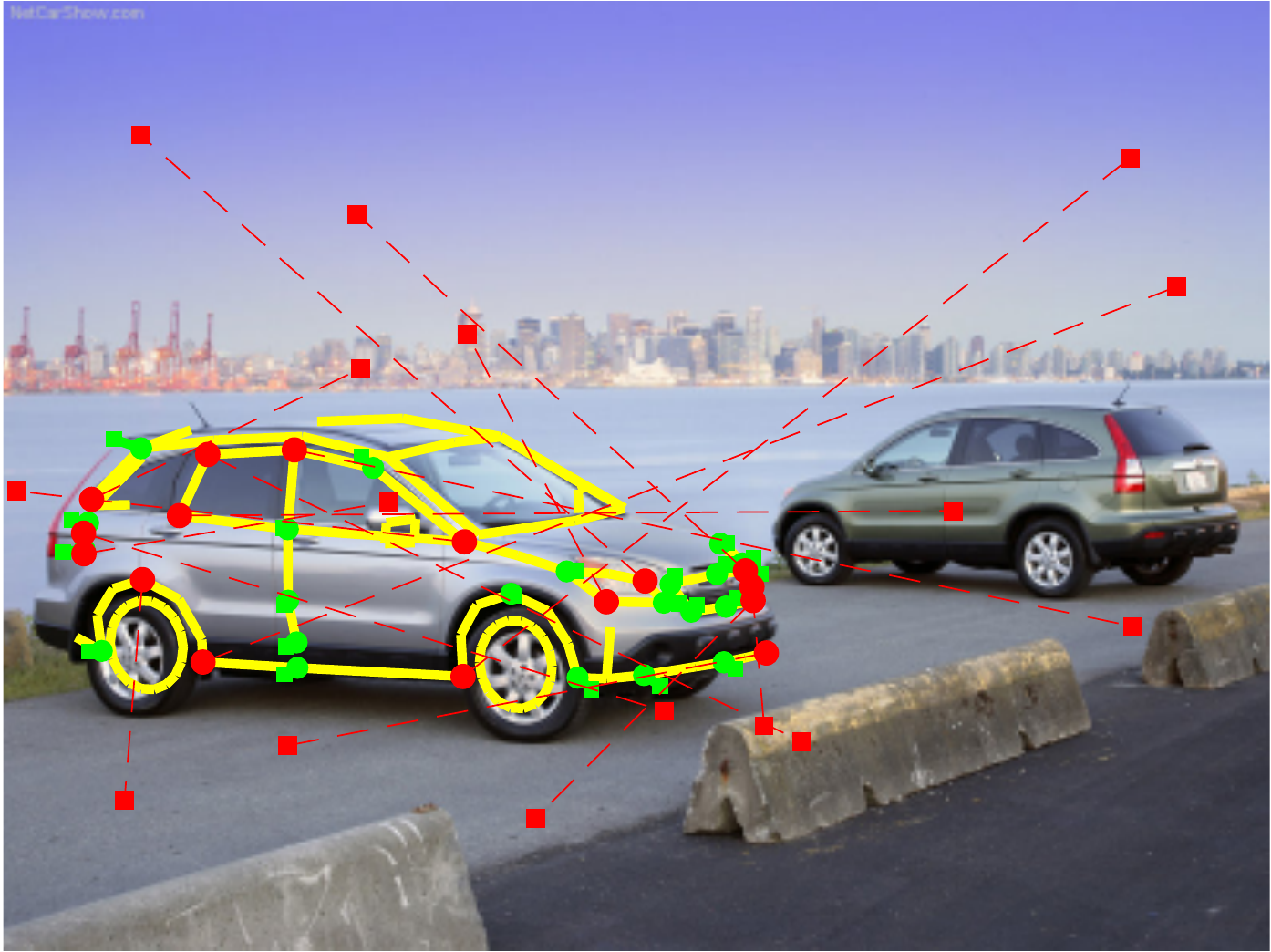} \\
	\vspace{1mm}
	\end{minipage} 
& \myhspace
	\begin{minipage}{\mpwthree}%
	\centering%
	\includegraphics[width=\columnwidth]{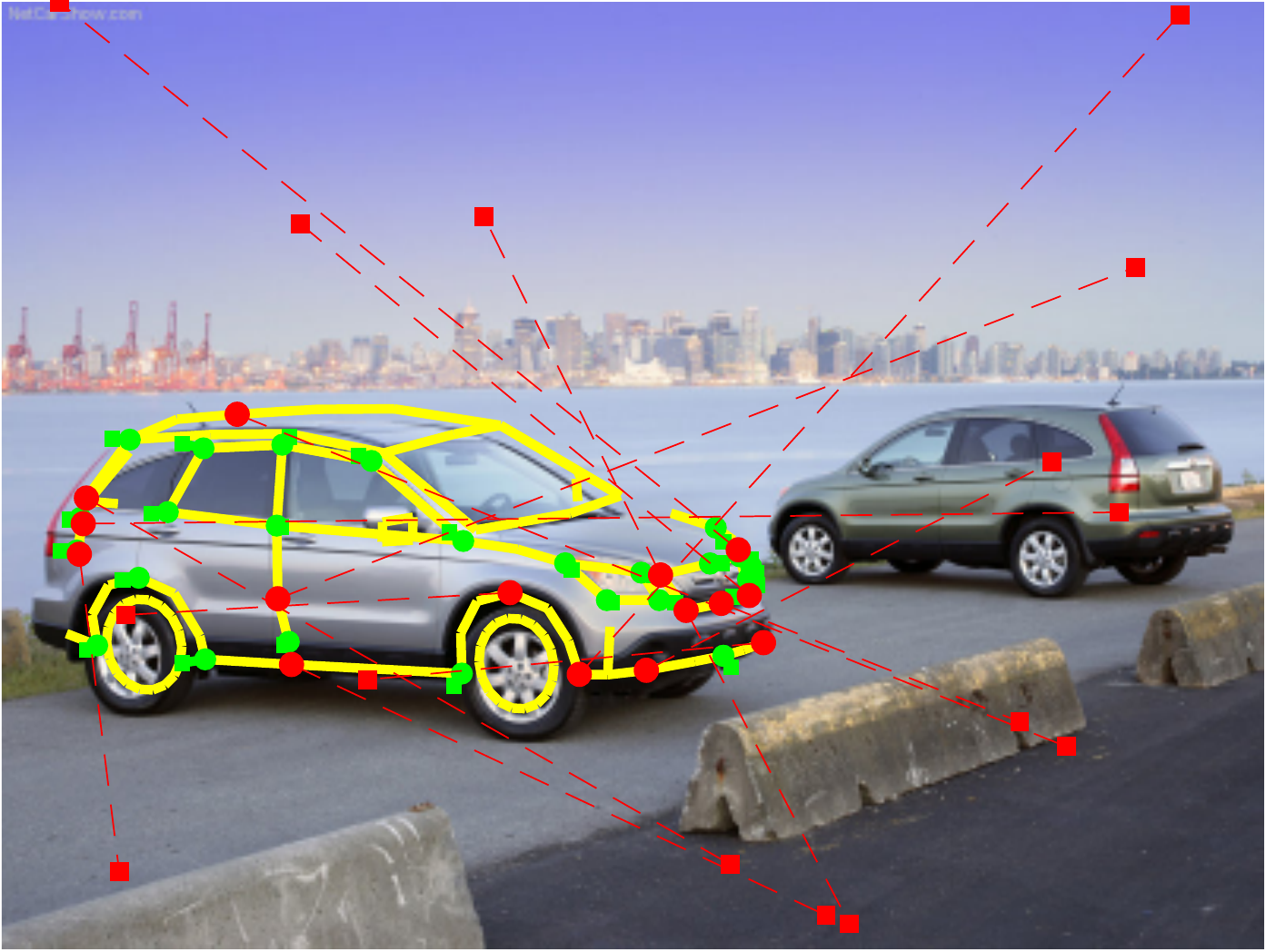} \\
	\vspace{1mm}
	\end{minipage}
& \myhspace
	\begin{minipage}{\mpwthree}%
	\centering%
	\includegraphics[width=\columnwidth]{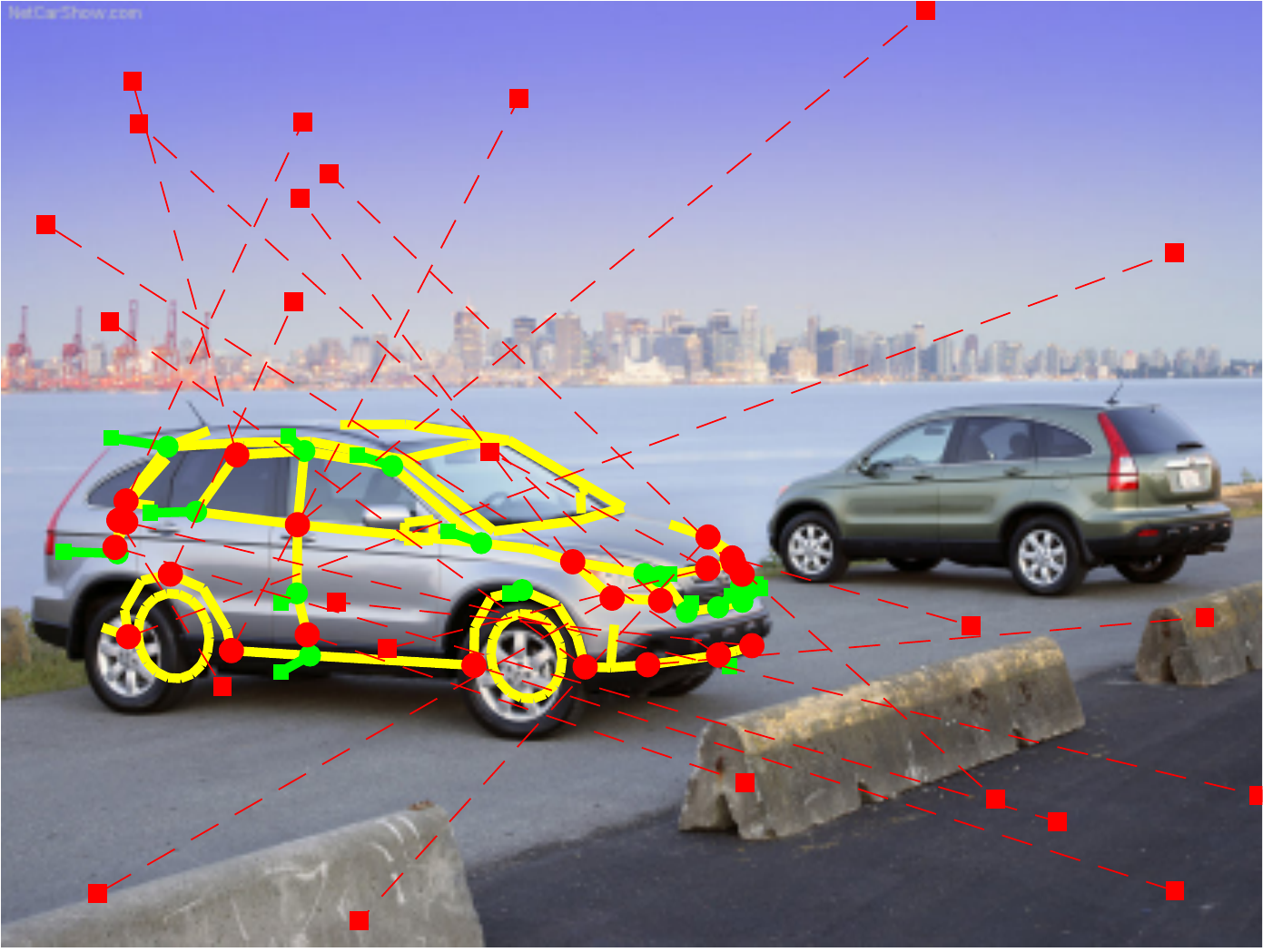} \\
	\vspace{1mm}
	\end{minipage}
&  \myhspace
	\begin{minipage}{\mpwthree}%
	\centering%
	\includegraphics[width=\columnwidth]{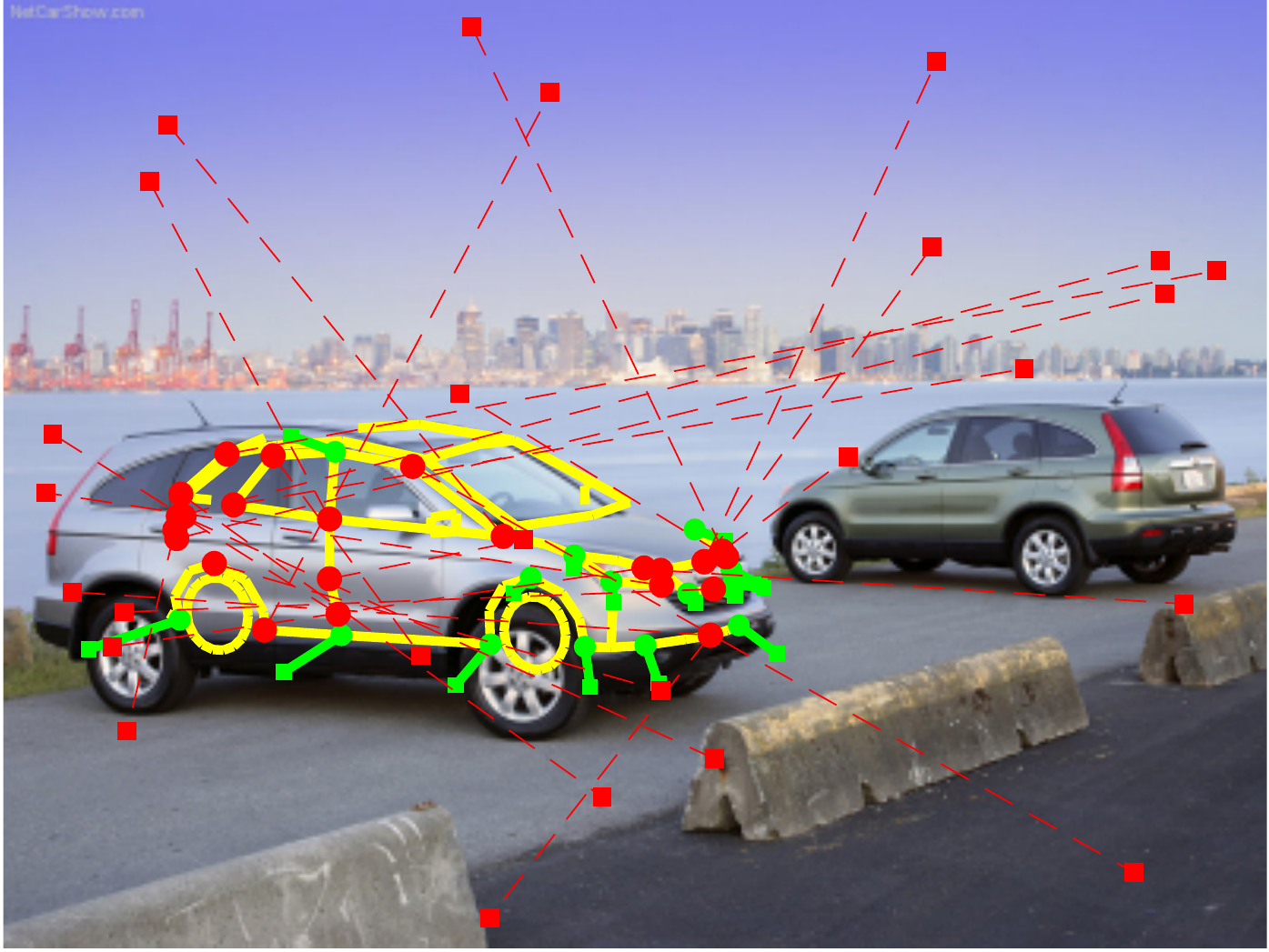} \\
	\vspace{1mm}
	\end{minipage} 
&  \myhspace
	\begin{minipage}{\mpwthree}%
	\centering%
	\includegraphics[width=\columnwidth]{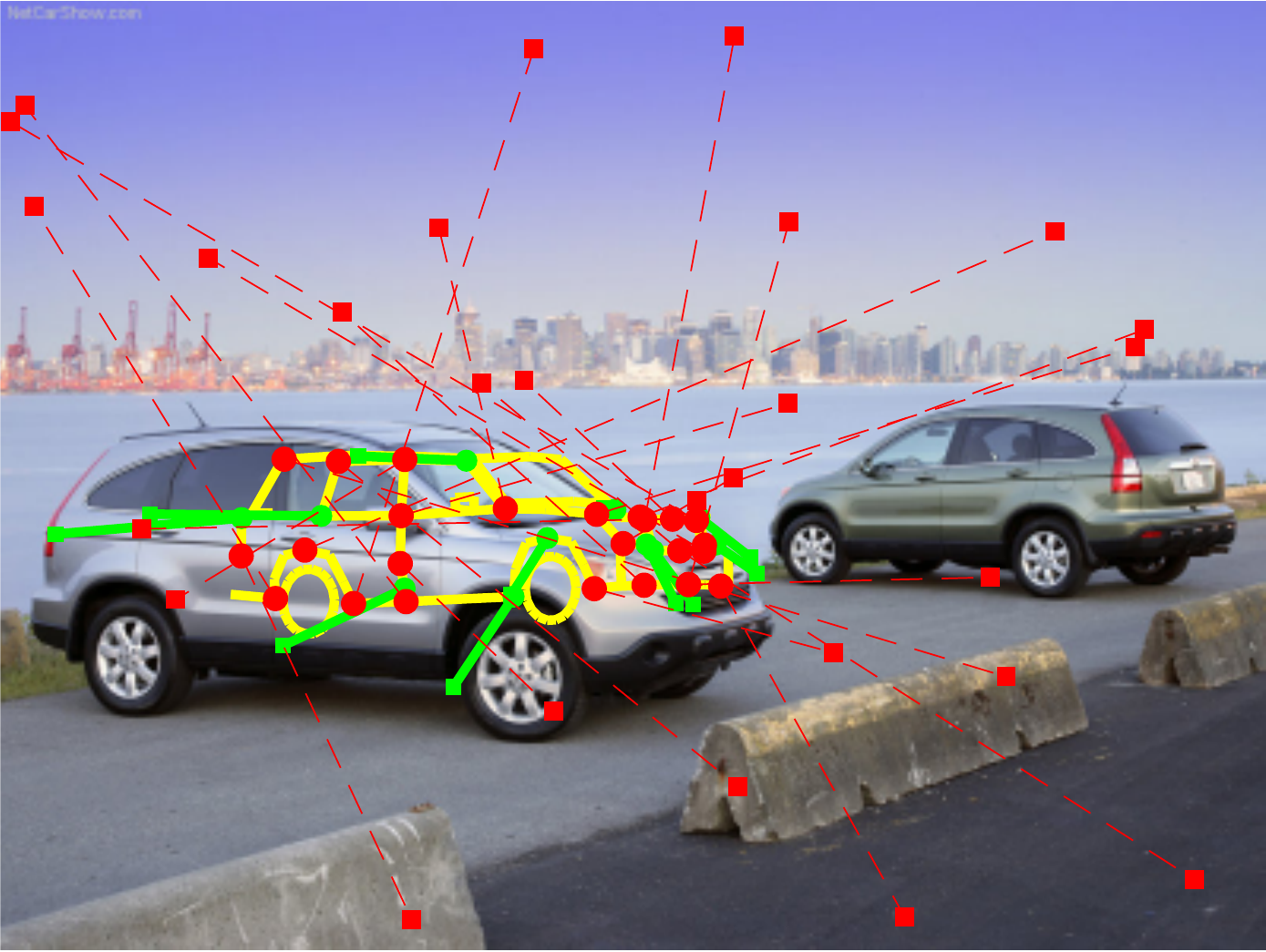} \\
	\vspace{1mm}
	\end{minipage} \\
\myhspace \myhspace \hspace{-2mm} \rotatebox{90}{\hspace{-3mm} {\smaller \blue{ \namerobust}} } & 
\myhspace
	\begin{minipage}{\mpwthree}%
	\centering%
	\includegraphics[width=\columnwidth]{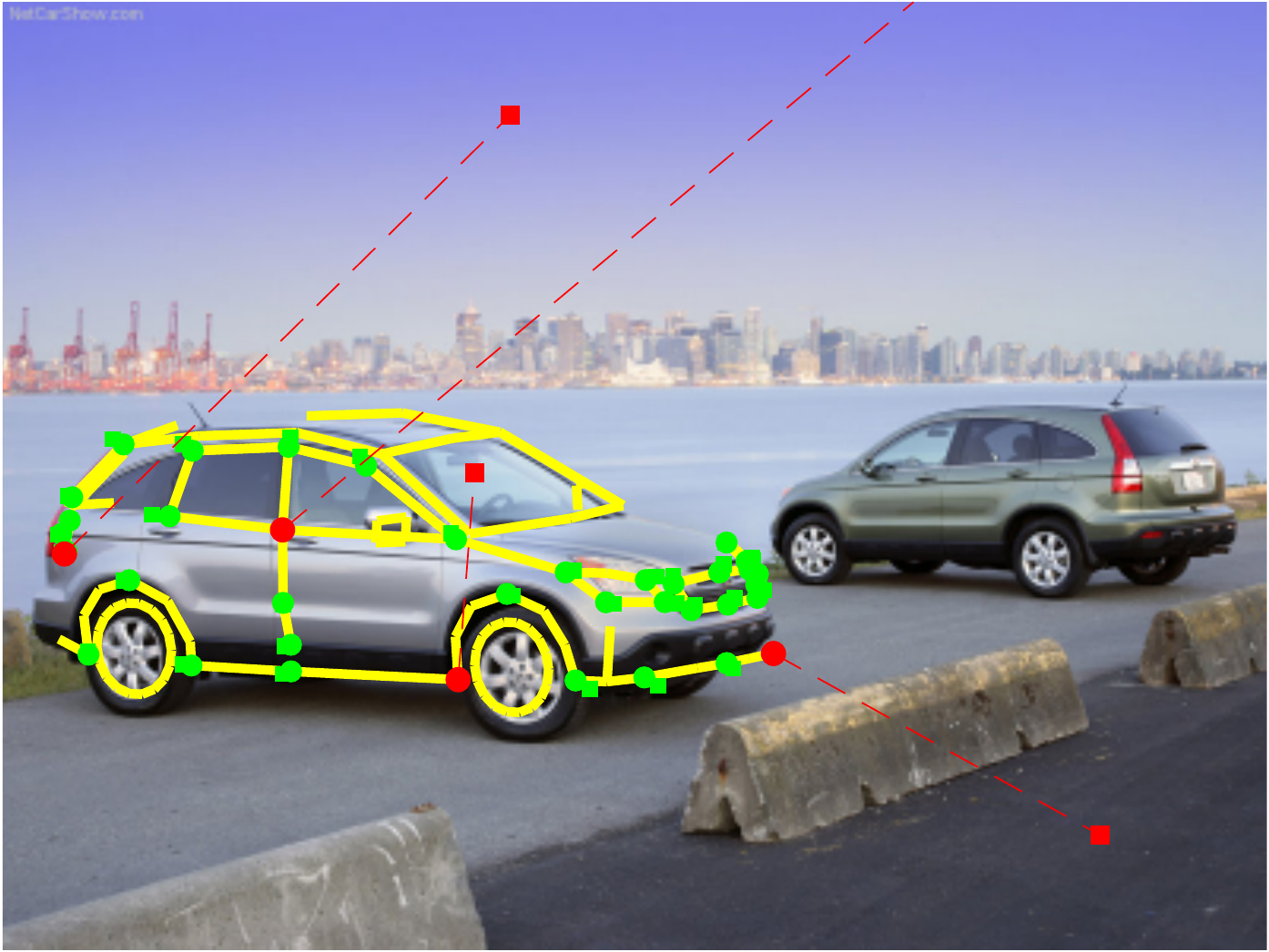} \\
	\vspace{1mm}
	\end{minipage}
& \myhspace
	\begin{minipage}{\mpwthree}%
	\centering%
	\includegraphics[width=\columnwidth]{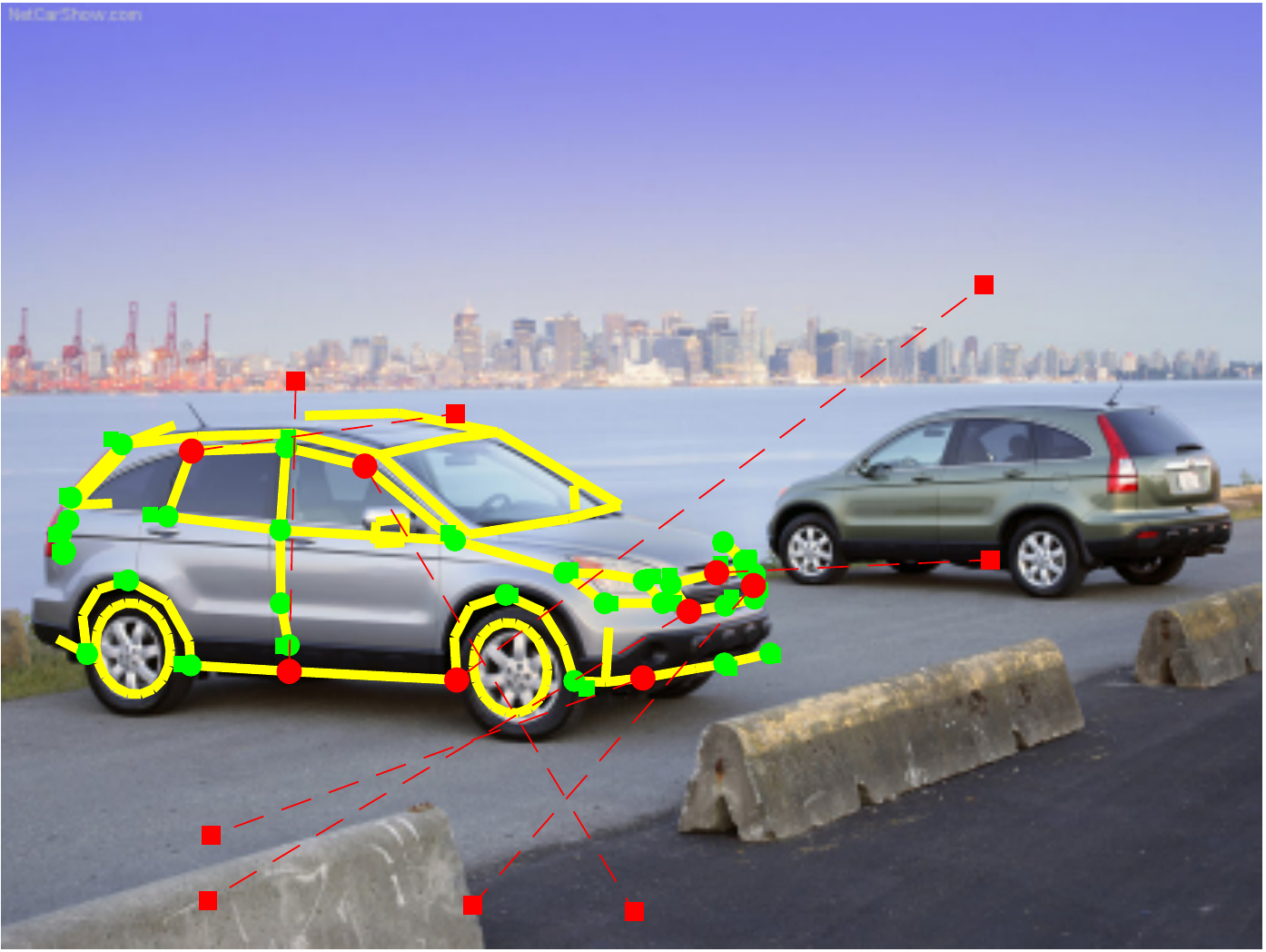} \\
	\vspace{1mm}
	\end{minipage}
& \myhspace
	\begin{minipage}{\mpwthree}%
	\centering%
	\includegraphics[width=\columnwidth]{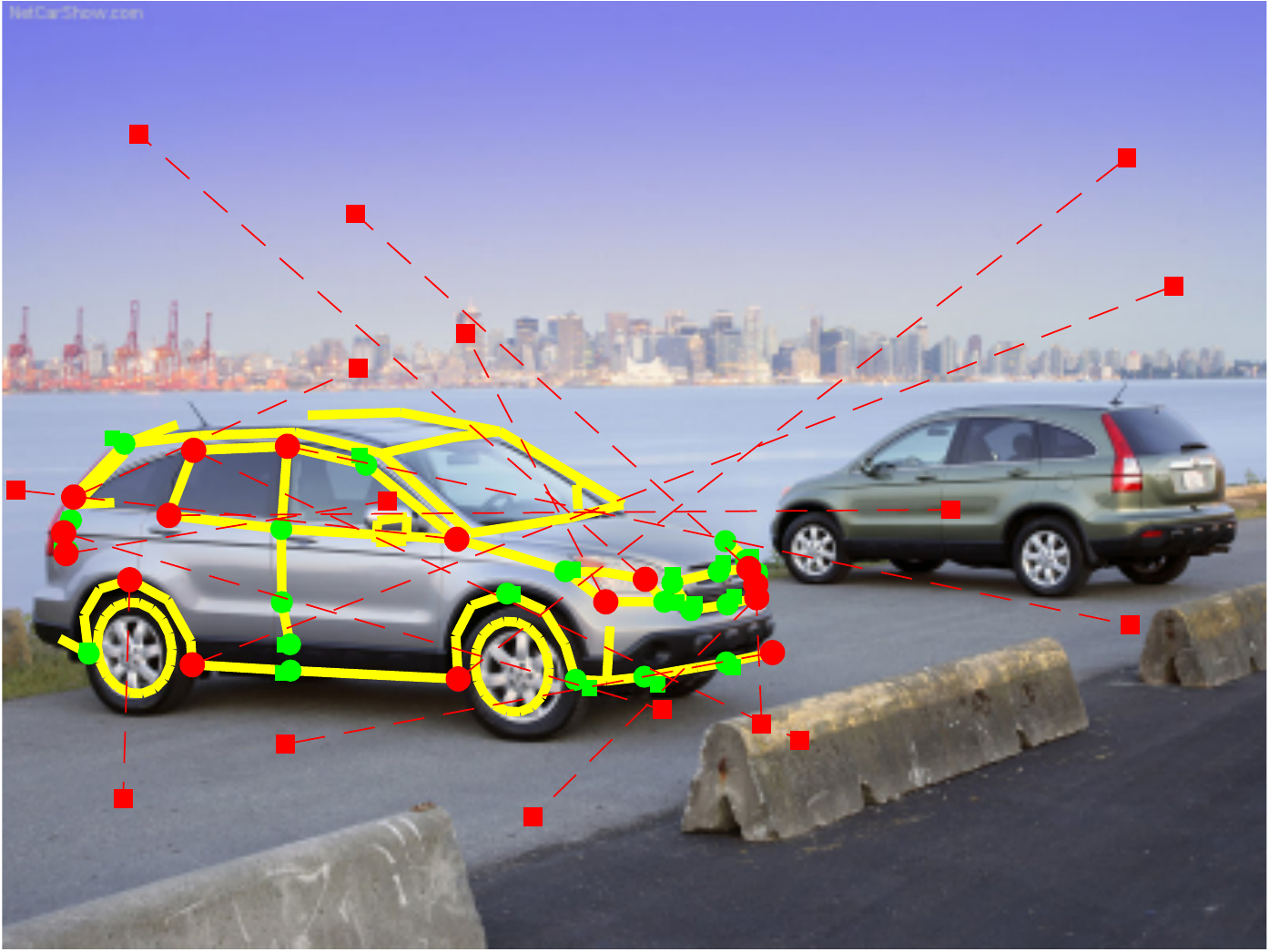} \\
	\vspace{1mm}
	\end{minipage} 
& \myhspace
	\begin{minipage}{\mpwthree}%
	\centering%
	\includegraphics[width=\columnwidth]{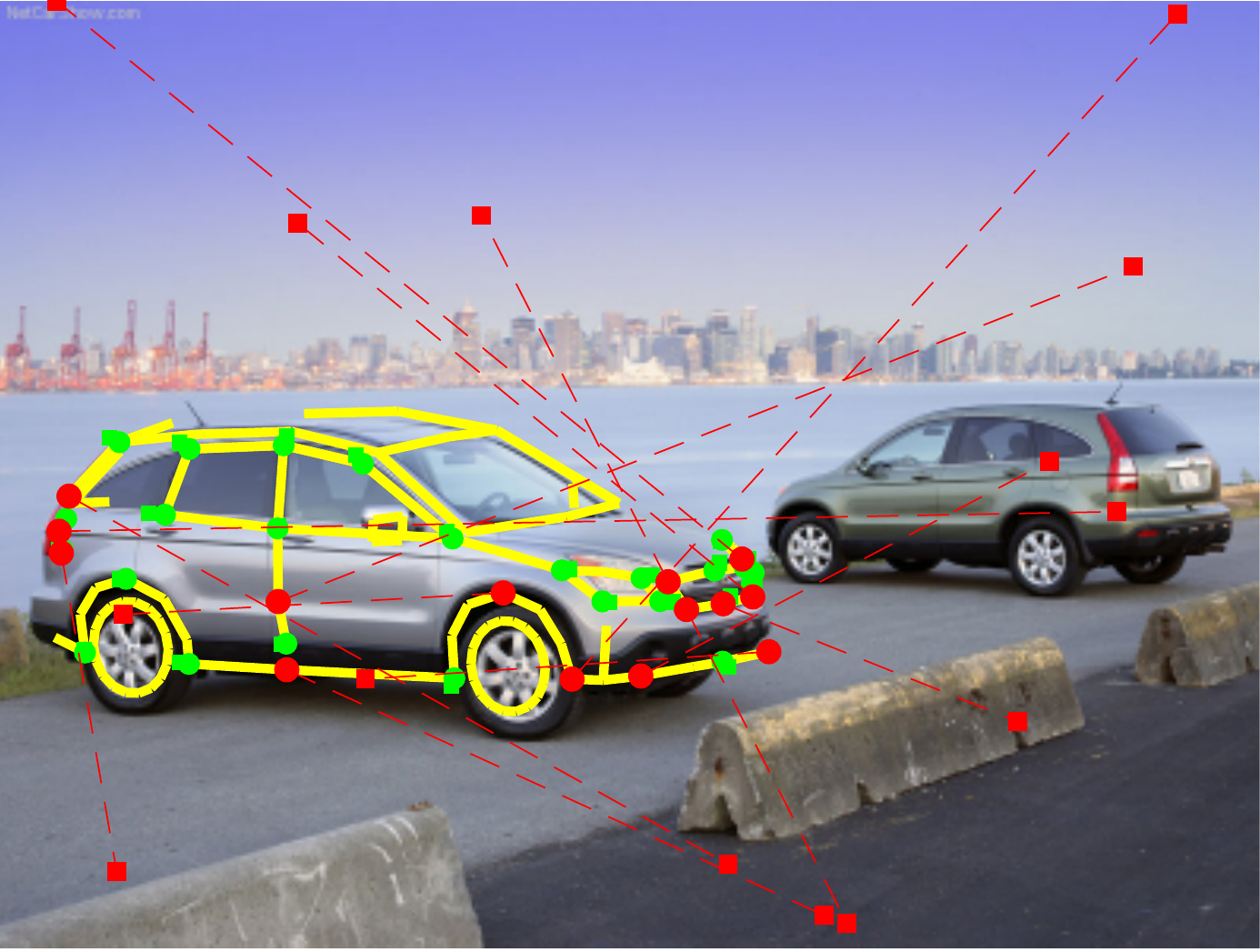} \\
	\vspace{1mm}
	\end{minipage}
& \myhspace
	\begin{minipage}{\mpwthree}%
	\centering%
	\includegraphics[width=\columnwidth]{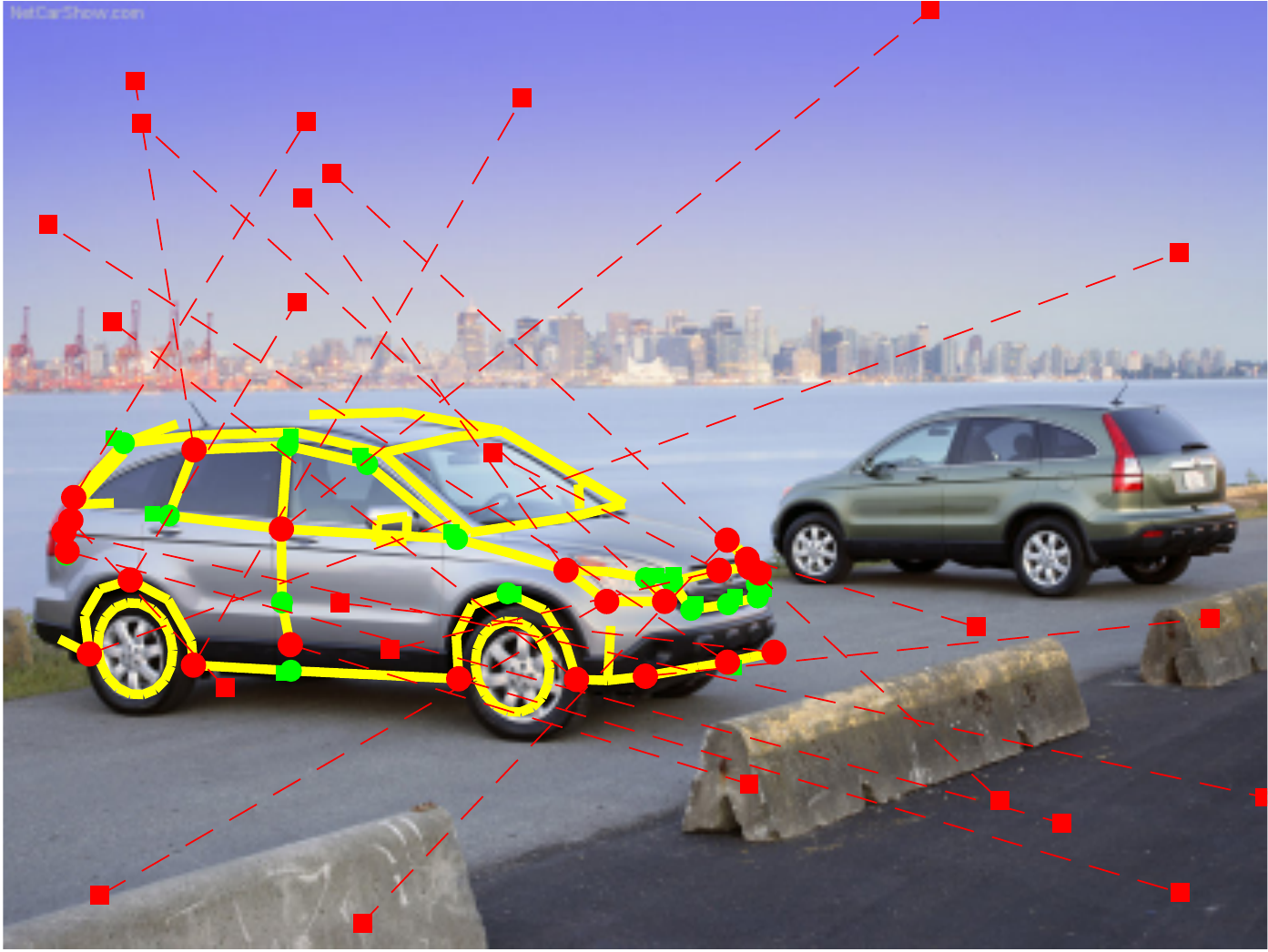} \\
	\vspace{1mm}
	\end{minipage}
&  \myhspace
	\begin{minipage}{\mpwthree}%
	\centering%
	\includegraphics[width=\columnwidth]{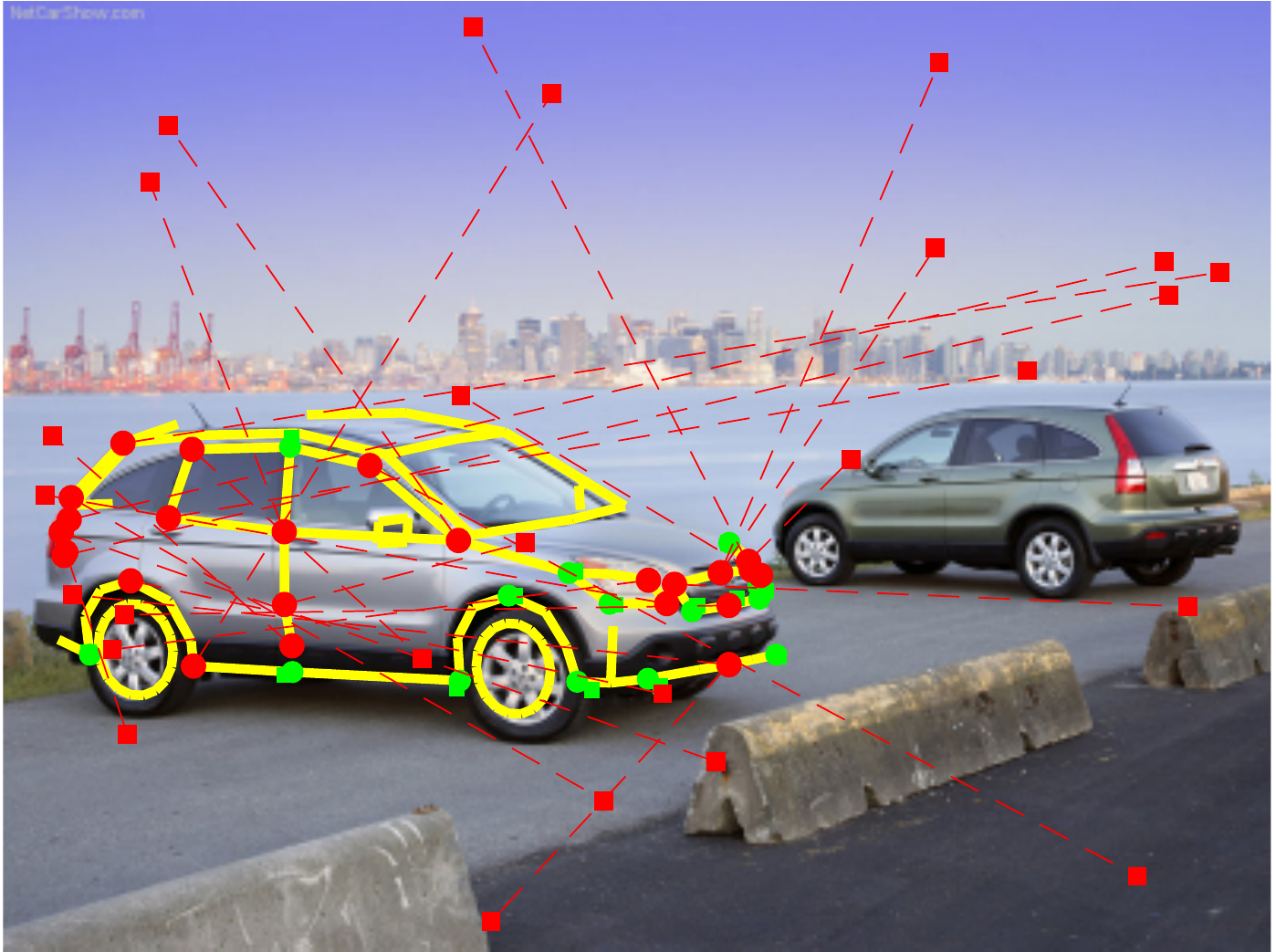} \\
	\vspace{1mm}
	\end{minipage} 
&  \myhspace
	\begin{minipage}{\mpwthree}%
	\centering%
	\includegraphics[width=\columnwidth]{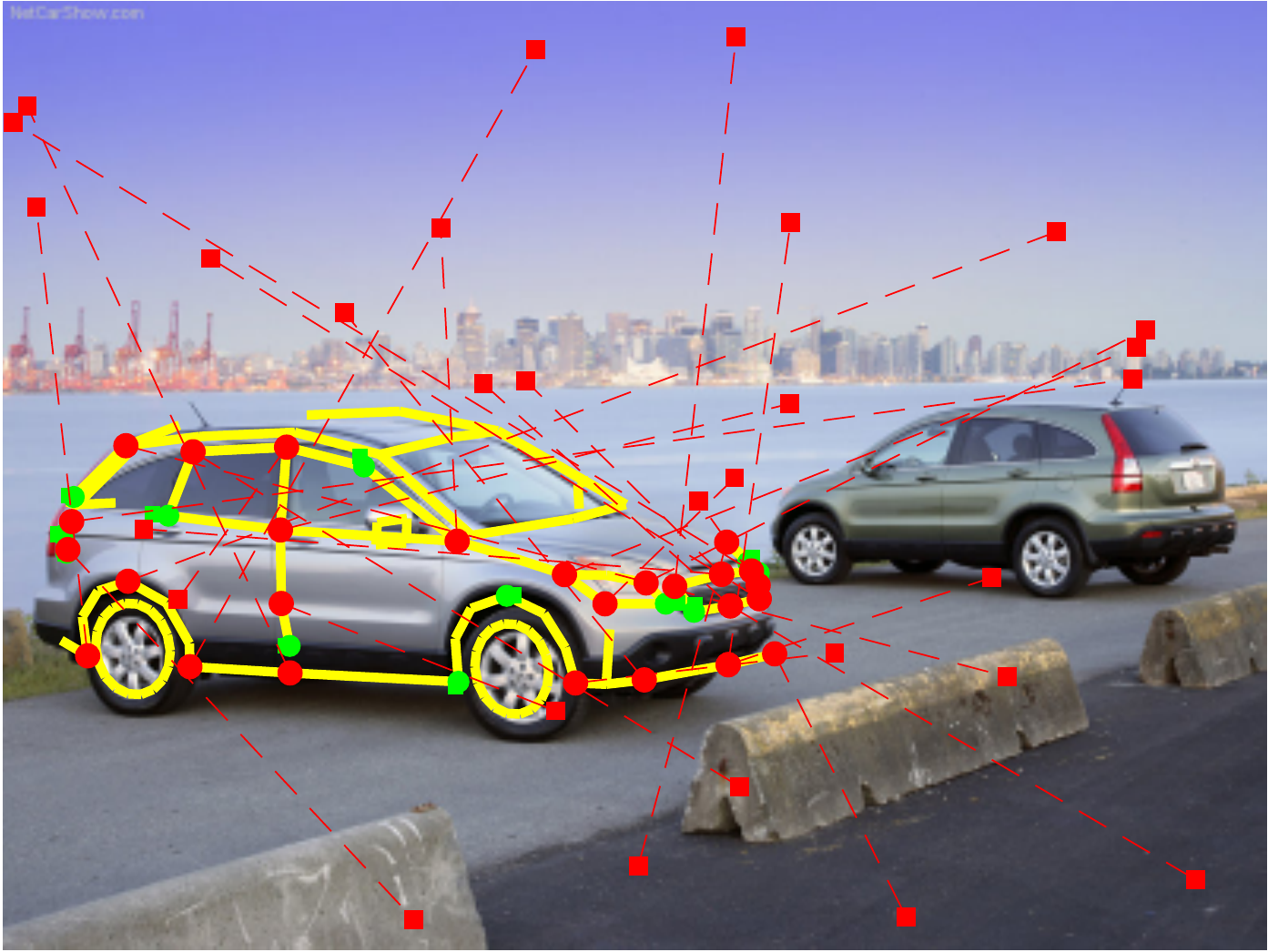} \\
	\vspace{1mm}
	\end{minipage} \\
\multicolumn{8}{c}{Honda CRV } \\

%% file: 173-Mercedes-Benz_GL450.tex
\myhspace \myhspace \hspace{-2mm} \rotatebox{90}{\hspace{-7mm} {\smaller \alternrobust} } & 
\myhspace
	\begin{minipage}{\mpwthree}%
	\centering%
	\includegraphics[width=\columnwidth]{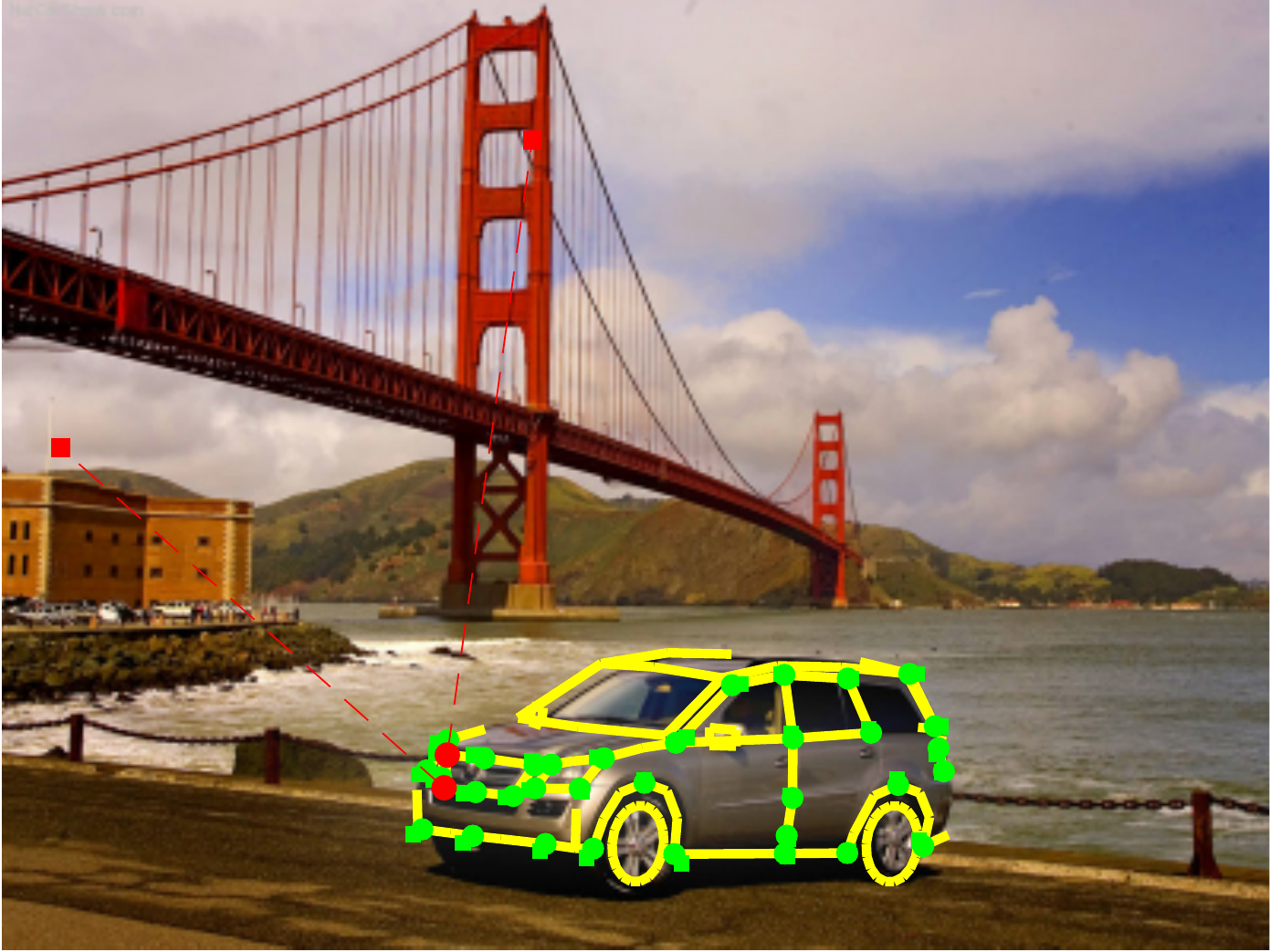} \\
	\vspace{1mm}
	\end{minipage}
& \myhspace
	\begin{minipage}{\mpwthree}%
	\centering%
	\includegraphics[width=\columnwidth]{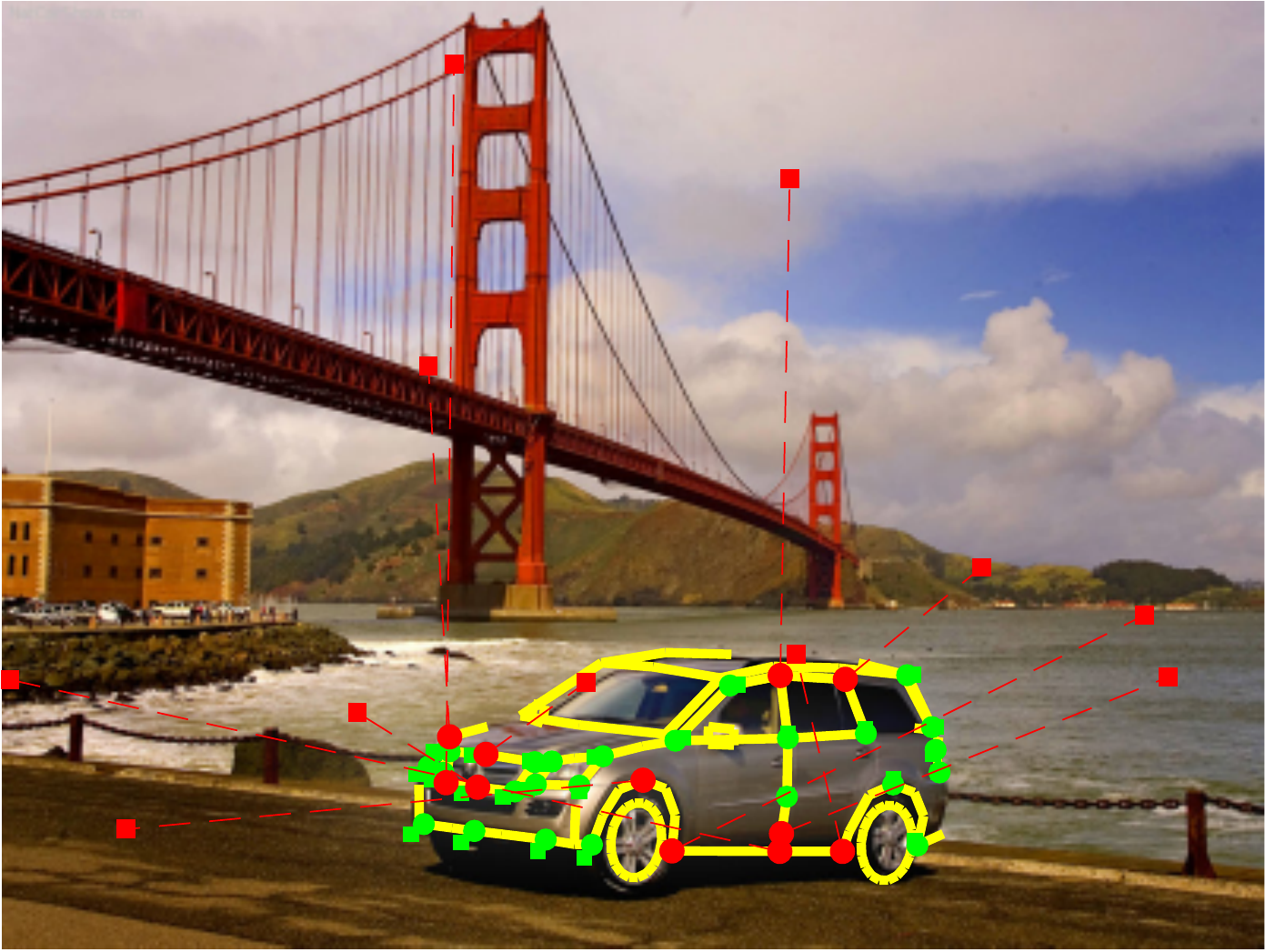} \\
	\vspace{1mm}
	\end{minipage}
& \myhspace
	\begin{minipage}{\mpwthree}%
	\centering%
	\includegraphics[width=\columnwidth]{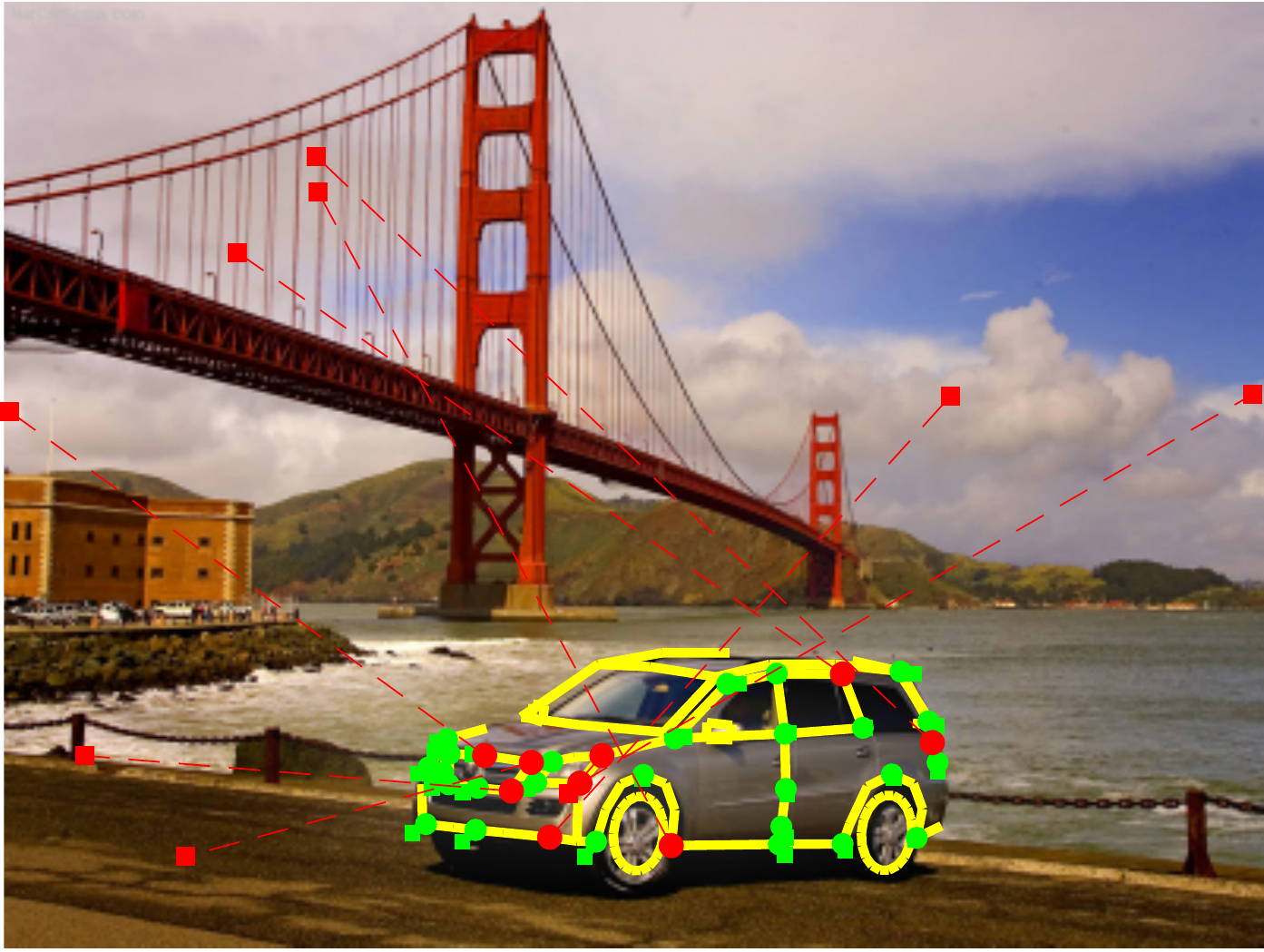} \\
	\vspace{1mm}
	\end{minipage} 
& \myhspace
	\begin{minipage}{\mpwthree}%
	\centering%
	\includegraphics[width=\columnwidth]{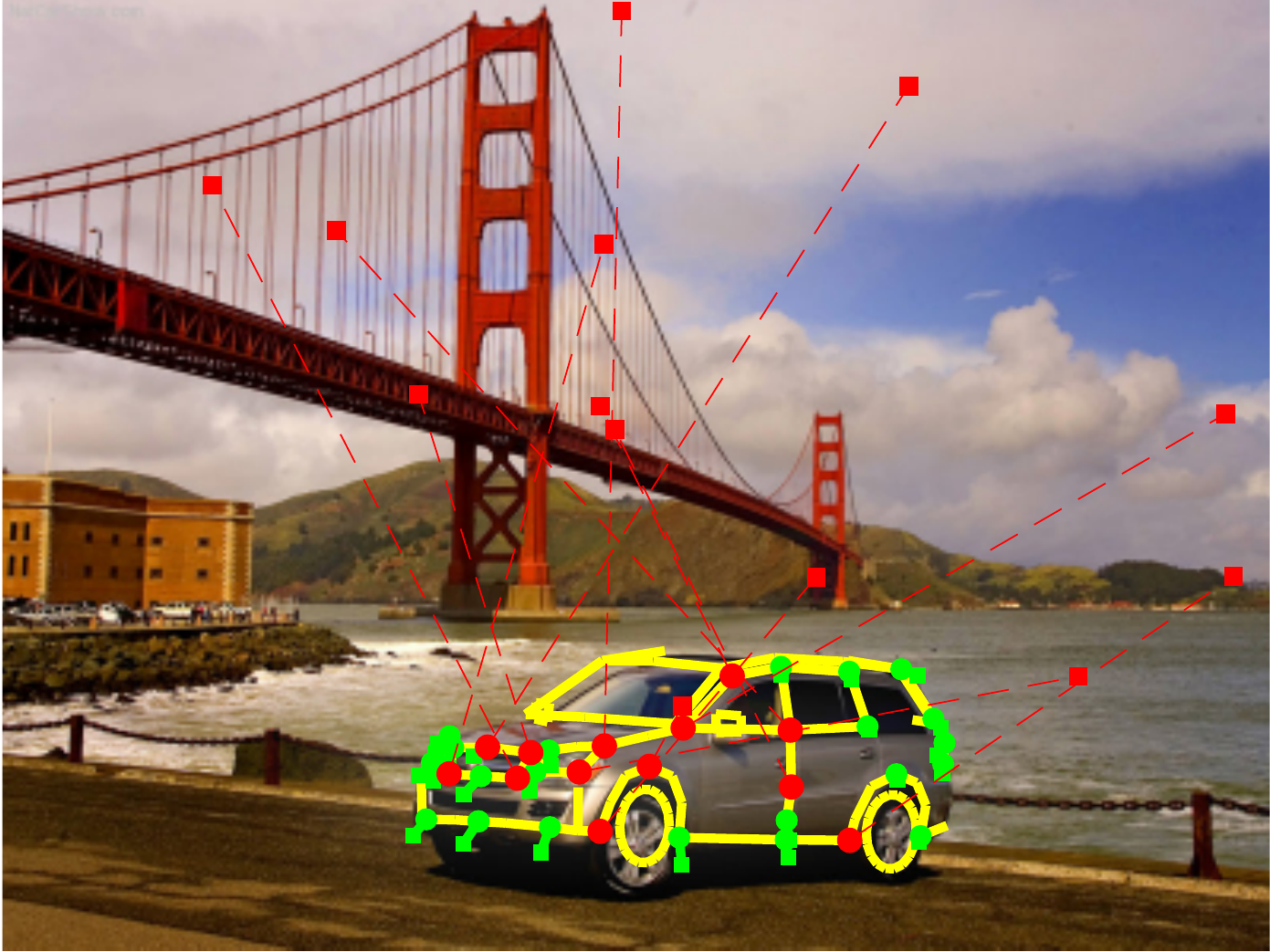} \\
	\vspace{1mm}
	\end{minipage}
& \myhspace
	\begin{minipage}{\mpwthree}%
	\centering%
	\includegraphics[width=\columnwidth]{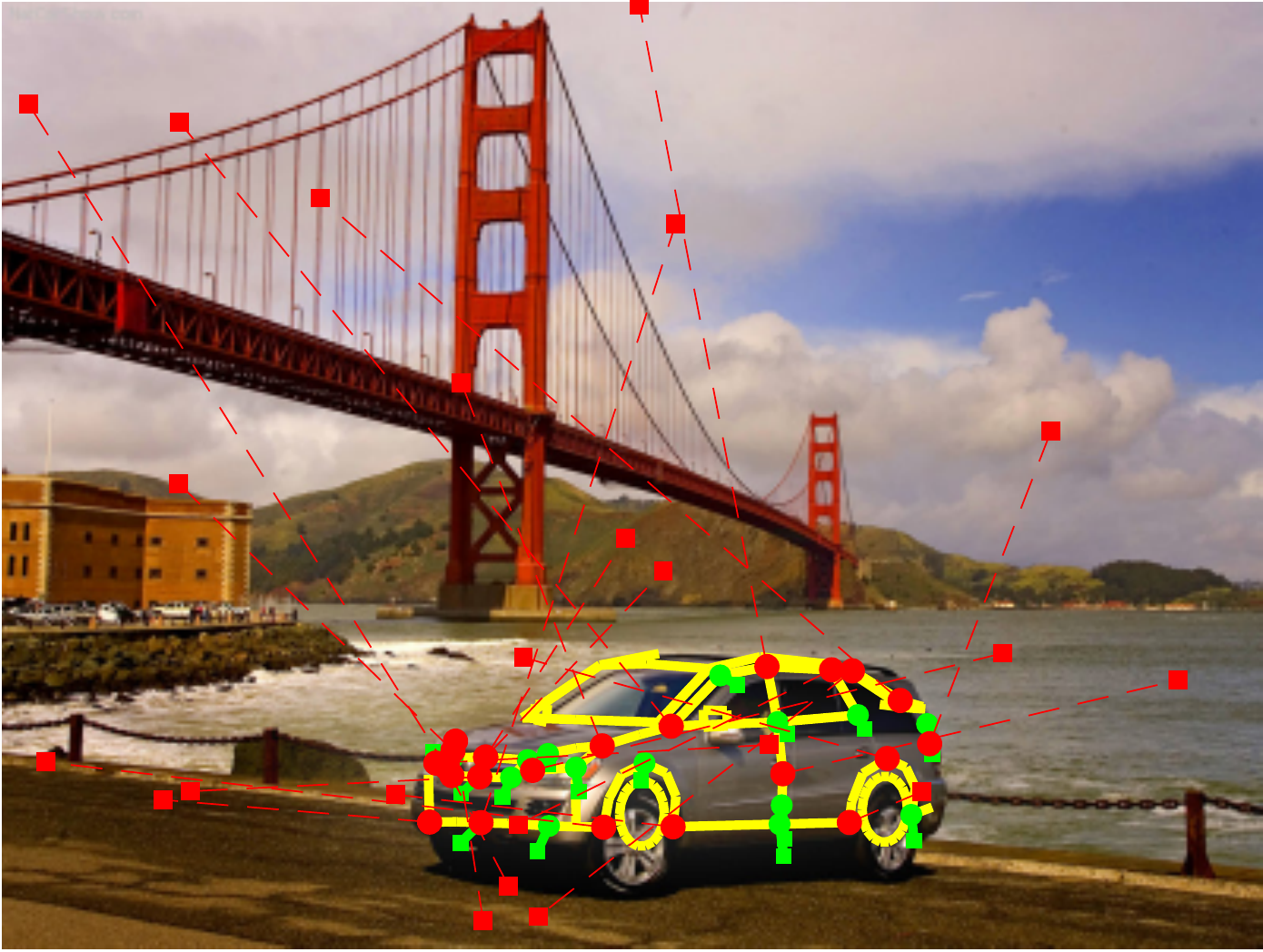} \\
	\vspace{1mm}
	\end{minipage}
&  \myhspace
	\begin{minipage}{\mpwthree}%
	\centering%
	\includegraphics[width=\columnwidth]{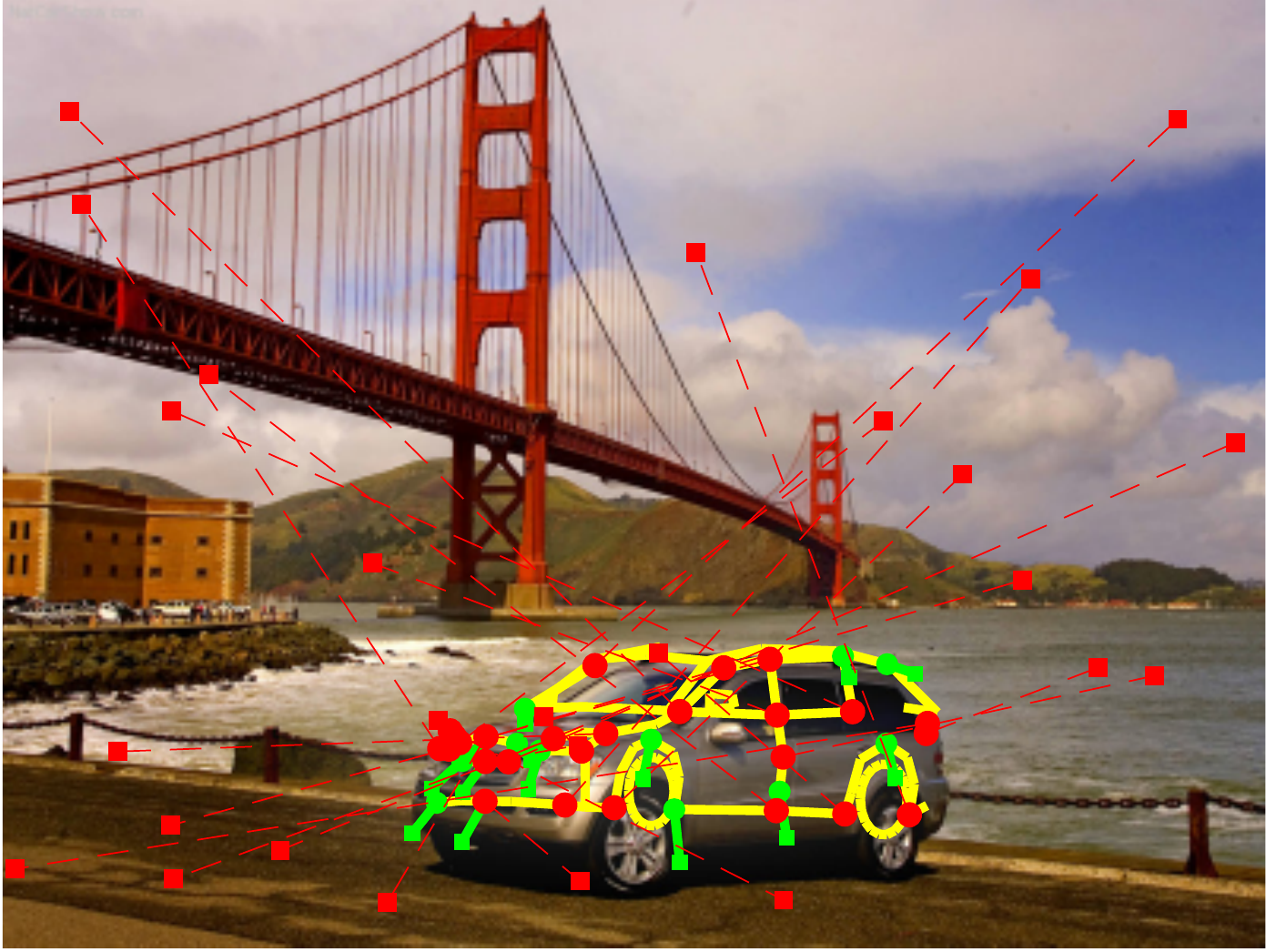} \\
	\vspace{1mm}
	\end{minipage} 
&  \myhspace
	\begin{minipage}{\mpwthree}%
	\centering%
	\includegraphics[width=\columnwidth]{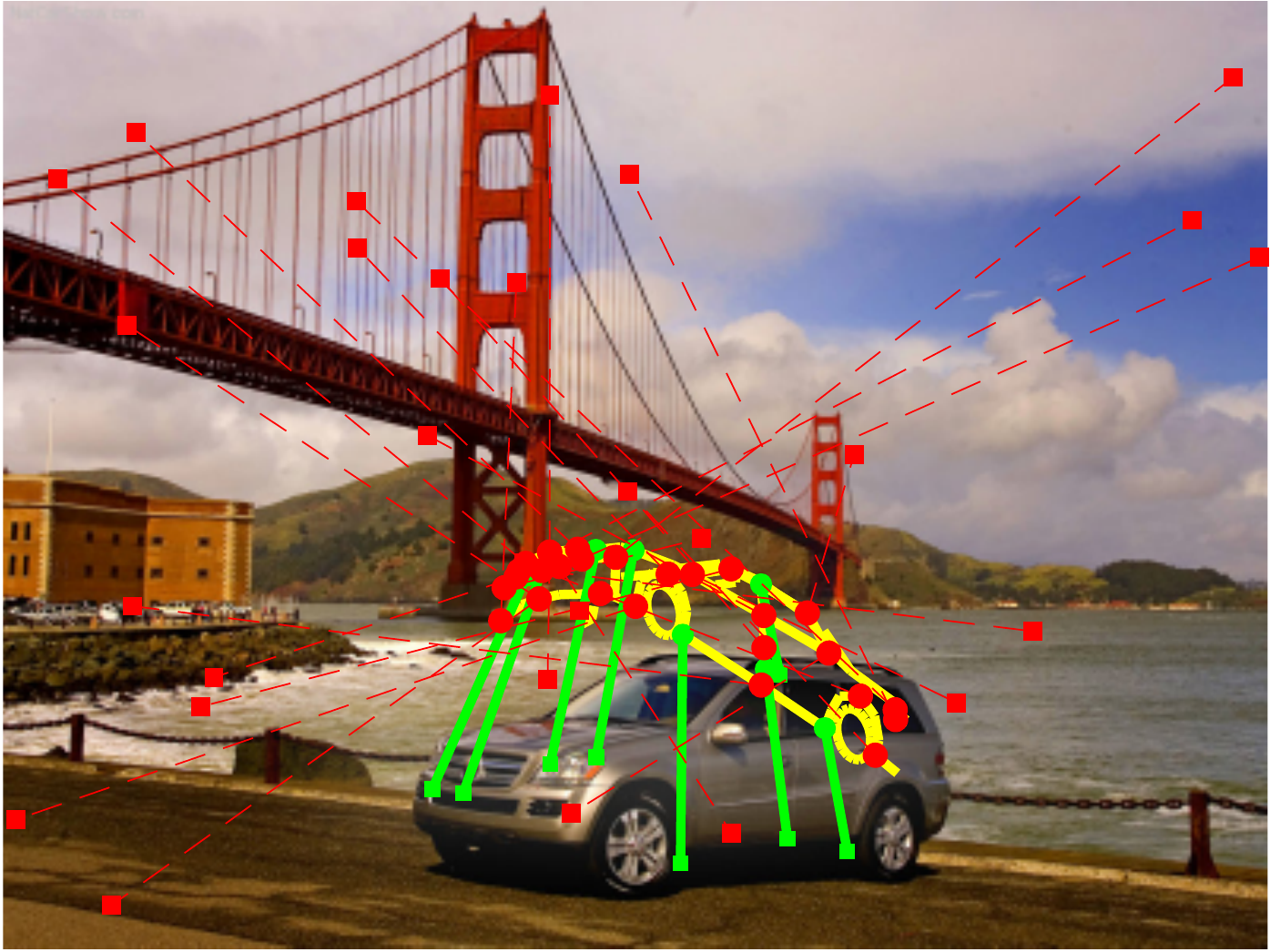} \\
	\vspace{1mm}
	\end{minipage} \\
\myhspace \myhspace \hspace{-2mm} \rotatebox{90}{\hspace{-8mm} {\smaller \convexrobust} } & 
\myhspace
	\begin{minipage}{\mpwthree}%
	\centering%
	\includegraphics[width=\columnwidth]{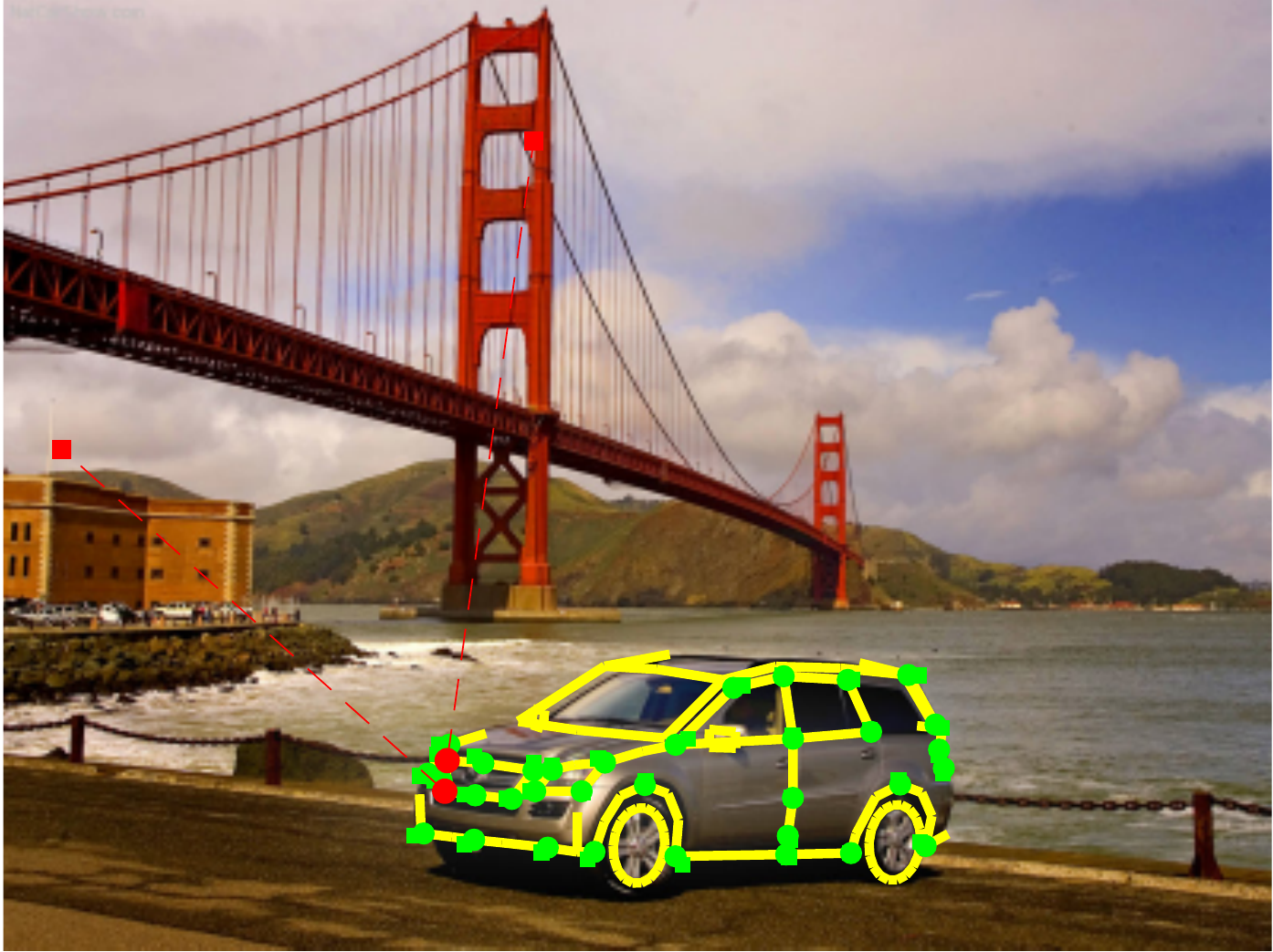} \\
	\vspace{1mm}
	\end{minipage}
& \myhspace
	\begin{minipage}{\mpwthree}%
	\centering%
	\includegraphics[width=\columnwidth]{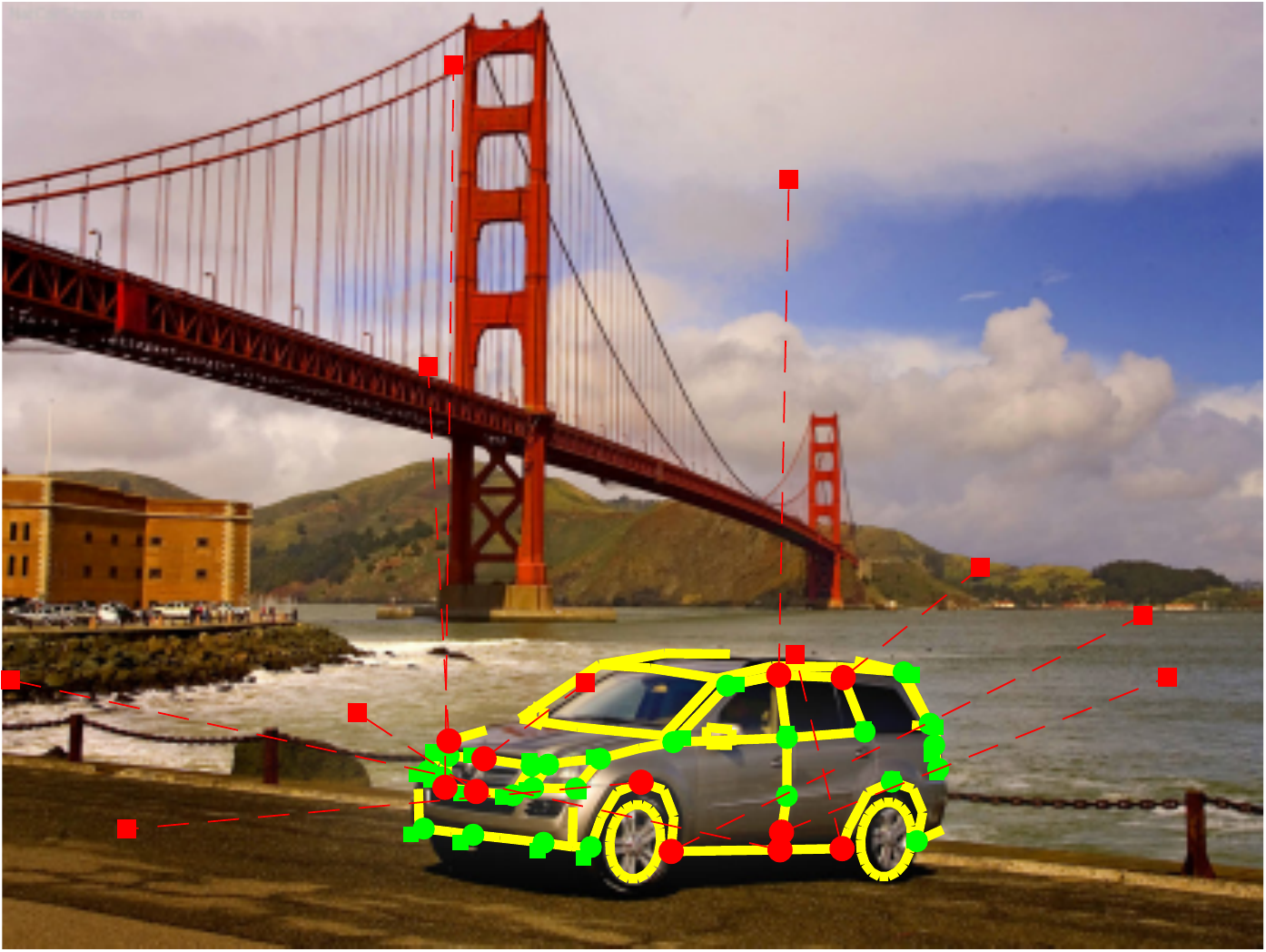} \\
	\vspace{1mm}
	\end{minipage}
& \myhspace
	\begin{minipage}{\mpwthree}%
	\centering%
	\includegraphics[width=\columnwidth]{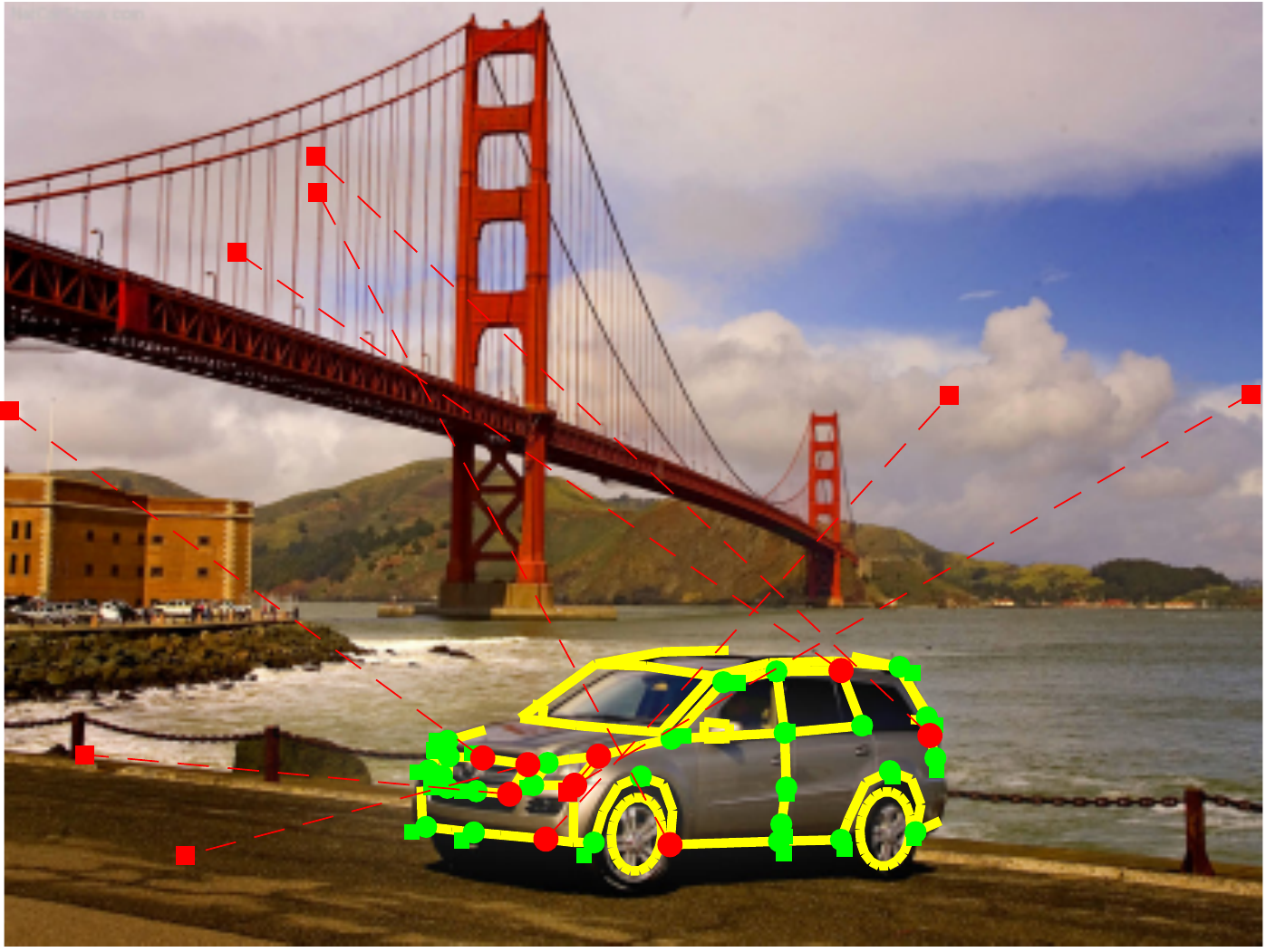} \\
	\vspace{1mm}
	\end{minipage} 
& \myhspace
	\begin{minipage}{\mpwthree}%
	\centering%
	\includegraphics[width=\columnwidth]{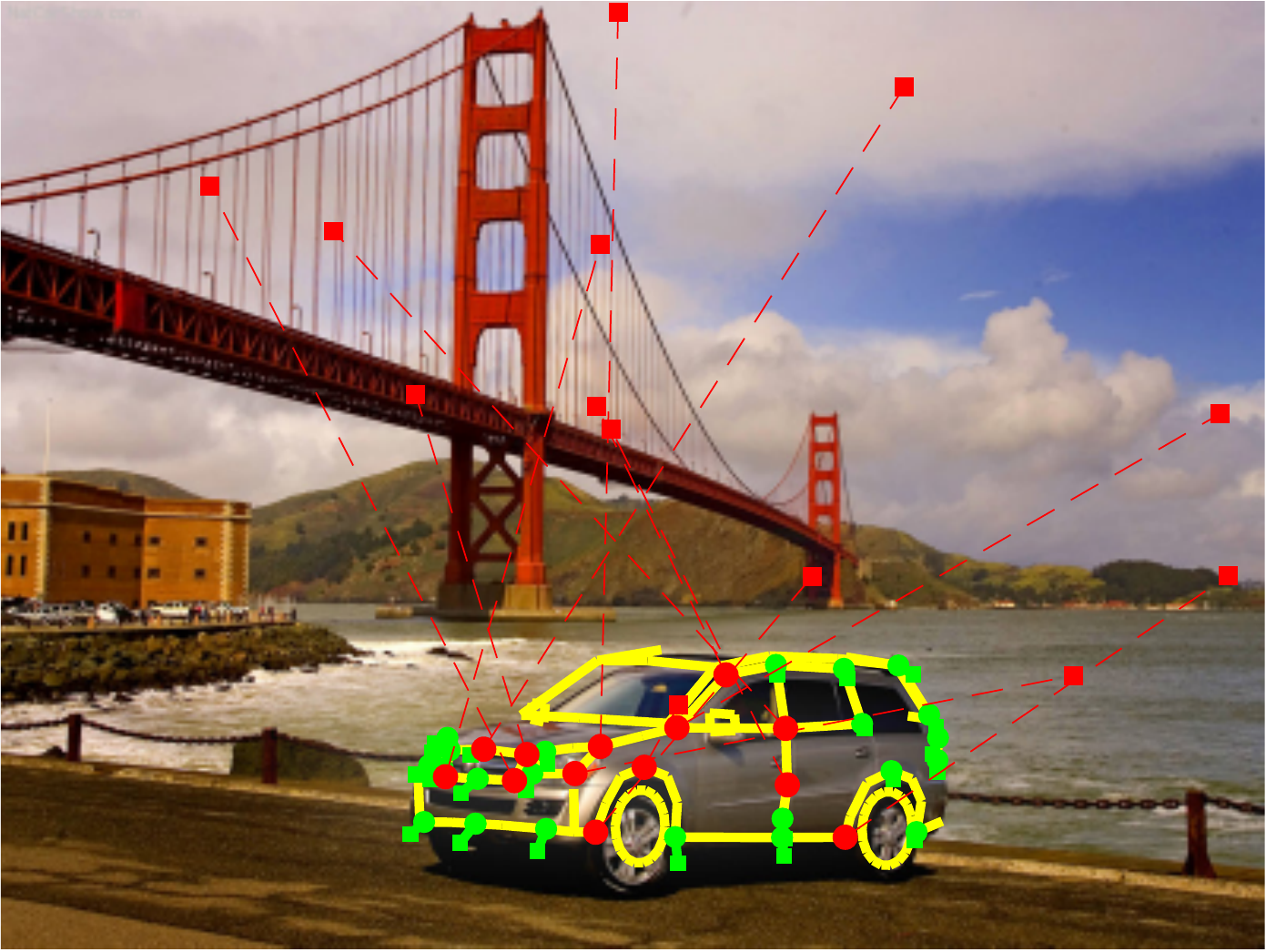} \\
	\vspace{1mm}
	\end{minipage}
& \myhspace
	\begin{minipage}{\mpwthree}%
	\centering%
	\includegraphics[width=\columnwidth]{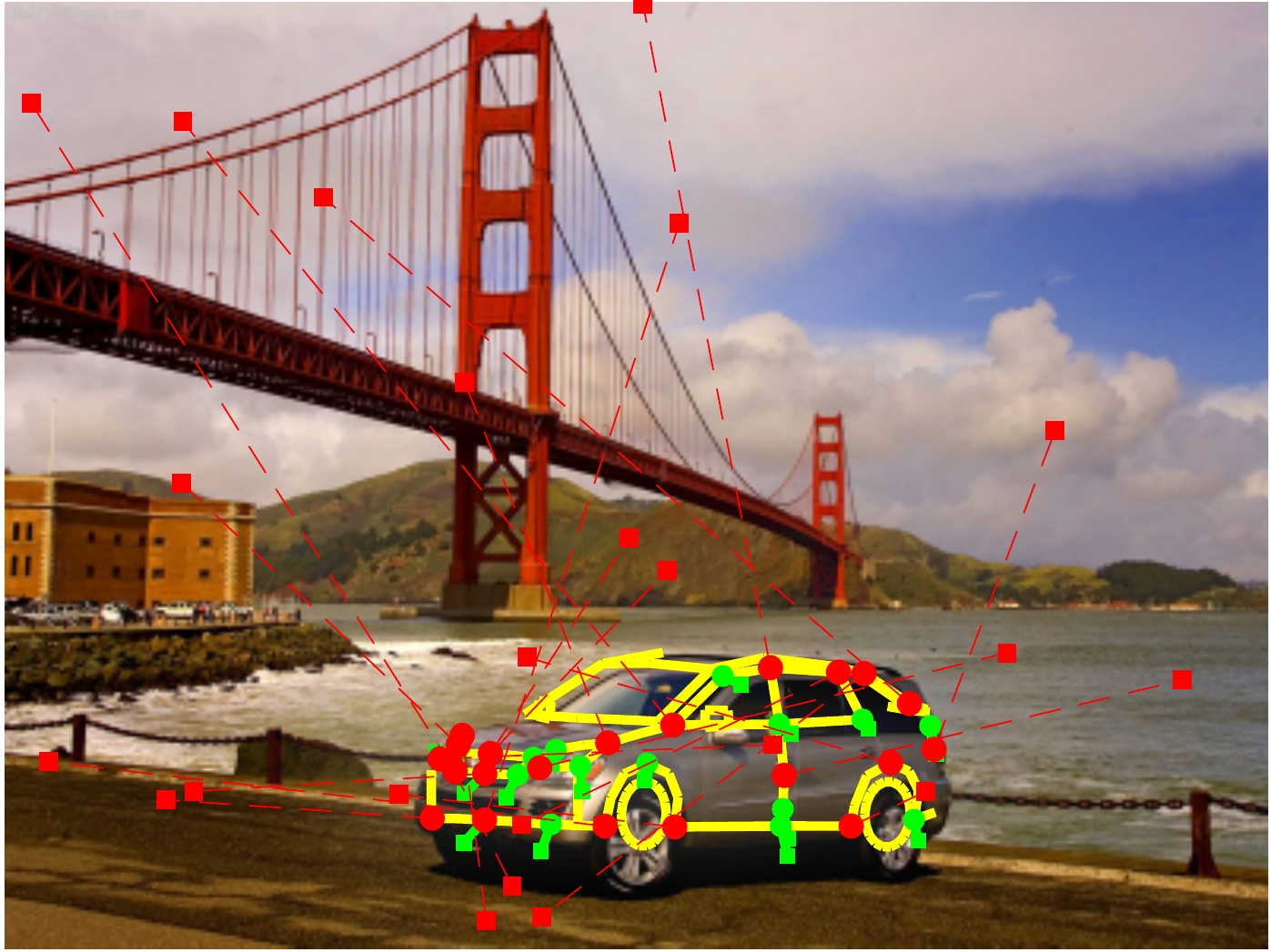} \\
	\vspace{1mm}
	\end{minipage}
&  \myhspace
	\begin{minipage}{\mpwthree}%
	\centering%
	\includegraphics[width=\columnwidth]{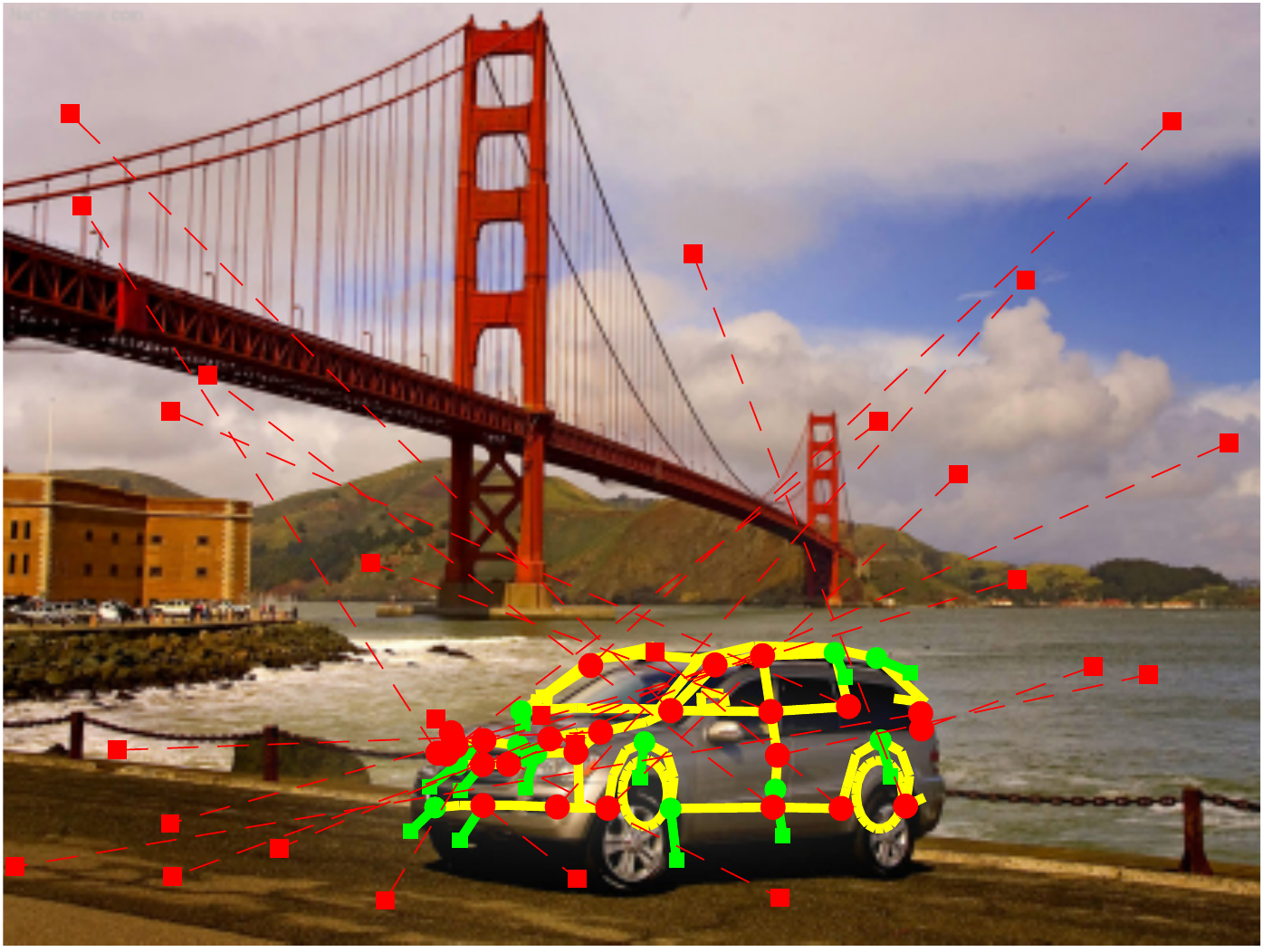} \\
	\vspace{1mm}
	\end{minipage} 
&  \myhspace
	\begin{minipage}{\mpwthree}%
	\centering%
	\includegraphics[width=\columnwidth]{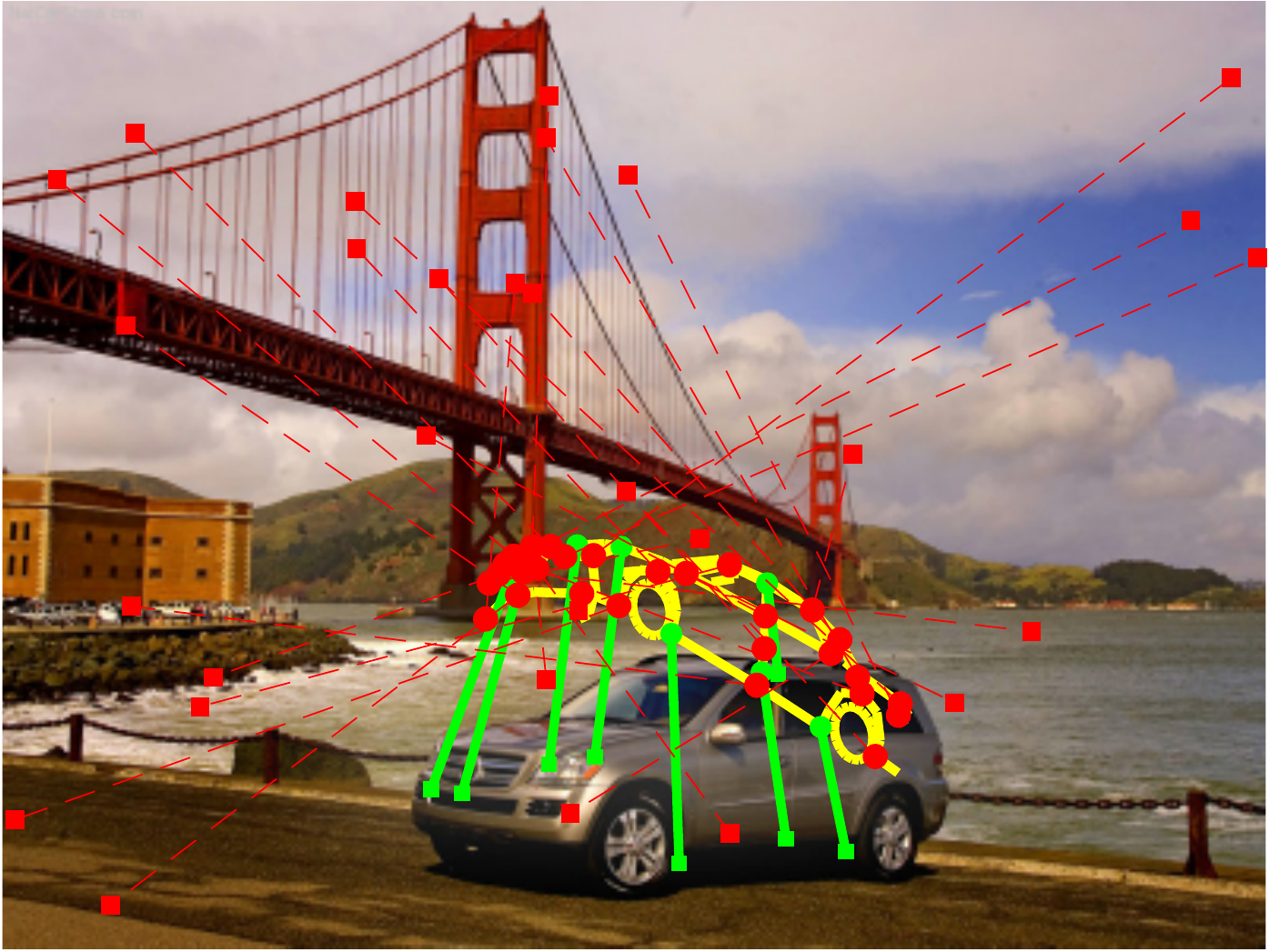} \\
	\vspace{1mm}
	\end{minipage} \\
\myhspace \myhspace \hspace{-2mm} \rotatebox{90}{\hspace{-3mm} {\smaller \blue{ \namerobust}} } & 
\myhspace
	\begin{minipage}{\mpwthree}%
	\centering%
	\includegraphics[width=\columnwidth]{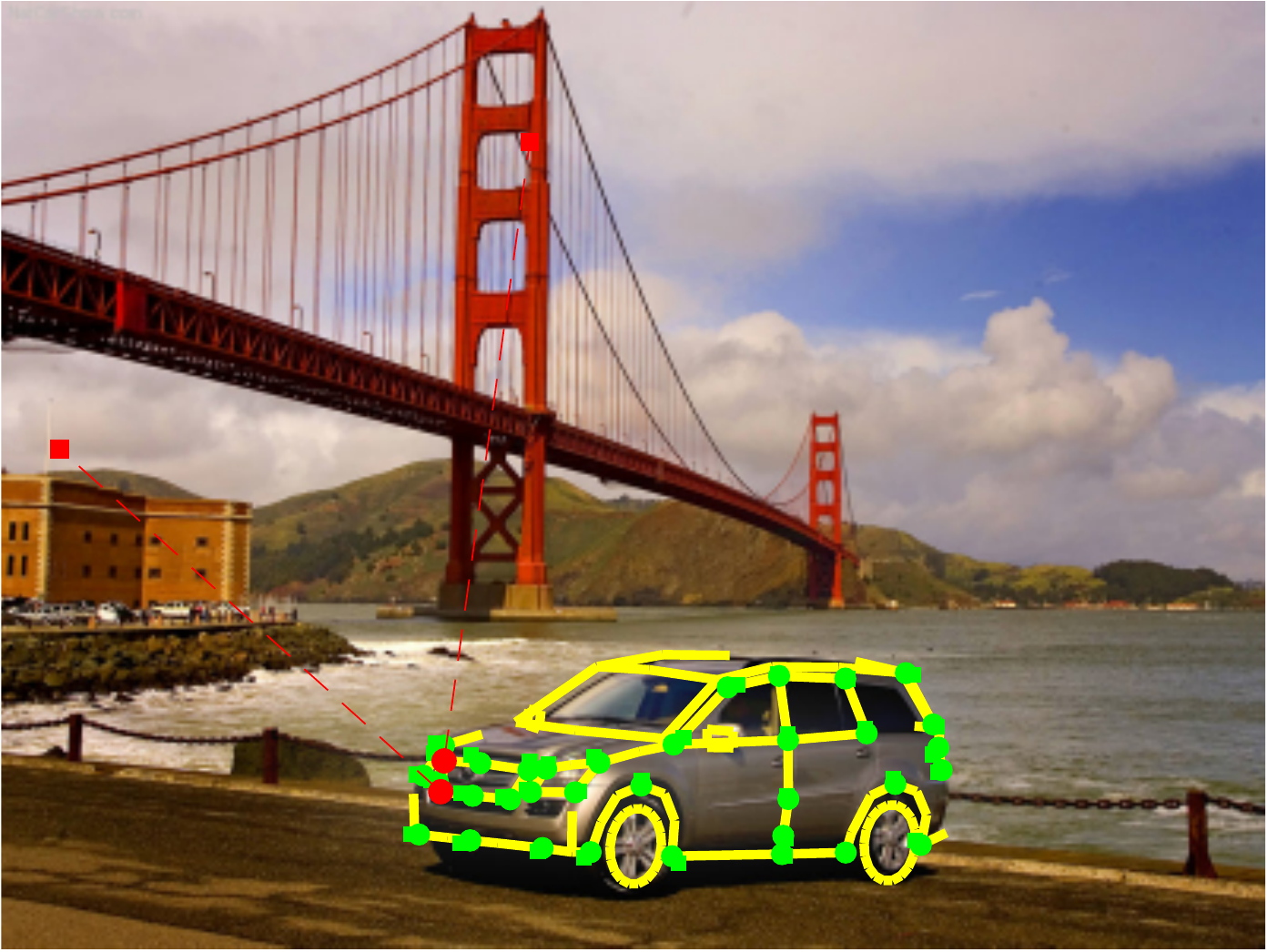} \\
	\vspace{1mm}
	\end{minipage}
& \myhspace
	\begin{minipage}{\mpwthree}%
	\centering%
	\includegraphics[width=\columnwidth]{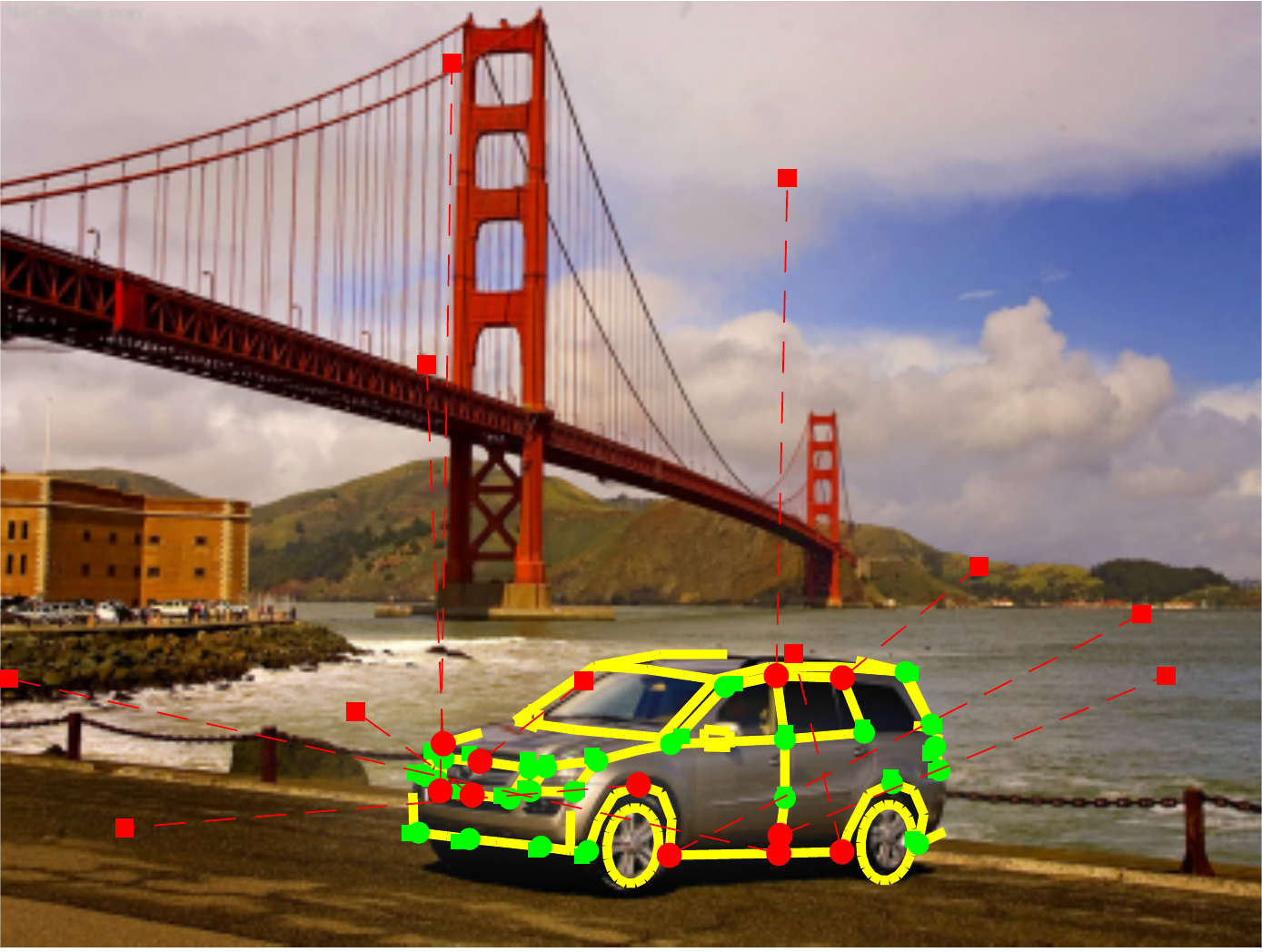} \\
	\vspace{1mm}
	\end{minipage}
& \myhspace
	\begin{minipage}{\mpwthree}%
	\centering%
	\includegraphics[width=\columnwidth]{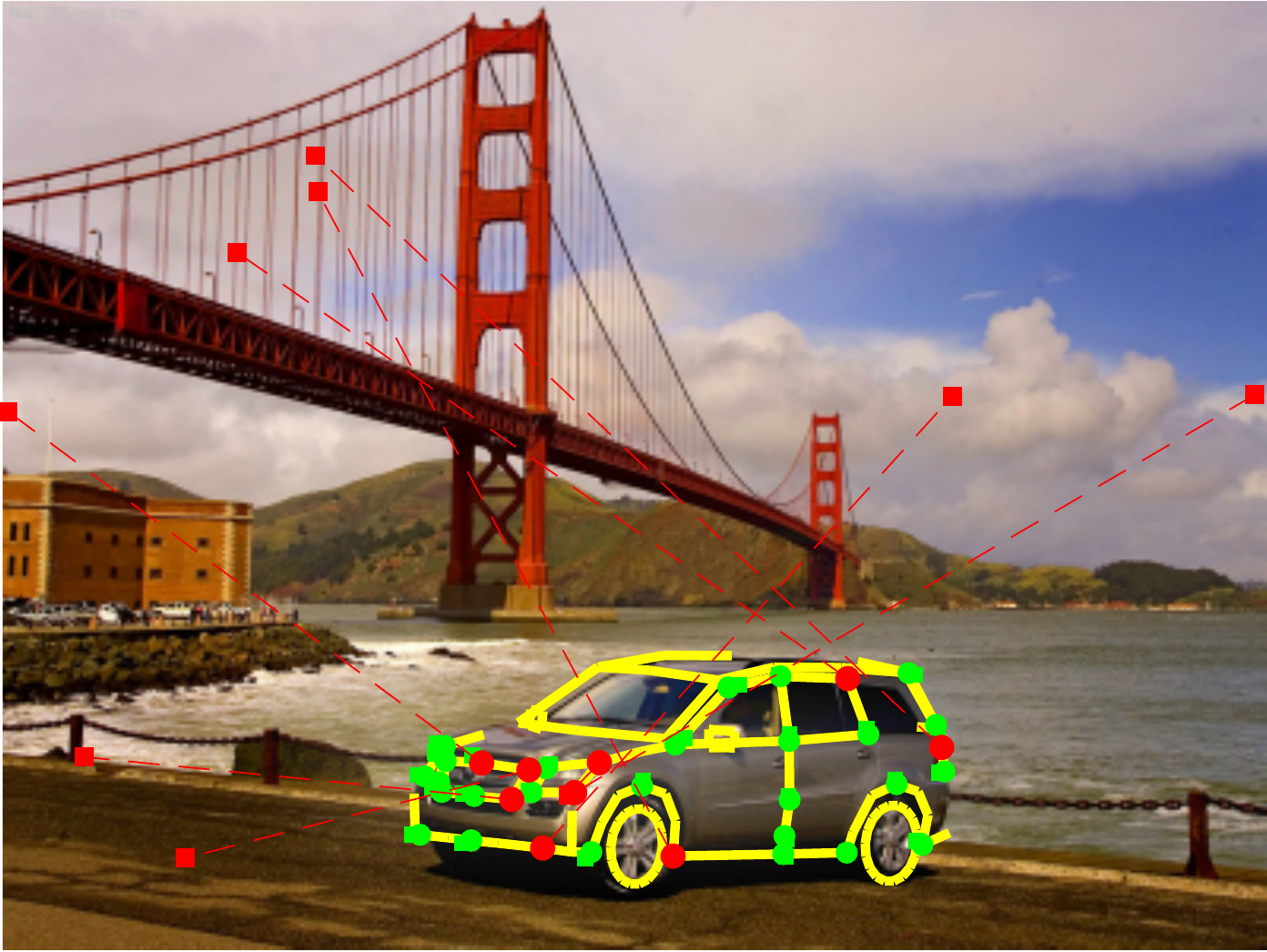} \\
	\vspace{1mm}
	\end{minipage} 
& \myhspace
	\begin{minipage}{\mpwthree}%
	\centering%
	\includegraphics[width=\columnwidth]{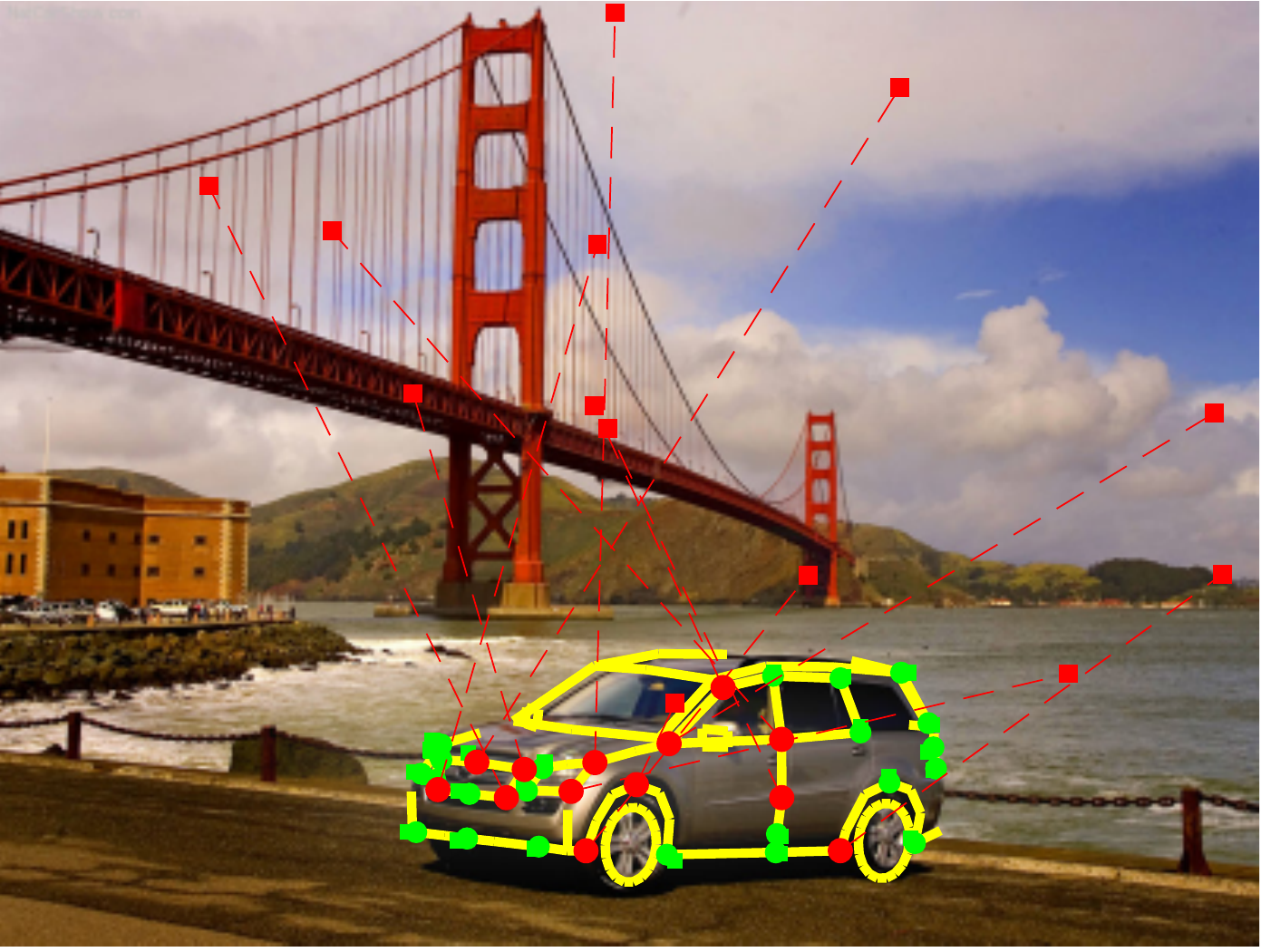} \\
	\vspace{1mm}
	\end{minipage}
& \myhspace
	\begin{minipage}{\mpwthree}%
	\centering%
	\includegraphics[width=\columnwidth]{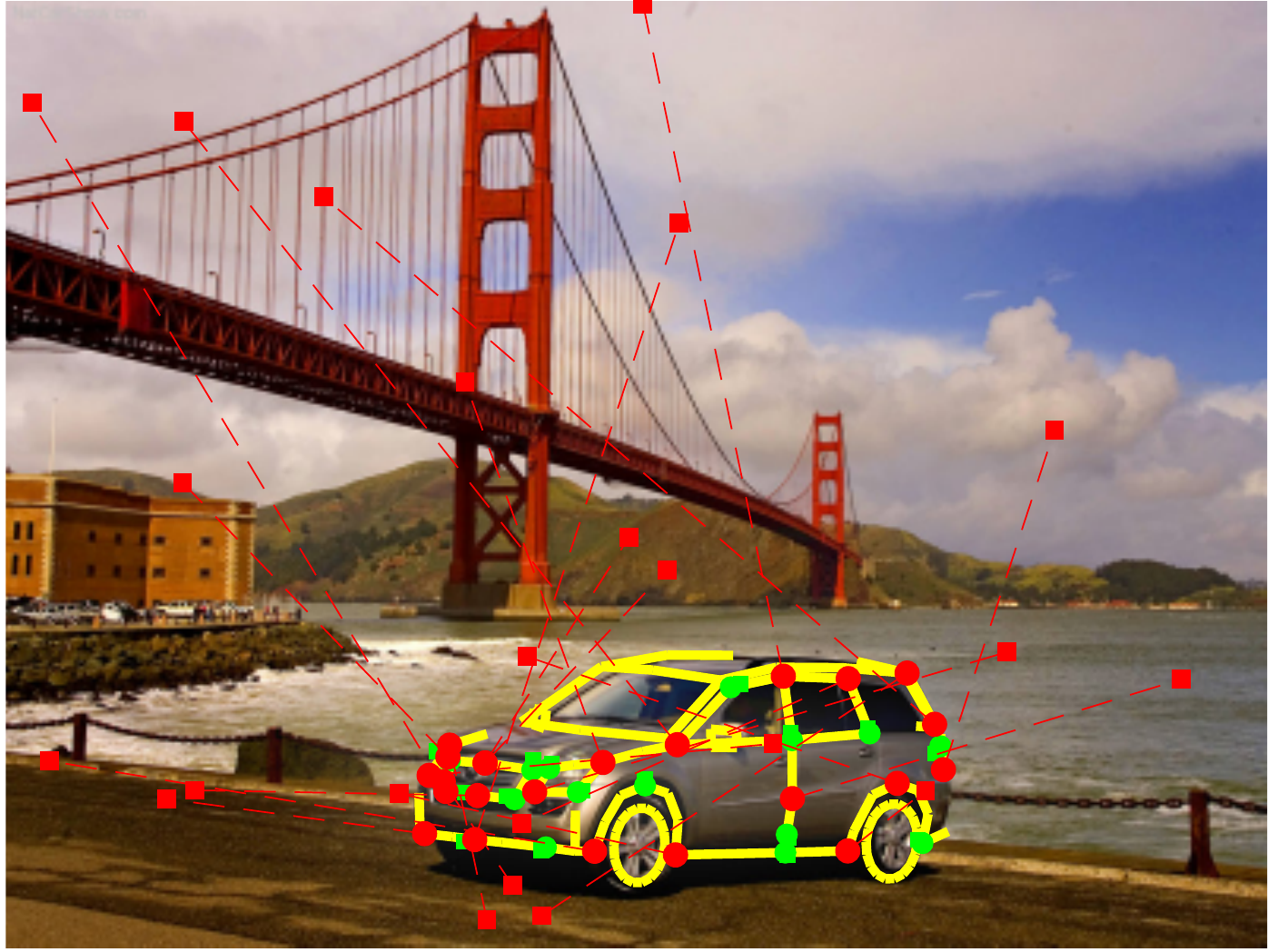} \\
	\vspace{1mm}
	\end{minipage}
&  \myhspace
	\begin{minipage}{\mpwthree}%
	\centering%
	\includegraphics[width=\columnwidth]{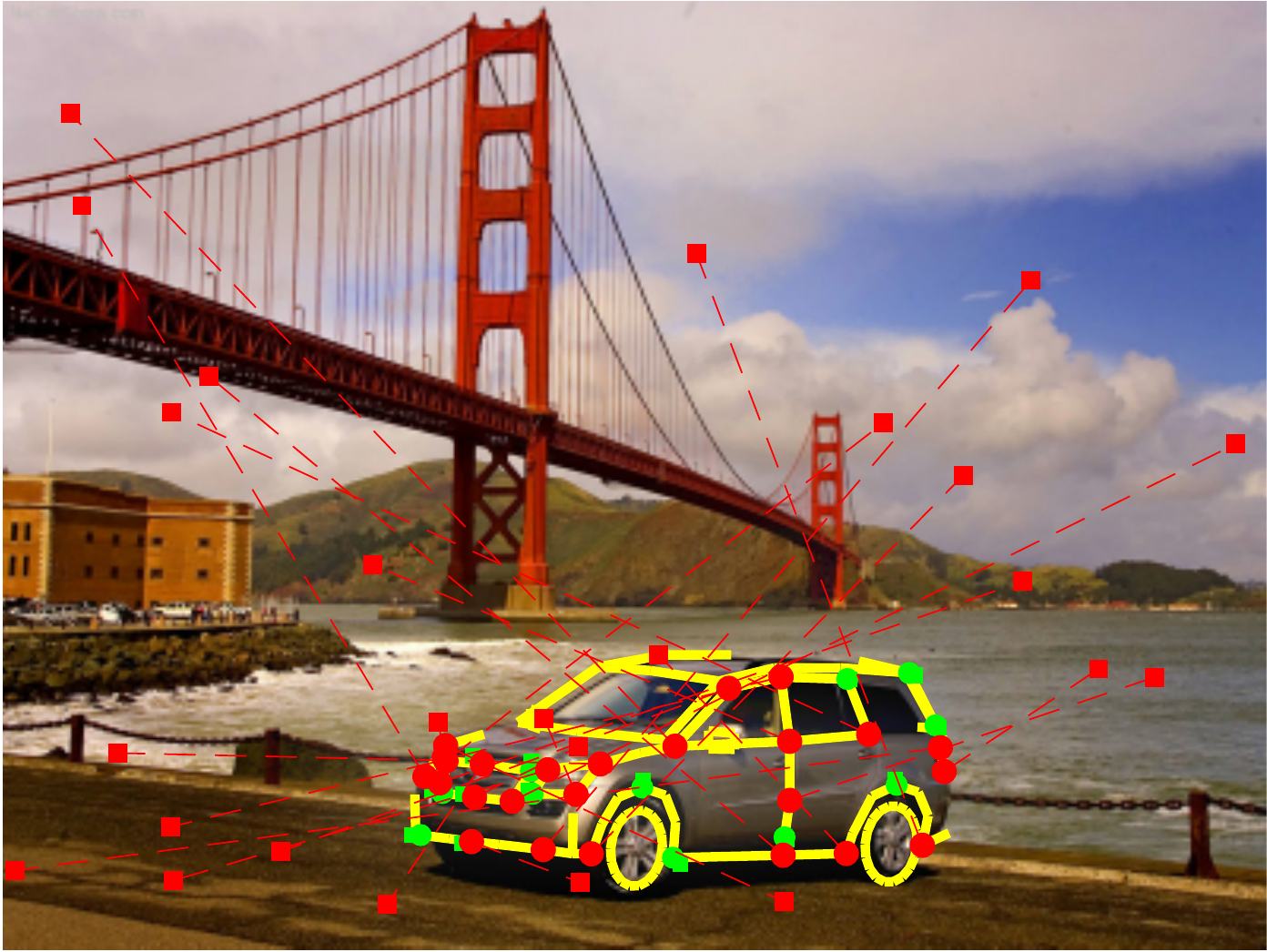} \\
	\vspace{1mm}
	\end{minipage} 
&  \myhspace
	\begin{minipage}{\mpwthree}%
	\centering%
	\includegraphics[width=\columnwidth]{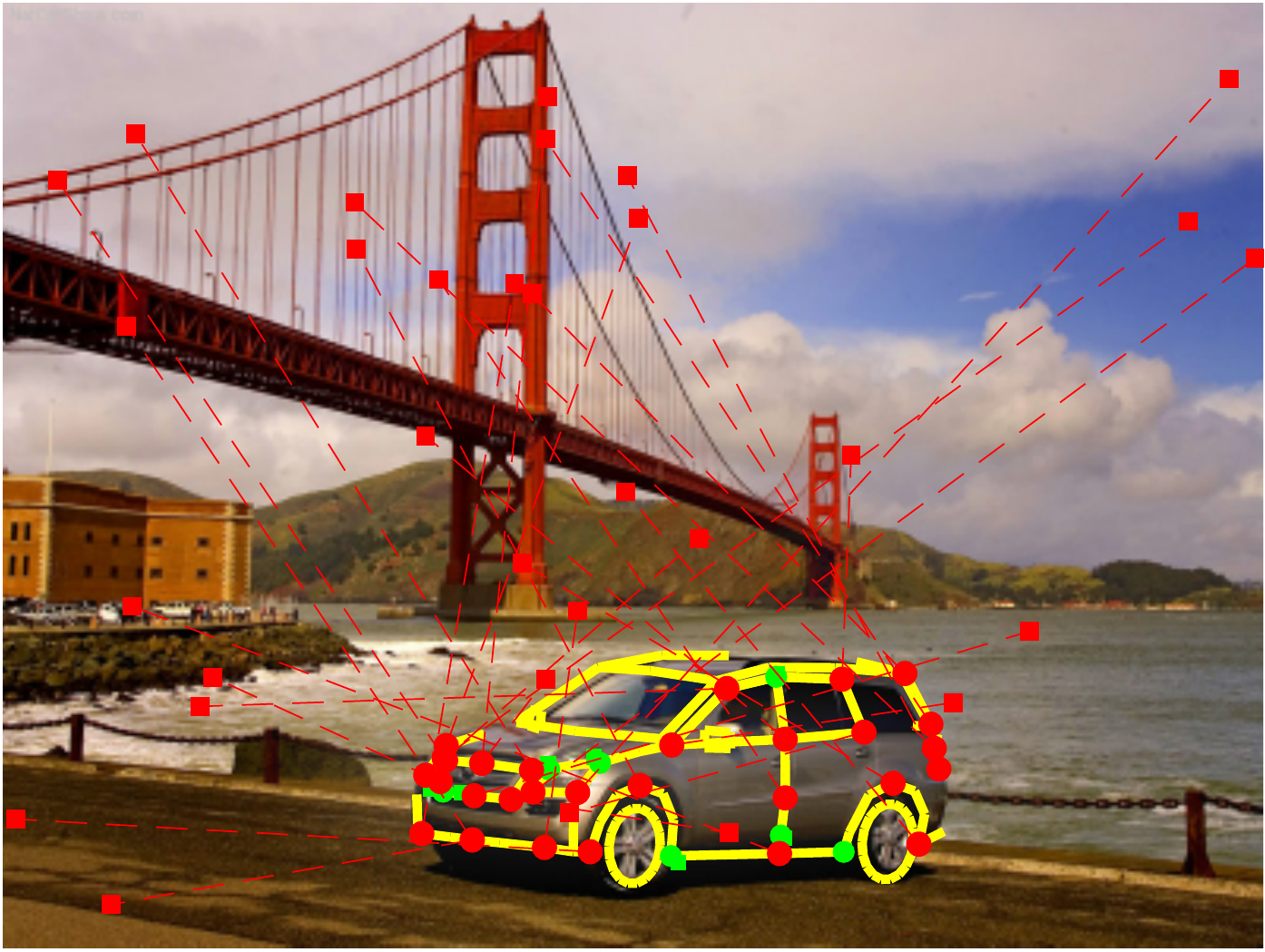} \\
	\vspace{1mm}
	\end{minipage} \\
\multicolumn{8}{c}{Mercedes-Benz GL450} \\

%% file: supp-fig-FG3DCar_qualitative-2.tex

\renewcommand{\mpwthree}{2.6cm}
\renewcommand{\myhspace}{\hspace{-3.5mm}}

\begin{figure*}[h]\ContinuedFloat
	\begin{center}{}
	\begin{minipage}{\textwidth}
	\begin{tabular}{cccccccc}%
		\input{278-Vauxhall_Zafira}
		\input{295-Volvo-V70}
		\input{217-saab_93}

		\end{tabular}
	\end{minipage}
	\vspace{-2mm} 
	\caption{Qualitative results on the \FGCar dataset~\cite{Lin14eccv-modelFitting} under $10-70\%$ outlier rates using \alternrobust~\cite{Zhou17pami-shapeEstimationConvex}, \convexrobust~\cite{Zhou17pami-shapeEstimationConvex}, and \namerobust. Yellow: shape reconstruction result projected onto the image.
	Green: inliers. Red: outliers. Circle: 3D landmark. Square: 2D landmark. (cont.) [Best viewed electronically.]
	\label{fig:supp_FG3DCar_qualitative}}
	\end{center}
\end{figure*}

%% file: 278-Vauxhall_Zafira.tex
\myhspace \myhspace \hspace{-2mm} \rotatebox{90}{\hspace{-7mm} {\smaller \alternrobust} } & 
\myhspace
	\begin{minipage}{\mpwthree}%
	\centering%
	\includegraphics[width=\columnwidth]{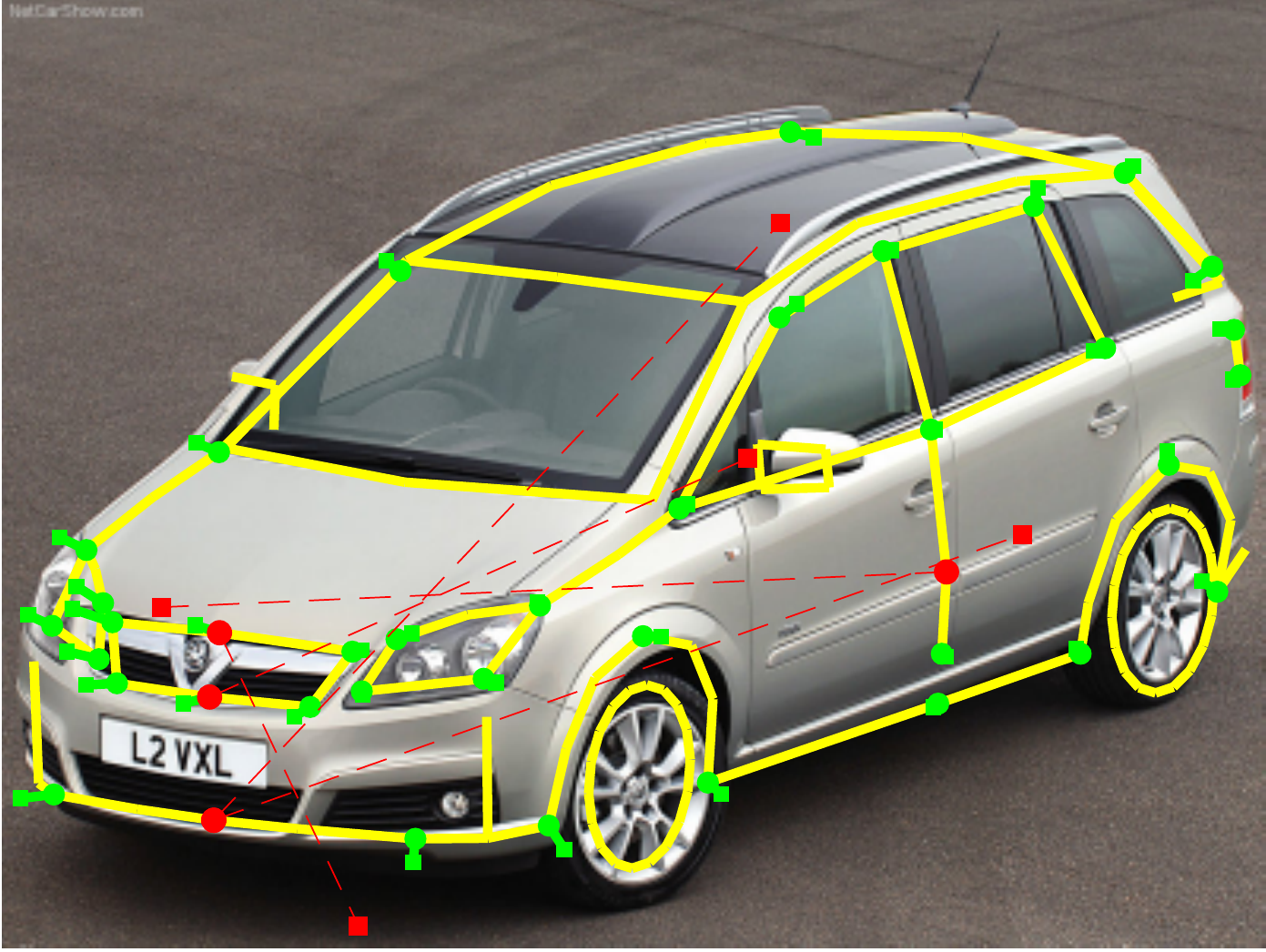} \\
	\vspace{1mm}
	\end{minipage}
& \myhspace
	\begin{minipage}{\mpwthree}%
	\centering%
	\includegraphics[width=\columnwidth]{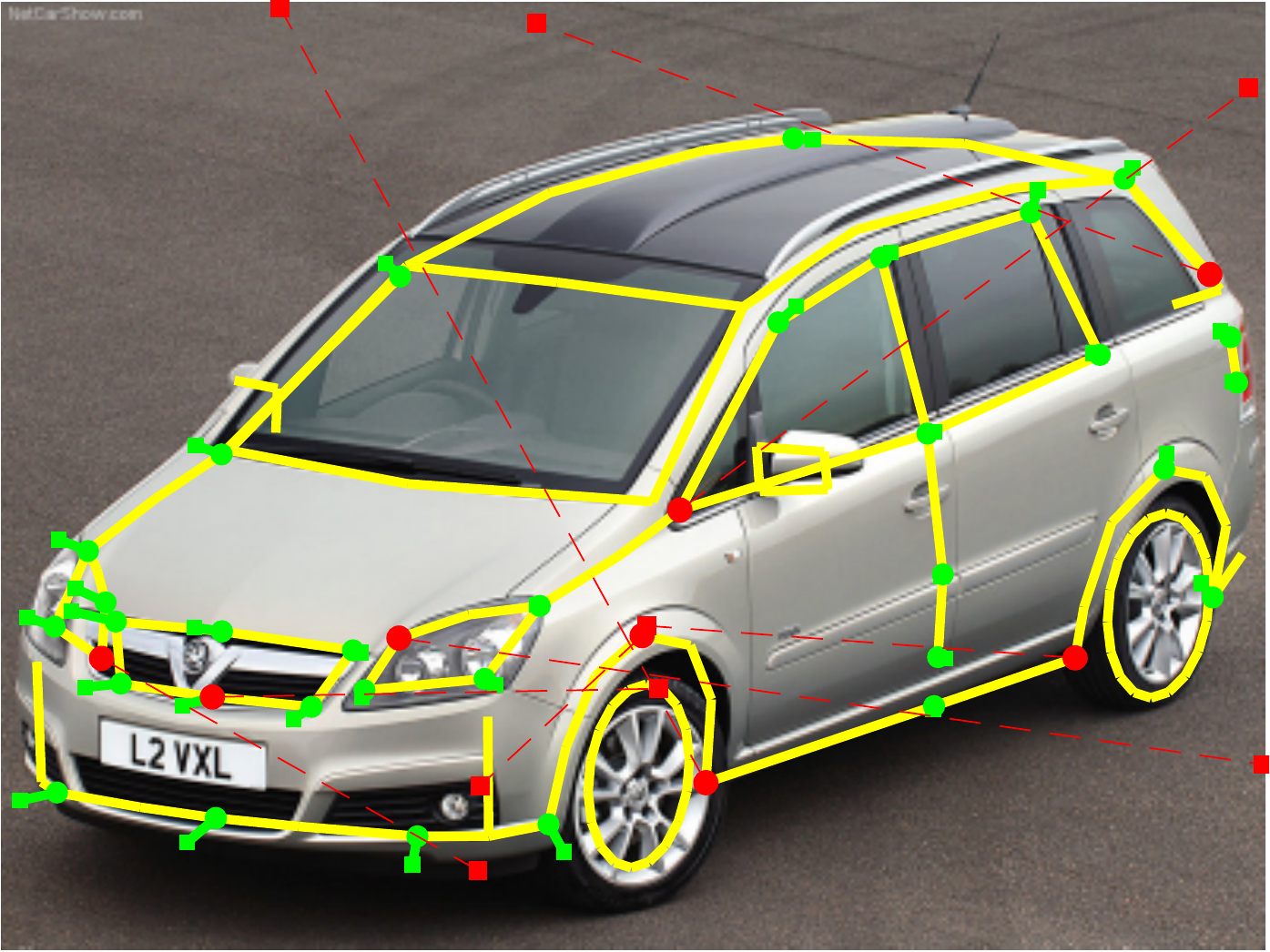} \\
	\vspace{1mm}
	\end{minipage}
& \myhspace
	\begin{minipage}{\mpwthree}%
	\centering%
	\includegraphics[width=\columnwidth]{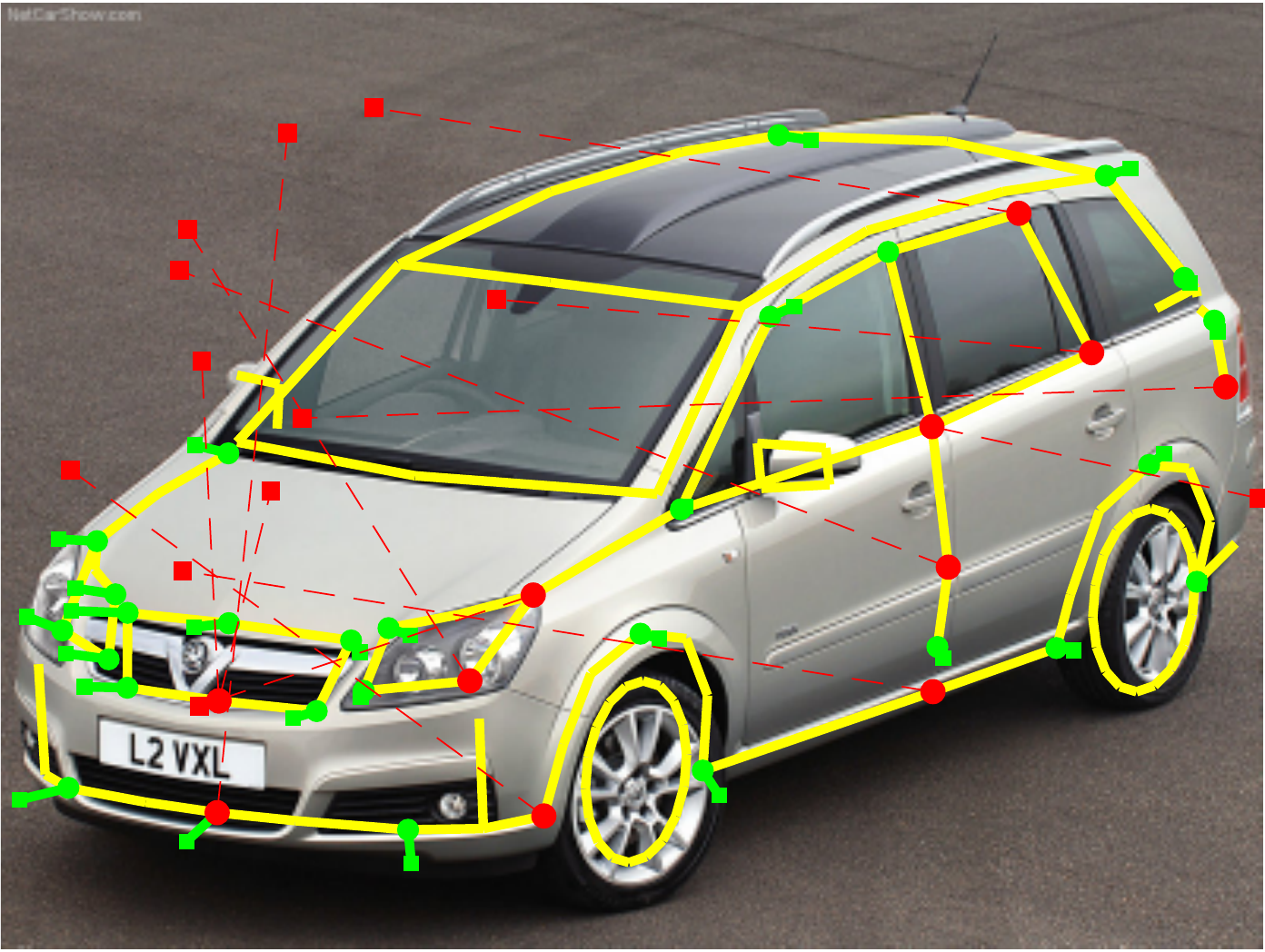} \\
	\vspace{1mm}
	\end{minipage} 
& \myhspace
	\begin{minipage}{\mpwthree}%
	\centering%
	\includegraphics[width=\columnwidth]{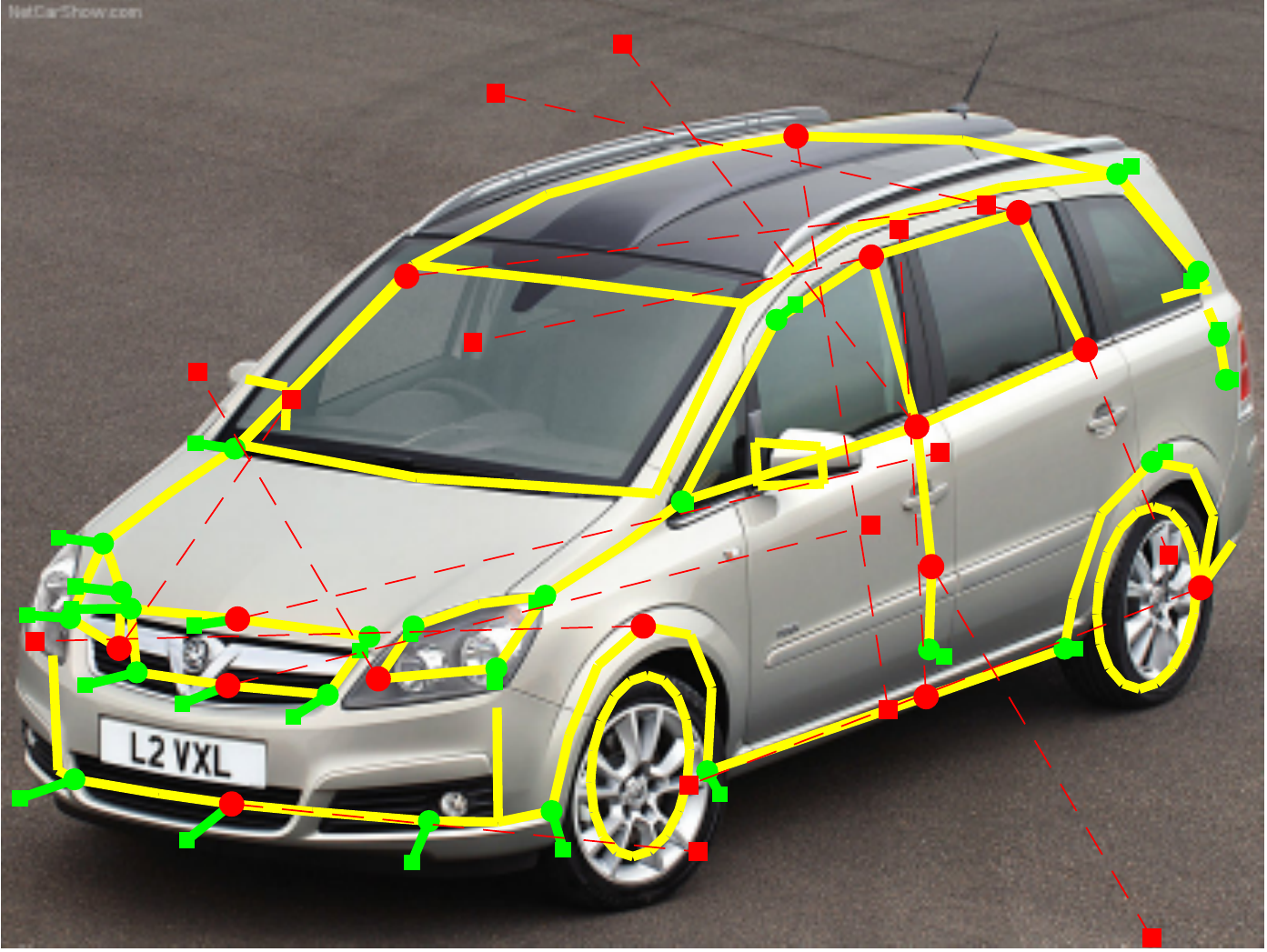} \\
	\vspace{1mm}
	\end{minipage}
& \myhspace
	\begin{minipage}{\mpwthree}%
	\centering%
	\includegraphics[width=\columnwidth]{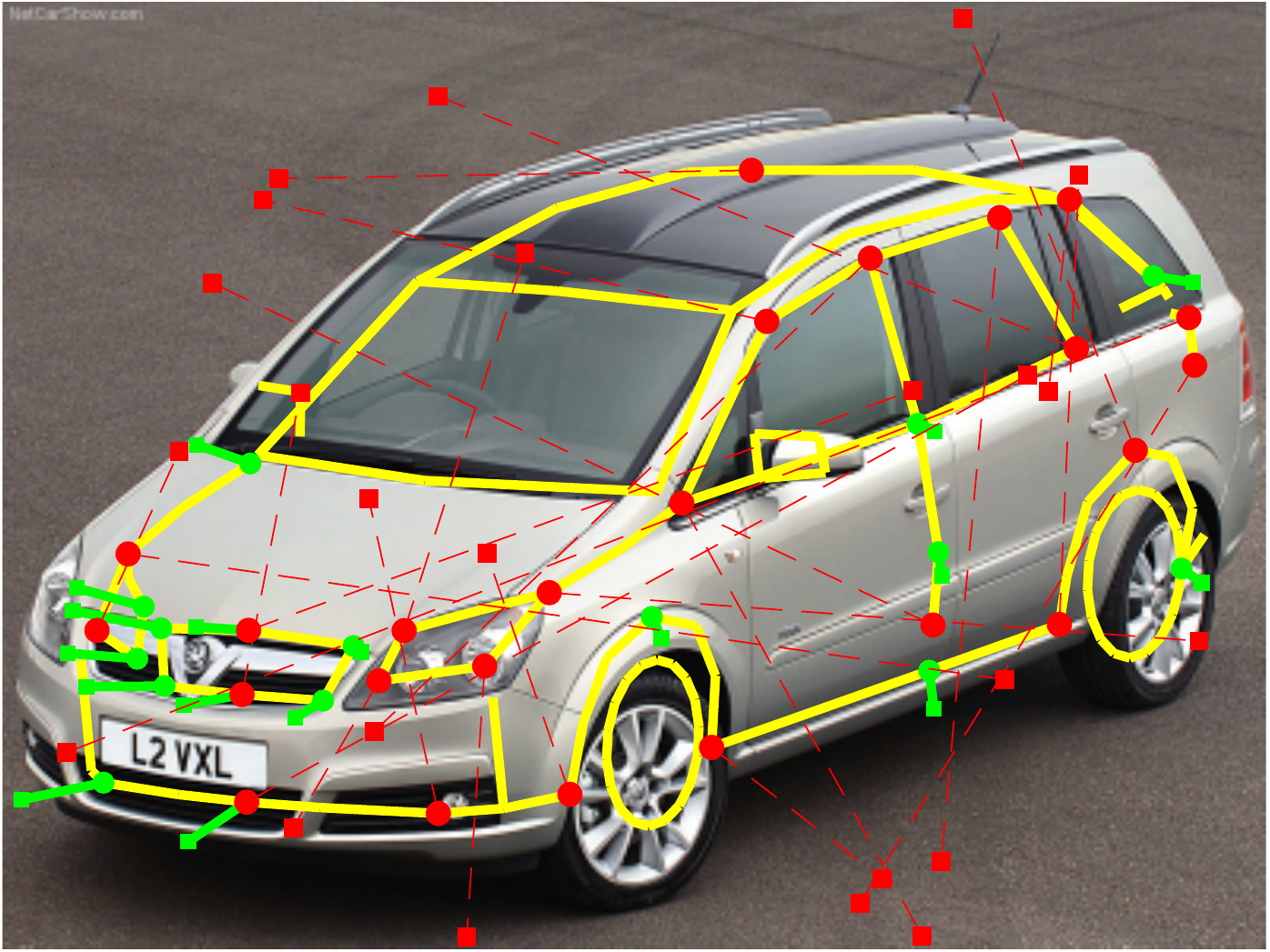} \\
	\vspace{1mm}
	\end{minipage}
&  \myhspace
	\begin{minipage}{\mpwthree}%
	\centering%
	\includegraphics[width=\columnwidth]{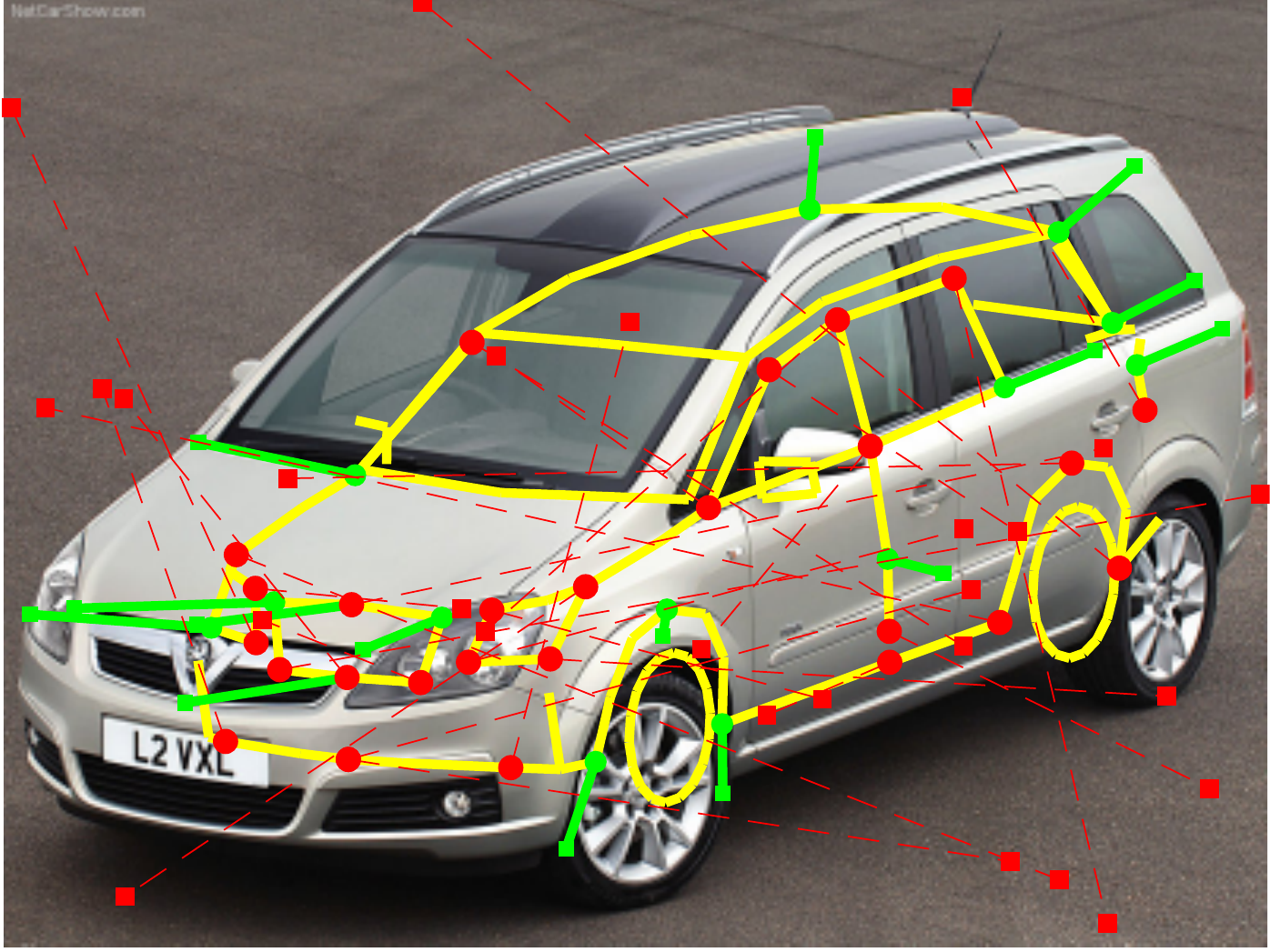} \\
	\vspace{1mm}
	\end{minipage} 
&  \myhspace
	\begin{minipage}{\mpwthree}%
	\centering%
	\includegraphics[width=\columnwidth]{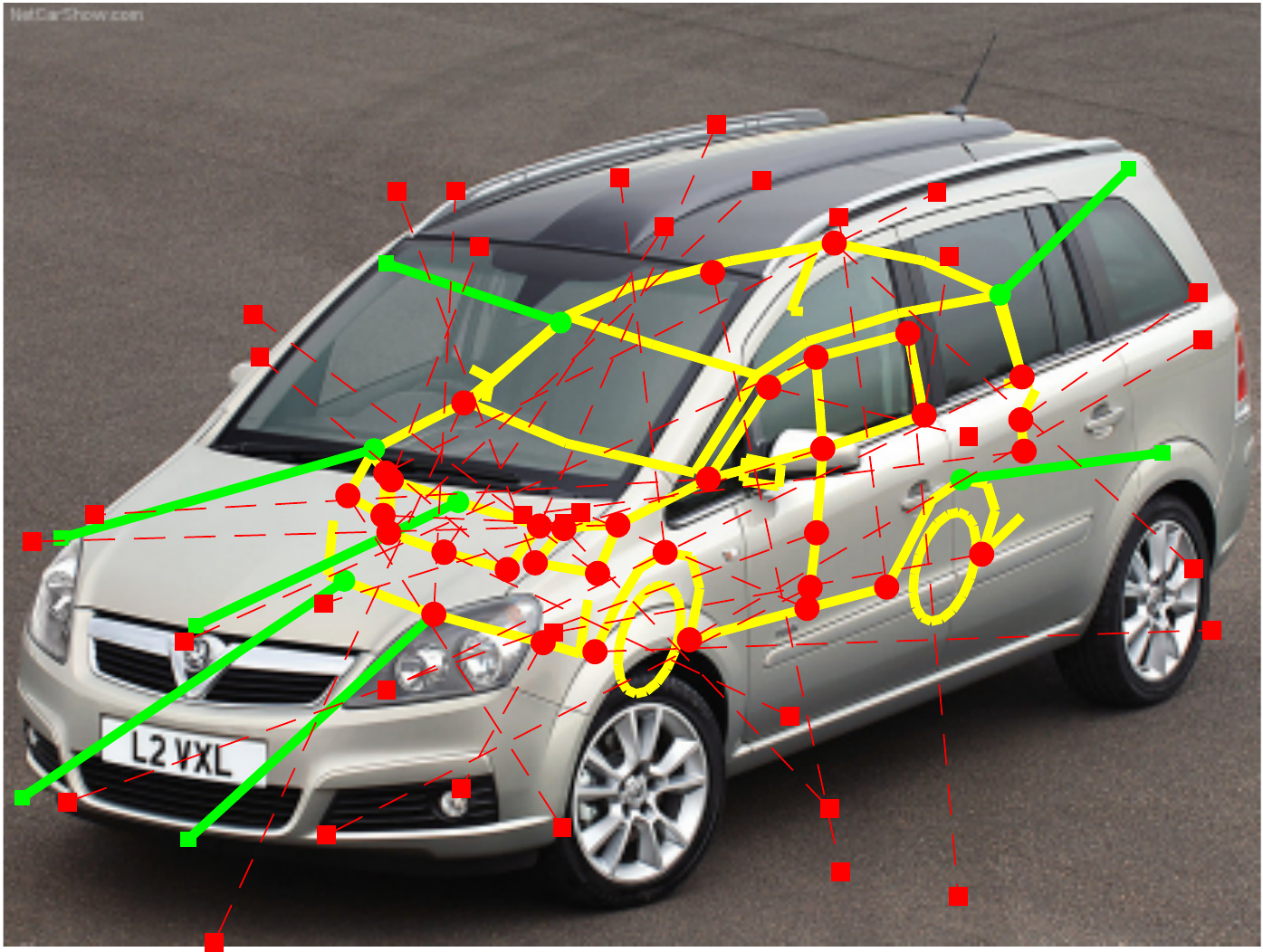} \\
	\vspace{1mm}
	\end{minipage} \\
\myhspace \myhspace \hspace{-2mm} \rotatebox{90}{\hspace{-8mm} {\smaller \convexrobust} } & 
\myhspace
	\begin{minipage}{\mpwthree}%
	\centering%
	\includegraphics[width=\columnwidth]{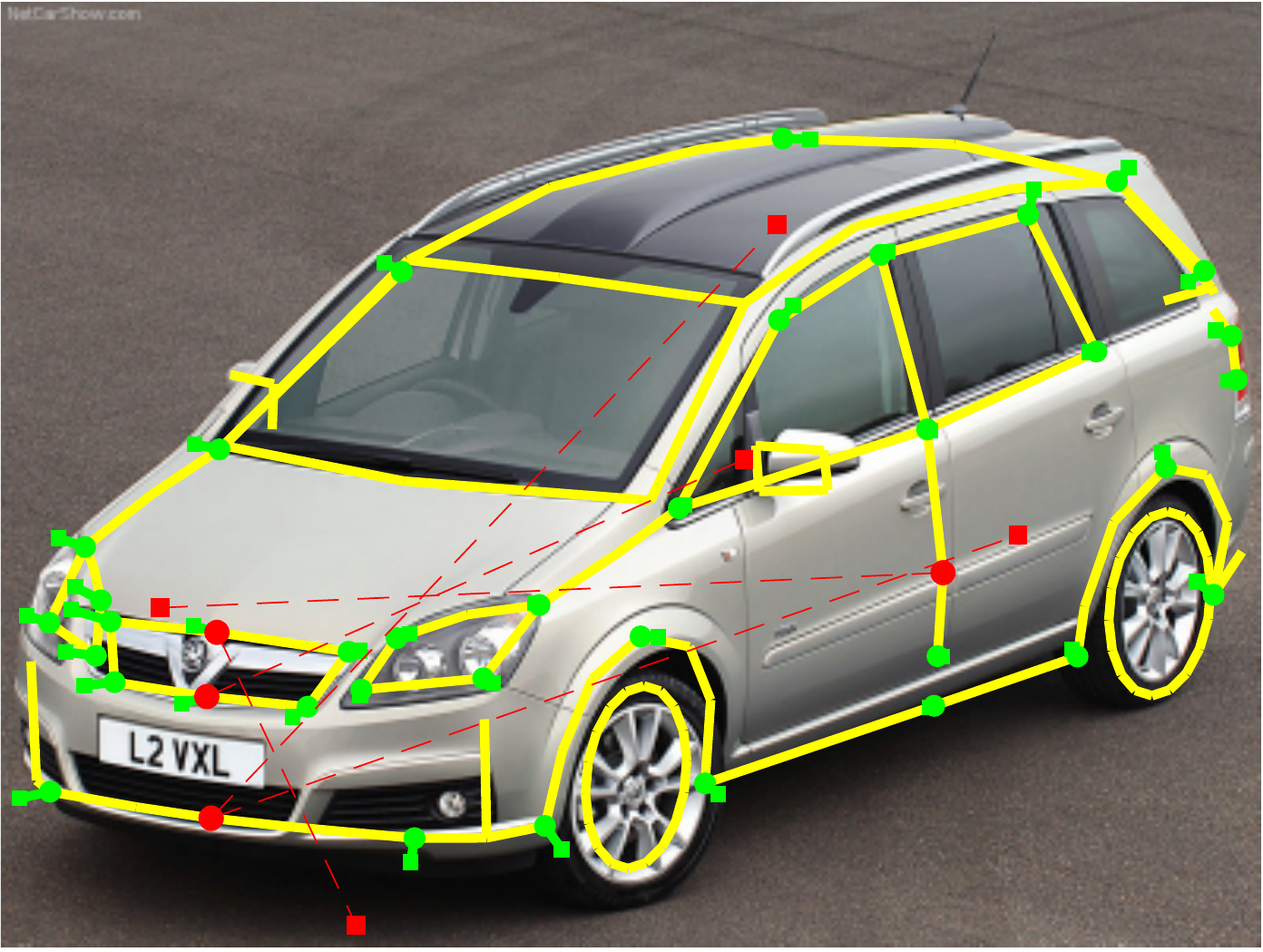} \\
	\vspace{1mm}
	\end{minipage}
& \myhspace
	\begin{minipage}{\mpwthree}%
	\centering%
	\includegraphics[width=\columnwidth]{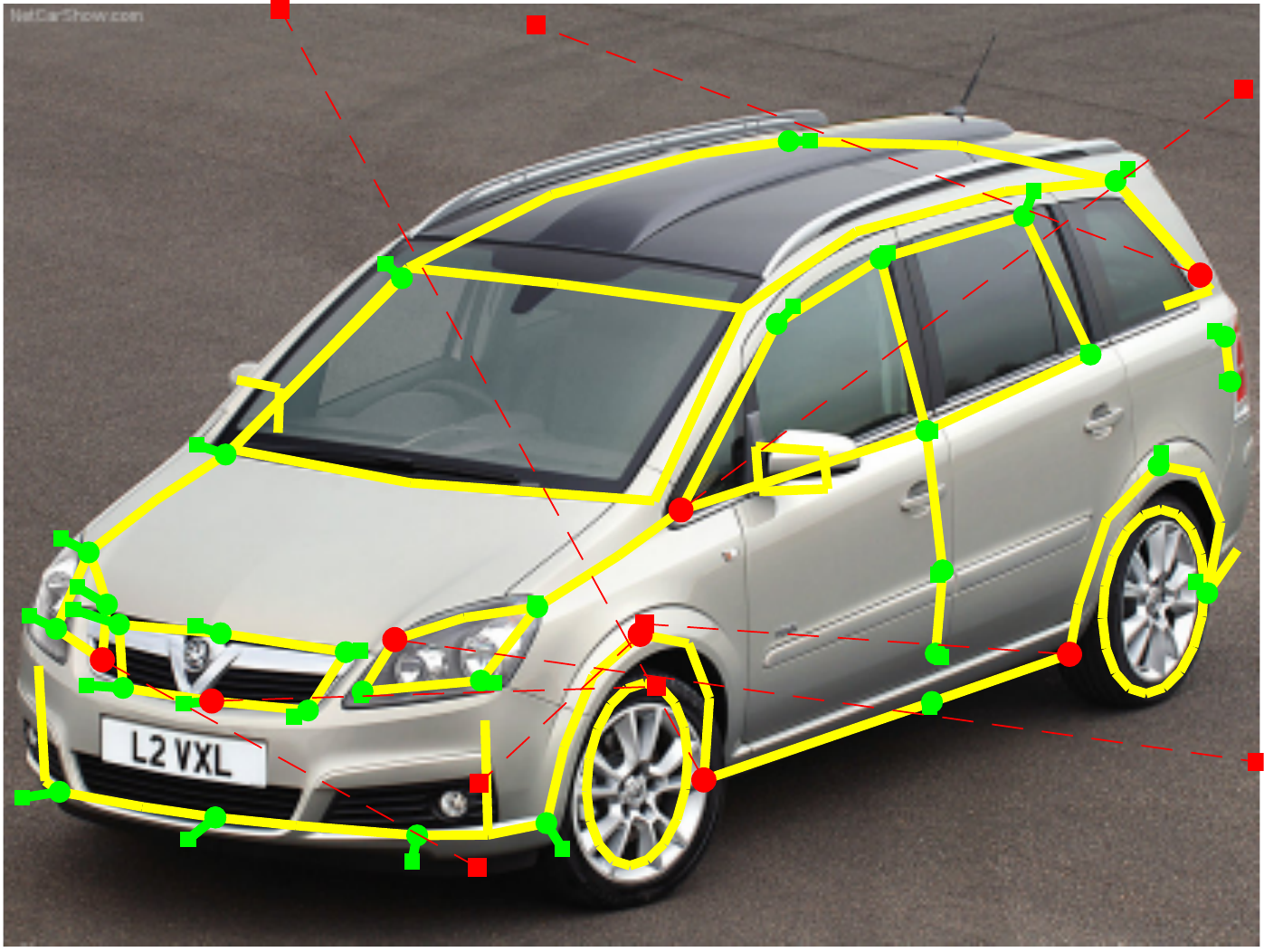} \\
	\vspace{1mm}
	\end{minipage}
& \myhspace
	\begin{minipage}{\mpwthree}%
	\centering%
	\includegraphics[width=\columnwidth]{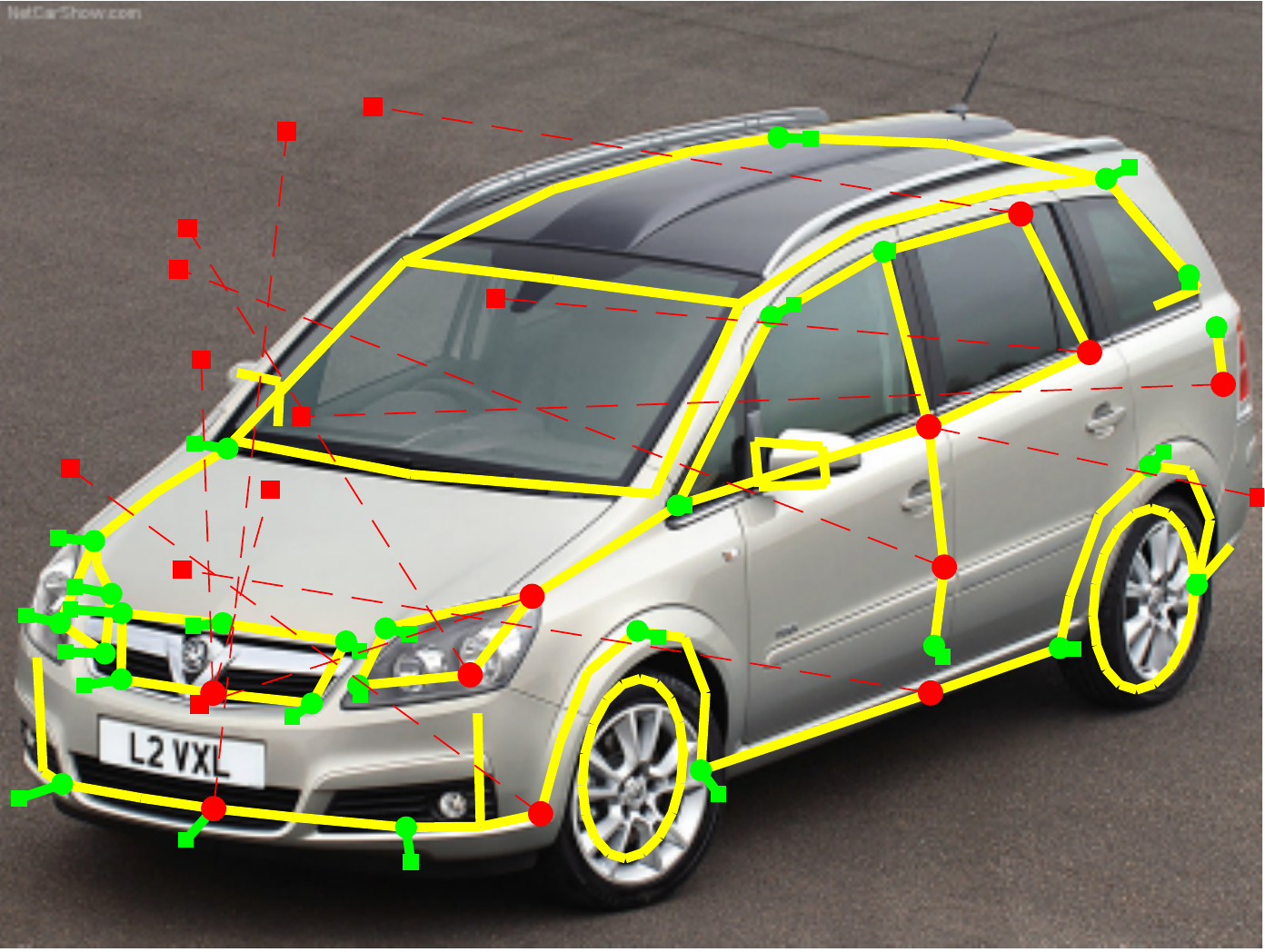} \\
	\vspace{1mm}
	\end{minipage} 
& \myhspace
	\begin{minipage}{\mpwthree}%
	\centering%
	\includegraphics[width=\columnwidth]{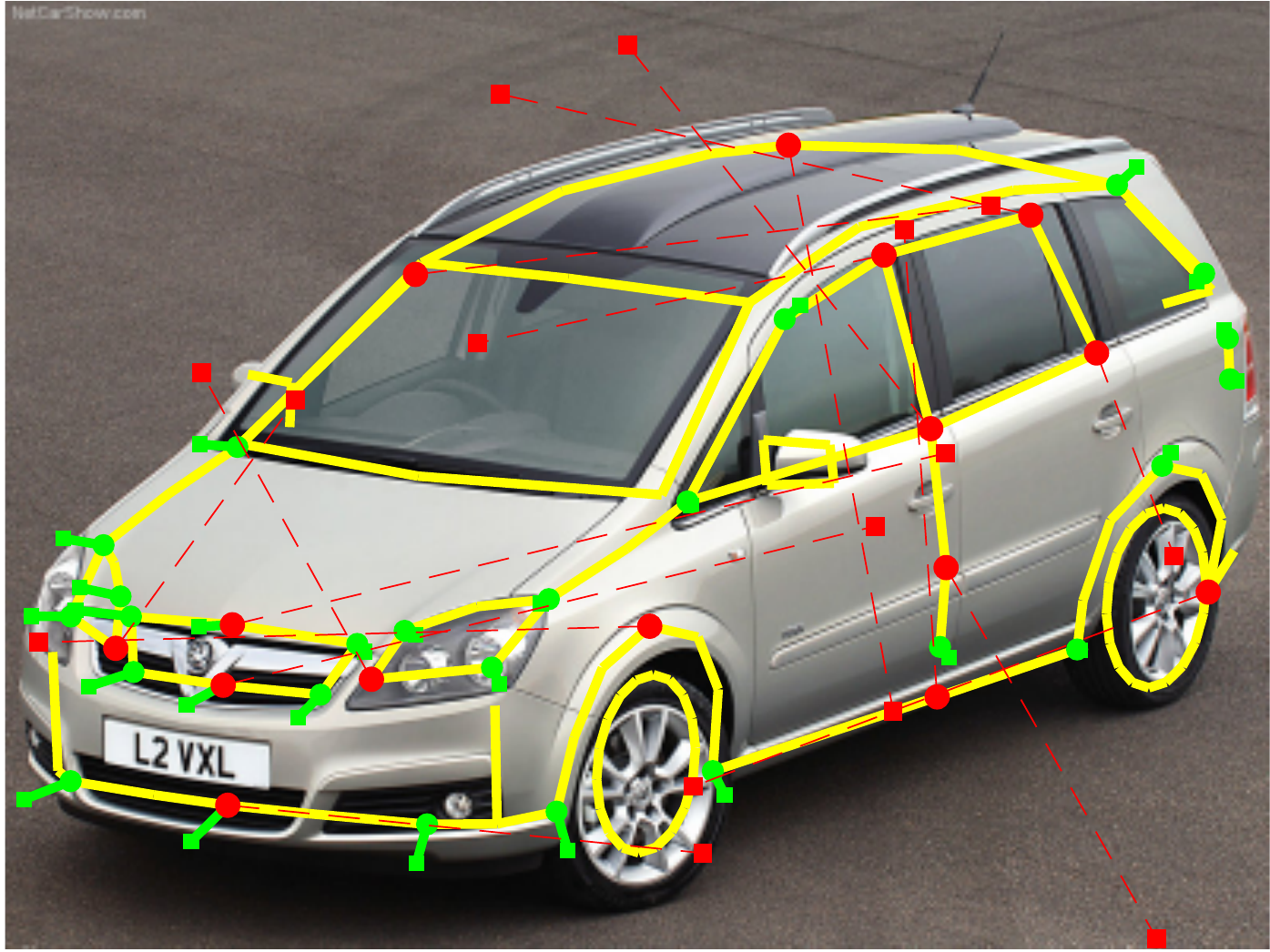} \\
	\vspace{1mm}
	\end{minipage}
& \myhspace
	\begin{minipage}{\mpwthree}%
	\centering%
	\includegraphics[width=\columnwidth]{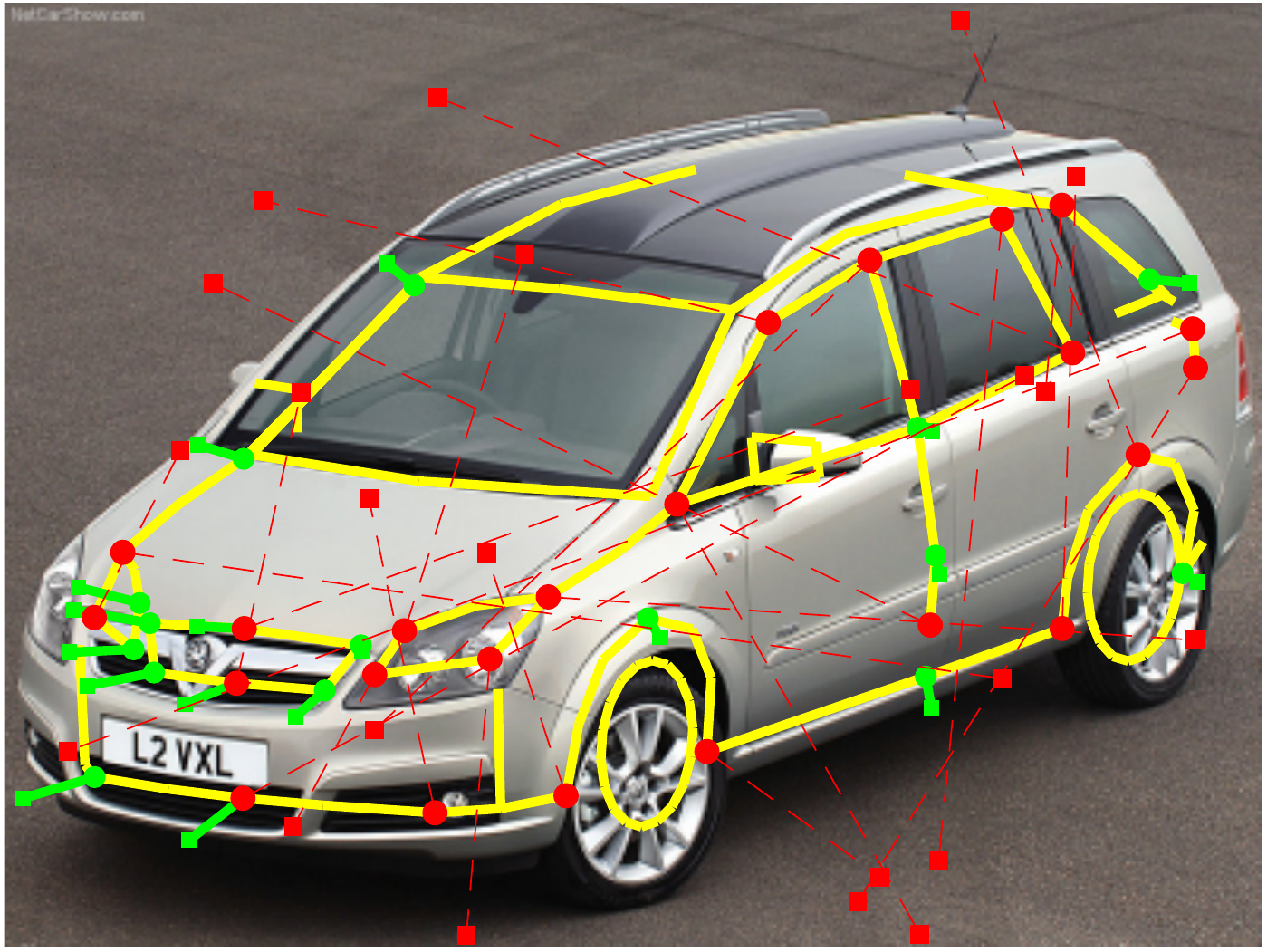} \\
	\vspace{1mm}
	\end{minipage}
&  \myhspace
	\begin{minipage}{\mpwthree}%
	\centering%
	\includegraphics[width=\columnwidth]{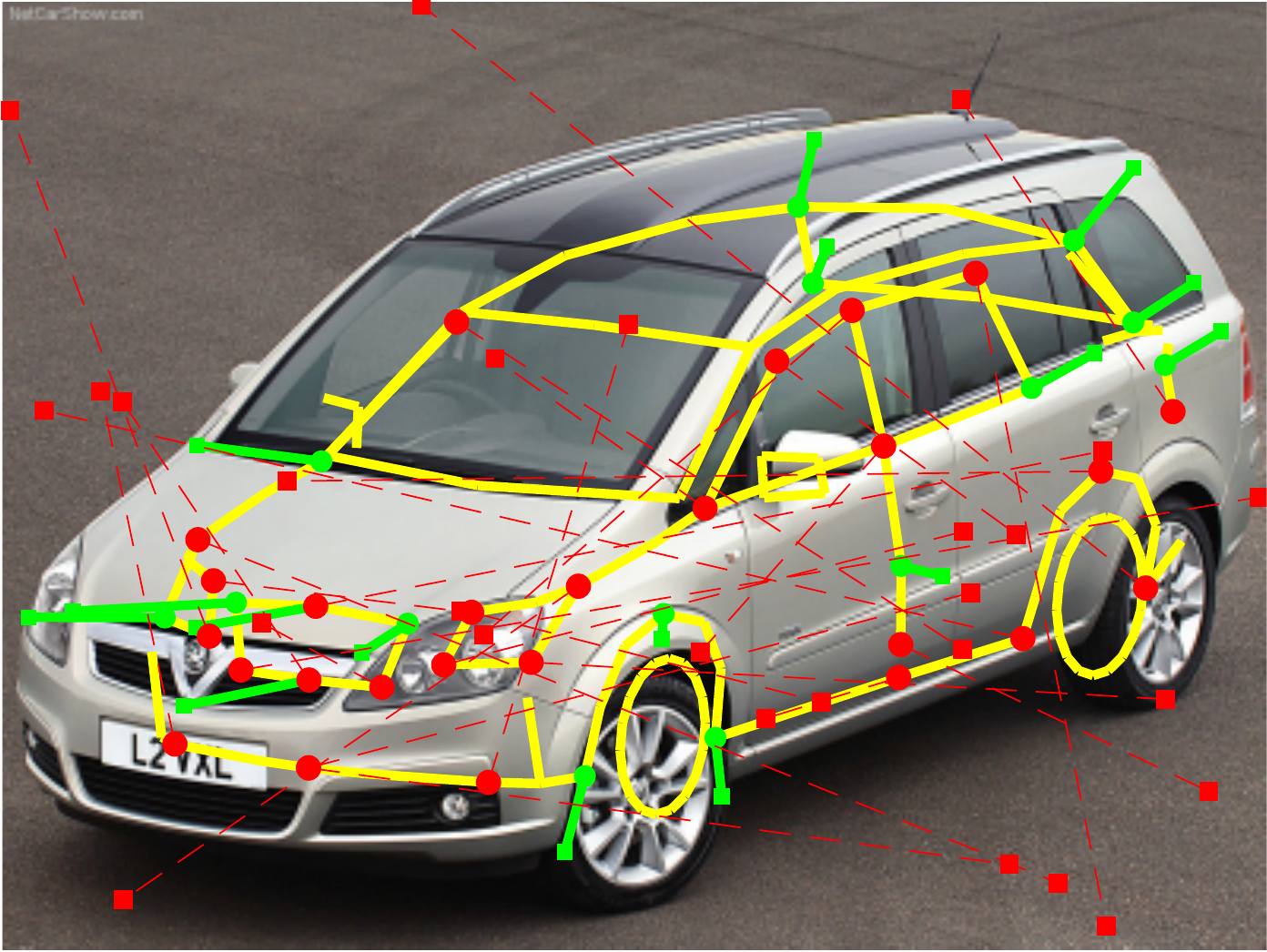} \\
	\vspace{1mm}
	\end{minipage} 
&  \myhspace
	\begin{minipage}{\mpwthree}%
	\centering%
	\includegraphics[width=\columnwidth]{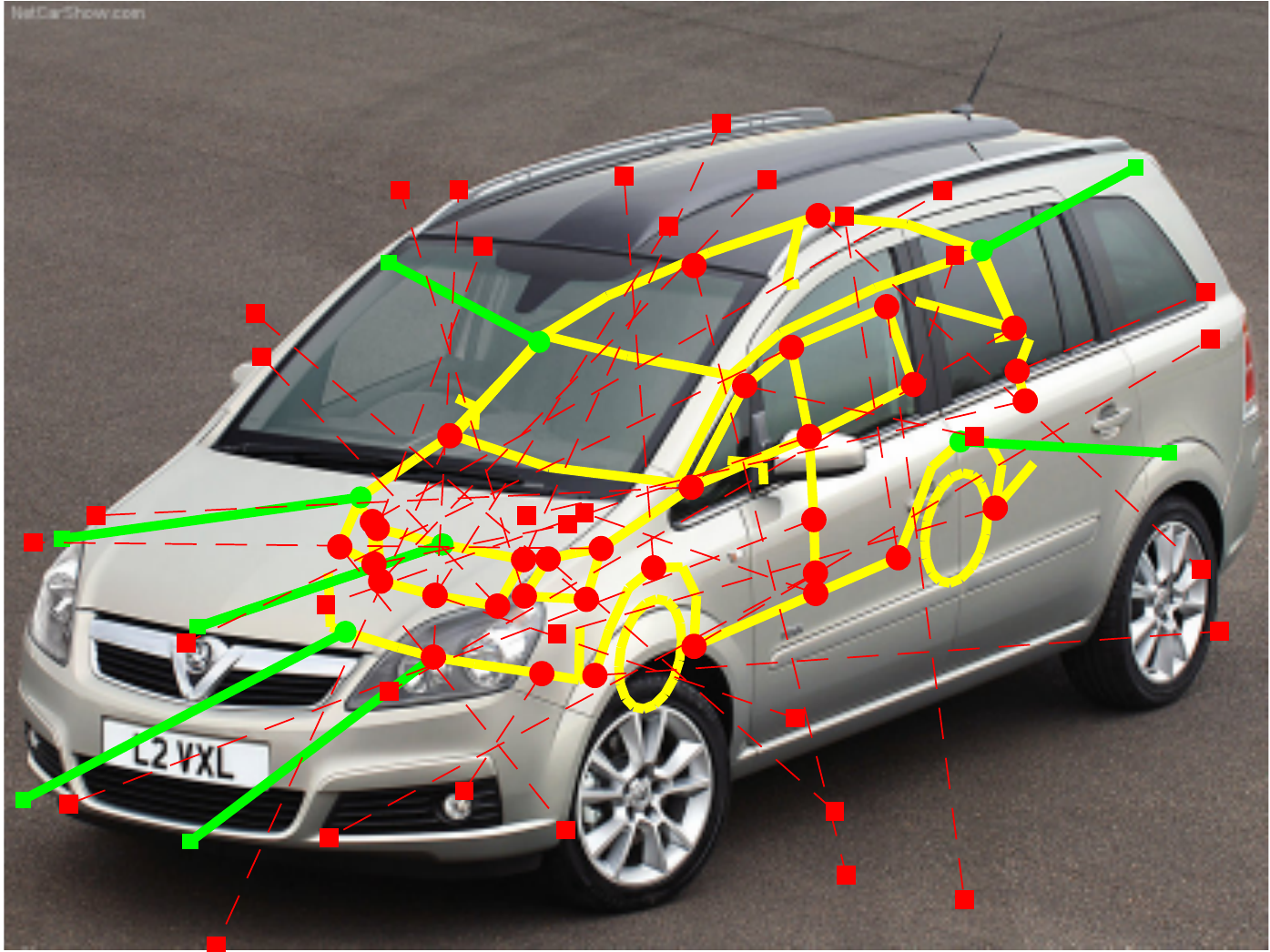} \\
	\vspace{1mm}
	\end{minipage} \\
\myhspace \myhspace \hspace{-2mm} \rotatebox{90}{\hspace{-3mm} {\smaller \blue{ \namerobust}} } & 
\myhspace
	\begin{minipage}{\mpwthree}%
	\centering%
	\includegraphics[width=\columnwidth]{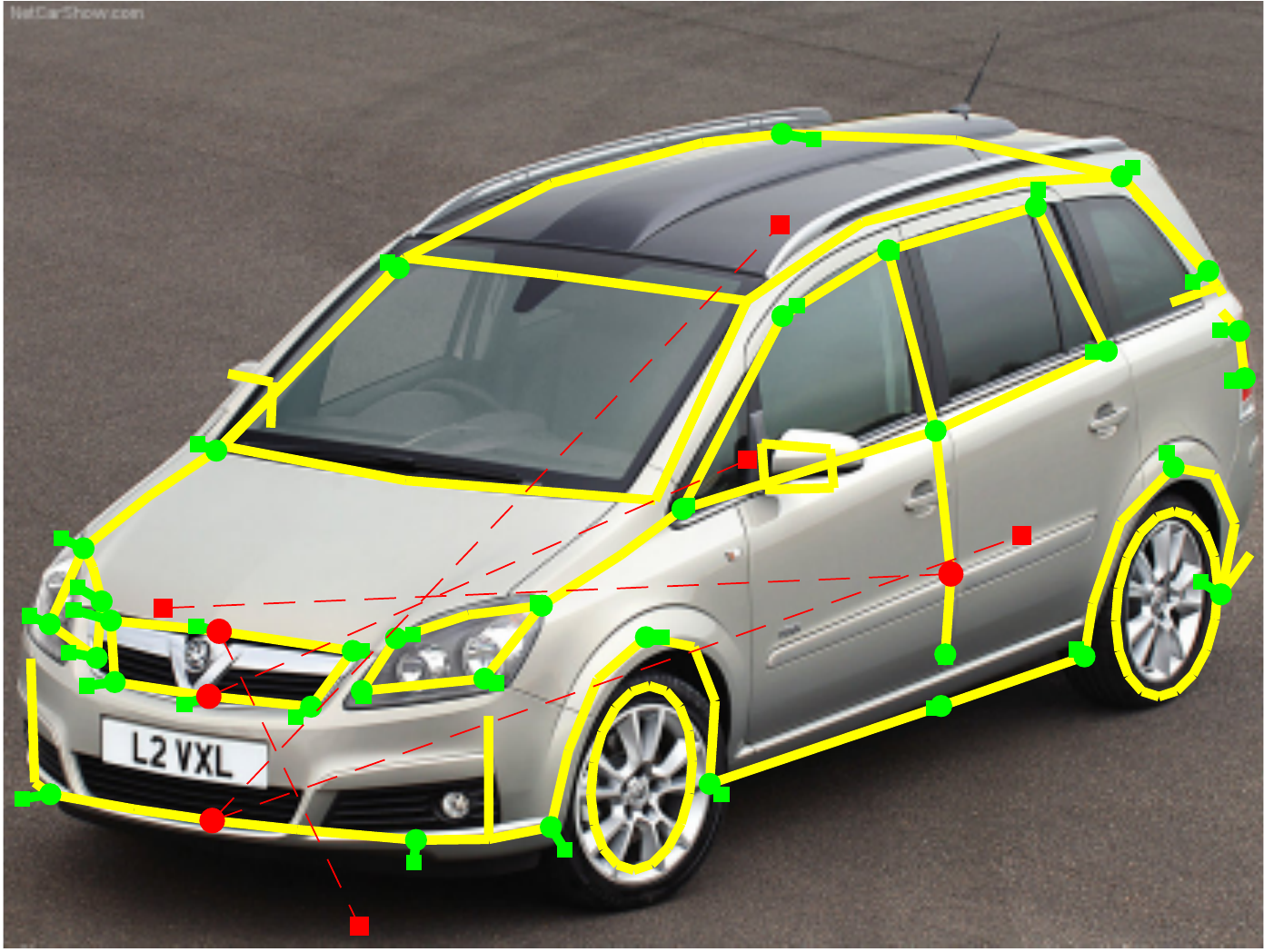} \\
	\vspace{1mm}
	\end{minipage}
& \myhspace
	\begin{minipage}{\mpwthree}%
	\centering%
	\includegraphics[width=\columnwidth]{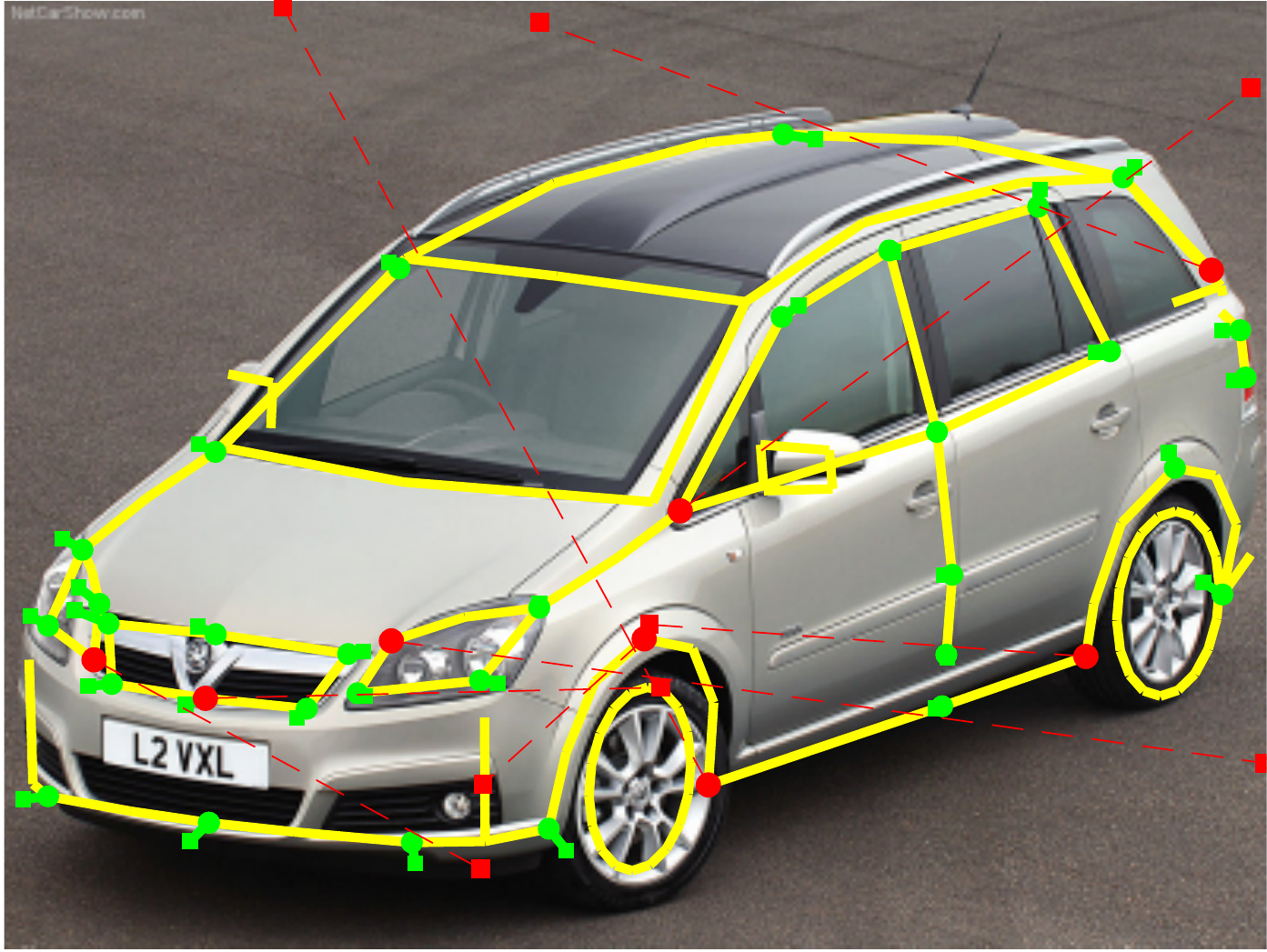} \\
	\vspace{1mm}
	\end{minipage}
& \myhspace
	\begin{minipage}{\mpwthree}%
	\centering%
	\includegraphics[width=\columnwidth]{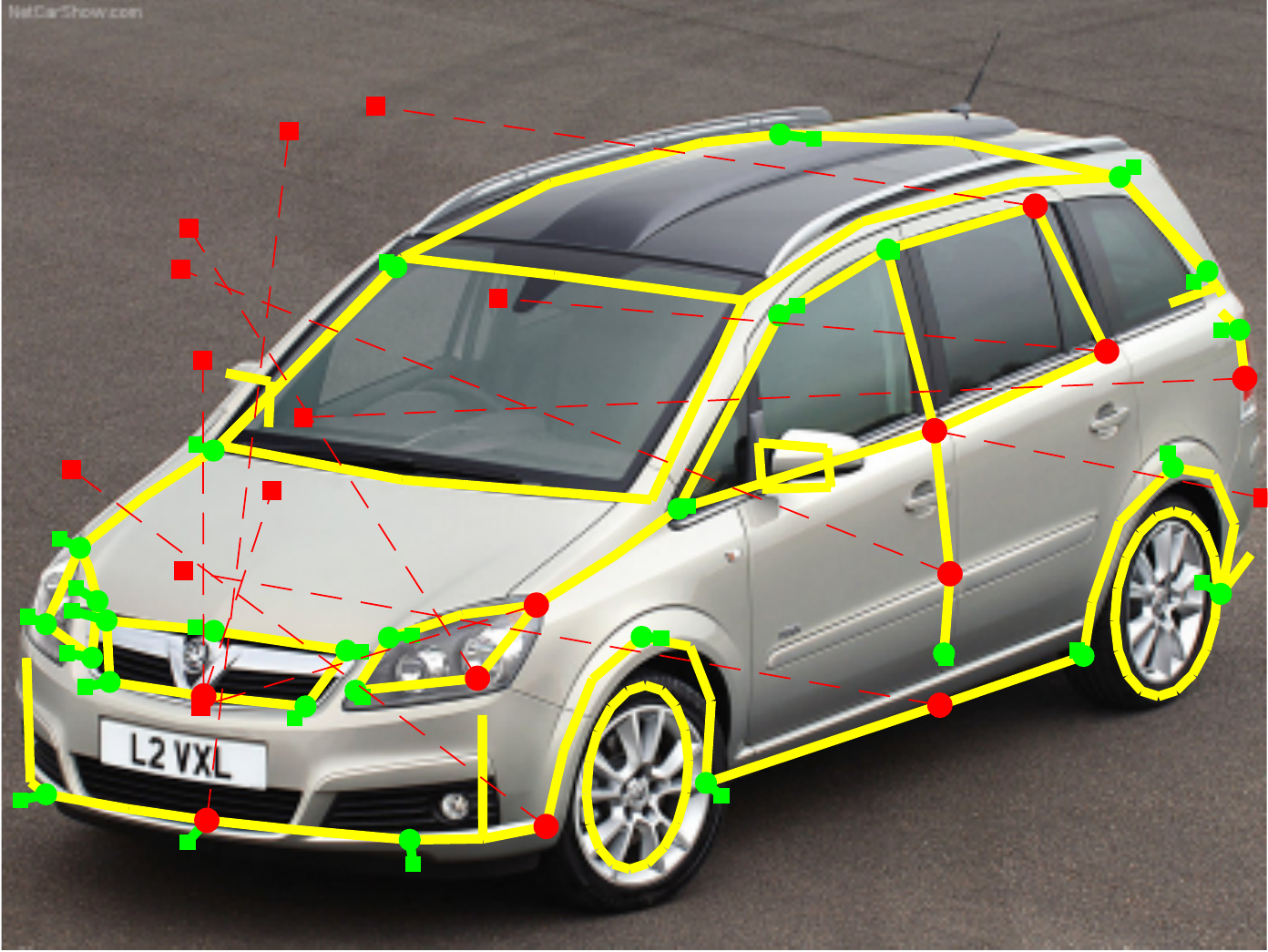} \\
	\vspace{1mm}
	\end{minipage} 
& \myhspace
	\begin{minipage}{\mpwthree}%
	\centering%
	\includegraphics[width=\columnwidth]{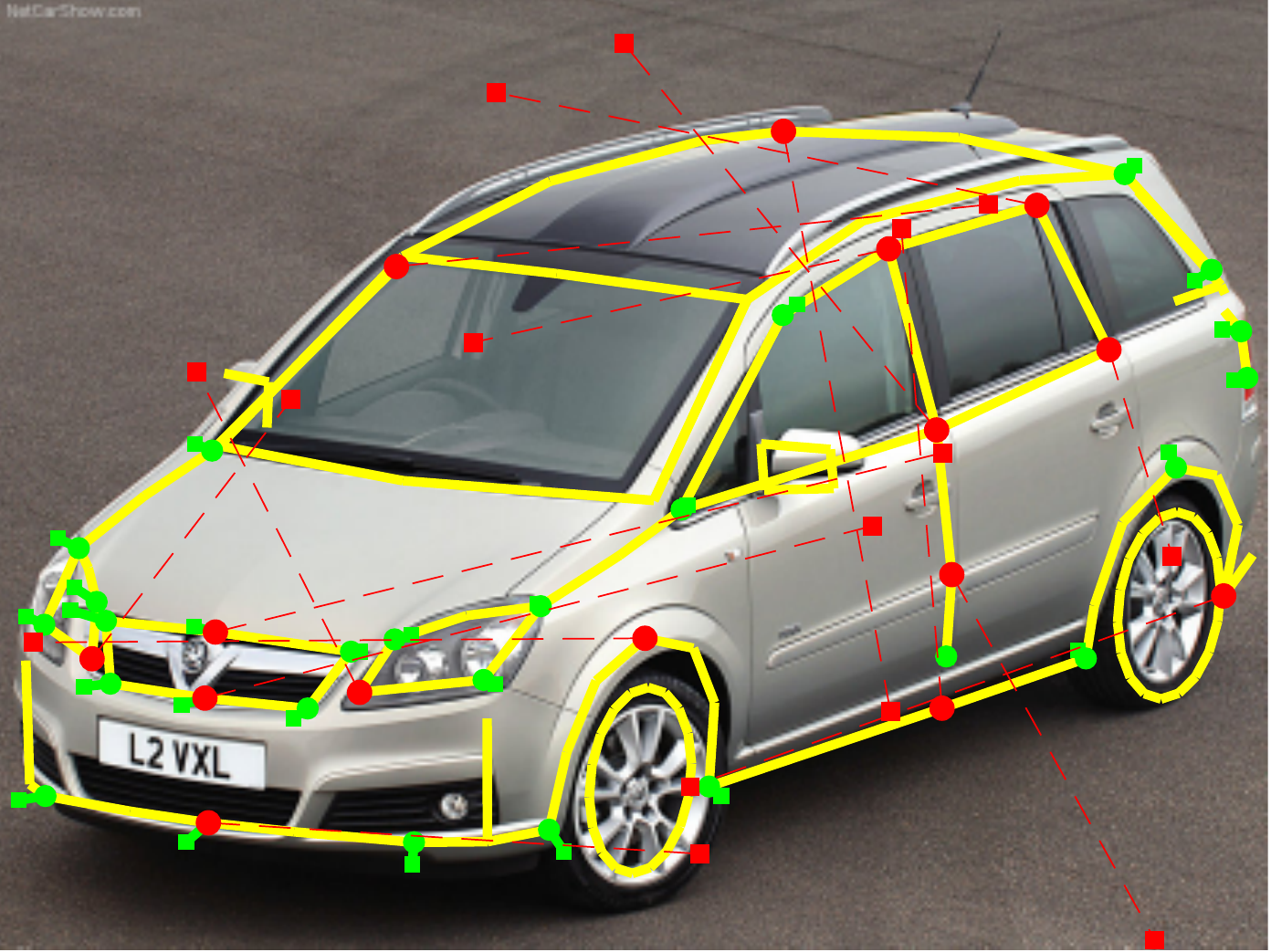} \\
	\vspace{1mm}
	\end{minipage}
& \myhspace
	\begin{minipage}{\mpwthree}%
	\centering%
	\includegraphics[width=\columnwidth]{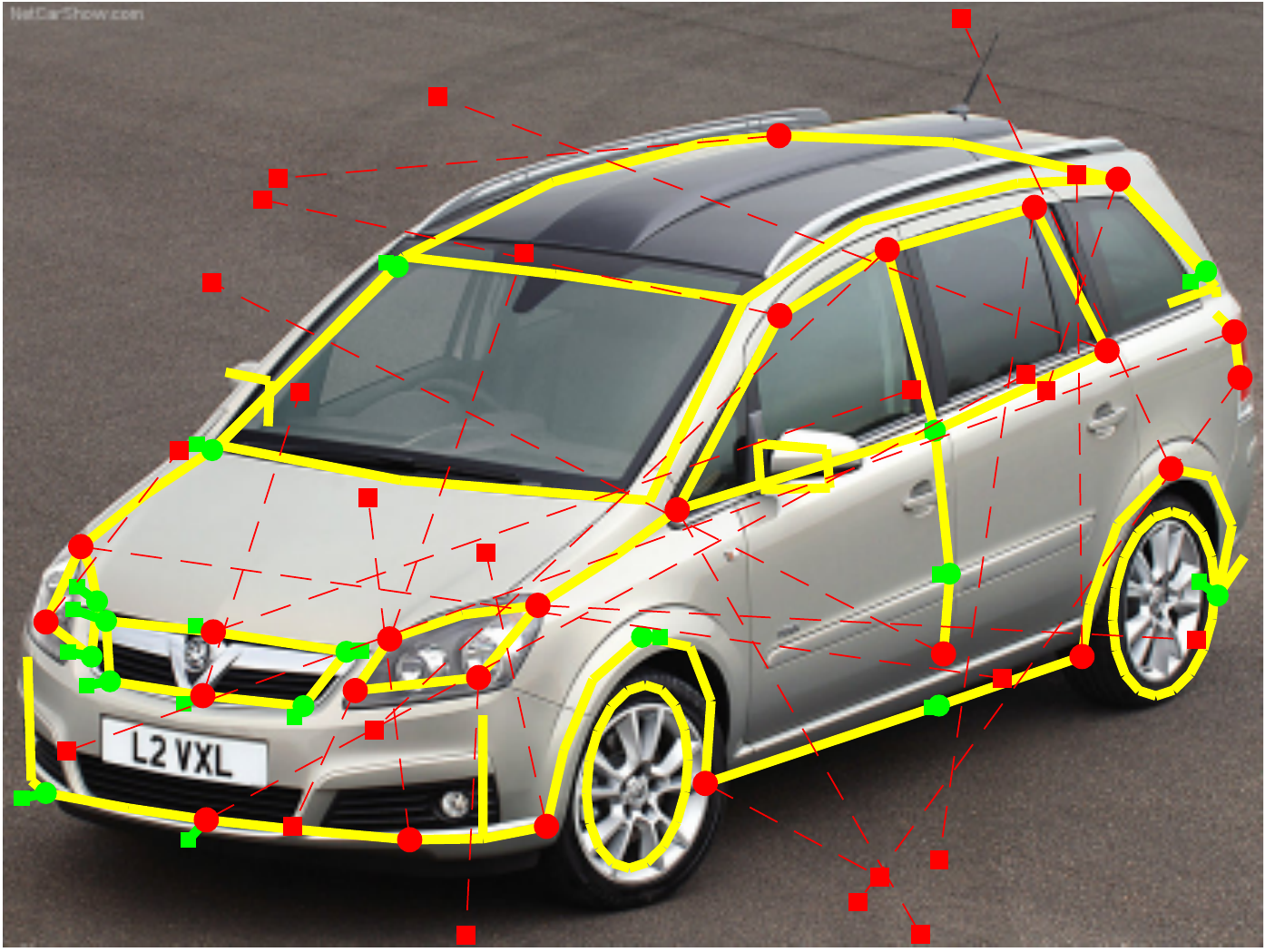} \\
	\vspace{1mm}
	\end{minipage}
&  \myhspace
	\begin{minipage}{\mpwthree}%
	\centering%
	\includegraphics[width=\columnwidth]{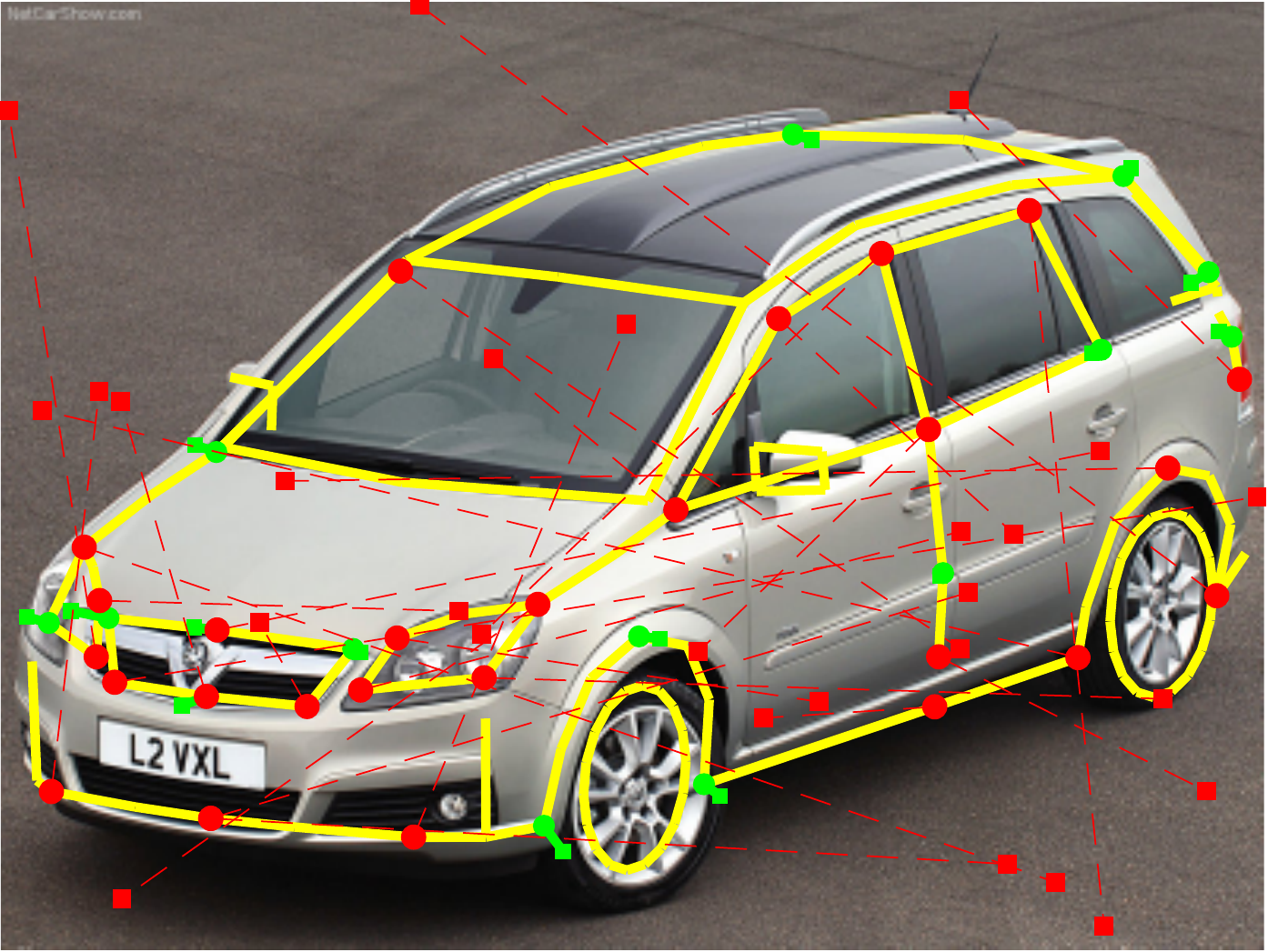} \\
	\vspace{1mm}
	\end{minipage} 
&  \myhspace
	\begin{minipage}{\mpwthree}%
	\centering%
	\includegraphics[width=\columnwidth]{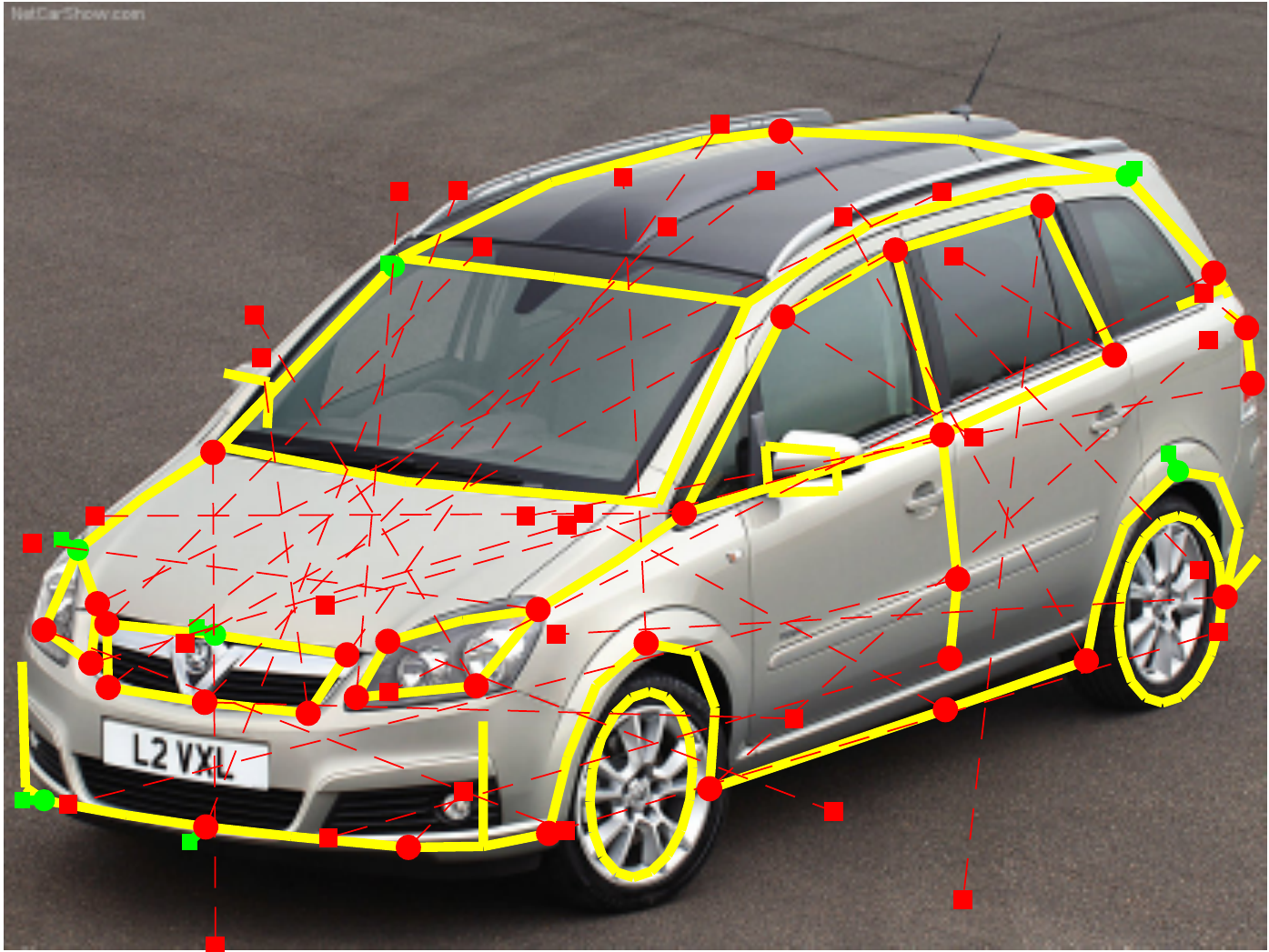} \\
	\vspace{1mm}
	\end{minipage} \\
\multicolumn{8}{c}{Vauxhall Zafira} \\

%% file: 295-Volvo-V70.tex
\myhspace \myhspace \hspace{-2mm} \rotatebox{90}{\hspace{-7mm} {\smaller \alternrobust} } & 
\myhspace
	\begin{minipage}{\mpwthree}%
	\centering%
	\includegraphics[width=\columnwidth]{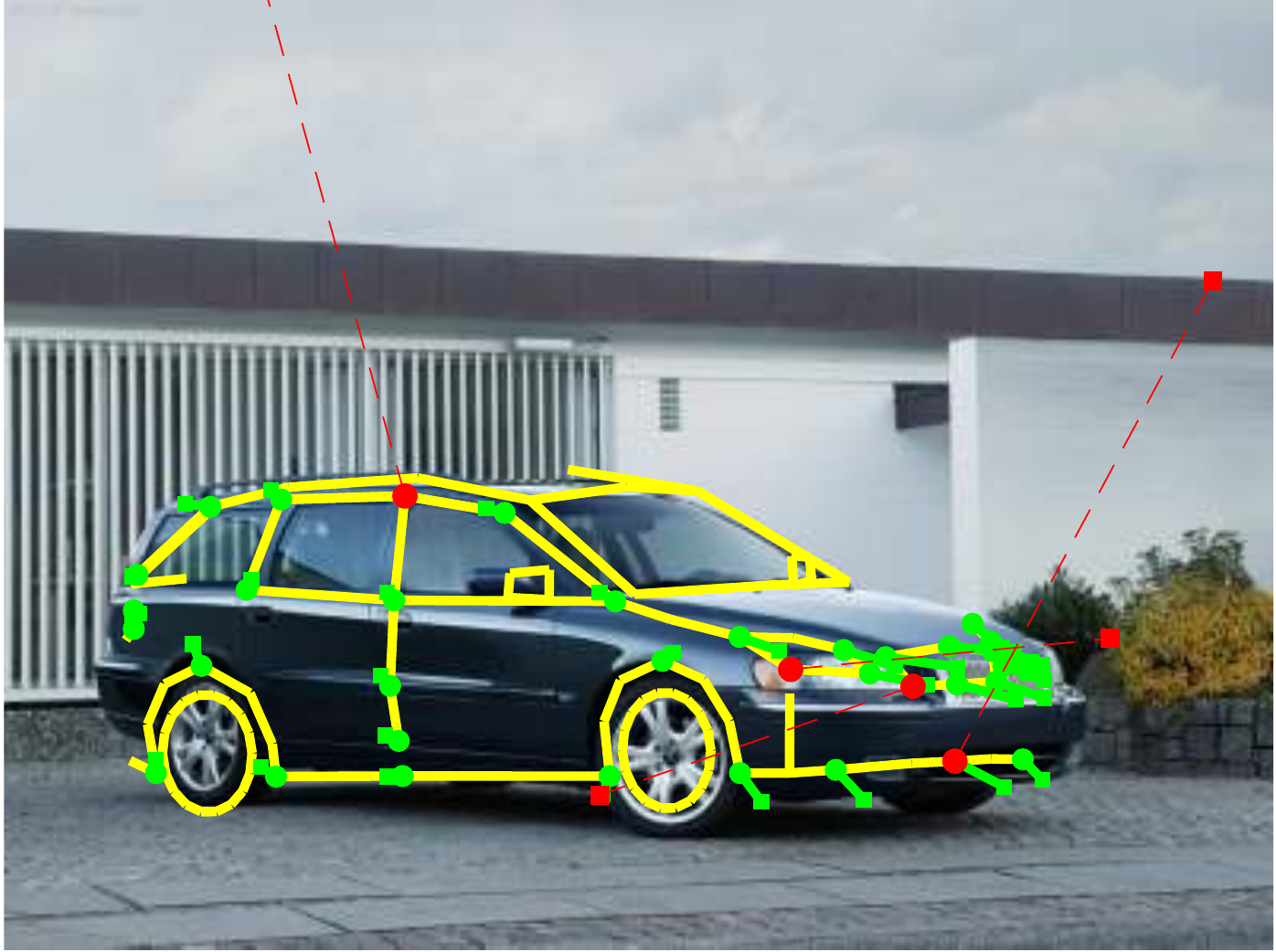} \\
	\vspace{1mm}
	\end{minipage}
& \myhspace
	\begin{minipage}{\mpwthree}%
	\centering%
	\includegraphics[width=\columnwidth]{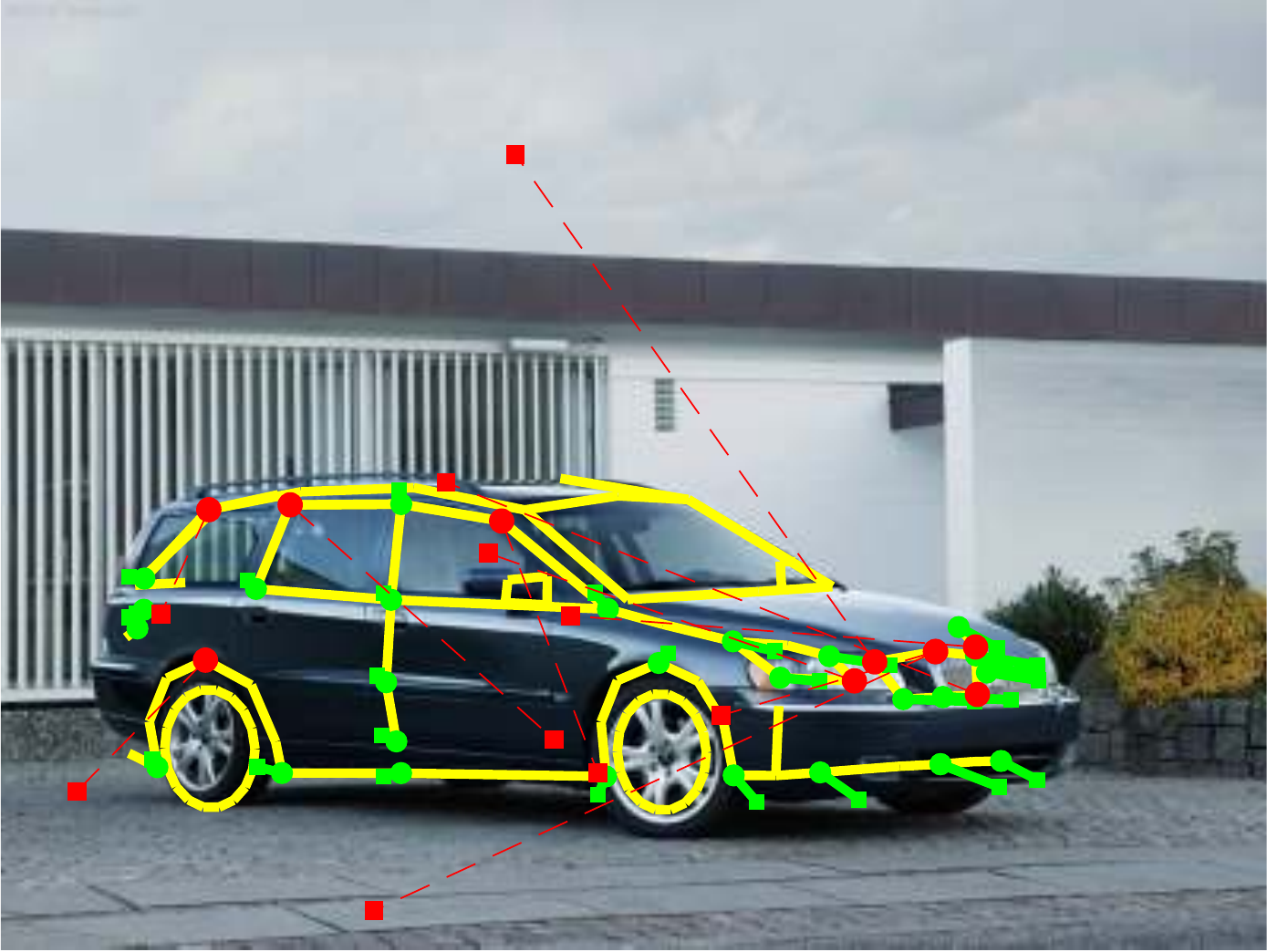} \\
	\vspace{1mm}
	\end{minipage}
& \myhspace
	\begin{minipage}{\mpwthree}%
	\centering%
	\includegraphics[width=\columnwidth]{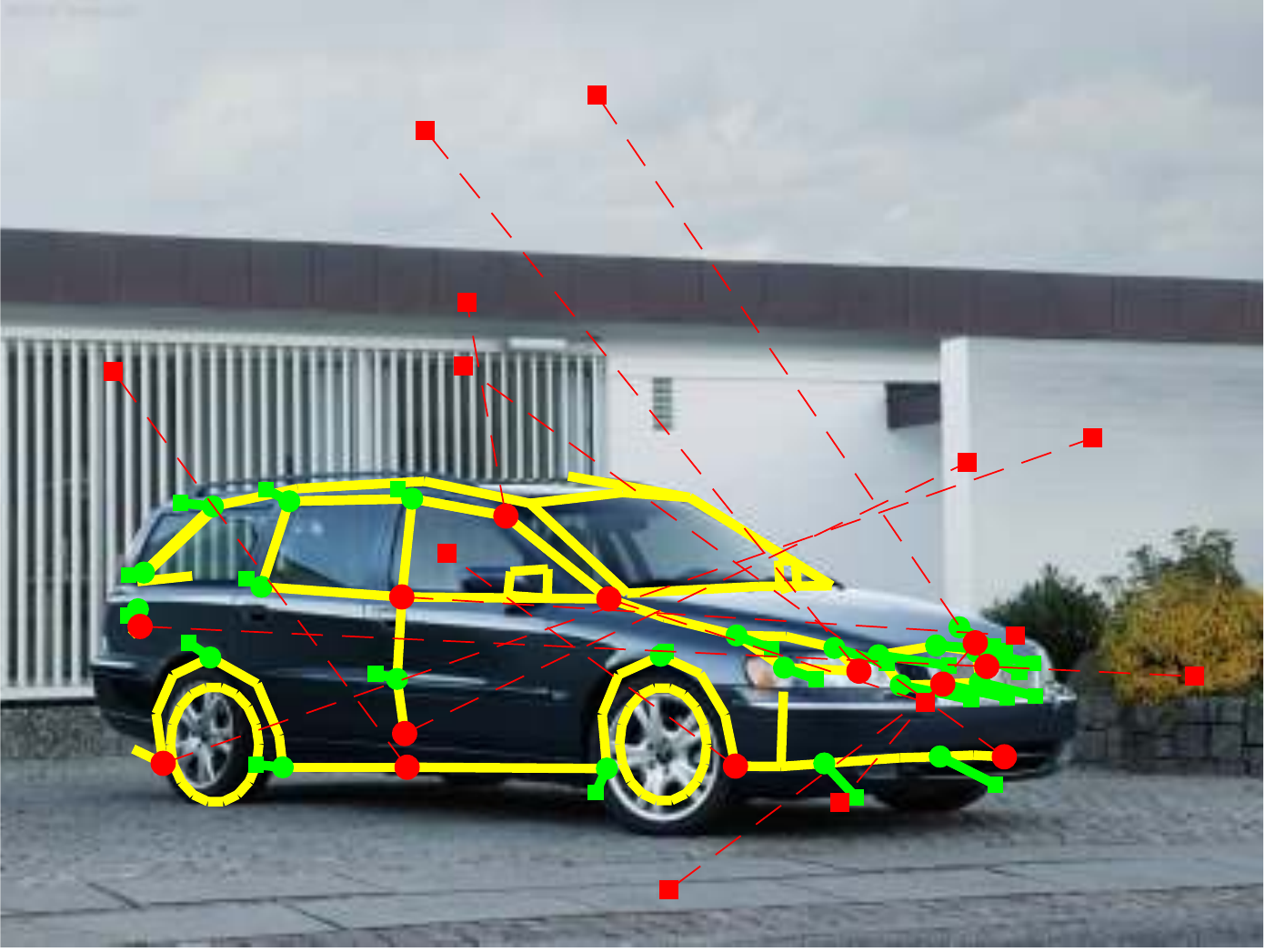} \\
	\vspace{1mm}
	\end{minipage} 
& \myhspace
	\begin{minipage}{\mpwthree}%
	\centering%
	\includegraphics[width=\columnwidth]{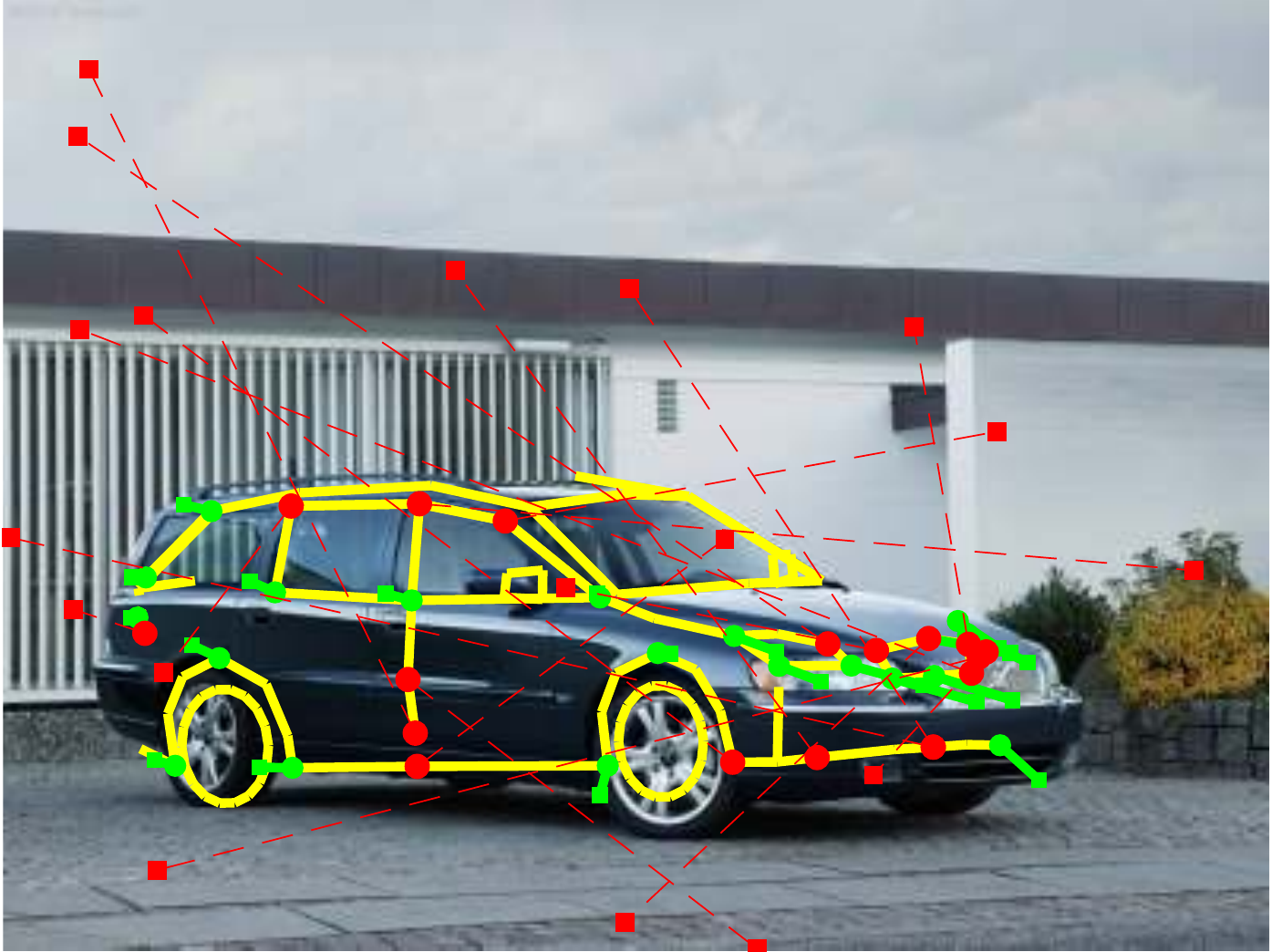} \\
	\vspace{1mm}
	\end{minipage}
& \myhspace
	\begin{minipage}{\mpwthree}%
	\centering%
	\includegraphics[width=\columnwidth]{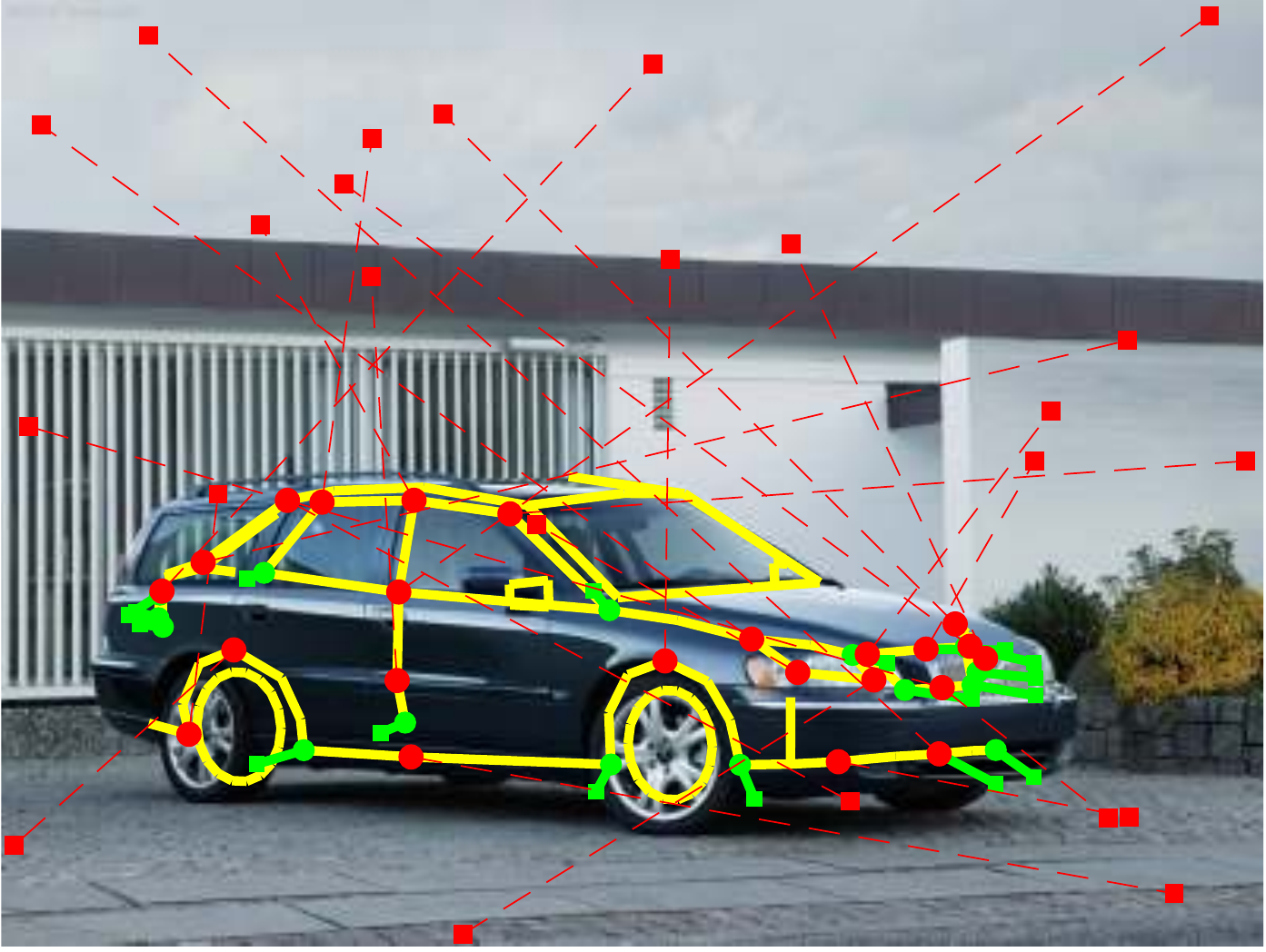} \\
	\vspace{1mm}
	\end{minipage}
&  \myhspace
	\begin{minipage}{\mpwthree}%
	\centering%
	\includegraphics[width=\columnwidth]{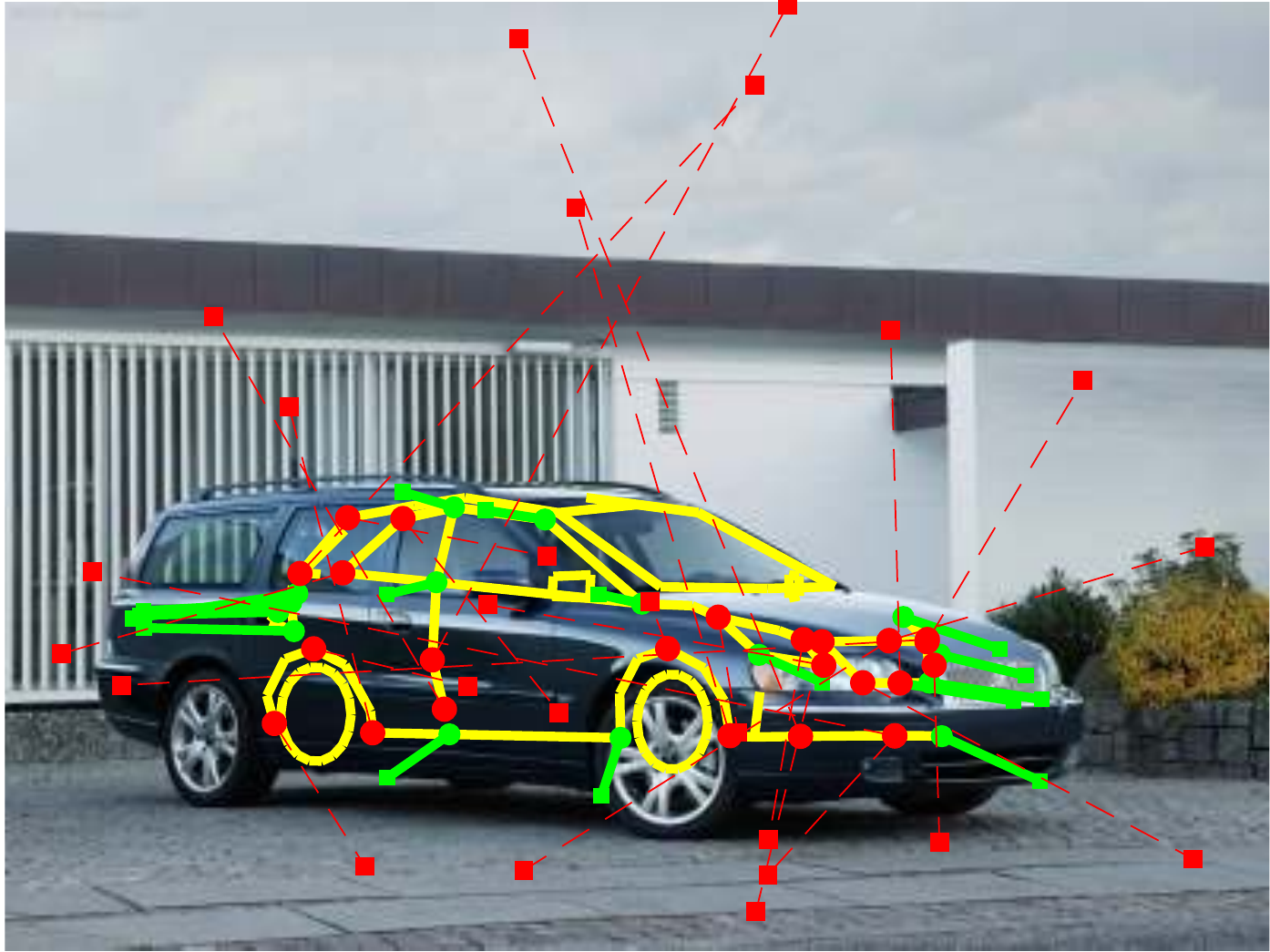} \\
	\vspace{1mm}
	\end{minipage} 
&  \myhspace
	\begin{minipage}{\mpwthree}%
	\centering%
	\includegraphics[width=\columnwidth]{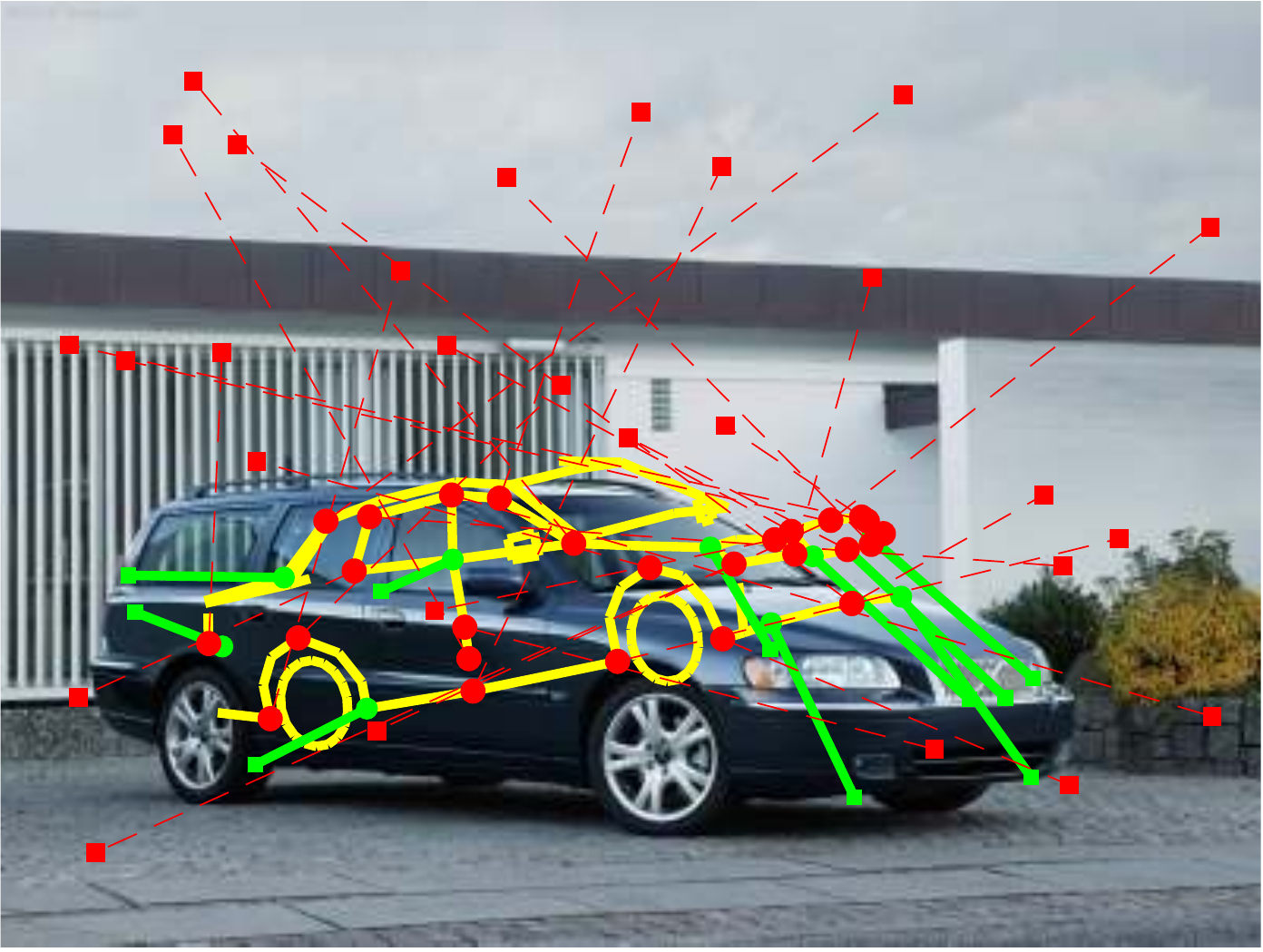} \\
	\vspace{1mm}
	\end{minipage} \\
\myhspace \myhspace \hspace{-2mm} \rotatebox{90}{\hspace{-8mm} {\smaller \convexrobust} } & 
\myhspace
	\begin{minipage}{\mpwthree}%
	\centering%
	\includegraphics[width=\columnwidth]{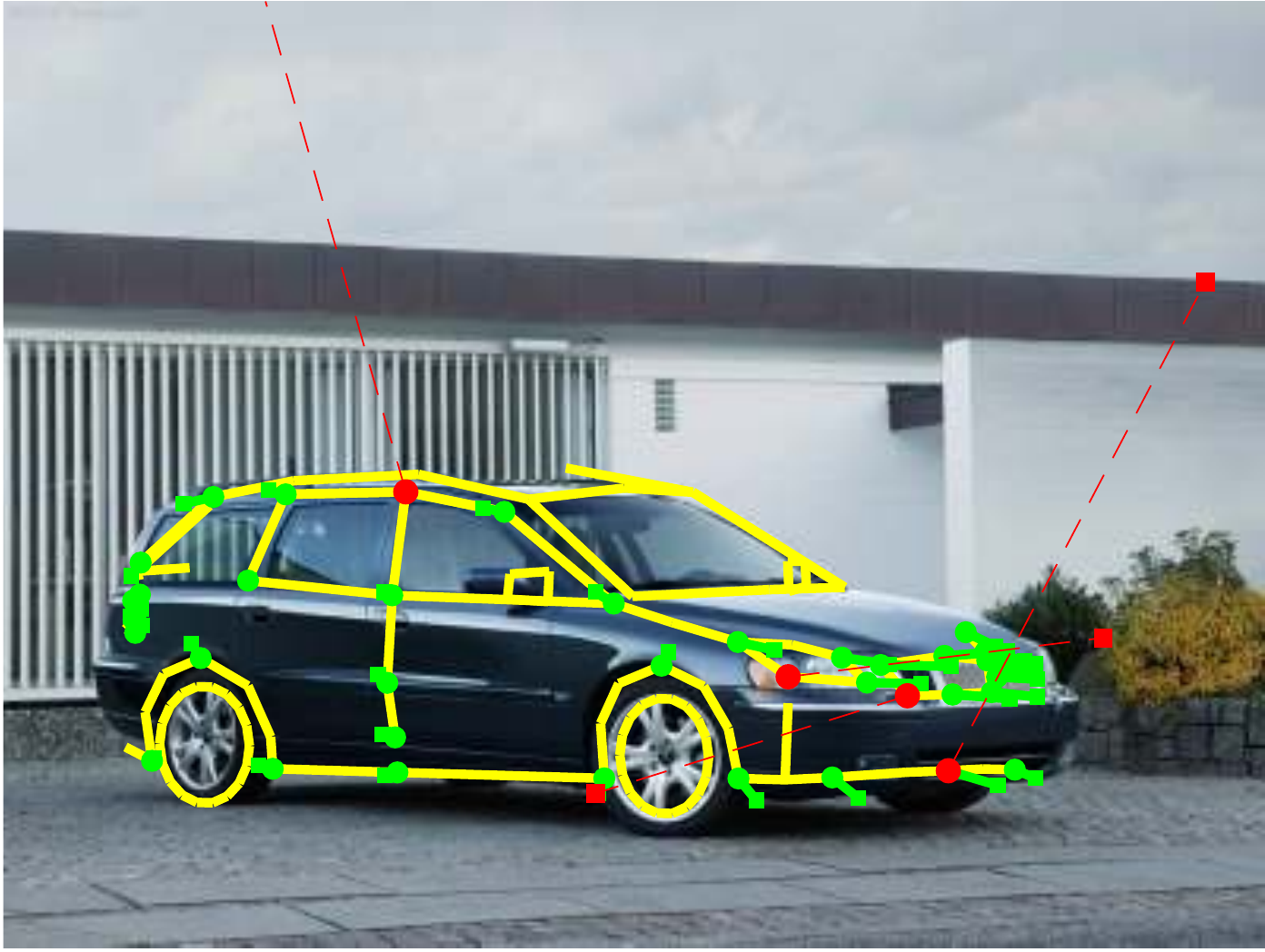} \\
	\vspace{1mm}
	\end{minipage}
& \myhspace
	\begin{minipage}{\mpwthree}%
	\centering%
	\includegraphics[width=\columnwidth]{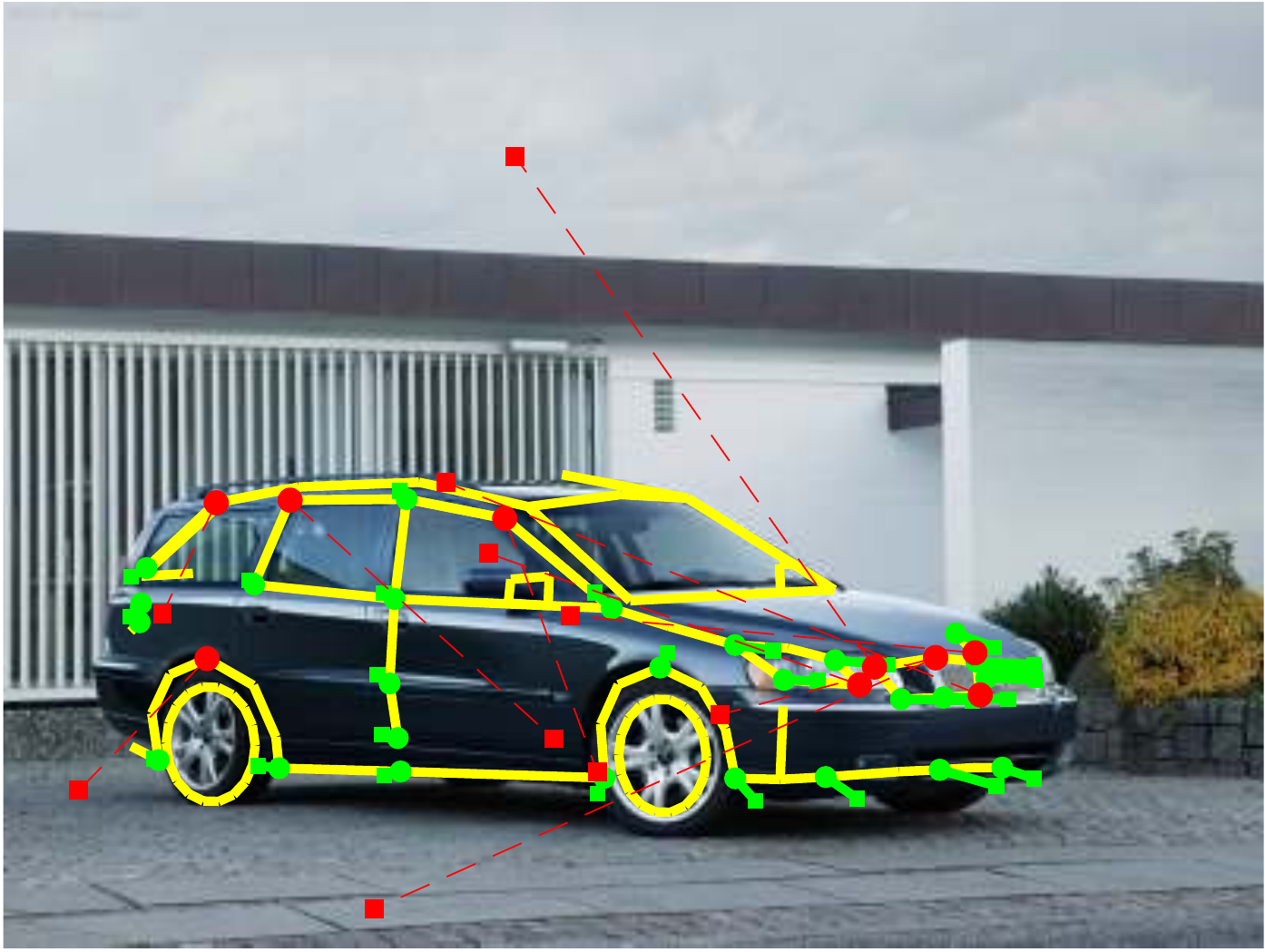} \\
	\vspace{1mm}
	\end{minipage}
& \myhspace
	\begin{minipage}{\mpwthree}%
	\centering%
	\includegraphics[width=\columnwidth]{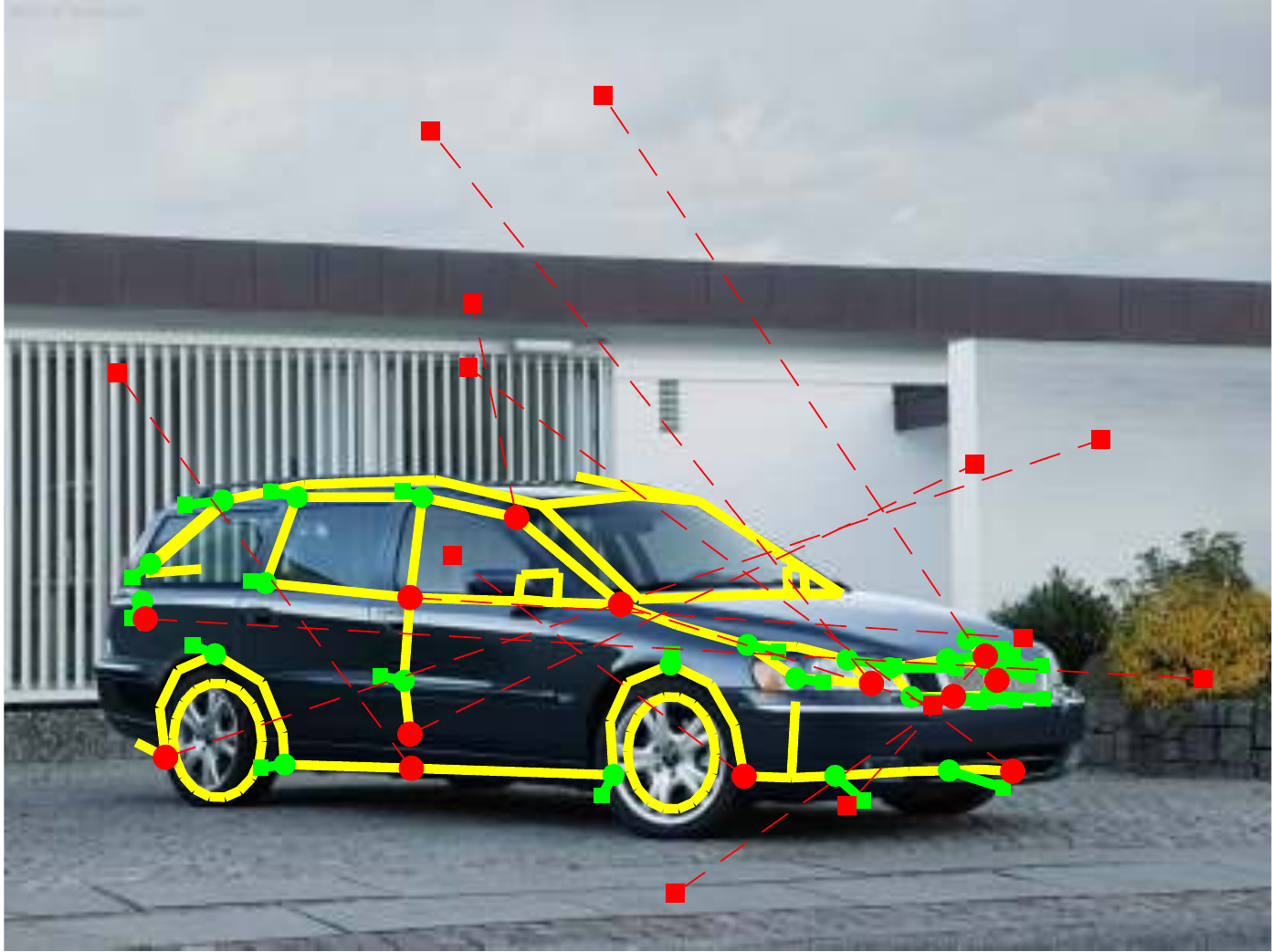} \\
	\vspace{1mm}
	\end{minipage} 
& \myhspace
	\begin{minipage}{\mpwthree}%
	\centering%
	\includegraphics[width=\columnwidth]{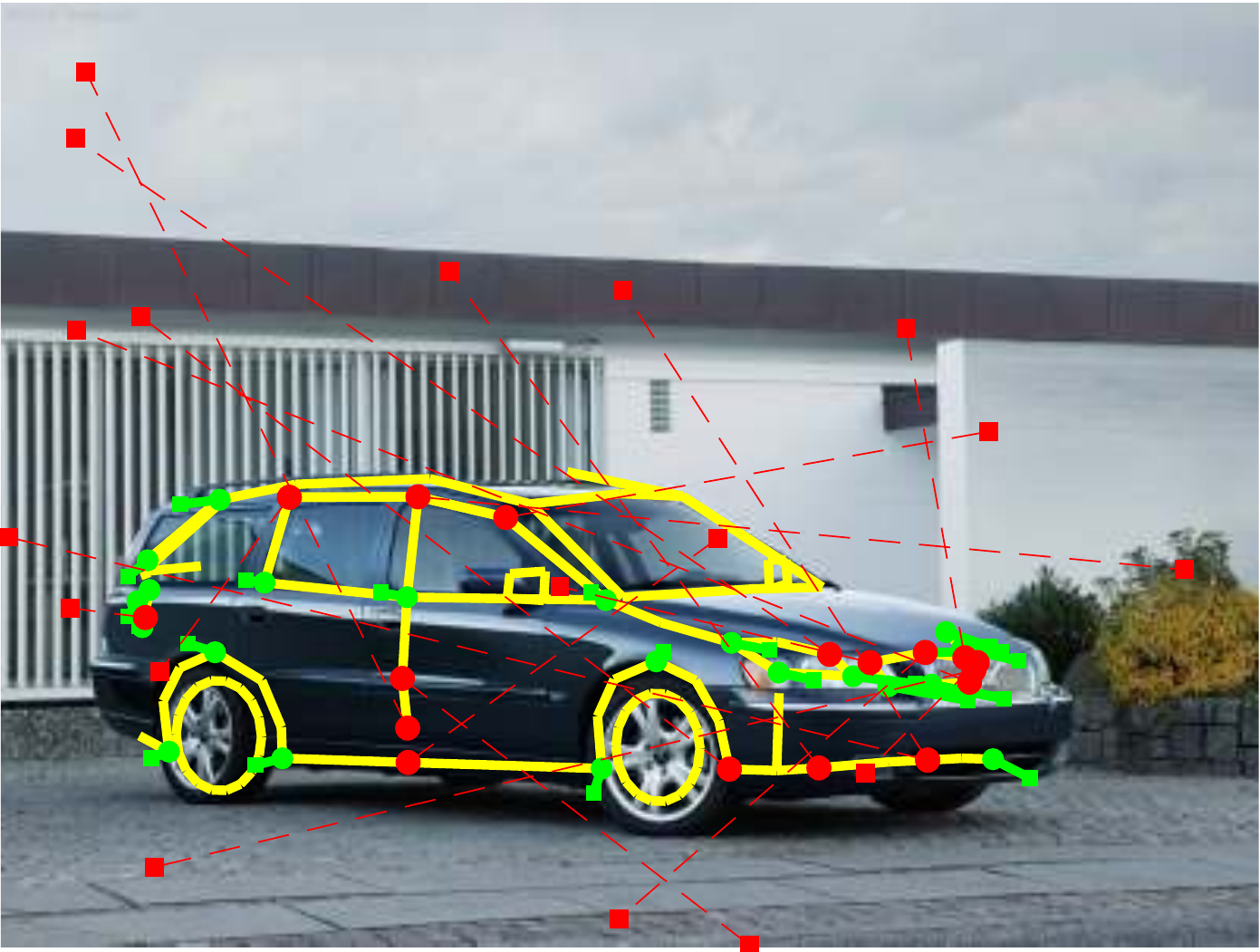} \\
	\vspace{1mm}
	\end{minipage}
& \myhspace
	\begin{minipage}{\mpwthree}%
	\centering%
	\includegraphics[width=\columnwidth]{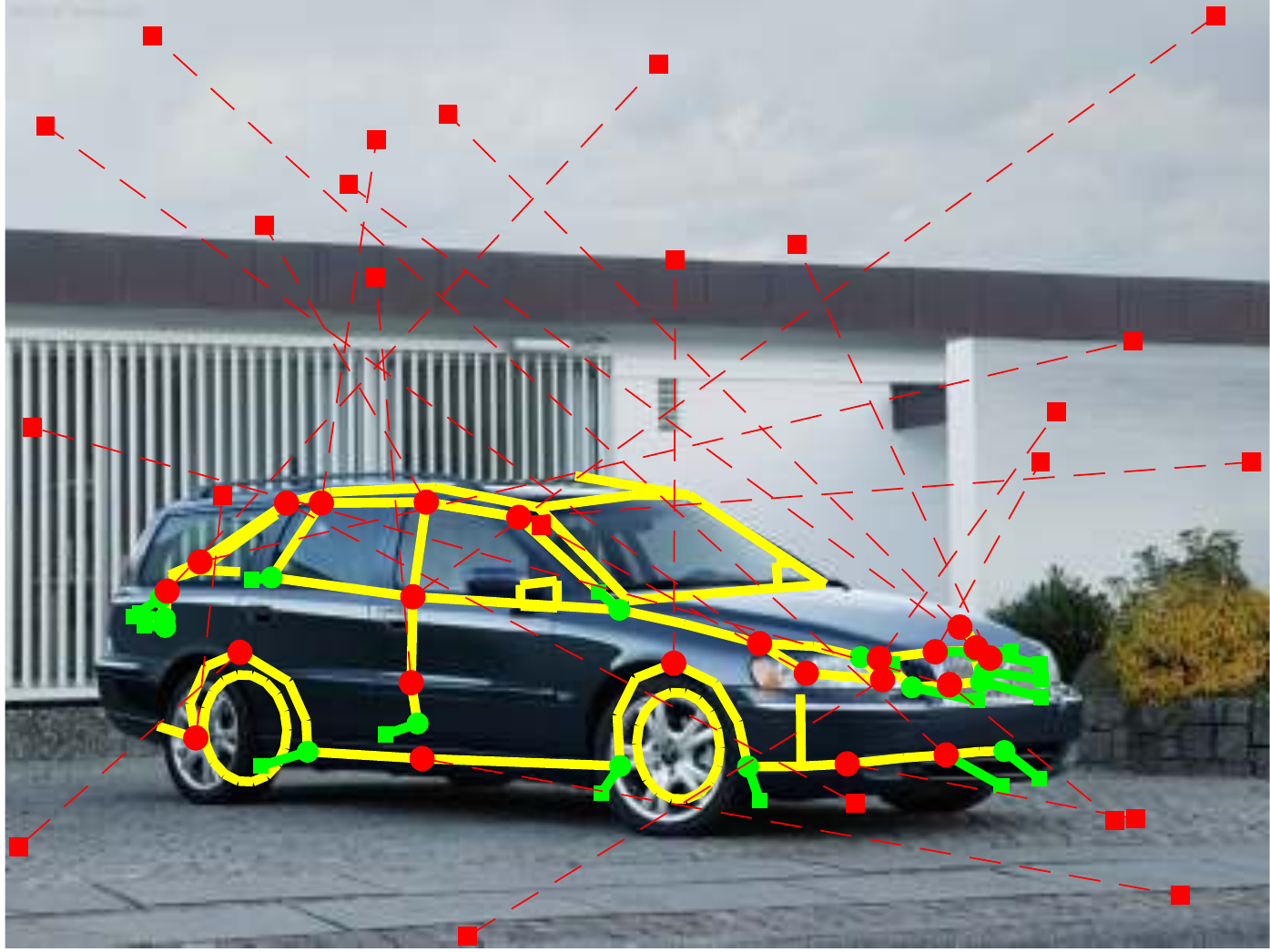} \\
	\vspace{1mm}
	\end{minipage}
&  \myhspace
	\begin{minipage}{\mpwthree}%
	\centering%
	\includegraphics[width=\columnwidth]{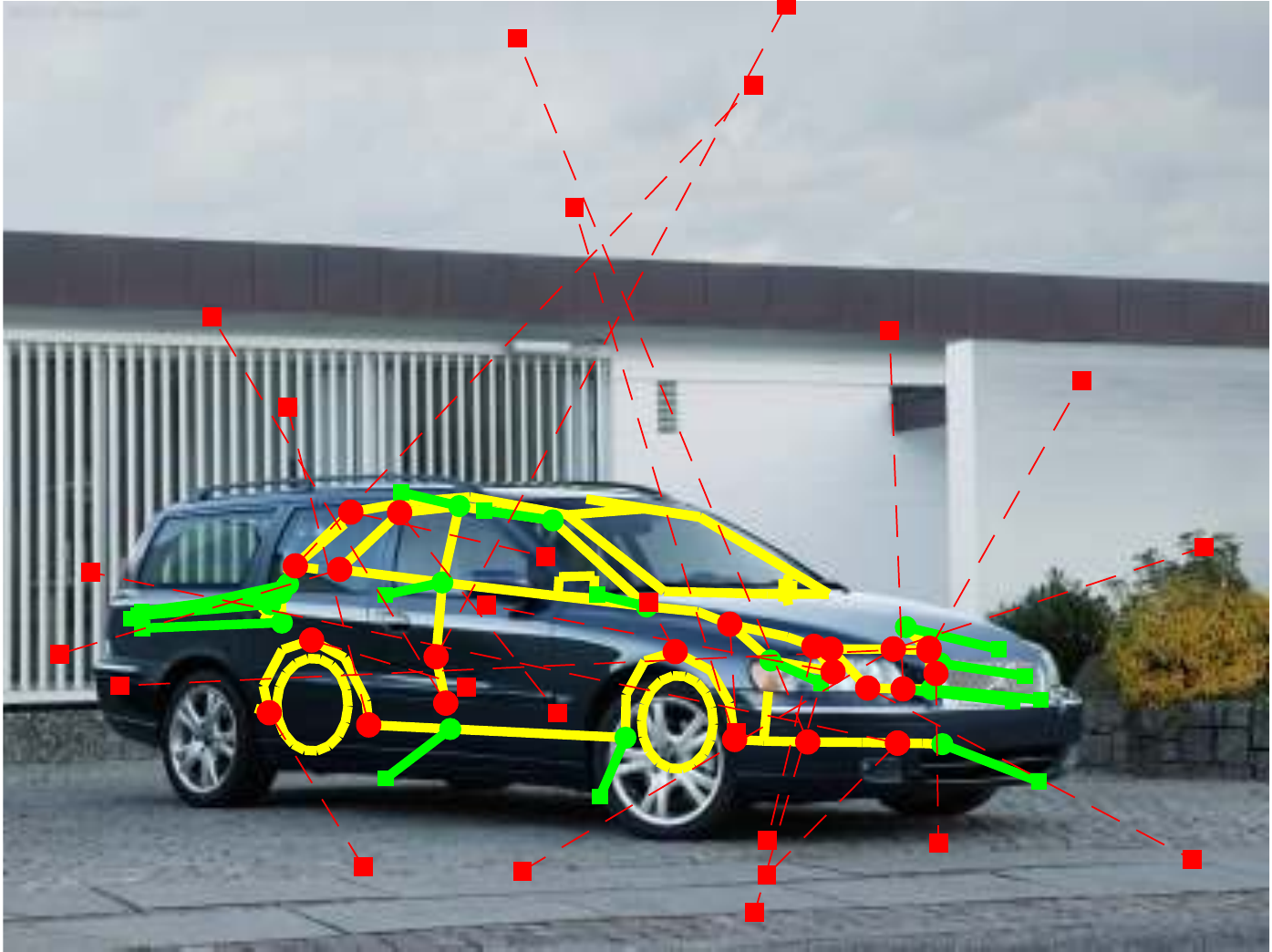} \\
	\vspace{1mm}
	\end{minipage} 
&  \myhspace
	\begin{minipage}{\mpwthree}%
	\centering%
	\includegraphics[width=\columnwidth]{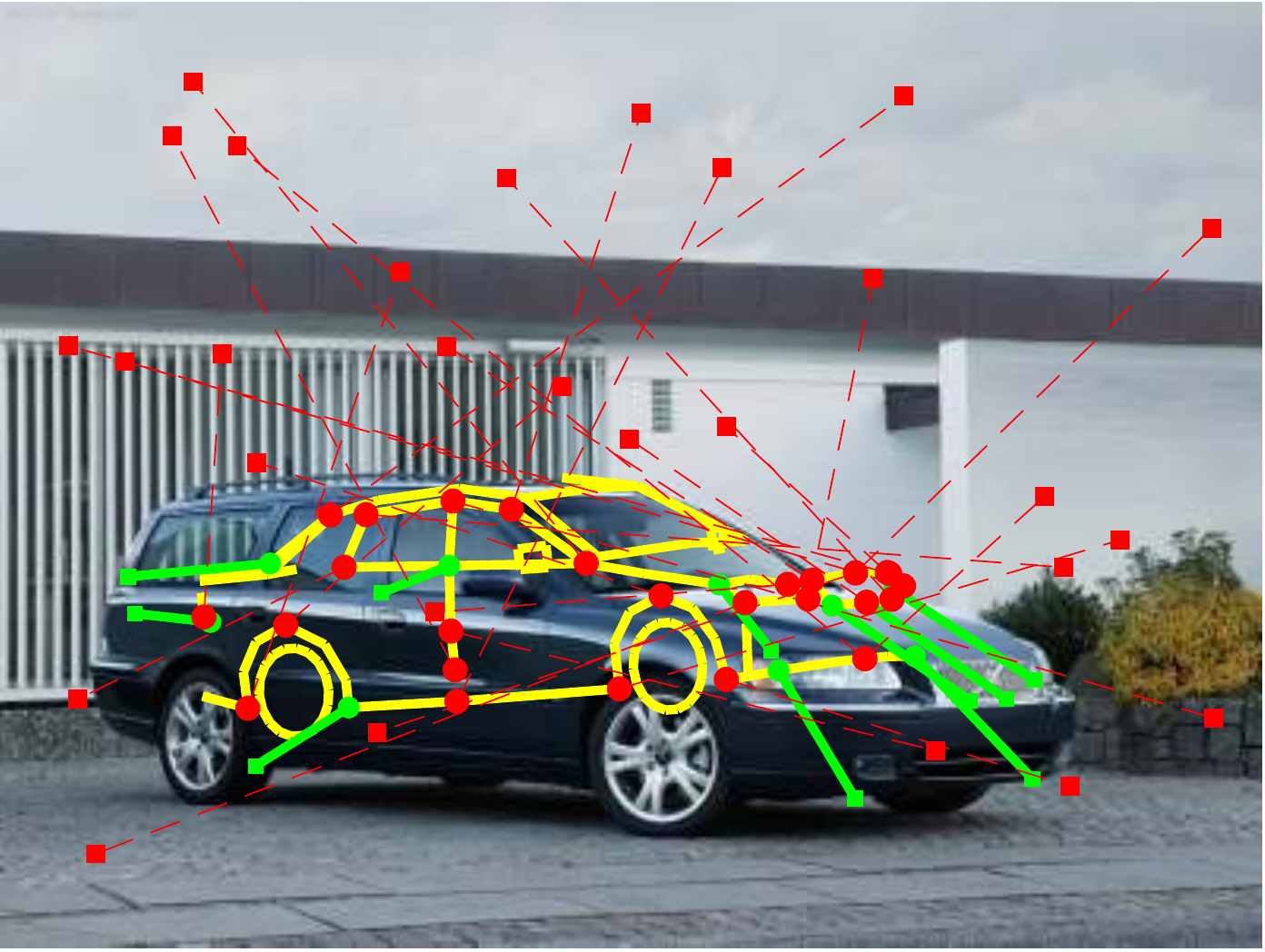} \\
	\vspace{1mm}
	\end{minipage} \\
\myhspace \myhspace \hspace{-2mm} \rotatebox{90}{\hspace{-3mm} {\smaller \blue{ \namerobust}} } & 
\myhspace
	\begin{minipage}{\mpwthree}%
	\centering%
	\includegraphics[width=\columnwidth]{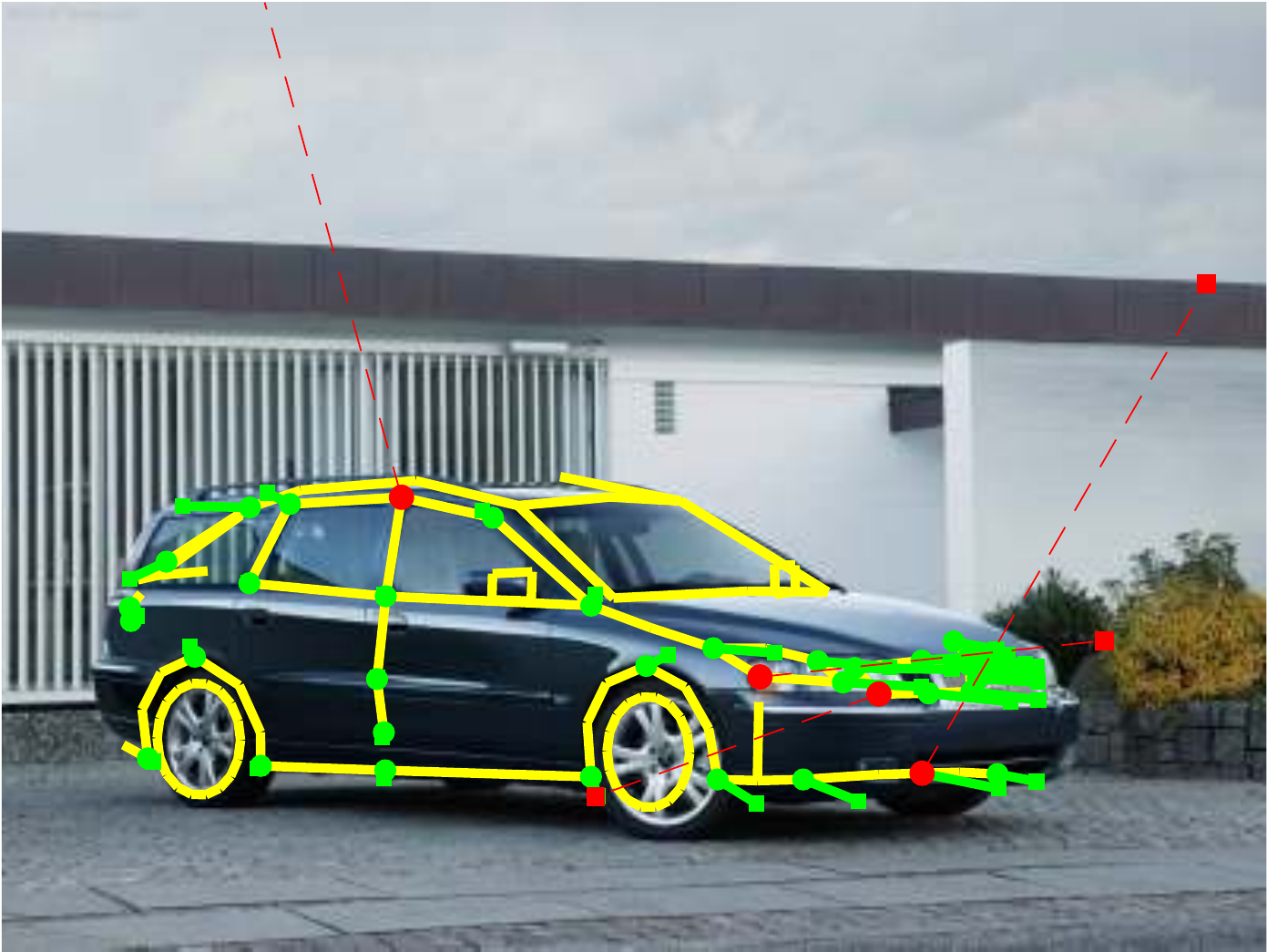} \\
	\vspace{1mm}
	\end{minipage}
& \myhspace
	\begin{minipage}{\mpwthree}%
	\centering%
	\includegraphics[width=\columnwidth]{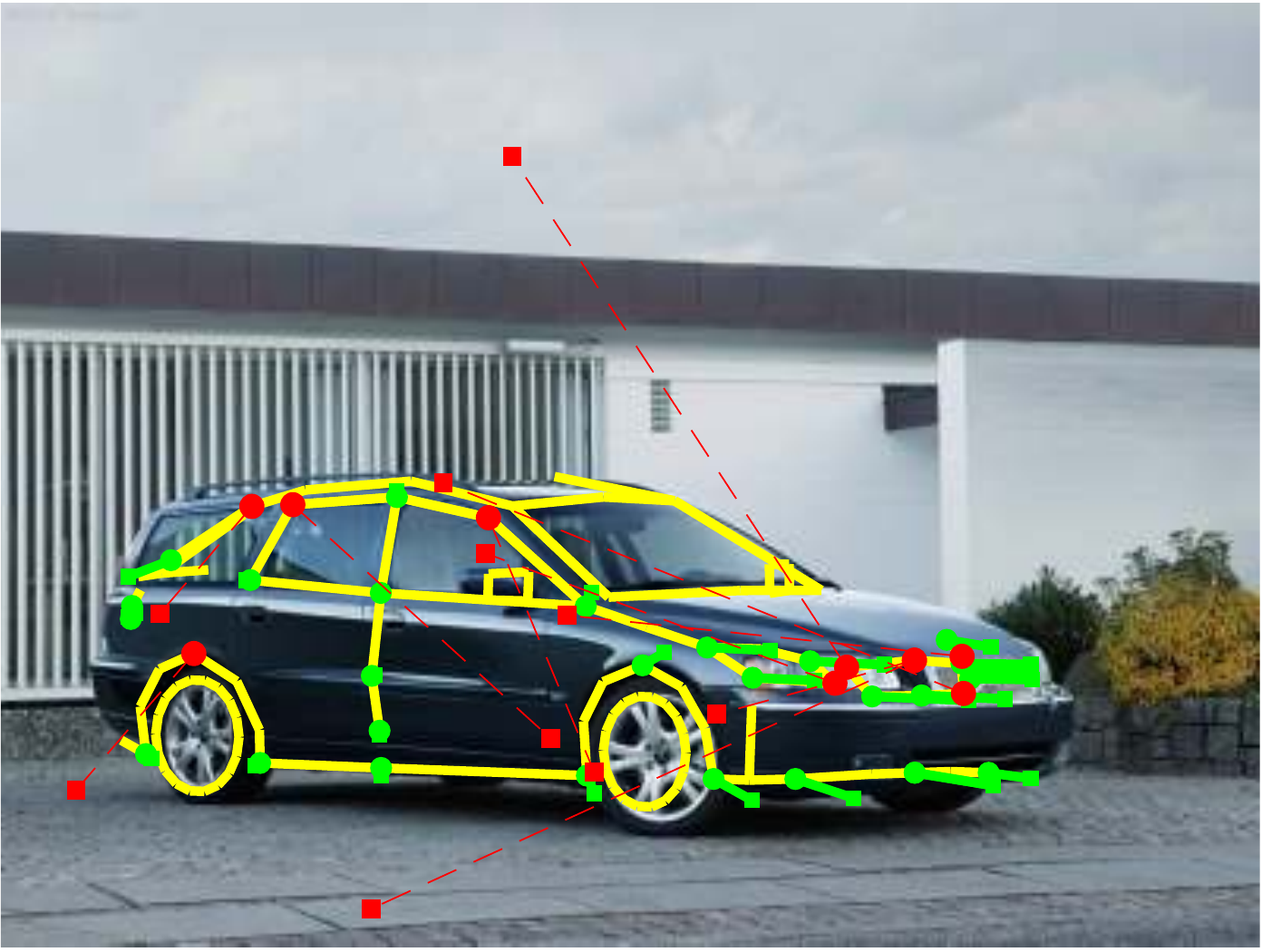} \\
	\vspace{1mm}
	\end{minipage}
& \myhspace
	\begin{minipage}{\mpwthree}%
	\centering%
	\includegraphics[width=\columnwidth]{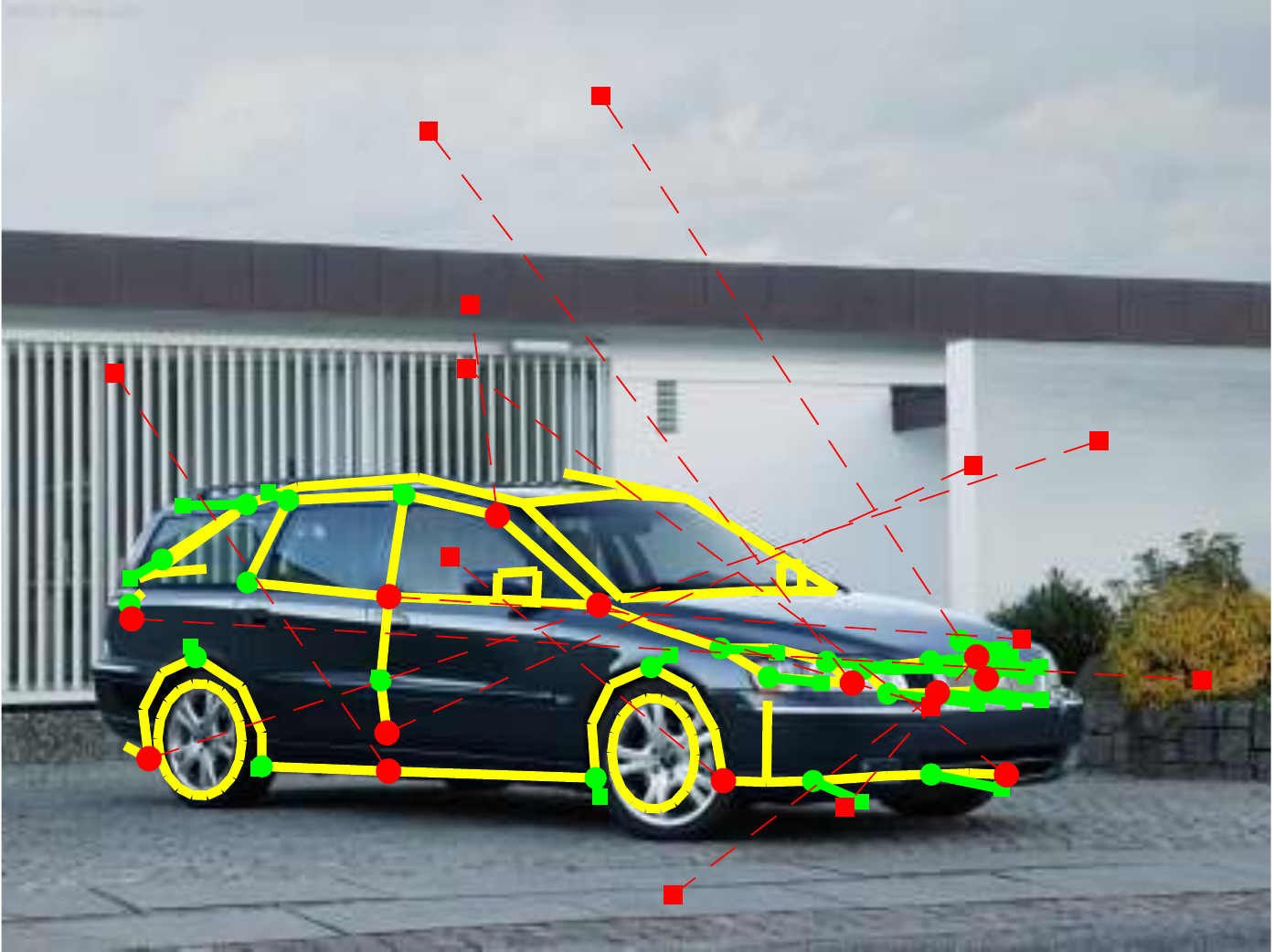} \\
	\vspace{1mm}
	\end{minipage} 
& \myhspace
	\begin{minipage}{\mpwthree}%
	\centering%
	\includegraphics[width=\columnwidth]{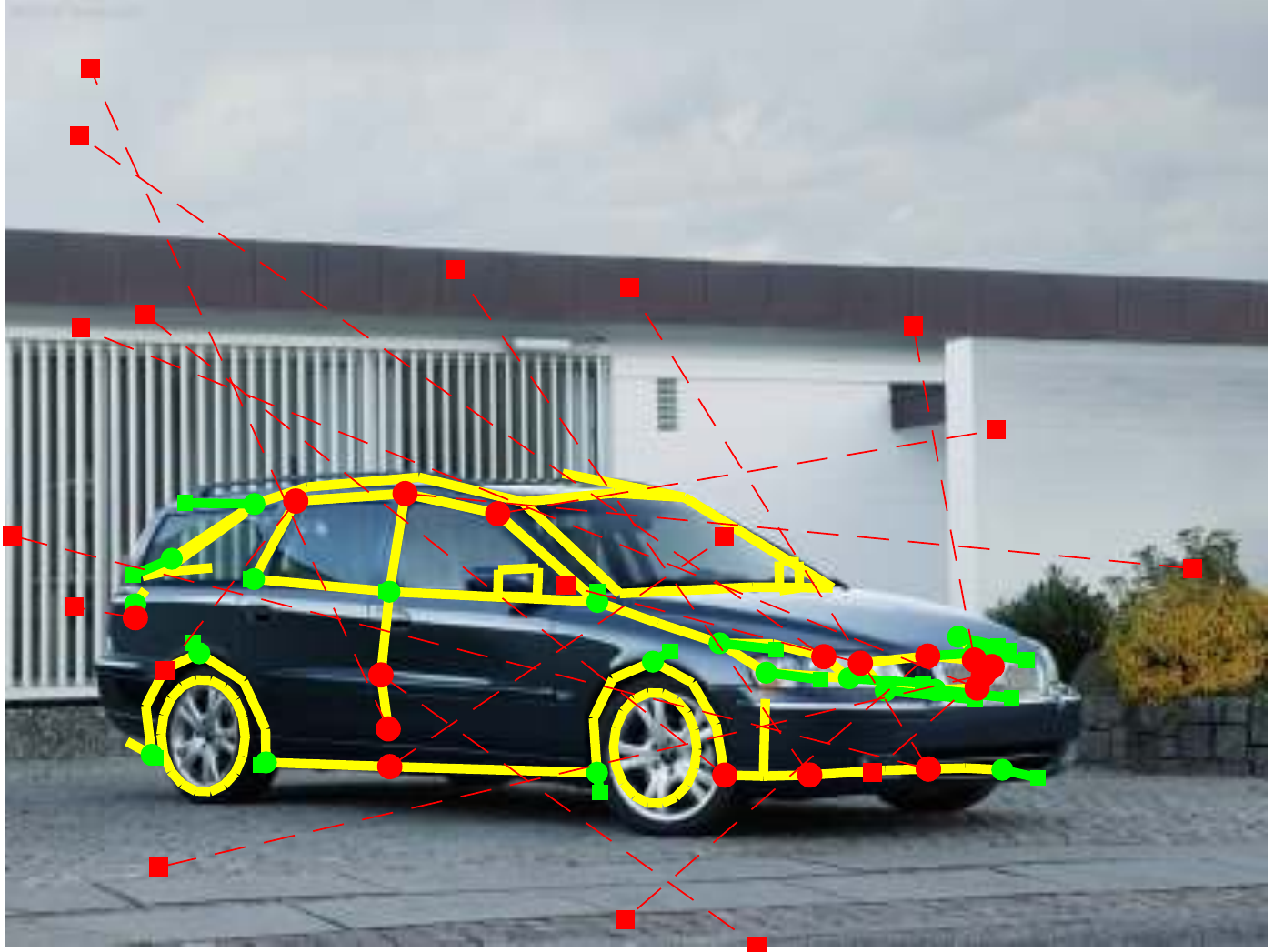} \\
	\vspace{1mm}
	\end{minipage}
& \myhspace
	\begin{minipage}{\mpwthree}%
	\centering%
	\includegraphics[width=\columnwidth]{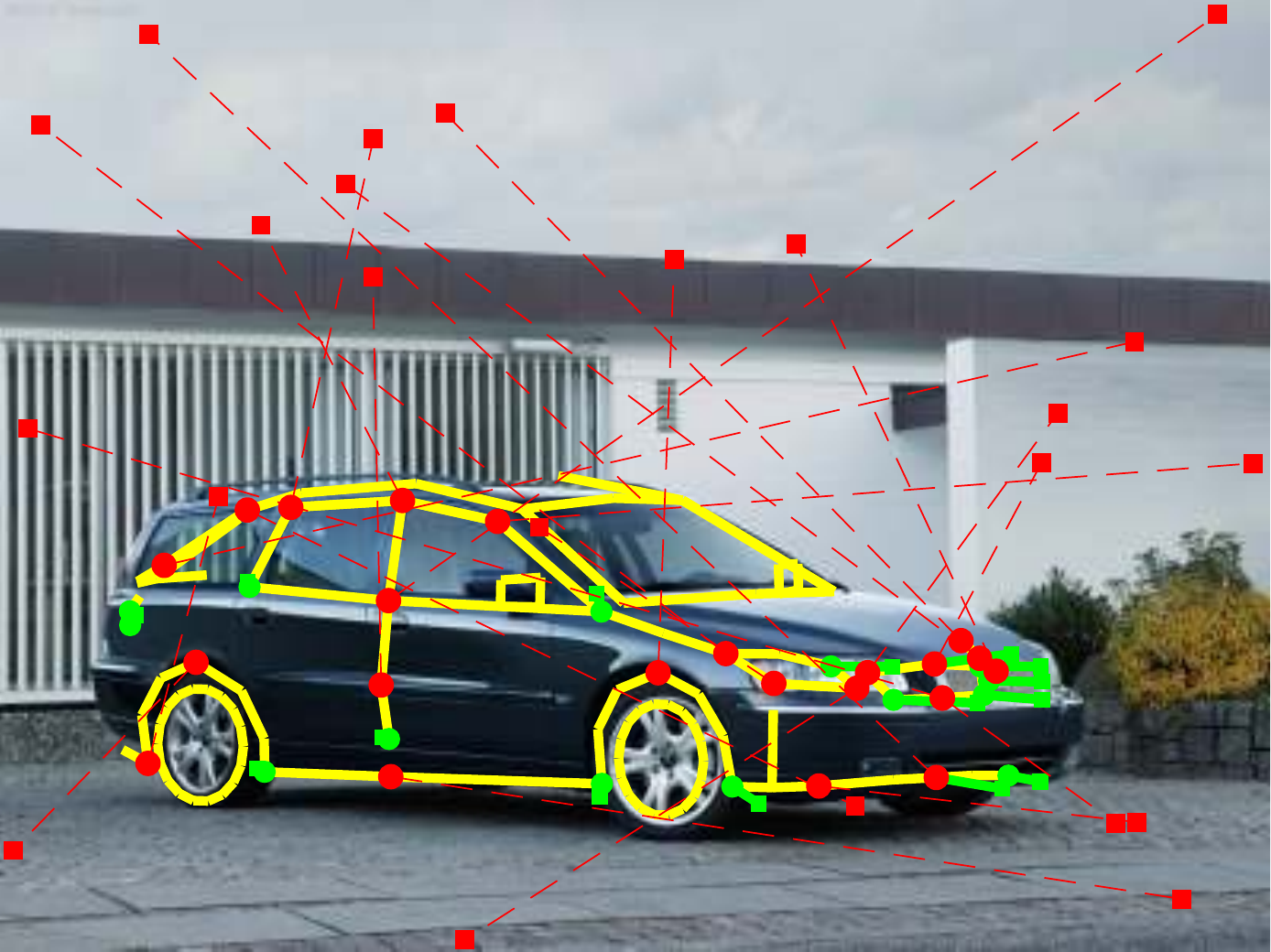} \\
	\vspace{1mm}
	\end{minipage}
&  \myhspace
	\begin{minipage}{\mpwthree}%
	\centering%
	\includegraphics[width=\columnwidth]{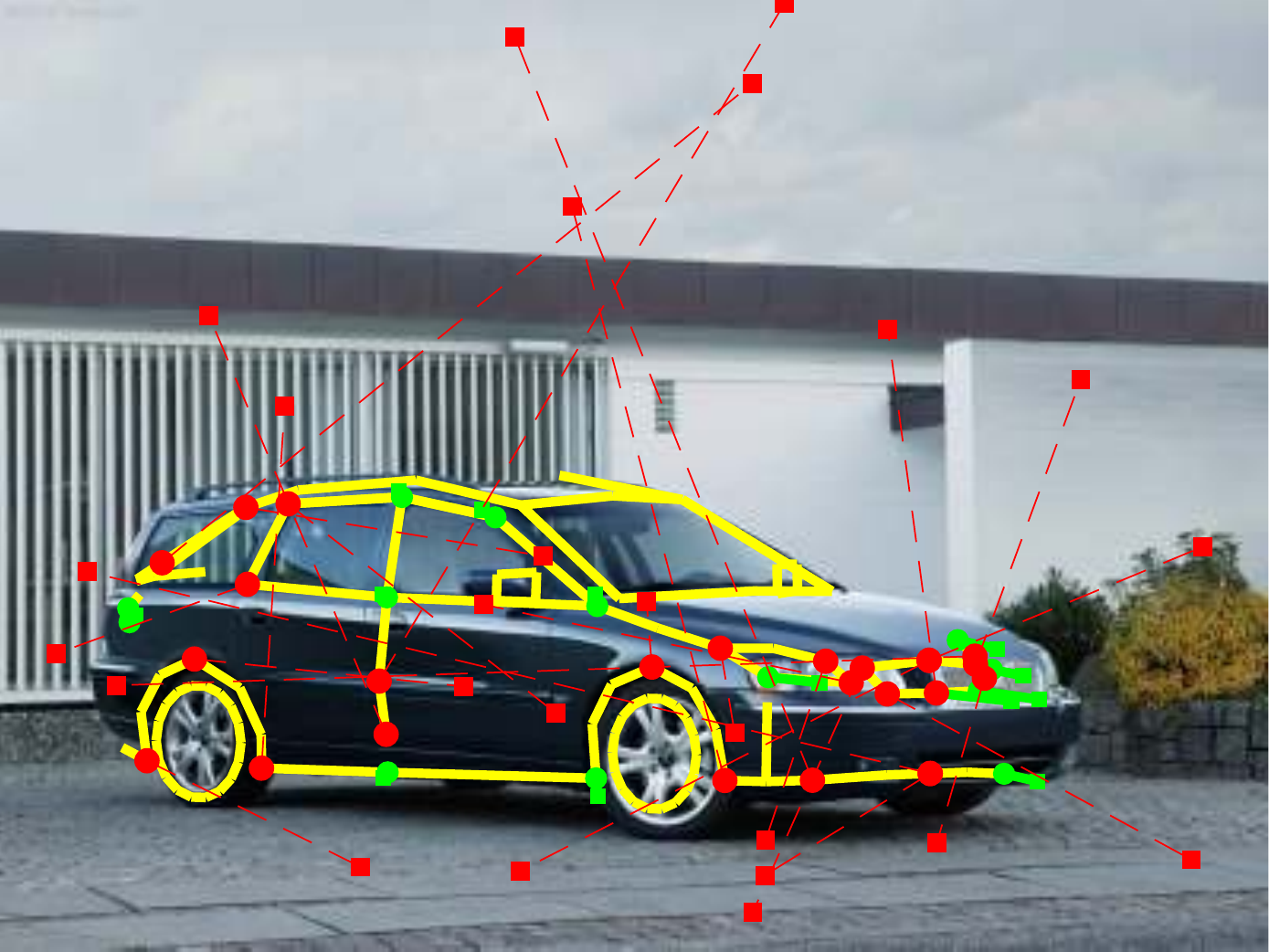} \\
	\vspace{1mm}
	\end{minipage} 
&  \myhspace
	\begin{minipage}{\mpwthree}%
	\centering%
	\includegraphics[width=\columnwidth]{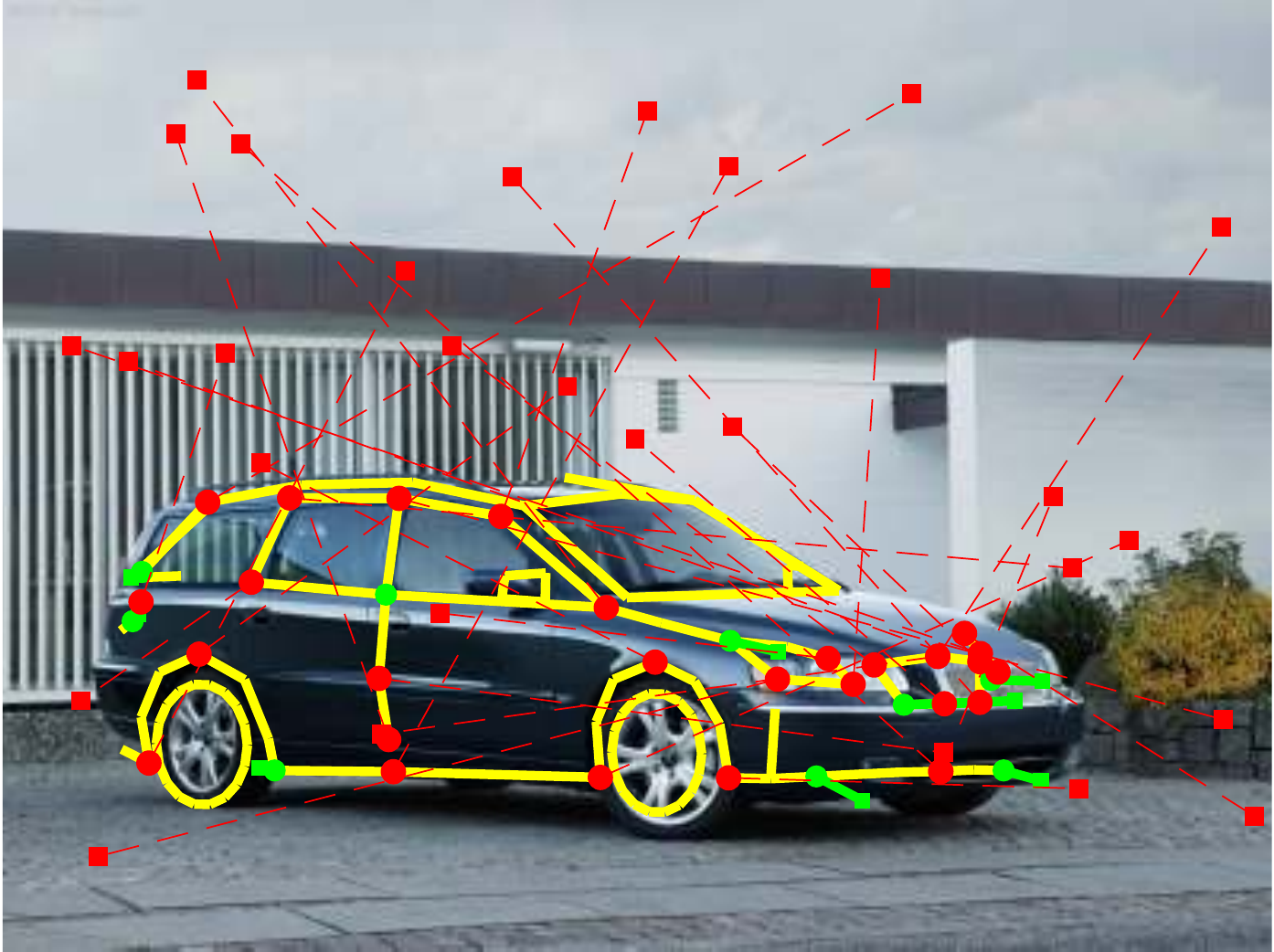} \\
	\vspace{1mm}
	\end{minipage} \\
\multicolumn{8}{c}{Volvo V70} \\

%% file: 217-saab_93.tex
\myhspace \myhspace \hspace{-2mm} \rotatebox{90}{\hspace{-7mm} {\smaller \alternrobust} } & 
\myhspace
	\begin{minipage}{\mpwthree}%
	\centering%
	\includegraphics[width=\columnwidth]{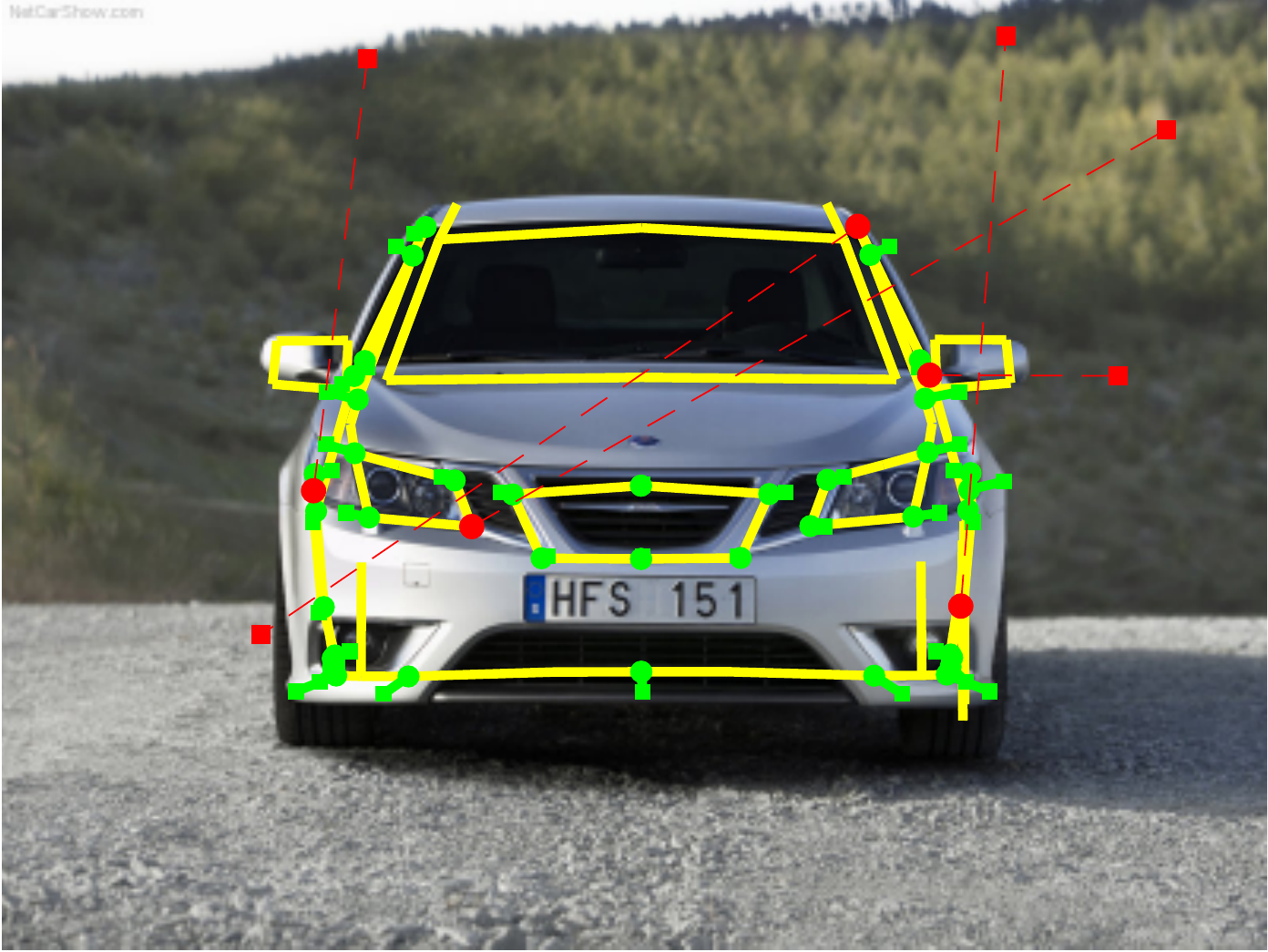} \\
	\vspace{1mm}
	\end{minipage}
& \myhspace
	\begin{minipage}{\mpwthree}%
	\centering%
	\includegraphics[width=\columnwidth]{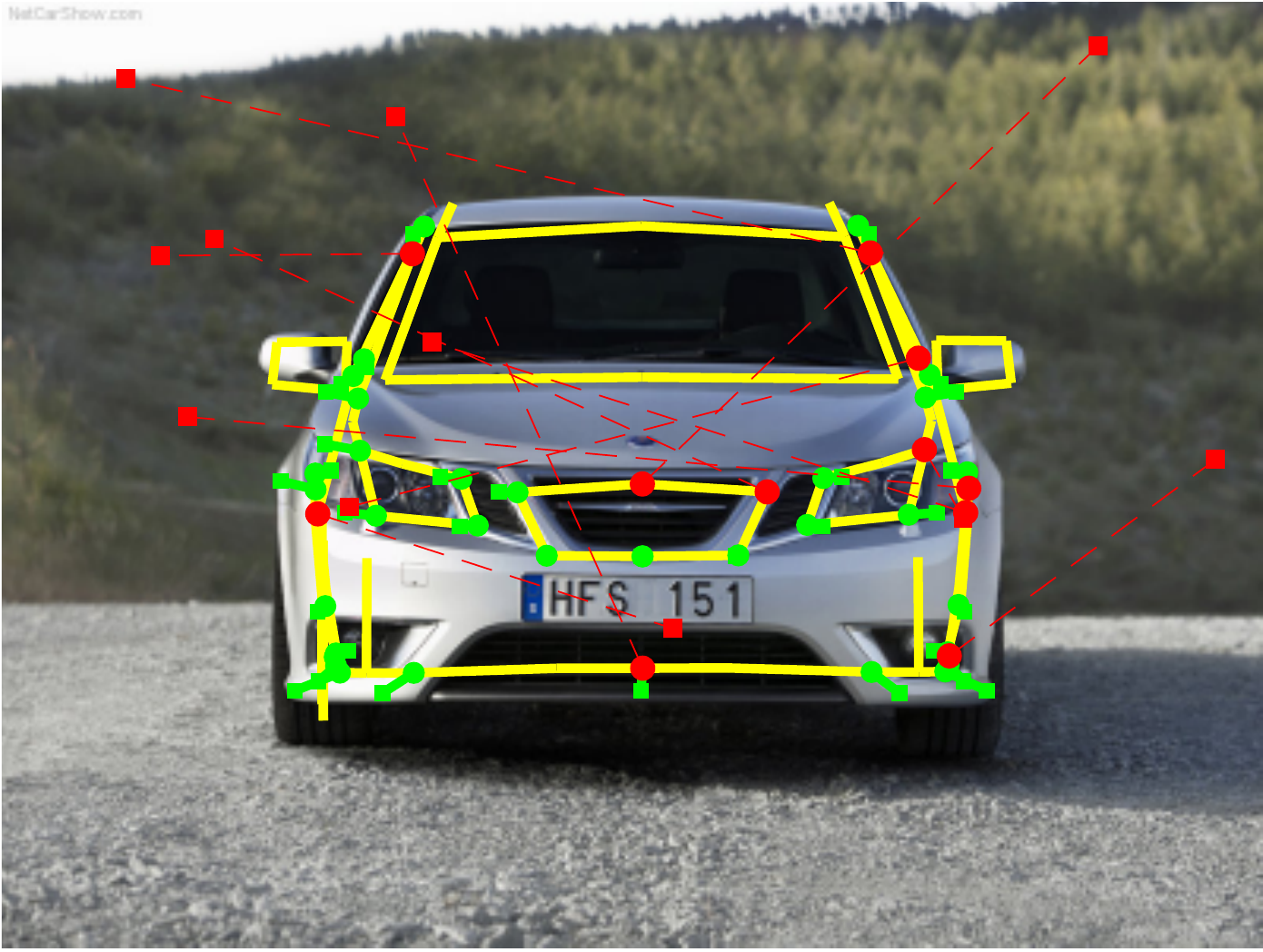} \\
	\vspace{1mm}
	\end{minipage}
& \myhspace
	\begin{minipage}{\mpwthree}%
	\centering%
	\includegraphics[width=\columnwidth]{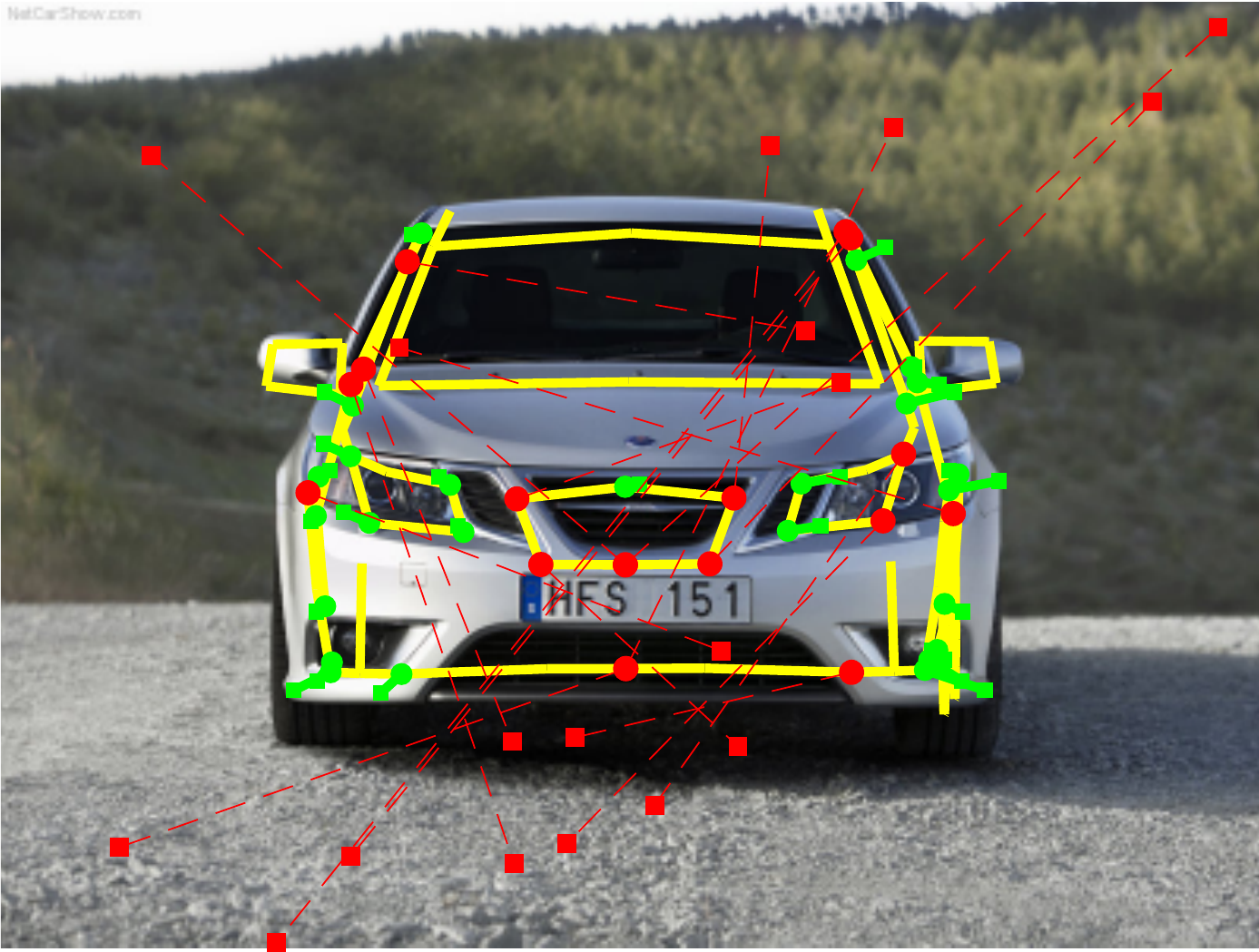} \\
	\vspace{1mm}
	\end{minipage} 
& \myhspace
	\begin{minipage}{\mpwthree}%
	\centering%
	\includegraphics[width=\columnwidth]{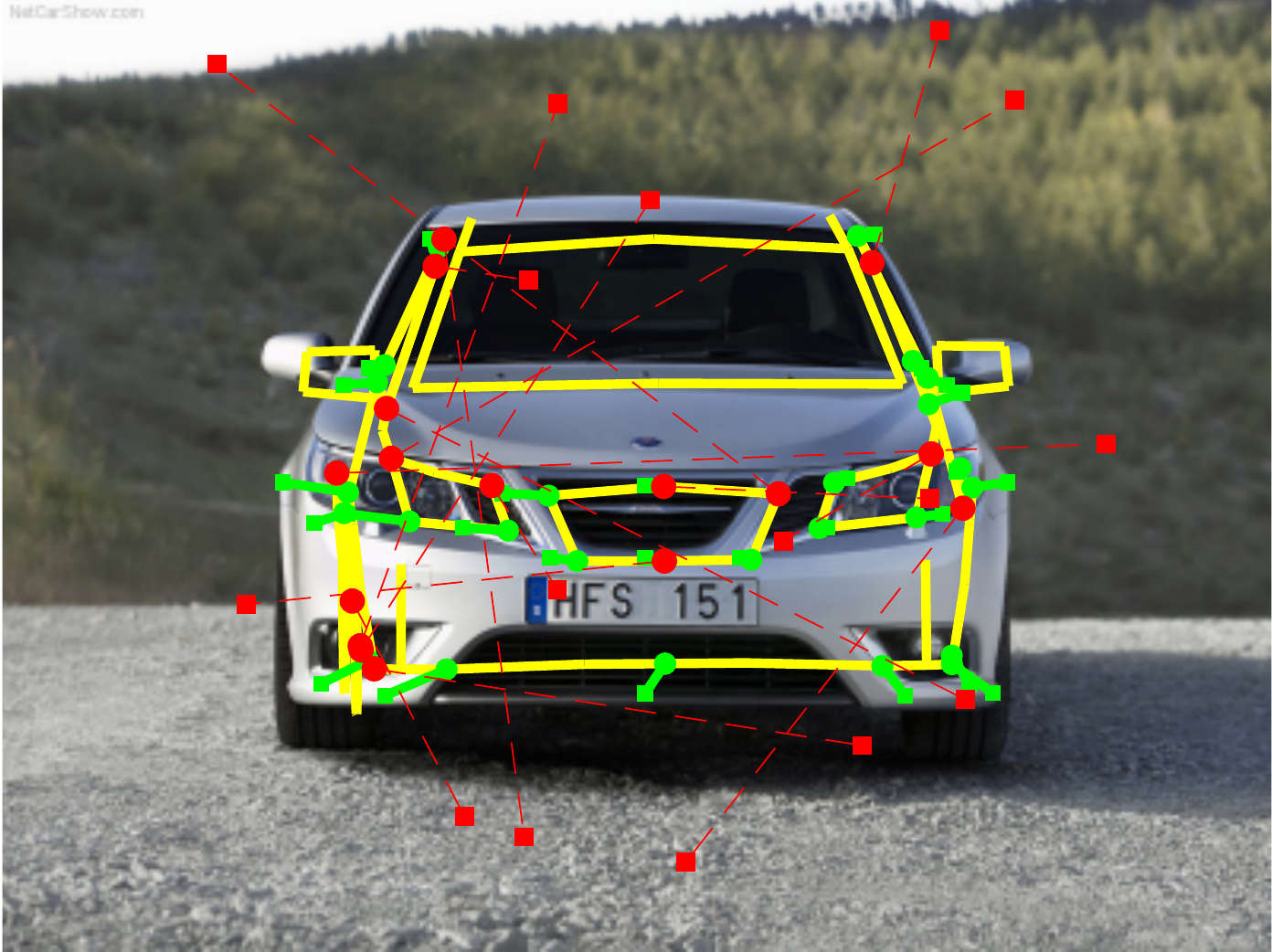} \\
	\vspace{1mm}
	\end{minipage}
& \myhspace
	\begin{minipage}{\mpwthree}%
	\centering%
	\includegraphics[width=\columnwidth]{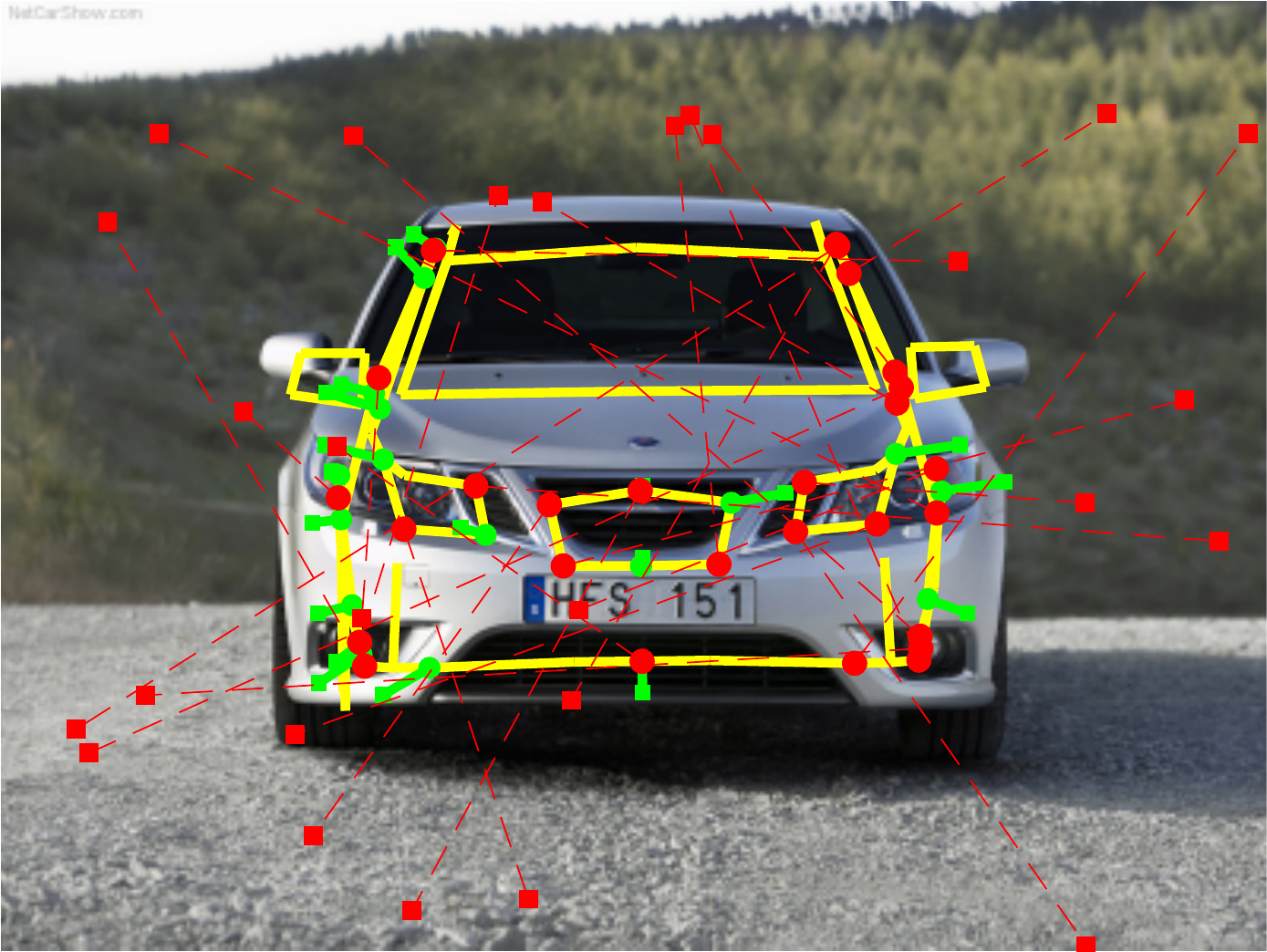} \\
	\vspace{1mm}
	\end{minipage}
&  \myhspace
	\begin{minipage}{\mpwthree}%
	\centering%
	\includegraphics[width=\columnwidth]{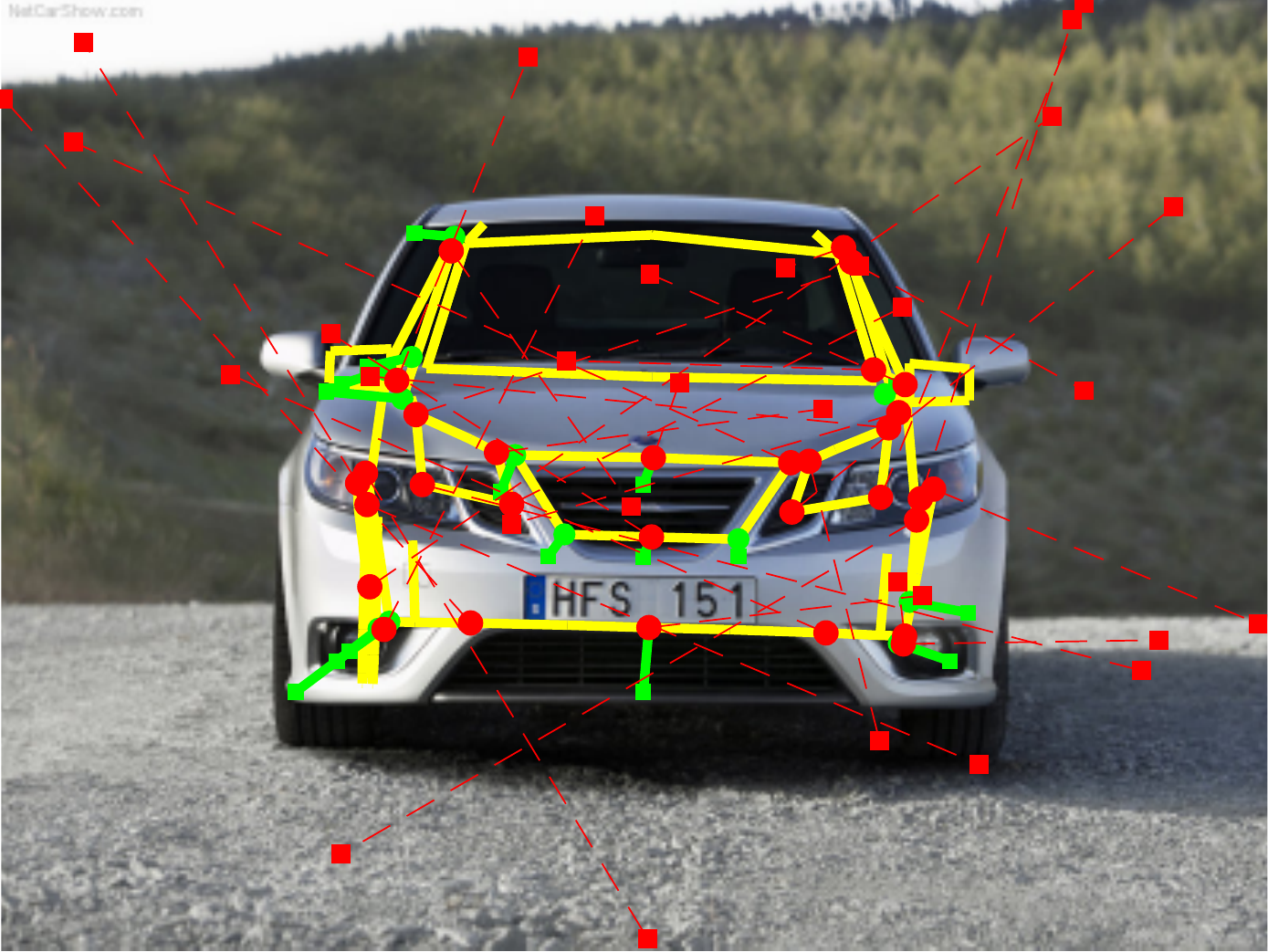} \\
	\vspace{1mm}
	\end{minipage} 
&  \myhspace
	\begin{minipage}{\mpwthree}%
	\centering%
	\includegraphics[width=\columnwidth]{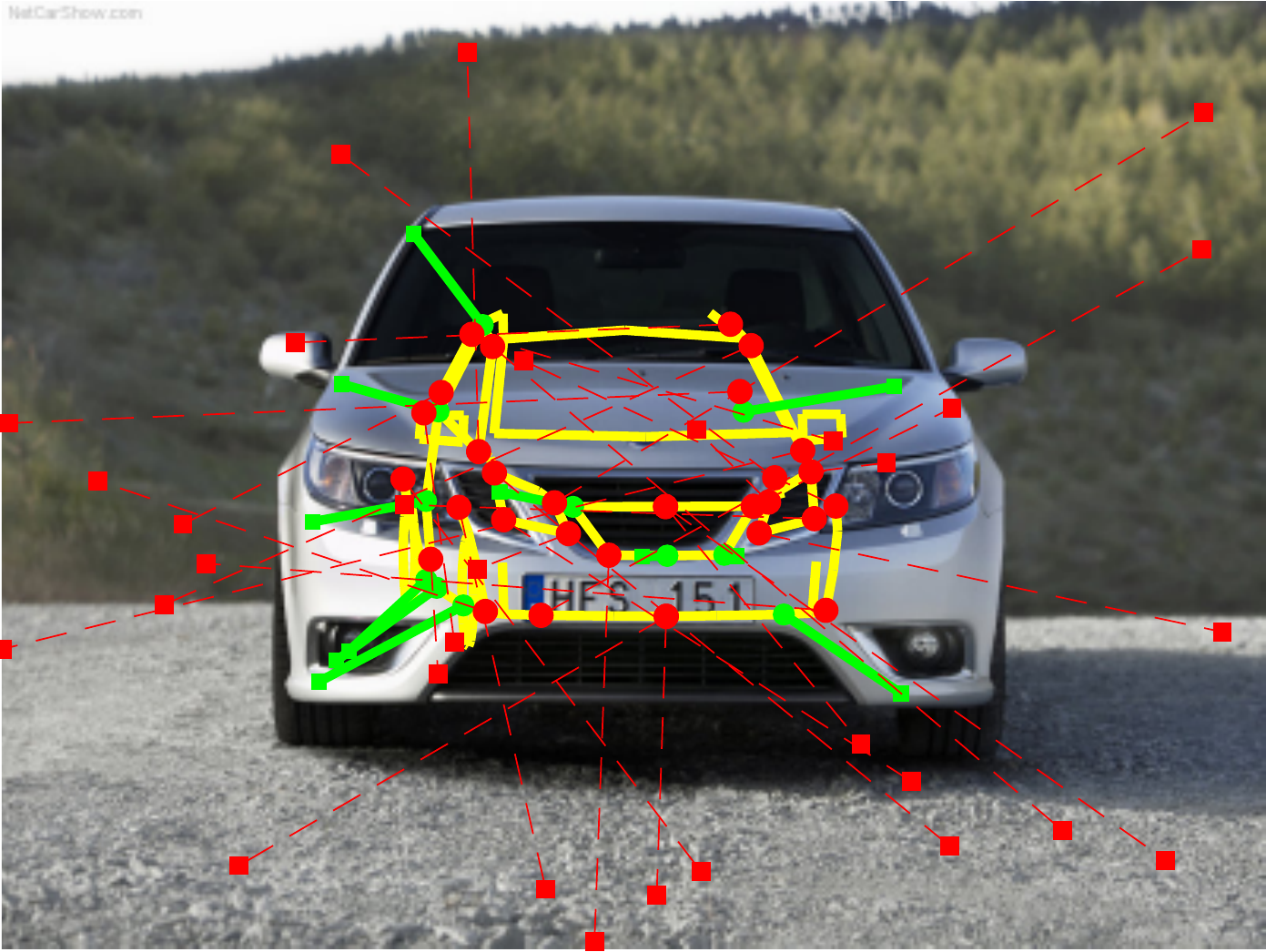} \\
	\vspace{1mm}
	\end{minipage} \\
\myhspace \myhspace \hspace{-2mm} \rotatebox{90}{\hspace{-8mm} {\smaller \convexrobust} } & 
\myhspace
	\begin{minipage}{\mpwthree}%
	\centering%
	\includegraphics[width=\columnwidth]{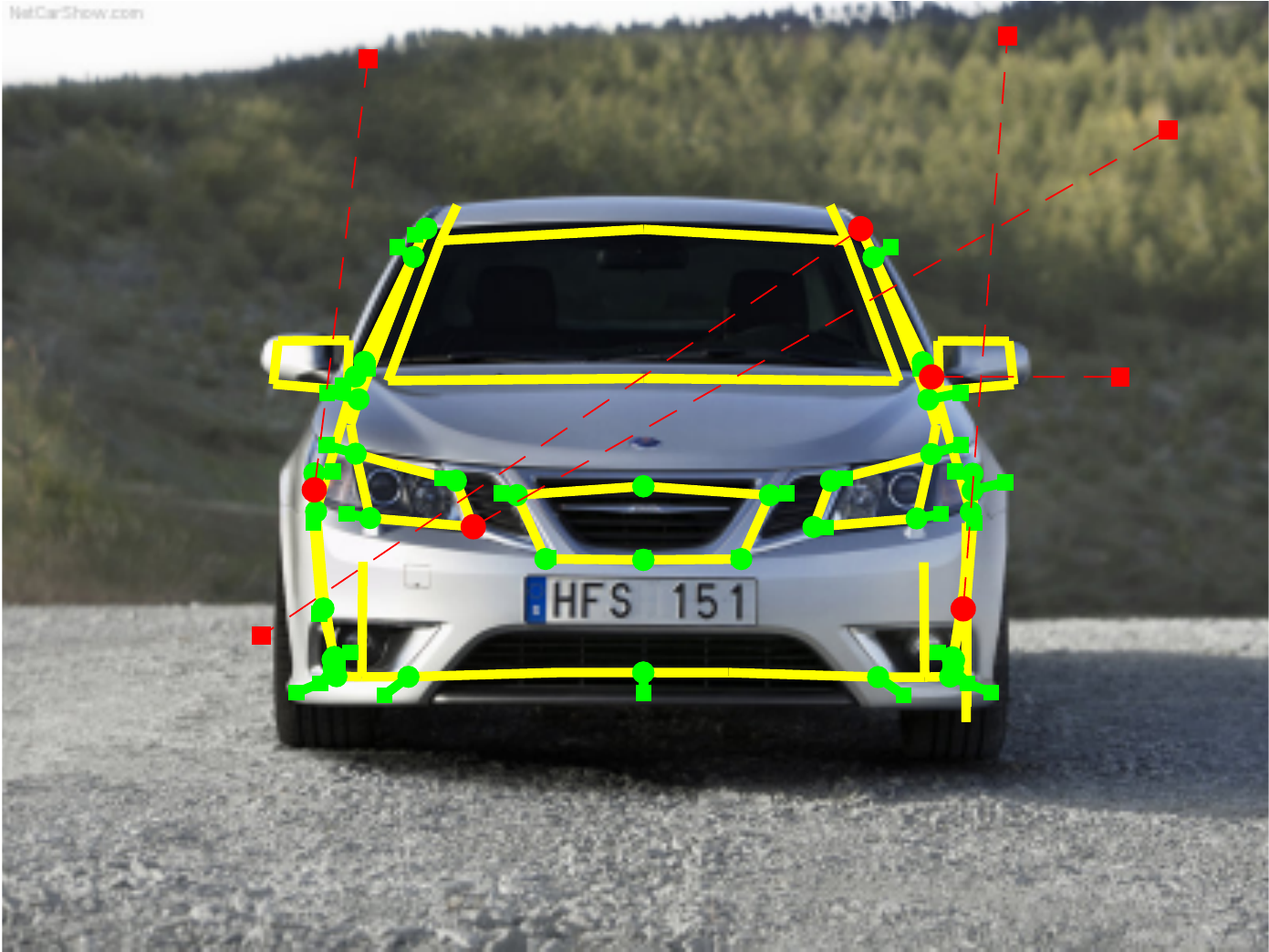} \\
	\vspace{1mm}
	\end{minipage}
& \myhspace
	\begin{minipage}{\mpwthree}%
	\centering%
	\includegraphics[width=\columnwidth]{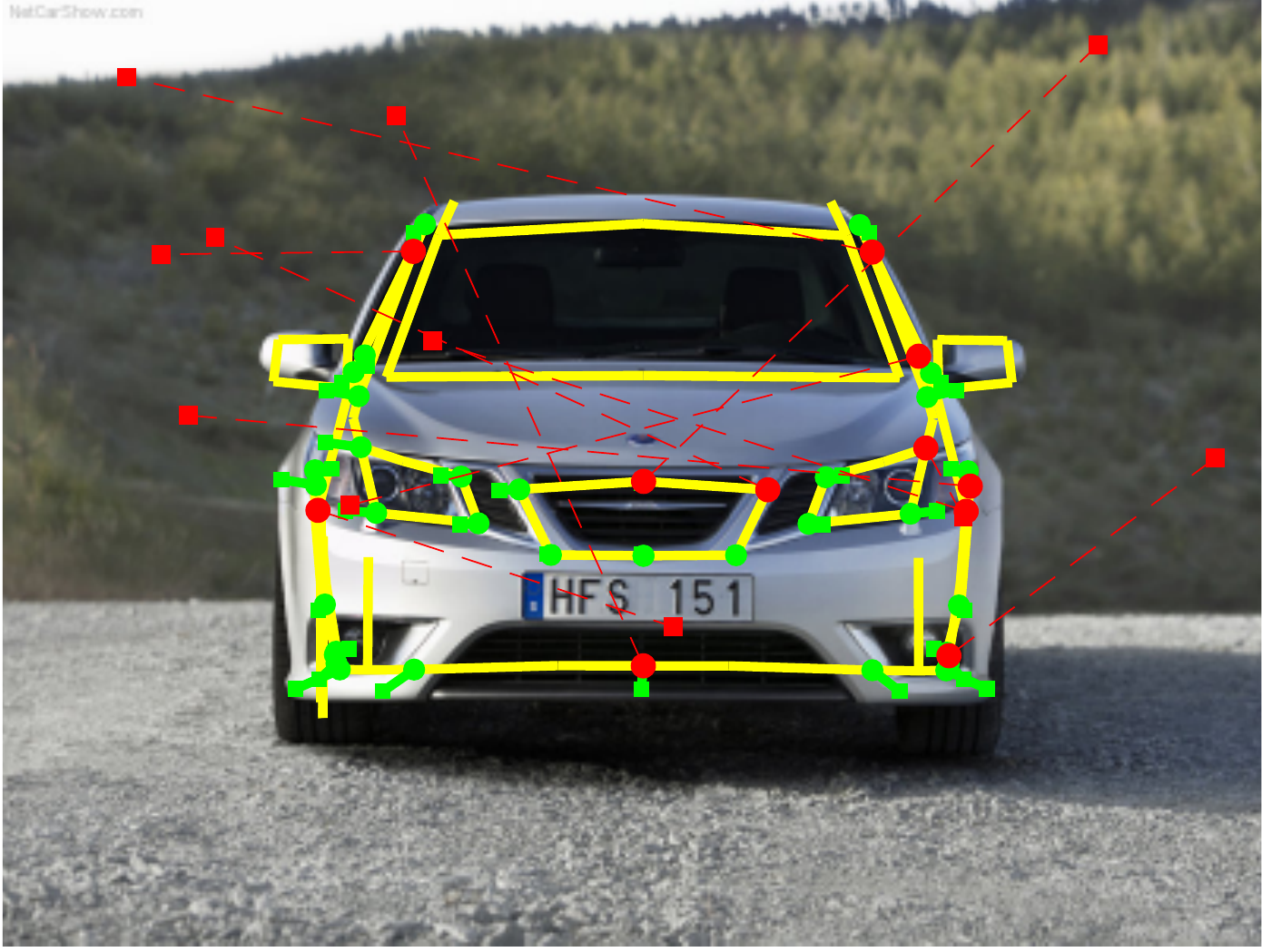} \\
	\vspace{1mm}
	\end{minipage}
& \myhspace
	\begin{minipage}{\mpwthree}%
	\centering%
	\includegraphics[width=\columnwidth]{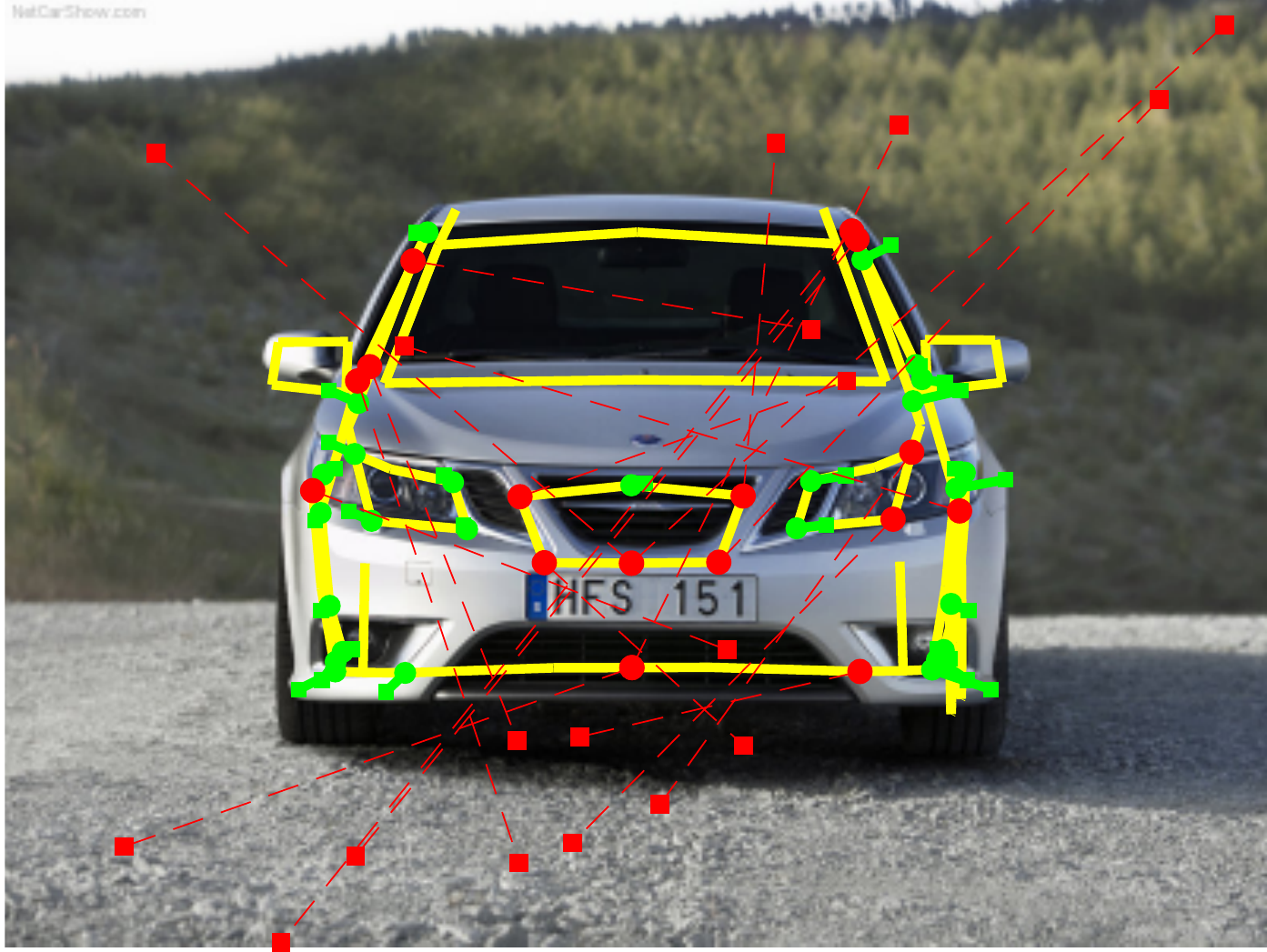} \\
	\vspace{1mm}
	\end{minipage} 
& \myhspace
	\begin{minipage}{\mpwthree}%
	\centering%
	\includegraphics[width=\columnwidth]{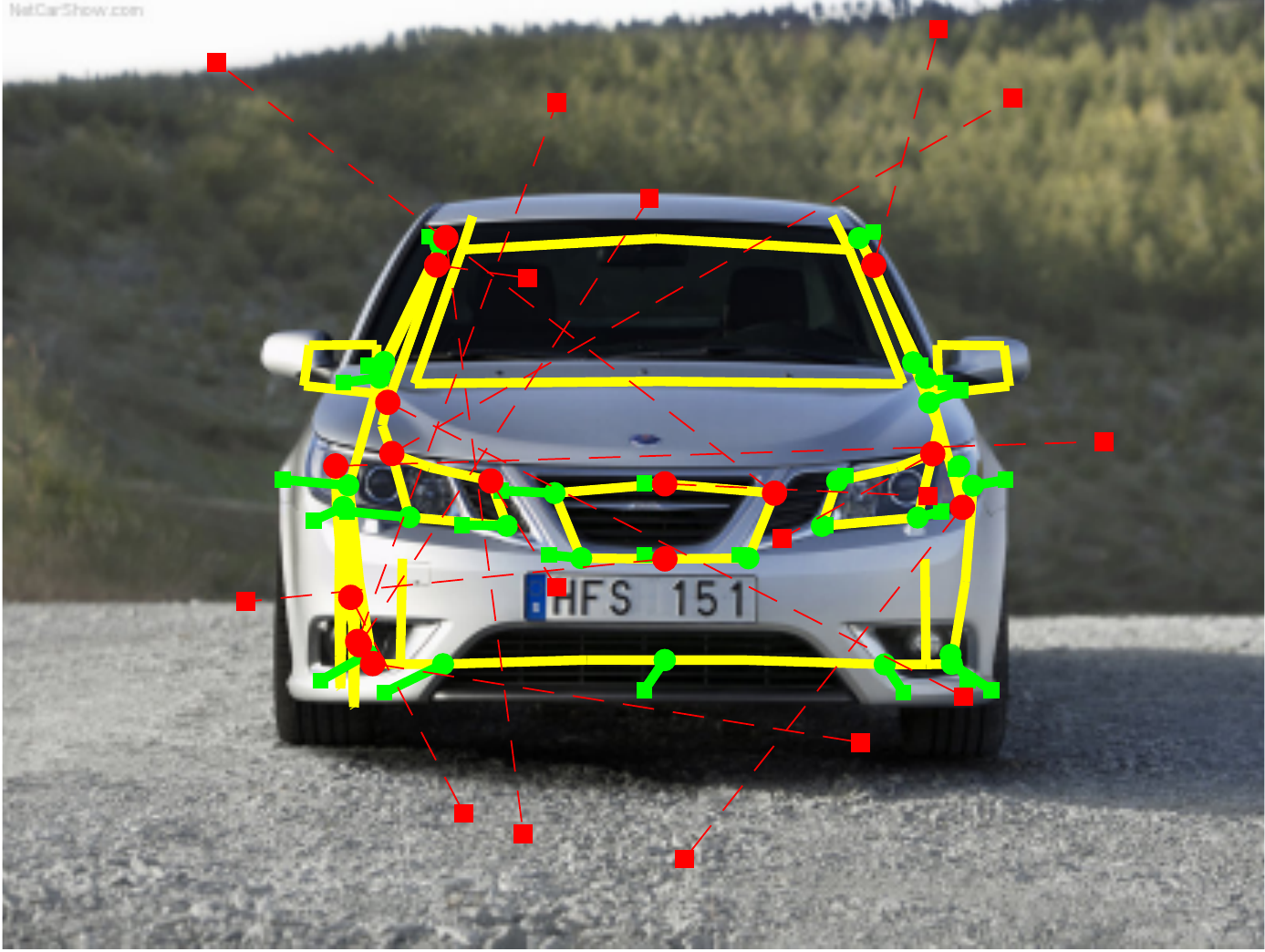} \\
	\vspace{1mm}
	\end{minipage}
& \myhspace
	\begin{minipage}{\mpwthree}%
	\centering%
	\includegraphics[width=\columnwidth]{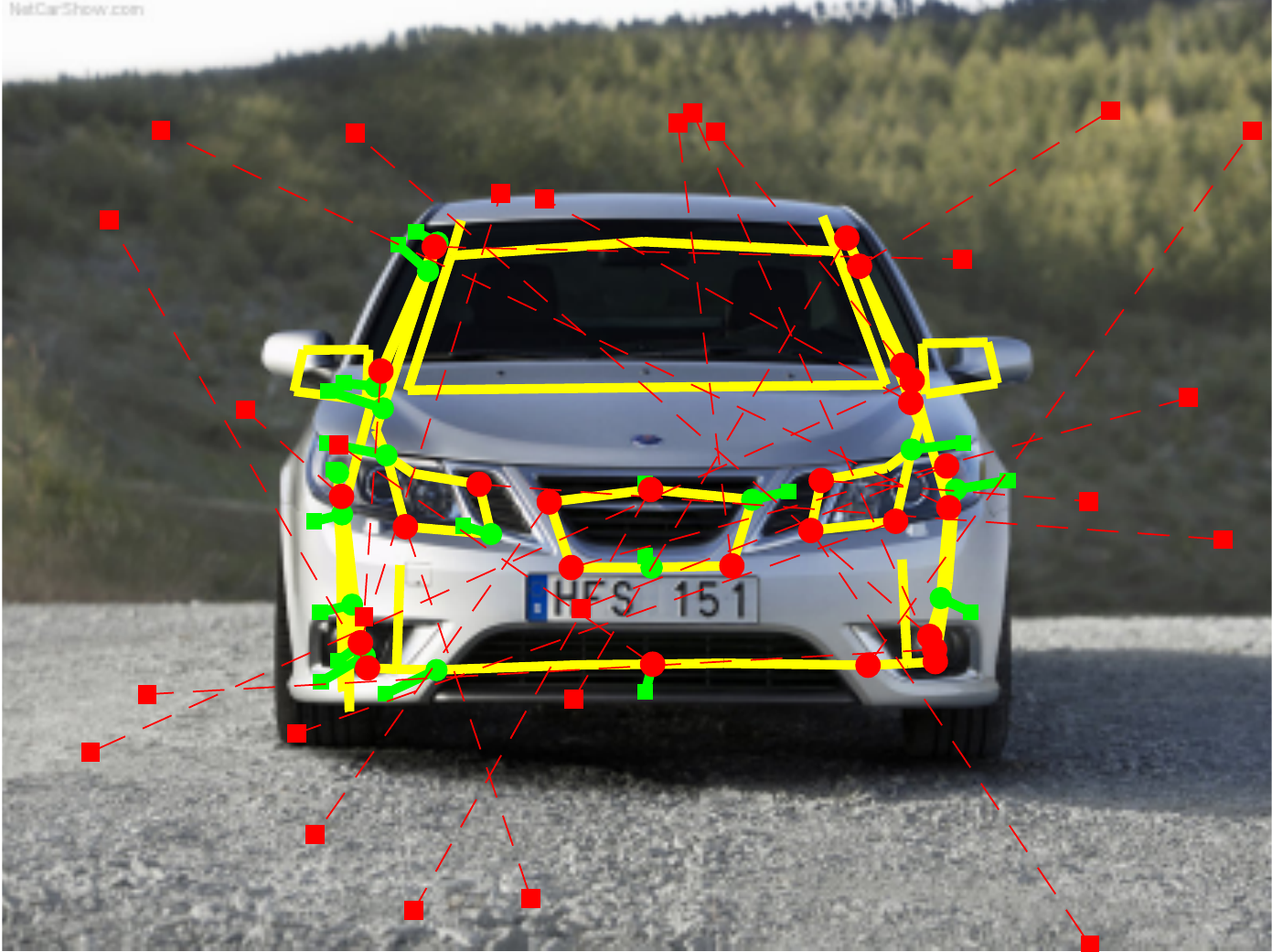} \\
	\vspace{1mm}
	\end{minipage}
&  \myhspace
	\begin{minipage}{\mpwthree}%
	\centering%
	\includegraphics[width=\columnwidth]{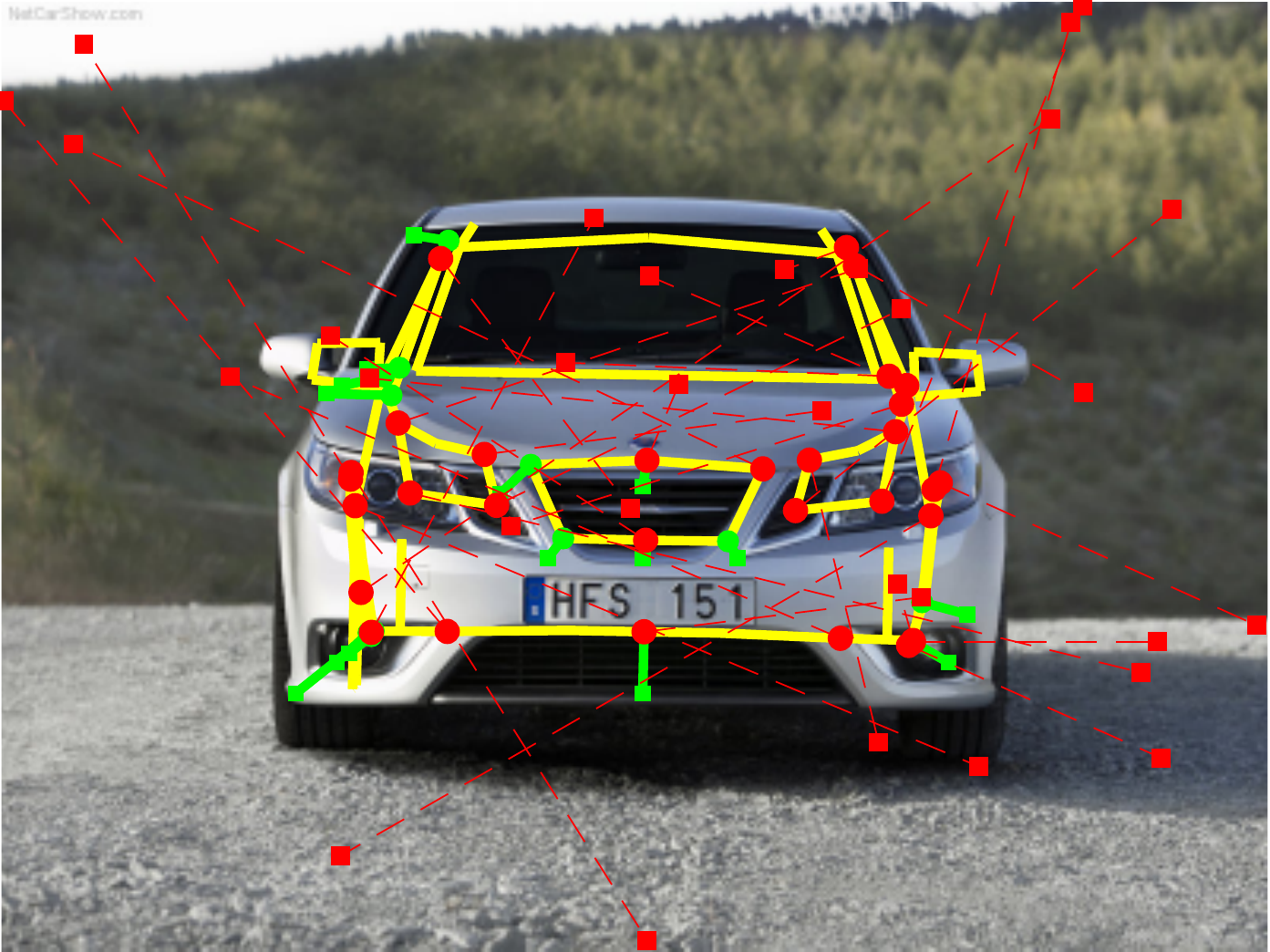} \\
	\vspace{1mm}
	\end{minipage} 
&  \myhspace
	\begin{minipage}{\mpwthree}%
	\centering%
	\includegraphics[width=\columnwidth]{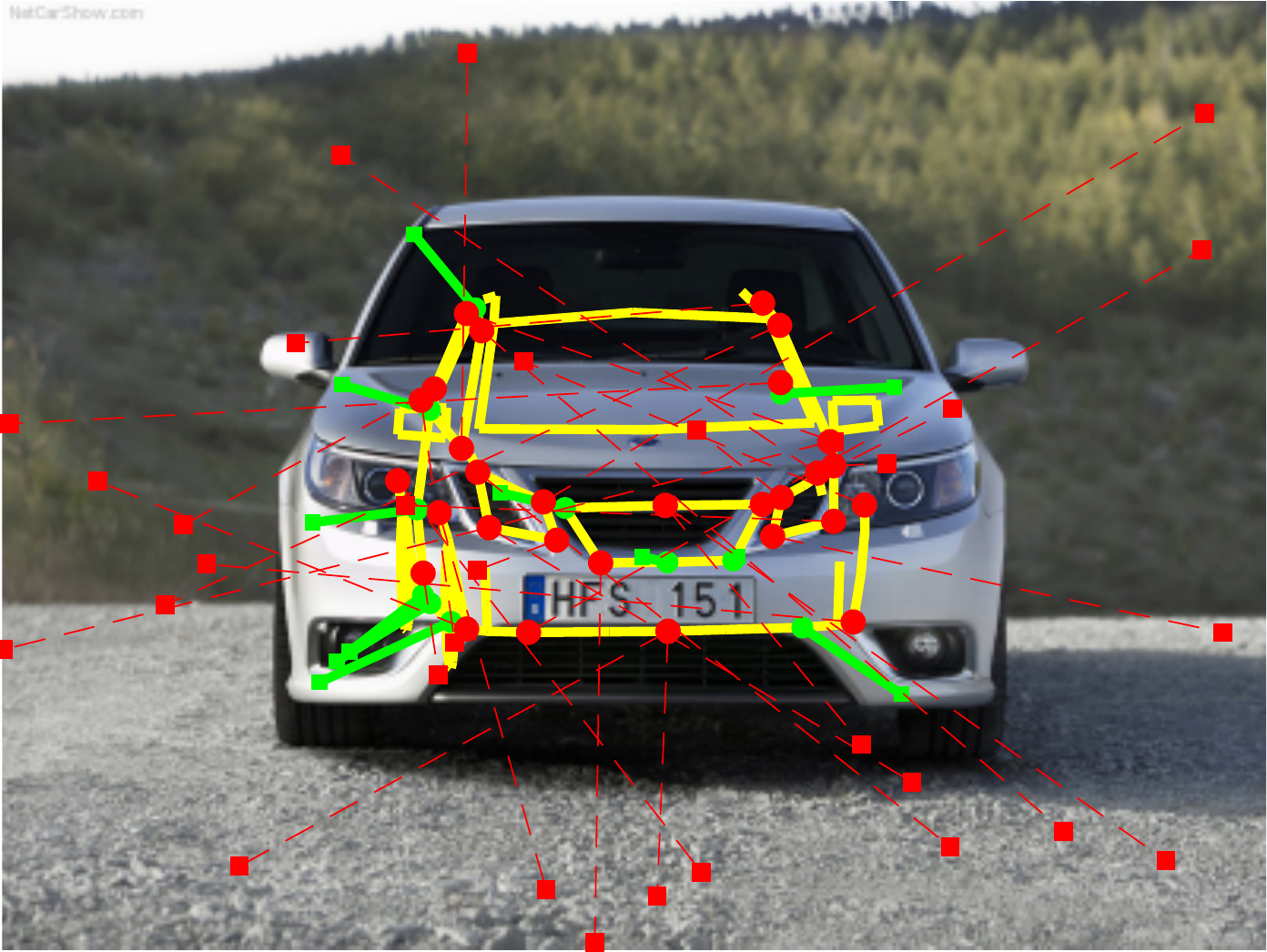} \\
	\vspace{1mm}
	\end{minipage} \\
\myhspace \myhspace \hspace{-2mm} \rotatebox{90}{\hspace{-3mm} {\smaller \blue{ \namerobust}} } & 
\myhspace
	\begin{minipage}{\mpwthree}%
	\centering%
	\includegraphics[width=\columnwidth]{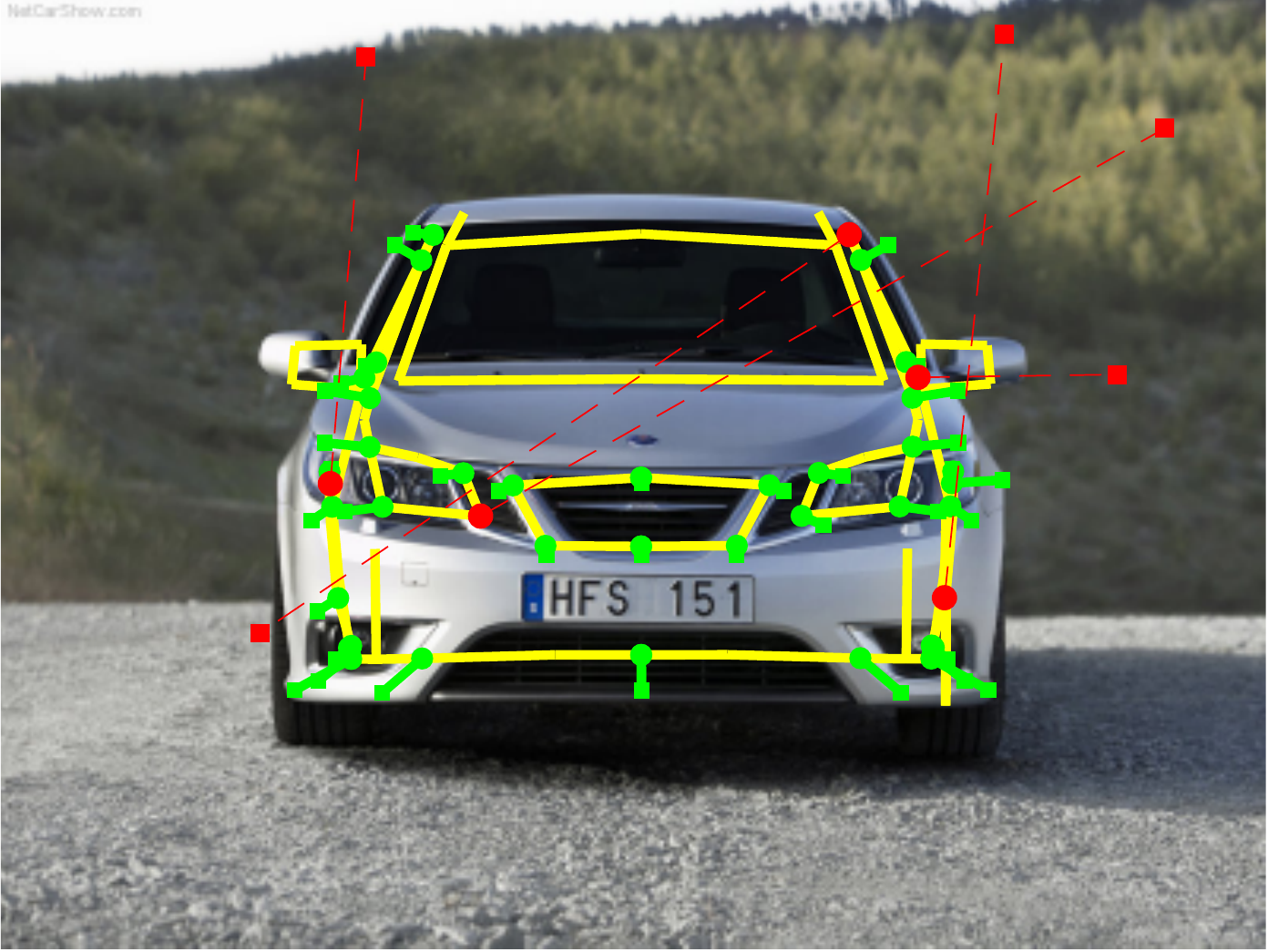} \\
	\vspace{1mm}
	\end{minipage}
& \myhspace
	\begin{minipage}{\mpwthree}%
	\centering%
	\includegraphics[width=\columnwidth]{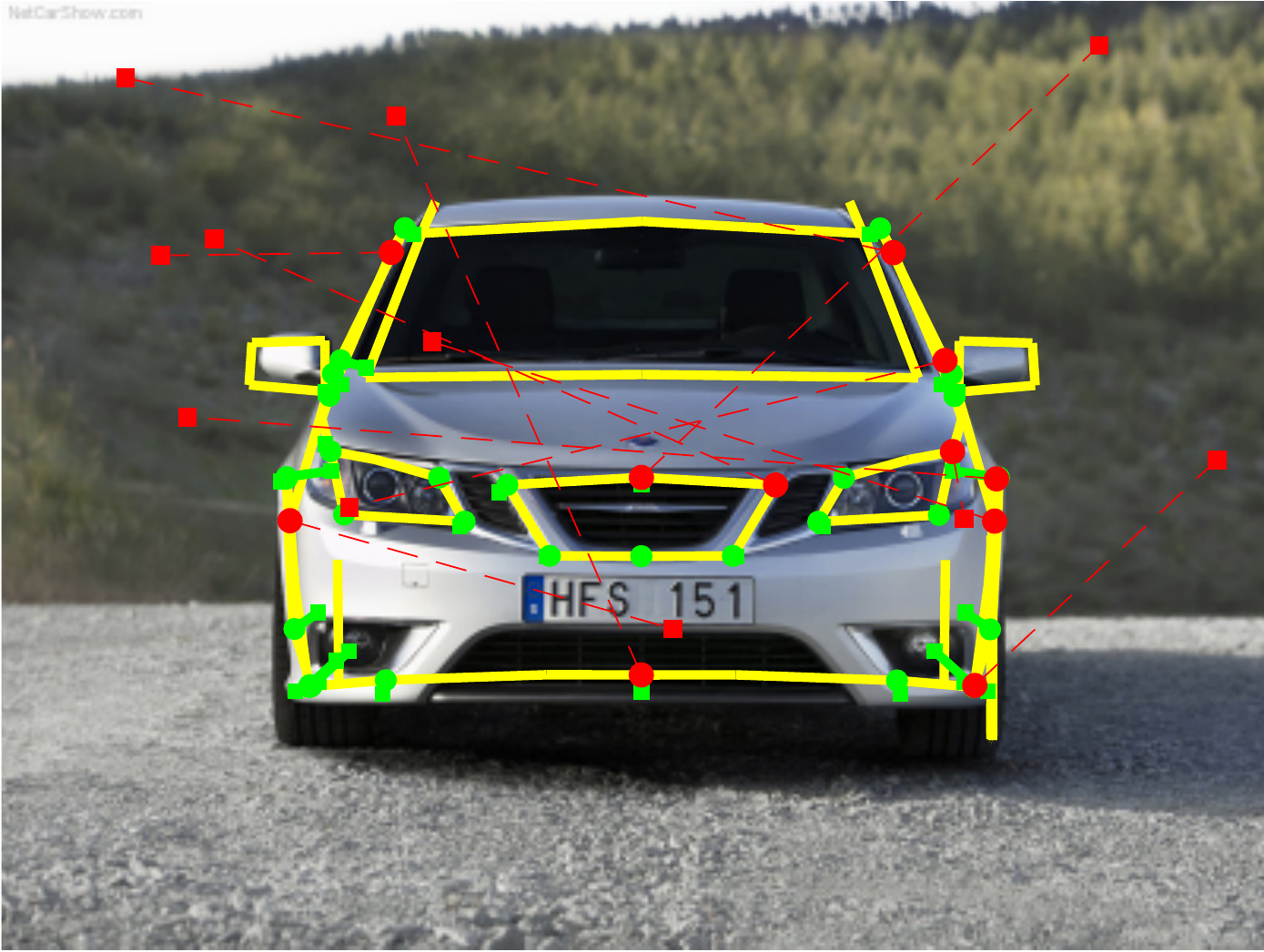} \\
	\vspace{1mm}
	\end{minipage}
& \myhspace
	\begin{minipage}{\mpwthree}%
	\centering%
	\includegraphics[width=\columnwidth]{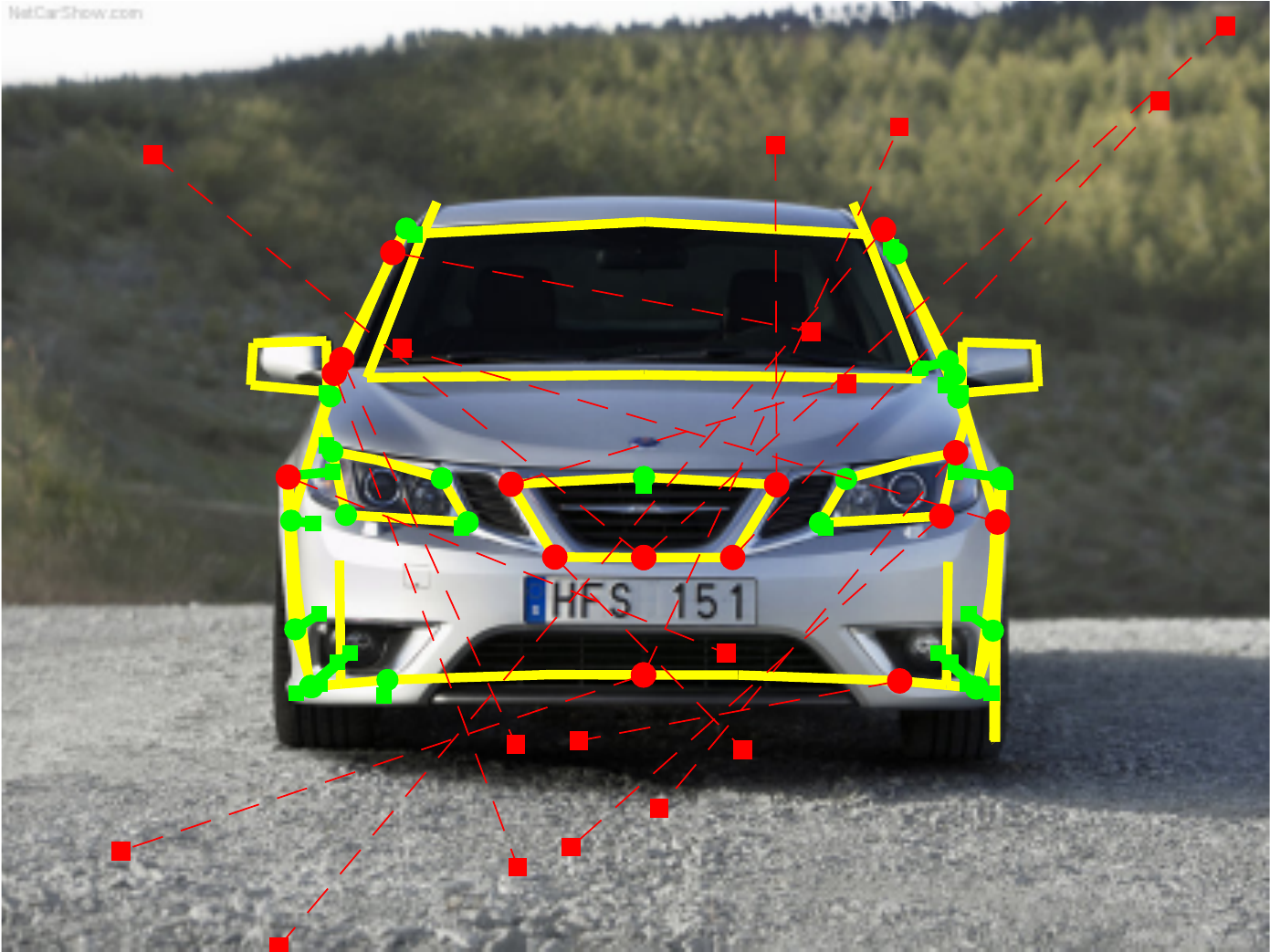} \\
	\vspace{1mm}
	\end{minipage} 
& \myhspace
	\begin{minipage}{\mpwthree}%
	\centering%
	\includegraphics[width=\columnwidth]{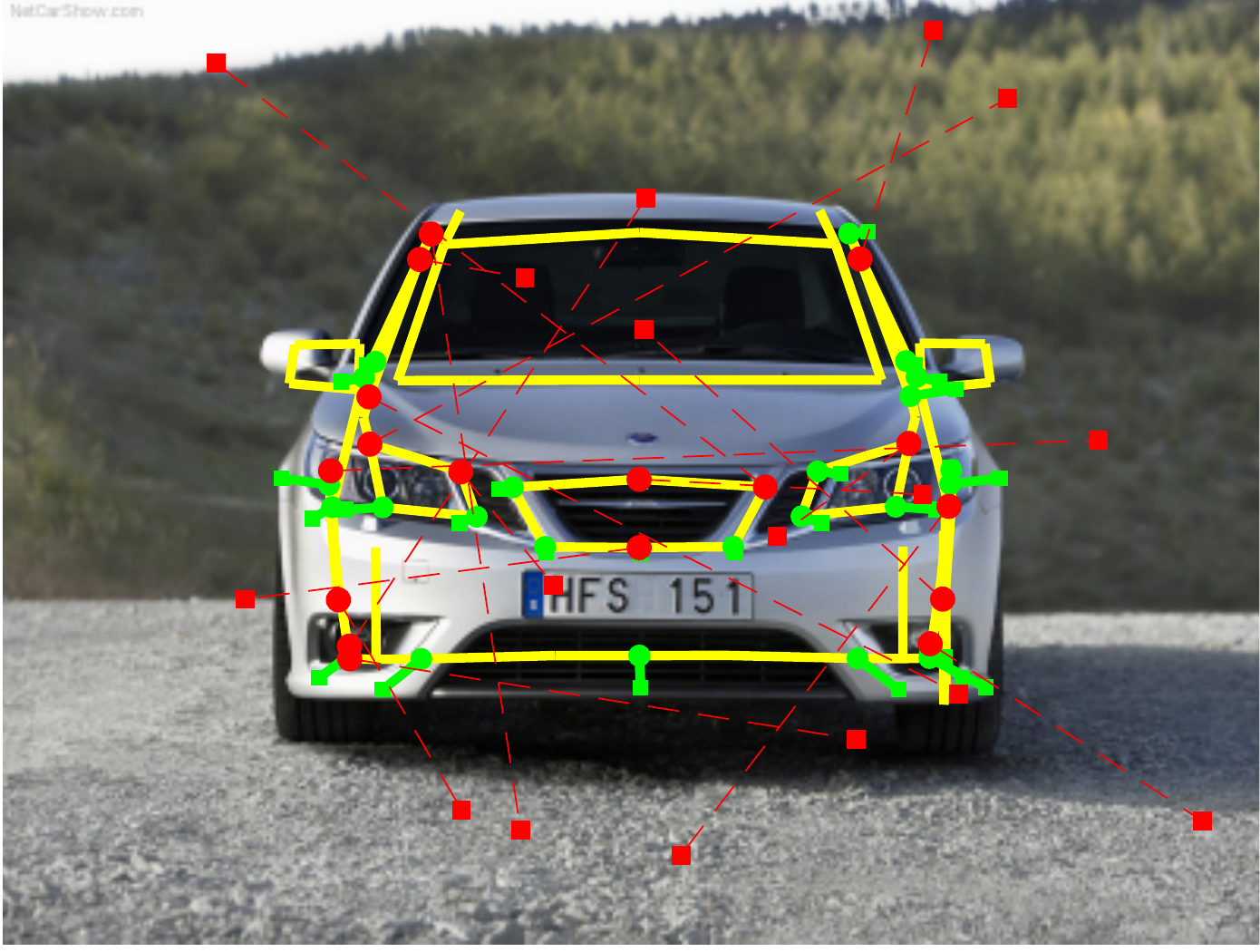} \\
	\vspace{1mm}
	\end{minipage}
& \myhspace
	\begin{minipage}{\mpwthree}%
	\centering%
	\includegraphics[width=\columnwidth]{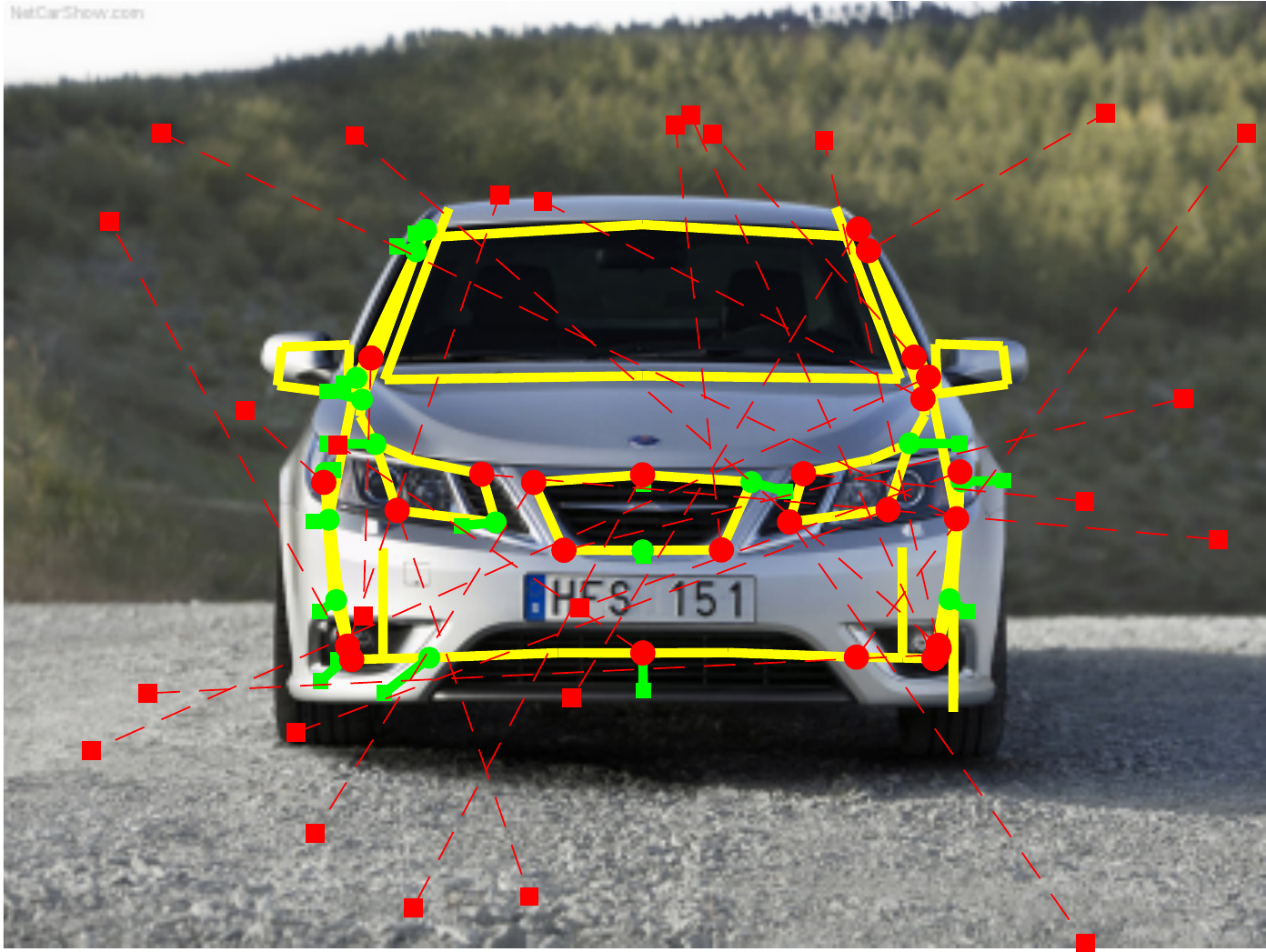} \\
	\vspace{1mm}
	\end{minipage}
&  \myhspace
	\begin{minipage}{\mpwthree}%
	\centering%
	\includegraphics[width=\columnwidth]{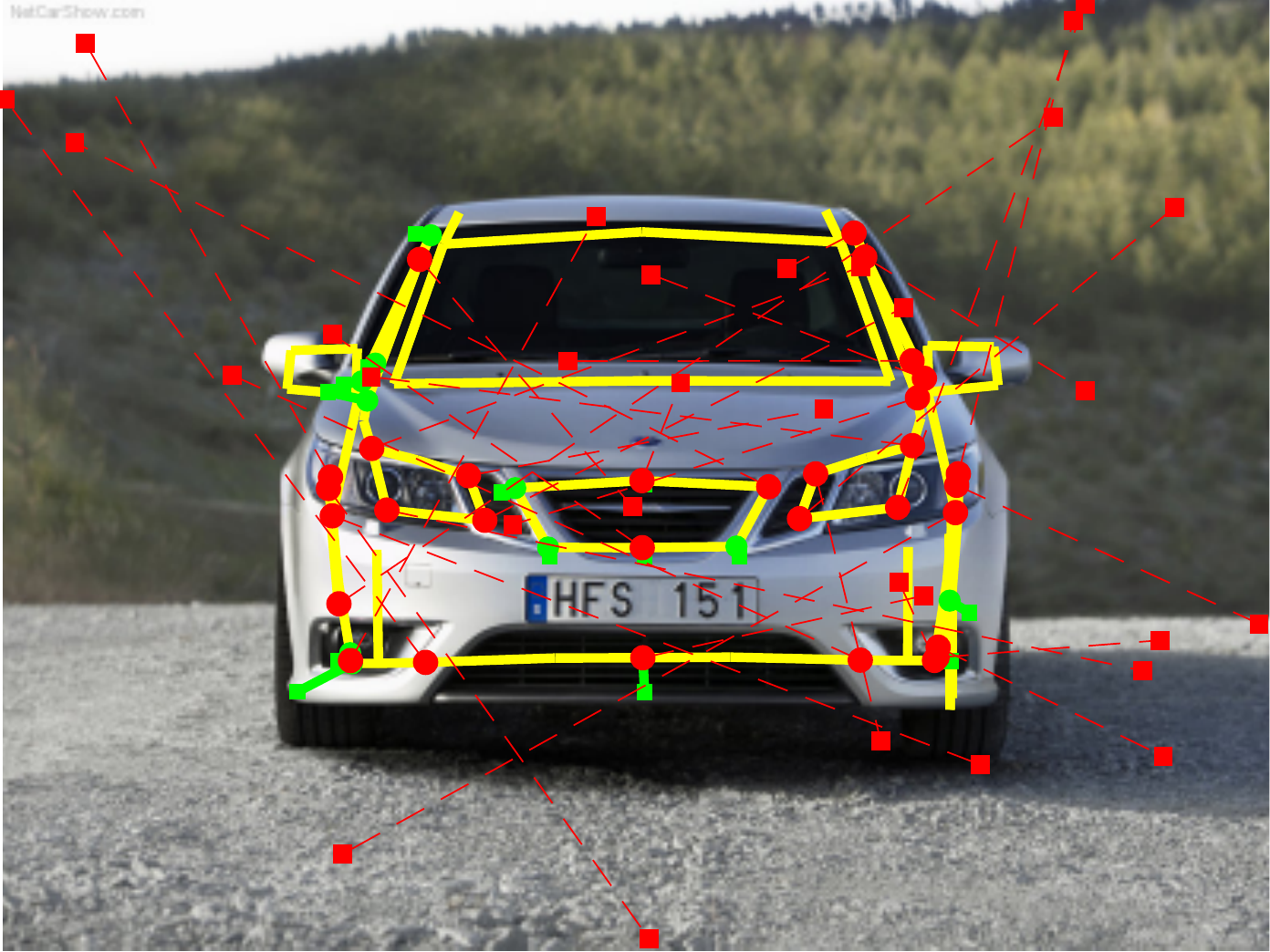} \\
	\vspace{1mm}
	\end{minipage} 
&  \myhspace
	\begin{minipage}{\mpwthree}%
	\centering%
	\includegraphics[width=\columnwidth]{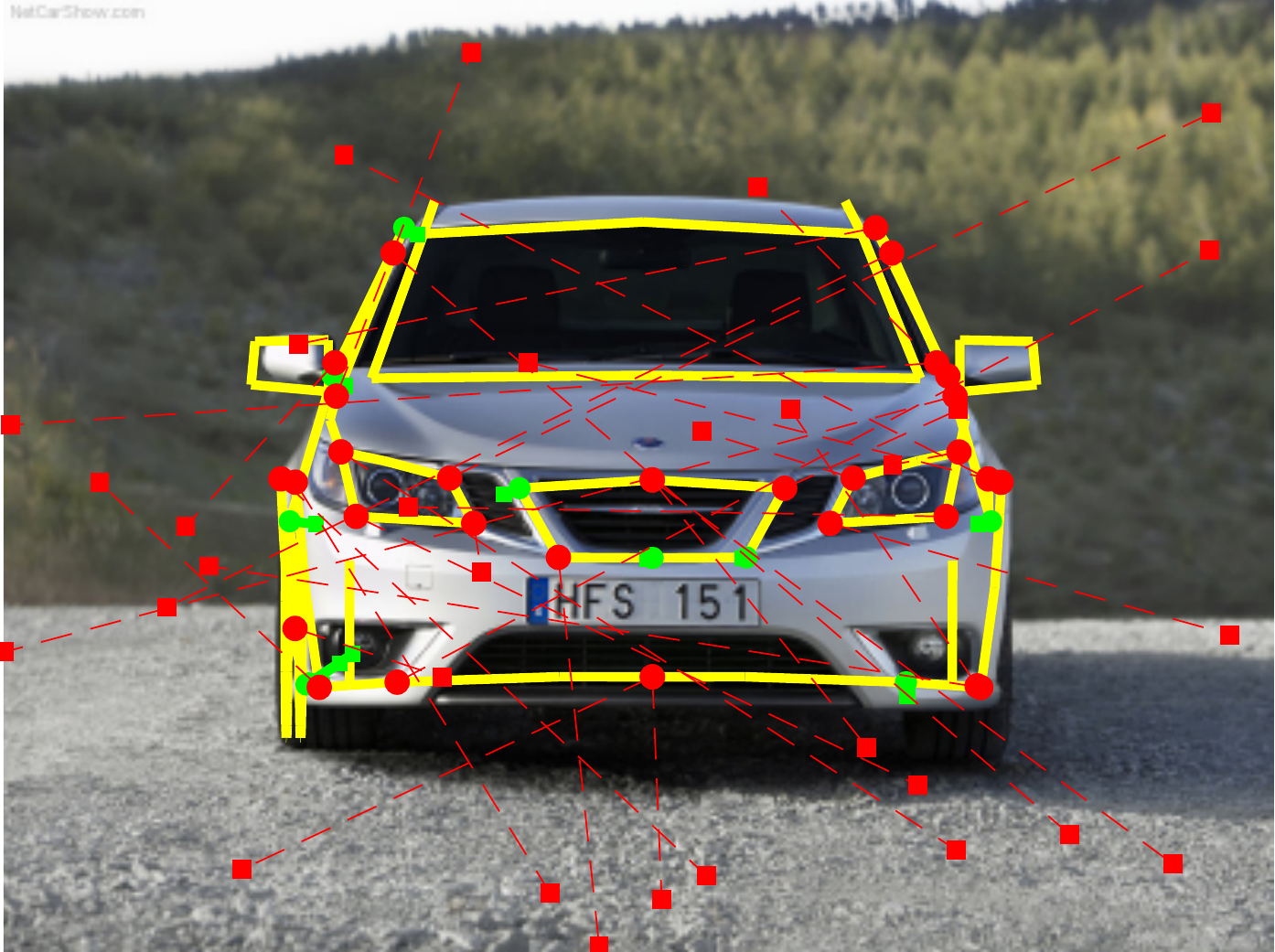} \\
	\vspace{1mm}
	\end{minipage} \\
\multicolumn{8}{c}{Saab 93} \\